\documentclass[format=acmsmall, review=false, screen=true, nonacm=true, authorversion=true]{acmart}

\usepackage{booktabs} 

\usepackage{amsfonts}
\usepackage{amsmath}
\usepackage{bm}
\usepackage{subfig}
\usepackage{graphicx}
\usepackage[ruled,vlined,linesnumbered,nokwfunc]{algorithm2e}
\usepackage{multirow}
\usepackage{mathtools}

\SetAlFnt{\small}
\SetAlCapFnt{\small}
\SetAlCapNameFnt{\small}
\SetAlCapHSkip{0pt}
\IncMargin{-\parindent}

\acmJournal{TKDD}
\acmVolume{V}
\acmNumber{N}
\acmArticle{1}
\acmYear{2019}
\acmMonth{9}
\copyrightyear{2019}

\newcommand{\iForest}{\mbox{iForest}}
\newcommand{\CFOF}{\mbox{CFOF}}

\newcommand{\KNN}{\mbox{KNN}}
\newcommand{\aKNN}{\mbox{aKNN}}
\newcommand{\LOF}{\mbox{LOF}}
\newcommand{\RNNc}{\mbox{RNNc}} 
\newcommand{\ODIN}{\mbox{ODIN}}
\newcommand{\antiHub}{\mbox{AntiHub$^2$}}
\newcommand{\aHub}{\mbox{aHub2}}
\newcommand{\FastABOD}{\mbox{FastABOD}}
\newcommand{\Vect}[1]{{\boldsymbol{#1}^{(d)}}}
\newcommand{\vect}[1]{{#1}^{(d)}}
\newcommand{\norm}[1]{\|#1\|}
\newcommand{\HardCFOF}{\mbox{\textit{hard}-$\CFOF$}}
\newcommand{\SoftCFOF}{\mbox{\textit{soft}-$\CFOF$}}
\newcommand{\FastCFOF}{\mbox{\textit{fast}-$\CFOF$}}
\newcommand{\CFOFprocpart}{\mbox{\textit{fast}-$\CFOF$}\textit{\_part}}

\newcommand{\DS}{{\bf DS}}
\newcommand{\dist}{{\rm dist}}

\newcommand{\NN}{{\rm NN}}
\newcommand{\nn}{{\it nn}}

\newcommand{\E}{\mathbf{E}}
\newcommand{\N}{{\rm N}}

\newcommand{\pwina}[1]{p_{win,AUC}(#1)}
\newcommand{\pwinp}[1]{p_{win,Prec}(#1)}

\begin{document}
\title[CFOF: A Concentration Free Measure for Anomaly Detection]{CFOF: A Concentration Free Measure for 
Anomaly Detection}
\titlenote{A preliminary version of this article appears in 
F. Angiulli, ``Concentration Free Outlier Detection'', Proc. of the 
European Conf. on Machine Learning and Knowledge Discovery in Databases
(ECMLPKDD), pp. 3-19, 2017
\cite{Angiulli17}.}

\author{Fabrizio Angiulli}
\orcid{0000-0002-9860-7569}
\affiliation{%
  \institution{University of Calabria}
  \department{DIMES--Dept. of Computer, Modeling, Electronics, and Systems Engineering}
  \streetaddress{Via P. Bucci, 41C}
  \city{Rende}
  \state{CS}
  \postcode{87036}
  \country{Italy}}
\email{fabrizio.angiulli@unical.it}

\begin{abstract}
We present a novel notion of outlier, called the Concentration Free Outlier 
Factor, or CFOF. As a main contribution, we formalize the notion of 
concentration of outlier scores and theoretically prove that CFOF does not 
concentrate in the Euclidean space for any arbitrary large dimensionality. To 
the best of our knowledge, there are no other proposals of data analysis 
measures related to the Euclidean distance for which it has been provided 
theoretical evidence that they are immune to the concentration effect. We 
determine the closed form of the distribution of CFOF scores in arbitrarily 
large dimensionalities and show that the CFOF score of a point depends on its 
squared norm standard score and on the kurtosis of the data distribution, thus 
providing a clear and statistically founded characterization of this notion. 
Moreover, we leverage this closed form to provide evidence that the definition 
does not suffer of the hubness problem affecting other measures in high 
dimensions. We prove that the number of CFOF outliers coming from each cluster 
is proportional to cluster size and kurtosis, a property that we call 
semi-locality. We leverage theoretical findings to shed lights on properties of 
well-known outlier scores. Indeed, we determine that semi-locality characterizes 
existing reverse nearest neighbor-based outlier definitions, thus clarifying the 
exact nature of their observed local behavior. We also formally prove that 
classical distance-based and density-based outliers concentrate both for bounded 
and unbounded sample sizes and for fixed and variable values of the neighborhood 
parameter. We introduce the fast-CFOF algorithm for detecting outliers in large 
high-dimensional dataset. The algorithm has linear cost, supports 
multi-resolution analysis, and is embarrassingly parallel. Experiments highlight 
that the technique is able to efficiently process huge datasets and to deal even 
with large values of the neighborhood parameter, to avoid concentration, and to 
obtain excellent accuracy.
\end{abstract}

%
%

\ccsdesc[300]{Computing methodologies~Unsupervised learning}
\ccsdesc[500]{Computing methodologies~Anomaly detection}%
%

\keywords{Outlier detection, high-dimensional data,
curse of dimensionality, concentration of distances, hubness,
kurtosis, huge datasets}

\maketitle

\renewcommand{\shortauthors}{F. Angiulli}

\section{Introduction}

Outlier detection is one of the main {data mining} and machine learning tasks,
whose goal is to single out anomalous observations, also called outliers 
\cite{Aggarwal2013}.
While the other data analysis approaches, such as classification, 
clustering or dependency detection,
consider outliers as noise that must be eliminated, as pointed out in 
\cite{HK01},
``one person's noise could
be another person's signal", thus outliers themselves are of
great interest in different settings,
e.g. fraud detection, ecosystem disturbances, 
intrusion detection, cybersecurity,
medical analysis, to cite a few.

Outlier analysis has its roots in statistics
\cite{DaviesGather93,BL94}.
Data mining outlier approaches to outlier detection
can be classified in supervised, 
semi-supervised, and unsupervised \cite{Hodge2004,ChandolaBK09}.
Supervised methods take in input data labeled 
as normal and abnormal and build a classifier.
The challenge there is posed by the 
fact that abnormal data form a rare class.
Semi-supervised methods,
also called one-class 
classifiers or domain description techniques,
take in input only normal examples and use them
to identify anomalies.
Unsupervised methods detect outliers
in an input dataset by assigning 
a score or anomaly degree to each object.

A commonly accepted definition fitting the
unsupervised setting is the following: 
``Given a set of data points or
objects, find those objects that are considerably dissimilar, exceptional or
inconsistent with respect to the remaining data'' \cite{HK01}. 
{
Unsupervised outlier detection methods can be categorized in several
approaches, each of which assumes a specific concept of outlier.
Among the most popular families there are
{statistical-based} \cite{DaviesGather93,BL94}, 
{deviation-based} \cite{AAR96}, 
{distance-based} \cite{KN98,RRS00,AP05,AngiulliF09}, 
{density-based} \cite{BKNS00,JTH01,PKGF03,JinTHW06}, 
{reverse nearest neighbor-based} \cite{HautamakiKF04,RadovanovicNI15}
{angle-based} \cite{KriegelSZ08},
{isolation-based} \cite{LiuTZ12},
subspace-based \cite{AY01,AngiulliFPTODS09,KellerMB12},
{ensemble-based} \cite{LazarevicK05,AggarwalS17}
and others \cite{ChandolaBK12,ZimekSK12,Aggarwal2013,AkogluTK15}.
}

\smallskip
This work focuses on unsupervised outlier detection
in the full feature space.
In particular, we present a novel notion of outlier,
the \textit{Concentration Free Outlier Factor} ($\CFOF$),
having the peculiarity to resist to the distance
concentration phenomenon
which is part of the so called curse of dimensionality problem 
{\cite{Bellman1961,demartines1994,BeyerGRS99,ChavezNBM01,Francois2007,Angiulli2018}}. 
Specifically,
the term distance concentration refers to the tendency of
distances to become almost indiscernible
as dimensionality increases.
This phenomenon may greatly affect the quality and performances
of data mining, machine learning, and information retrieval techniques,
since all these techniques rely on the concept
of distance, or dissimilarity, among data items in order to retrieve or analyze
information.
Whereas low-dimensional spaces show good agreement between geometric
proximity and the notion of similarity, as dimensionality increases, 
counterintuitive phenomena like distance concentration and hubness
may be harmful to traditional techniques.
In fact, the concentration problem also affects
outlier scores of different families
due to the specific role played by distances 
in their formulation.

This characteristics of high dimensional data has generated 
in data analysis and data management applications
the need for dimensionality resistant notions of similarity,
that are similarities not affected by the poor separation between
the furthest and the nearest neighbor in high dimensional space.
Among the desiderata that a good distance resistant to dimensionality should possess, 
there are the to be \textit{contrasting}
and \textit{statistically sensitive}, 
that is meaningfully refractory to concentration, 
and to be \textit{compact}, 
or efficiently computable in terms of time and space
\cite{Aggarwal2001}.

In the context of unsupervised outlier detection, \cite{ZimekSK12} 
identified different issues related to the treatment of high-dimensional data,
among which the \textit{concentration of scores}, 
in that derived outlier score become numerically similar,
\textit{interpretability} of scores, 
that fact that the scores often no longer convey a semantic meaning,
and \textit{hubness}, the fact that certain points occur 
more frequently in neighbor lists than others \cite{AucouturierP08,RadovanovicNI09,Angiulli2018}.

Specifically, consider the number $\N_k(x)$
of observed points that have
$x$ among their $k$ nearest neighbors,
also called $k$-occurrences or reverse $k$-nearest neighbor count,
or RNNc, for short, in the following.
It is known that in low dimensional spaces, the distribution of $\N_k(x)$ over all $x$
complies with the binomial distribution and, in particular, for uniformly i.i.d.
data in low dimensions, that it can be modeled as node in-degrees in the $k$-nearest
neighbor graph, which follows the Erd\H{o}s-R\'enyi graph model \cite{Erdos59}.
However, it has been observed that as the dimensionality increases, the 
distribution of $\N_k$ becomes skewed to the right, resulting in the emergence
of \textit{hubs}, which are points whose reverse nearest neighbors counts
tend to be meaningfully larger than that associated with any other point.

Thus, the circumstance that the outlier scores
tend to be similar poses some challenges in terms of
their intelligibility,
absence of a clear separation between outliers and inliers, 
and loss of efficiency of pruning rules aiming at
reducing the computational effort.

\smallskip
The 
$\CFOF$ score is a \textit{reverse nearest neighbor-based score}.
Loosely speaking, it corresponds to measure
how many nearest neighbors have to be taken into account
in order for a point to be close to a sufficient fraction $\varrho$
of the data population.
We notice that this kind of notion
of perceiving the abnormality of an observation
is completely different from any other
notion so far introduced.
In the literature,
there are other outlier detection approaches
resorting to reverse nearest neighbor
counts \cite{HautamakiKF04,JinTHW06,LinED08,RadovanovicNI15}.
{Methods such INFLO \cite{JinTHW06} are density-based techniques considering 
both direct and reverse nearest neighbors
when estimating the outlierness of a point.}
Early \textit{reverse nearest neighbor-based}
approaches, that are $\ODIN$ \cite{HautamakiKF04}, which uses $\N_k(x)$ as outlier score
of $x$, and the one proposed in \cite{LinED08}, which returns as outliers 
those points $x$ having $\N_k(x)=0$,
are prone to the hubness phenomenon,
that is the concentration of the scores towards the values associated with outliers, 
due to direct used of the function $\N_k$.
Hence, to mitigate the hubness effect,
\cite{RadovanovicNI15} proposed
a simple heuristic method, namely $\antiHub$, which refines the scores
produced by the $\ODIN$ method
by returning the weighted mean of the sum of the $\N_k$ scores of 
the neighbors of the point and of the $\N_k$ score of the point itself.

\smallskip
In this work
we both empirically and theoretically show that 
the here introduced $\CFOF$ outlier score
complies with all of the above mentioned desiderata.
As a main contribution,
we formalize the notion of \textit{concentration of outlier scores},
and theoretically prove that $\CFOF$ does not concentrate 
in the Euclidean space $\mathbb{R}^d$
for any arbitrarily large dimensionality $d\rightarrow\infty$.
To the best of our knowledge, 
there are no other proposals of outlier detection measures,
and probably also of other {data analysis}
measures related to the Euclidean distance,
for which it has been provided the theoretical evidence
that they are immune to the concentration effect.

We recognize that the \textit{kurtosis} $\kappa$ of the data population,
a well-known measure of tailedness of a probability distribution
originating with Karl Pearson \cite{Pearson1905,FioriZ09,Westfall14},
is a key parameter for characterizing from the outlier detection perspective
the unknown distribution underlying the data,
a fact that has been neglected at least within the data mining literature.
The kurtosis may range from $\kappa=1$, 
for platykurtic distributions such as the Bernoulli 
distribution with success probability $0.5$,
to $\kappa\rightarrow\infty$, for extreme leptokurtic or heavy-tailed distributions.
Each outlier score must concentrate for $\kappa=1$ due to
the absolute absence of outliers.
We prove that
$\CFOF$ does not concentrate for any $\kappa>1$.

We determine the closed form of the distribution of the $\CFOF$
scores for arbitrarily large dimensionalities and show that
the $\CFOF$ score of a point $x$ depends,
other than on the parameter $\varrho$ employed,
on its squared norm standard score $z_x$ 
and on the kurtosis $\kappa$ of the data distribution.
The squared norm standard score of a data point
is the standardized squared norm of the point 
under the assumption that 
the origin of the feature space coincides with 
the mean of the distribution generating points.
We point out that the knowledge of the theoretical distribution
of an outlier score is a rare, if not unique, peculiarity.
We prove that the probability to observe larger scores
increases with the kurtosis.

As for the \textit{hubness} phenomenon, 
by exploiting the closed form of the $\CFOF$ scores distribution,
we provide evidence that $\CFOF$ does not suffer
of the hubness problem, since points associated with the largest scores 
always correspond to a small fraction
of the data.
Moreover, while previously known RNNc scores present large false positive rates for values of
the parameter $k$ which are not comparable with $n$,
$\CFOF$ is able to establish a clear separation between
outliers and inliers for any value of the parameter $\varrho$.

We theoretically prove that the $\CFOF$ score 
is both translation and scale-invariant.
This allows to establish that CFOF has connections with 
\textit{local} scores.
Indeed, if we consider a dataset consisting of multiple translated and scaled copies 
of the same seed cluster, 
the set of the $\CFOF$ outliers 
consists of the same points from each cluster.
More in the general,
in the presence of clusters having different generating 
distributions,
the number of outliers coming from each cluster
is directly proportional to its size and to its kurtosis,
a property that we called \textit{semi--locality}.

\smallskip
As an equally important contribution, 
the design of the novel outlier score 
and the study of its theoretical properties
allowed us to shed lights also on different properties
of well-known outlier detection scores.

First, we determine that the semi--locality is a
peculiarity of reverse nearest neighbor counts.
This discovery clarifies the exact nature of 
the reverse nearest neighbor family of outlier scores:
while in the literature this family of scores has been observed to
be adaptive to different density levels, the exact behavior of 
this adaptivity was unclear till now.

Second,
we identify the property each outlier score 
which is monotone increasing with respect to the squared norm standard score 
must possess in order to avoid concentration.
We leverage this property to formally show that 
classic distance-based and density-based outlier scores 
are subject to concentration
both for bounded and unbounded dataset sizes $n$,
and both for fixed and variable values of the parameter $k$.
Moreover, the convergence rate towards concentration 
of these scores is inversely 
proportional to the kurtosis of the data.

Third, 
as a theoretical confirmation of the proneness of $\N_k$ to false positives,
we show that the ratio between the amount of variability
of the $\CFOF$ outlier scores and that of the RNNc outlier scores
corresponds to of several orders of magnitude and, moreover, that 
the above ratio is even increasing with the kurtosis.

\medskip
Local outlier detection methods,
showing adaptivity to different density levels,
are usually identified in the literature with those methods that
compute the output scores by
comparing the neighborhood of each point 
with the neighborhood of its neighbors. 
We point out that,
as far as $\CFOF$ is concerned, its local behavior is obtained
without the need to explicitly perform such a kind of comparison.
Rather, since from a conceptual point of view
computing $\CFOF$ scores can be assimilated to estimate a probability,
we show that $\CFOF$ scores can be reliably computed
by exploiting sampling techniques.
The reliability of this approach descends from the
fact that $\CFOF$ outliers are the points less prone to bad estimations.

Specifically,
to deal with very 
large and high-dimensional datasets,
we introduce the $\FastCFOF$ technique
which exploits sampling to avoid
the computation of exact nearest neighbors
and, hence, from the computational point of view
does not suffer of the dimensionality curse
affecting (reverse) nearest neighbor search techniques.
The cost of $\FastCFOF$ is linear both in the dataset
size and dimensionality.
The $\FastCFOF$ algorithm
is efficiently parallelizable,
and we provide a multi-core (MIMD) vectorized (SIMD)
implementation.

The algorithm has an unique parameter $\varrho\in(0,1)$,
representing a fraction of the data population.
The $\FastCFOF$ algorithm supports 
multi-resolution analysis
regarding the dataset at different scales, 
since different $\varrho$ values can be managed simultaneously
by the algorithm,
with no additional computational effort.

\smallskip
Experimental results highlight that $\FastCFOF$
is able to achieve very good accuracy with reduced sample
sizes $s$ and, hence, to efficiently process huge datasets.
Moreover, since its asymptotic cost does not depend on the actual 
value of the parameter $\varrho$, 
$\CFOF$ can efficiently manage even large values of this parameter, 
a property 
which is considered a challenge for different existing outlier methods.
Moreover,
experiments involving the $\CFOF$ score
witness for the absence of concentration
on real data,
show that $\CFOF$ shows excellent accuracy
performances on distribution data, 
and that 
$\CFOF$ is likely to admit configurations ranking 
the outliers better than other approaches
on labelled data.

\smallskip
The study of theoretical properties is conducted
by considering the Euclidean distance as dissimilarity measure,
but from \cite{Angiulli2018} it is expected they are
also valid for any Minkowski's metrics.
Moreover, it is worth to notice that
the applicability of the technique is not confined to the Euclidean space
or to vector spaces.
It can be applied both in metric
and non-metric spaces equipped with a distance function.
Moreover, while
effectiveness and efficiency of the method do not deteriorate
with the dimensionality, 
its application
is {perfectly reasonable}
even in low dimensions.

\smallskip
We believe the $\CFOF$ technique and 
the properties presented in this work
provide insights within the scenario
of outlier detection and, more in the general,
of high-dimensional data analysis.

\smallskip
The rest of the work is organized as follows.
Section \ref{sect:definition} introduces the $\CFOF$ score and 
provides empirical evidence of its behavior.
Section \ref{sect:concfree}
studies theoretical properties
of the $\CFOF$ outlier score.
Section \ref{sect:algorithm} presents the $\FastCFOF$ algorithm
for detecting outliers in large high-dimensional datasets.
Section \ref{sect:experiments} describes experiments involving
the propose approach.
Finally, Section \ref{sect:conclusions} draws conclusions and depicts future 
work.

\section{The Concentration Free Outlier Factor}
\label{sect:definition}

In this section, we introduce the Concentration Free Outlier
Factor ($\CFOF$), a novel outlier detection measure.

After presenting the definition of $\CFOF$ score
(see Section \ref{sect:cfofdef}),
we provide empirical evidence of its behavior
by discussing relationship with the distance concentration
phenomenon (see Section \ref{sect:cfof_distconc})
and with the hubness phenomenon (see Section \ref{sect:cfof_hubness}).
Theoretical properties of $\CFOF$
will be taken into account
in subsequent Section \ref{sect:concfree}.

\subsection{Definition}\label{sect:cfofdef}

Let $\DS = \{x_1,x_2,\ldots,x_n\}$ denote a dataset of $n$ points,
also said objects,
belonging to an object space
$\mathbb{U}$ equipped with a distance function $\dist$.
In the following, we assume 
that $\mathbb{U}$ is a vector space
of the form
$\mathbb{U} = \mathbb{D}^d$, 
where $d\in \mathbb{N}^+$, the \textit{dimensionality} of the space,
is a positive natural number and
$\mathbb{D}$ is usually the set $\mathbb{R}$ of real numbers.
However, we point out that the method can be applied in any object space
equipped with a distance function (not necessarily a metric).

Given a dataset object $x$ and a positive integer $k$, 
the \emph{$k$-th nearest neighbor} of $x$
is the dataset object $\nn_k(x)$ such that there exists exactly
$k-1$ dataset objects lying at distance 
smaller than $\dist(x,\nn_k(x))$ from $x$.
It always holds that $x=\nn_1(x)$.
We assume that ties are non-deterministically ordered.

The \emph{$k$ nearest neighbors set} $\NN_k(x)$ of $x$,
where $k$ is also said the \textit{neighborhood width},
is the set of objects $\{ \nn_i(x) \mid 1\le i\le k \}$.

By $\N_k(x)$ we denote the number of objects
having $x$ among their $k$ nearest neighbors:
\[ \N_k(x) = |\{ y : x\in\NN_k(y) \}|, \]
also referred to as 
\textit{$k$-occurrences function} or
\textit{reverse neighborhood size} or 
\textit{reverse $k$ nearest neighbor count} or 
RNNc, for short.

\begin{definition}[CFOF outlier score]
Given a parameter $\varrho\in(0,1)$,
the \textit{Concentration Free Outlier Score},
also referred to as $\CFOF$ (or $\varrho$--$\CFOF$ if the 
value of the parameter $\varrho$ is not clear from the context),
is defined as:
\begin{equation}
\label{eq:cfof_hard}
\CFOF(x) = \min_{1\le k'\le n} \left\{ \frac{k'}{n} : \N_{k'}(x) \ge n\varrho \right\}.
\end{equation}
Thus, the CFOF score of $x$ represents the smallest
neighborhood width, normalized with respect to $n$,
for which $x$ exhibits a reverse neighborhood of size
at least $n\varrho$.

The CFOF score belongs to the interval $[0,1]$.
In some cases, we will use absolute $\CFOF$
score values, ranging from $1$ to $n$.
\end{definition}

{
For complying with 
existing outlier detection measures that employ 
the neighborhood width $k$
as an input parameter, 
when we refer to the input parameter $k$,
we assume that, as far as $\CFOF$ is concerned, $k$ represents a shorthand 
for the parameter $\varrho = k/n$.
}

\begin{figure}[t]
\centering
\begin{tabular}{lcccr}
\subfloat[\label{fig:example_scoreA}
Example dataset of $n=14$ points with associated 
absolute $\CFOF$ scores for $k=3$.
The top outlier is the point located in the left upper corner
which has score $k'=10$. Recall that every point is the 
first neighbor of itself.
]
{{\includegraphics[width=0.275\textwidth]{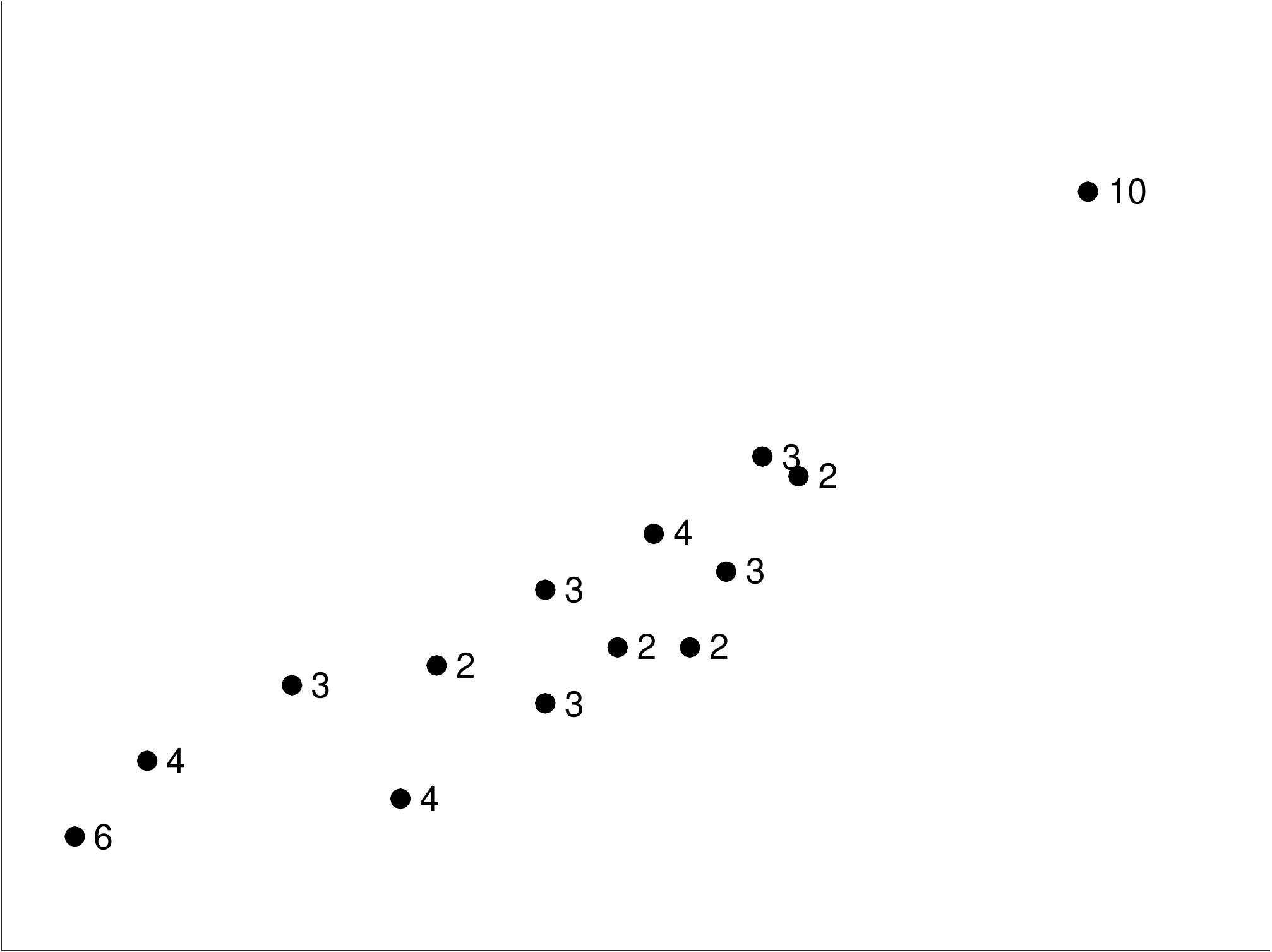}}}
&
&
\subfloat[\label{fig:example_scoreD}
Consider the value $k'=9$. Only $2$ points, 
instead of $k=3$,
namely points A and C, have the outlier point A point among their $k'$ nearest neighbors.
Thus, the score of A must be strictly greater than $2$.]
{{\includegraphics[width=0.275\textwidth]{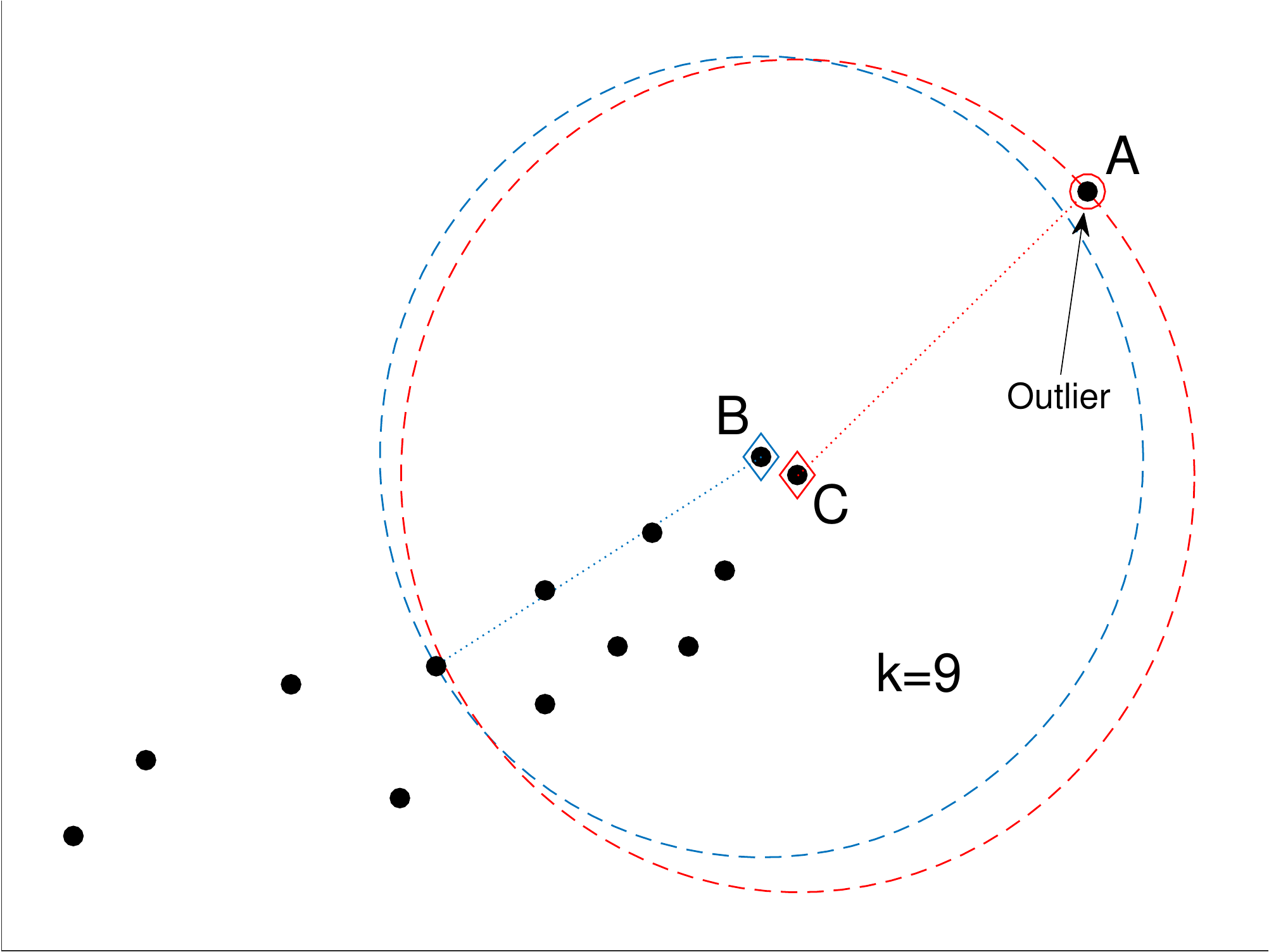}}}
&
&
\subfloat[\label{fig:example_scoreC}
The $\CFOF$ score of A is $k'=10$, since $10$ 
is the smallest natural number for which
there exists at least
$k=3$ points, namely points A, B, and C,
having the point A among their $k'$ nearest neighbors.]
{{\includegraphics[width=0.275\textwidth]{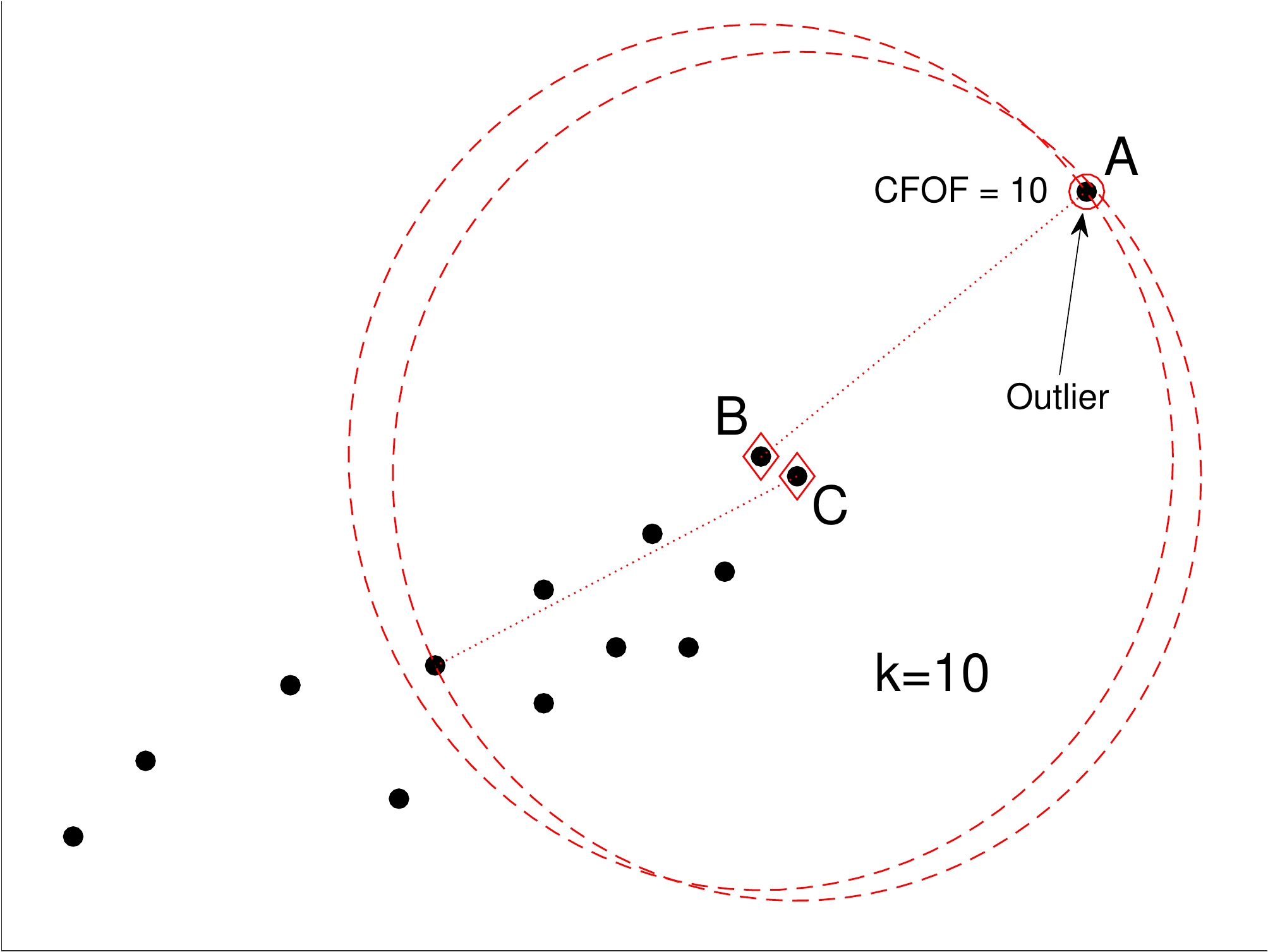}}}
\end{tabular}
\caption{Example illustrating the computation of the $\CFOF$ score.}
\label{fig:example_score}
\end{figure}

\medskip
Figure \ref{fig:example_score} 
illustrates the computation of the $\CFOF$ score on
a two-dimensional example dataset.

Intuitively, the $\CFOF$ score
measures how many neighbors
have to be taken into account
in order for the object to be considered
close by an
appreciable fraction of the
dataset objects.
We point out that this kind of notion
of perceiving the abnormality of an observation
is completely different from any other
notion so far introduced in the literature.

In particular,
the point of view here is in some sense reversed
with respect to distance-based outliers,
since we are interested in
determining the smallest neighborhood width $k'$
for which the object
is a neighbor of at least $n \varrho$ other objects,
while distance-based outliers
(and, specifically, the definition considering the
distance from the $k$th nearest neighbor)
determine the smallest radius
of a region centered in the object
which contains
at least $k$ other objects.

\subsection{Relationship with the distance concentration phenomenon}
\label{sect:cfof_distconc}

One of the main peculiarities of the $\CFOF$ definition is 
its resistance to
the distance concentration phenomenon,
which is part of the so called curse of dimensionality
problem \cite{demartines1994,BeyerGRS99,Francois2007,Angiulli2018}.
As already recalled,
the term \textit{curse of dimensionality} 
is used to refer to difficulties arising when high-dimensional
data must be taken into account,
and one of the main aspects of this curse is 
\textit{distance concentration},
that is the tendency of distances
to become almost indiscernible
as dimensionality increases.

In this scenario \cite{demartines1994}
has shown that the expectation of the Euclidean distance
of i.i.d. random vectors increases as the square root of the dimension,
whereas its variance tends toward a constant. This implies
that high-dimensional points appear to be distributed
around the surface of a sphere 
and distances between pairs of points tend to be similar:
{according to \cite{Angiulli2018}, the expected squared Euclidean distance
of the points from their mean is $d\sigma^2$ and the 
expected squared Euclidean inter-point distance is $2d\sigma^2$,
where $\sigma$ is the standard deviation of the random variable
used to generate points.}

The distance concentration phenomenon is usually characterized in
the literature by means of a ratio between some measure related to the spread
and some measure related to the magnitude of the distances. In particular,
the conclusion is that there is concentration
when the above ratio converges to $0$ as the dimensionality
tends to infinity.
The \textit{relative variance}
$RV = \sigma_{dist} / \mu_{dist}$ \cite{Francois2007}  is a measure
of concentration for distributions, corresponding to the ratio
between the standard deviation $\sigma_{dist}$
and the expected value $\mu_{dist}$ of the distance
between pairs of points.
In the case of the Euclidean distance,
or of any other Minkowski's metric,
the relative variance of data points
generated by i.i.d. random vectors 
tends to zero as the dimensionality tends to infinity,
independently from the sample size.
As a consequence, the separation between the nearest
neighbor and the farthest neighbor of a given point tend
to become increasingly indistinct as the dimensionality
increases.

\begin{figure}[t]
\centering
\subfloat[\label{fig:curseC}]
{\includegraphics[width=0.48\columnwidth]{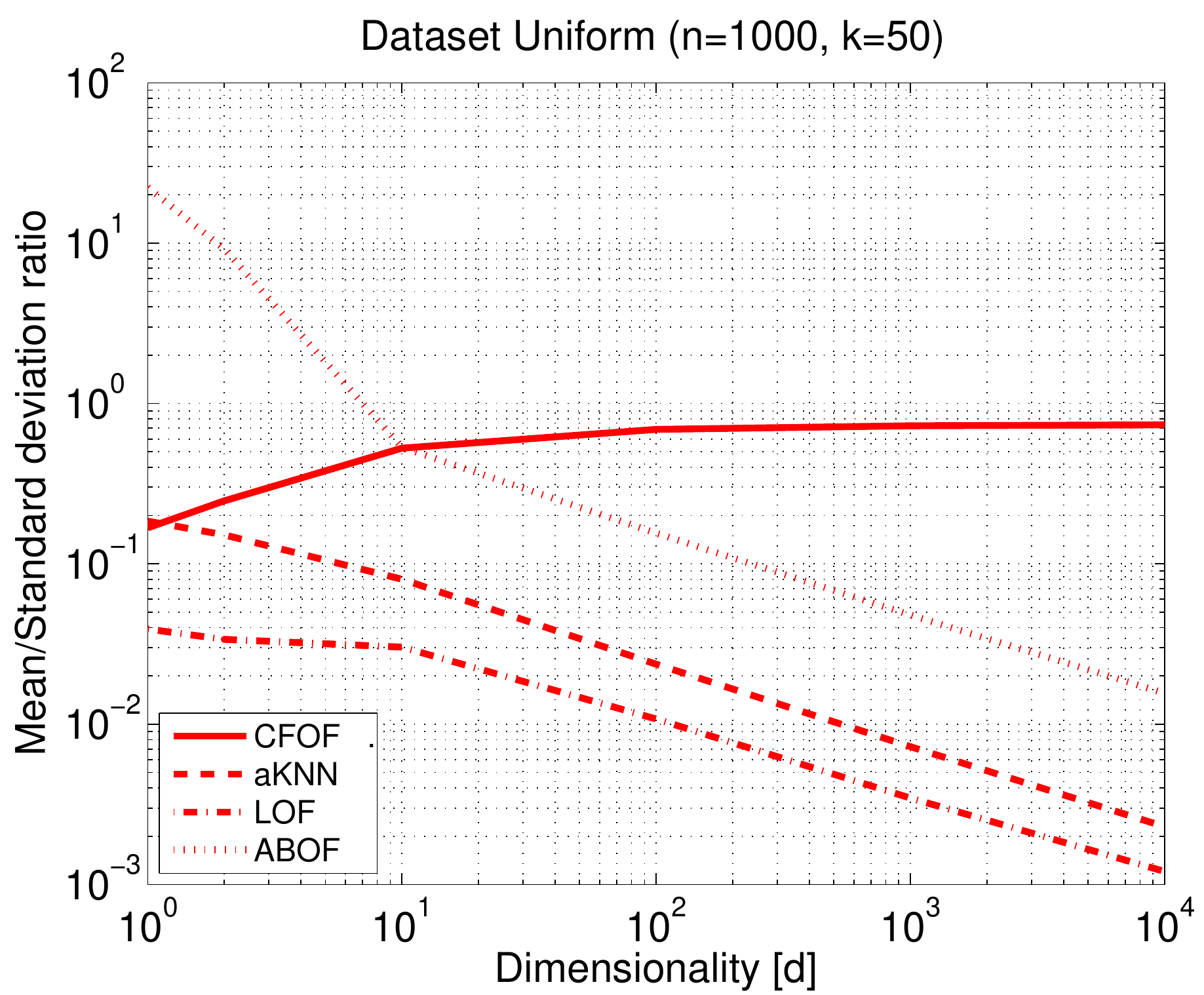}}
~
\subfloat[\label{fig:curseD}]
{\includegraphics[width=0.48\columnwidth]{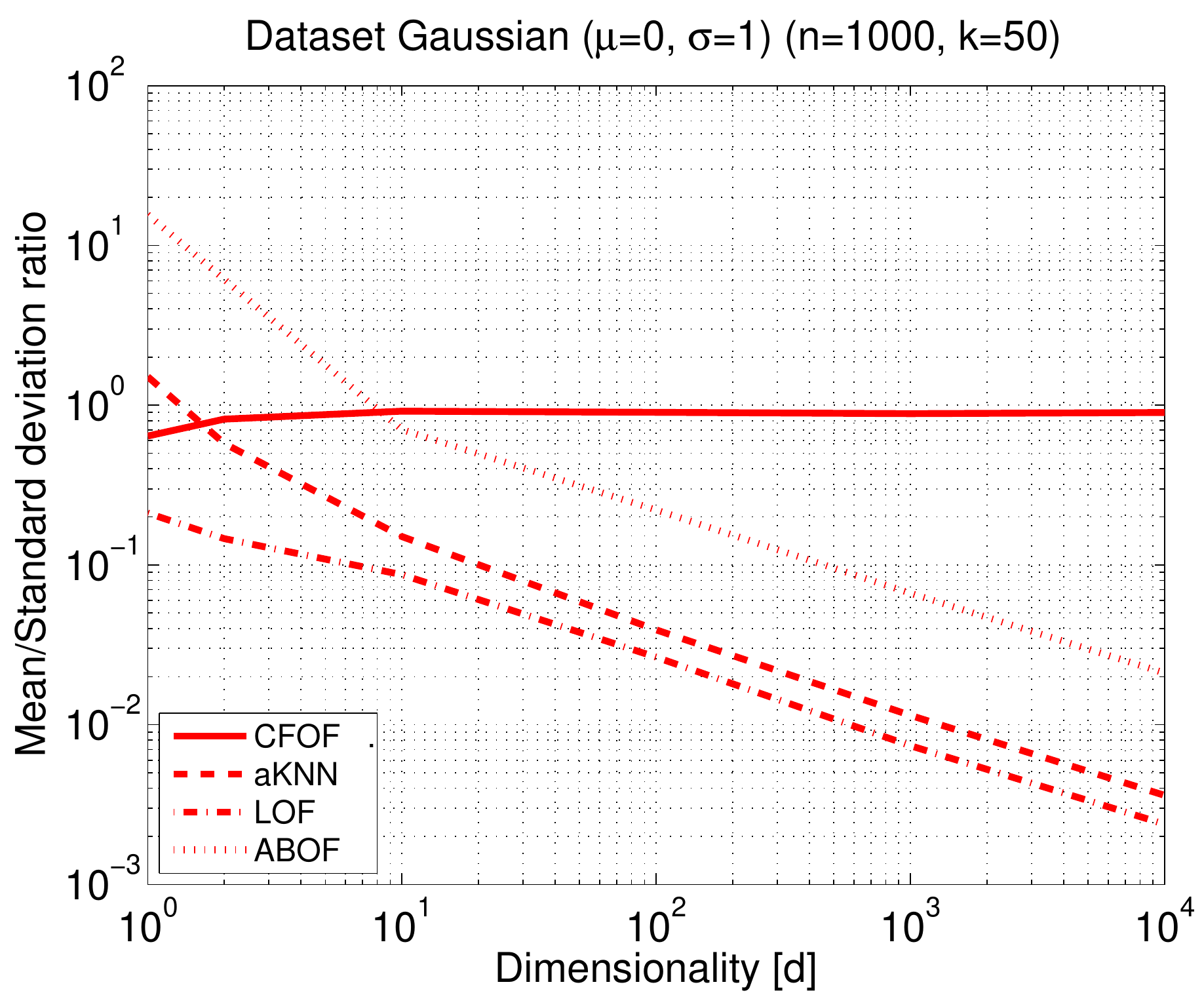}} 
\caption{Ratio $\sigma_{sc}/\mu_{sc}$ between the 
standard deviation $\sigma_{sc}$ and the mean $\mu_{sc}$
of different outlier scores $sc$
($\aKNN$, $\LOF$, ABOF, and $\CFOF$ methods)
versus the data dimensionality $d$
on uniform and normal data.}
\label{fig:curse2}
\end{figure}

\smallskip
Due to the specific role played by 
distances in their formulation,
the concentration problem also affects 
outlier scores.

To illustrate,
Figure \ref{fig:curse2}
shows the ratio 
${\sigma_{sc}}/{\mu_{sc}}$
between the standard deviation $\sigma_{sc}$
and the mean $\mu_{sc}$ of different 
families of outlier scores $sc$,
that are the distance-based method $\aKNN$ \cite{AP05}, 
the density-based method $\LOF$ \cite{BKNS00}, 
the angle-based method ABOF \cite{KriegelSZ08}, 
and $\CFOF$,
associated with a family of 
uniformly distributed (Figure \ref{fig:curseC})
and a normally distributed (Figure \ref{fig:curseD}) 
datasets having fixed size $n=1,\!000$ 
and increasing dimensionality $d\in[10^0,10^4]$,
for $k=50$.
We also computed the ratio ${\sigma_{sc}}/{\mu_{sc}}$ 
for the 
$\iForest$ \cite{LiuTZ12} outlier score $sc$,
obtaining the decreasing sequence
$0.103$, $0.078$, $0.032$, 
$0.023$, and $0.018$ for $d\in[10^0,10^4]$
on normally distributed data and similar values on uniform data.\footnote{The 
values for the parameters of the $\iForest$ algorithm
employed throughout the paper
correspond to the default values suggested in \cite{LiuTZ12}, that is to say the
sub-sampling size $\psi=256$ and the number of isolation
trees $t=100$.}

Results for each dimensionality value $d$
are obtained by ($i$) considering ten 
randomly generated different datasets, 
($ii$) computing outlier scores associated with each dataset,
($iii$) sorting scores of each dataset, 
and ($iv$) taking the average value for each rank position.
The figure and the values above reported highlight that, except for $\CFOF$,
the other scores, belonging to three different families
of techniques, exhibit a concentration effect.

Figure \ref{fig:score_distr}  
reports the sorted scores of the uniform
datasets above discussed.
For $\aKNN$ (Figure \ref{fig:score_distr_aknn})
the mean score value raises while the spread
stay limited.
For $\LOF$ (Figure \ref{fig:score_distr_lof})
all the values tend to $1$ as the dimensionality increases.
For ABOF (Figure \ref{fig:score_distr_abof})
both the mean and the standard deviation decrease 
of various orders of magnitude with the latter term
varying at a faster rate than the former one.
As for $\iForest$ (Figure \ref{fig:score_distr_iforest})
the mean stay fixed while the spread decreases.
As for $\CFOF$ (Figure \ref{fig:score_distr_cfof})
the score distributions for $d>100$ are very close
and exhibit only small differences.

\begin{figure}[t]
\centering
\subfloat[\label{fig:score_distr_aknn}]
{\includegraphics[width=0.28\textwidth]{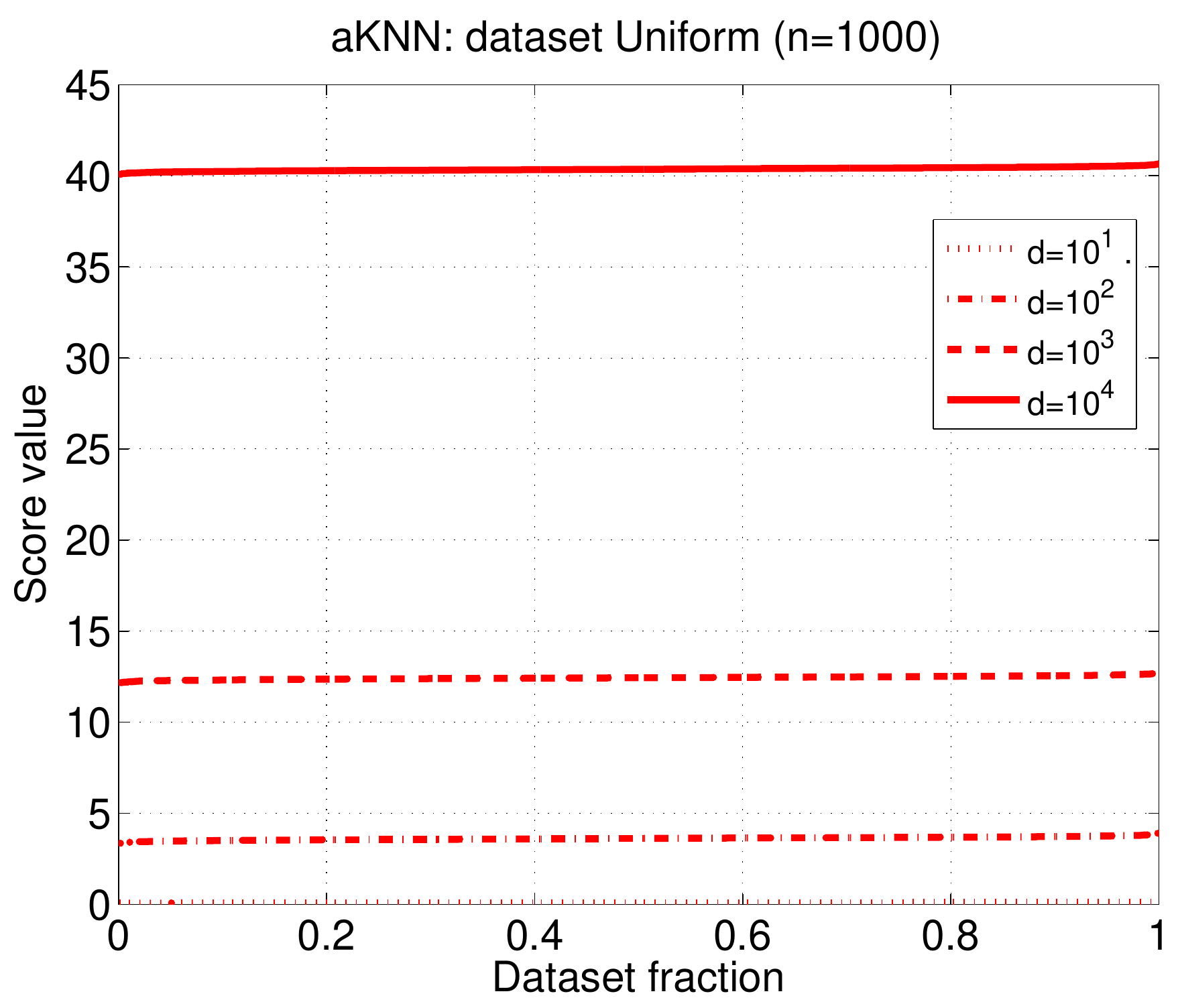}}
\subfloat[\label{fig:score_distr_lof}]
{\includegraphics[width=0.28\textwidth]{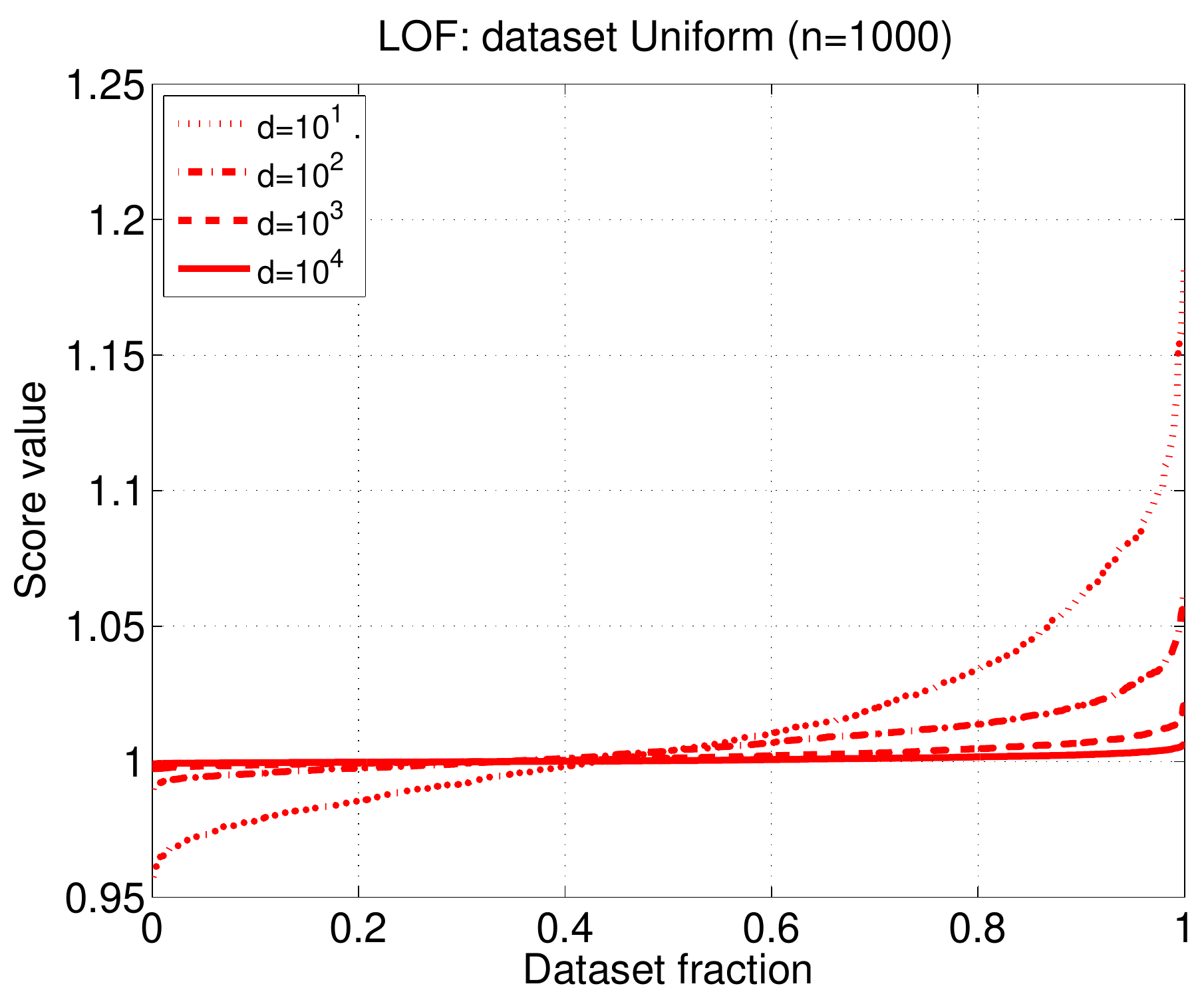}}
\subfloat[\label{fig:score_distr_abof}]
{\includegraphics[width=0.28\textwidth]{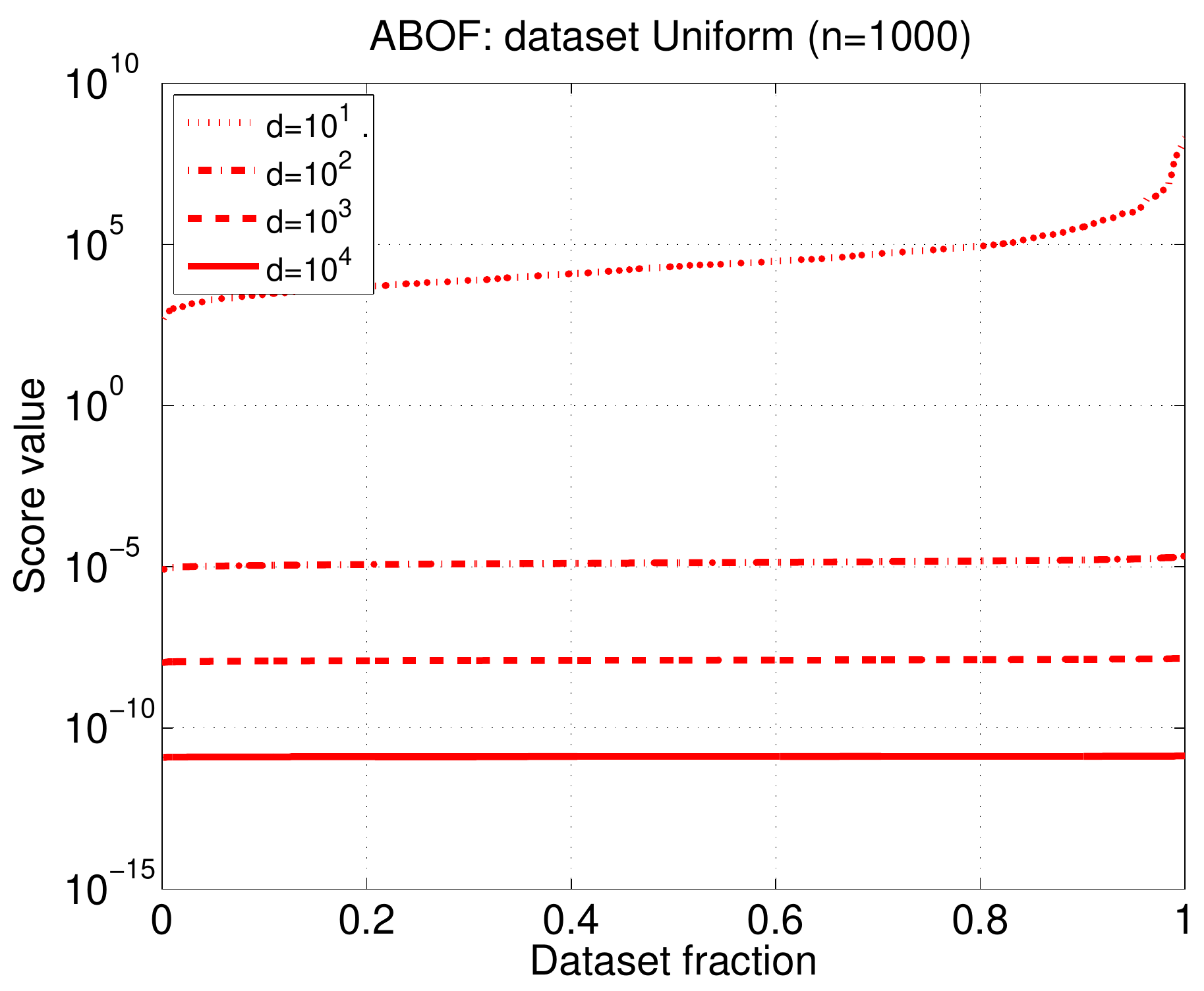}} 
\\
%
\subfloat[\label{fig:score_distr_iforest}]
{\includegraphics[width=0.28\columnwidth]{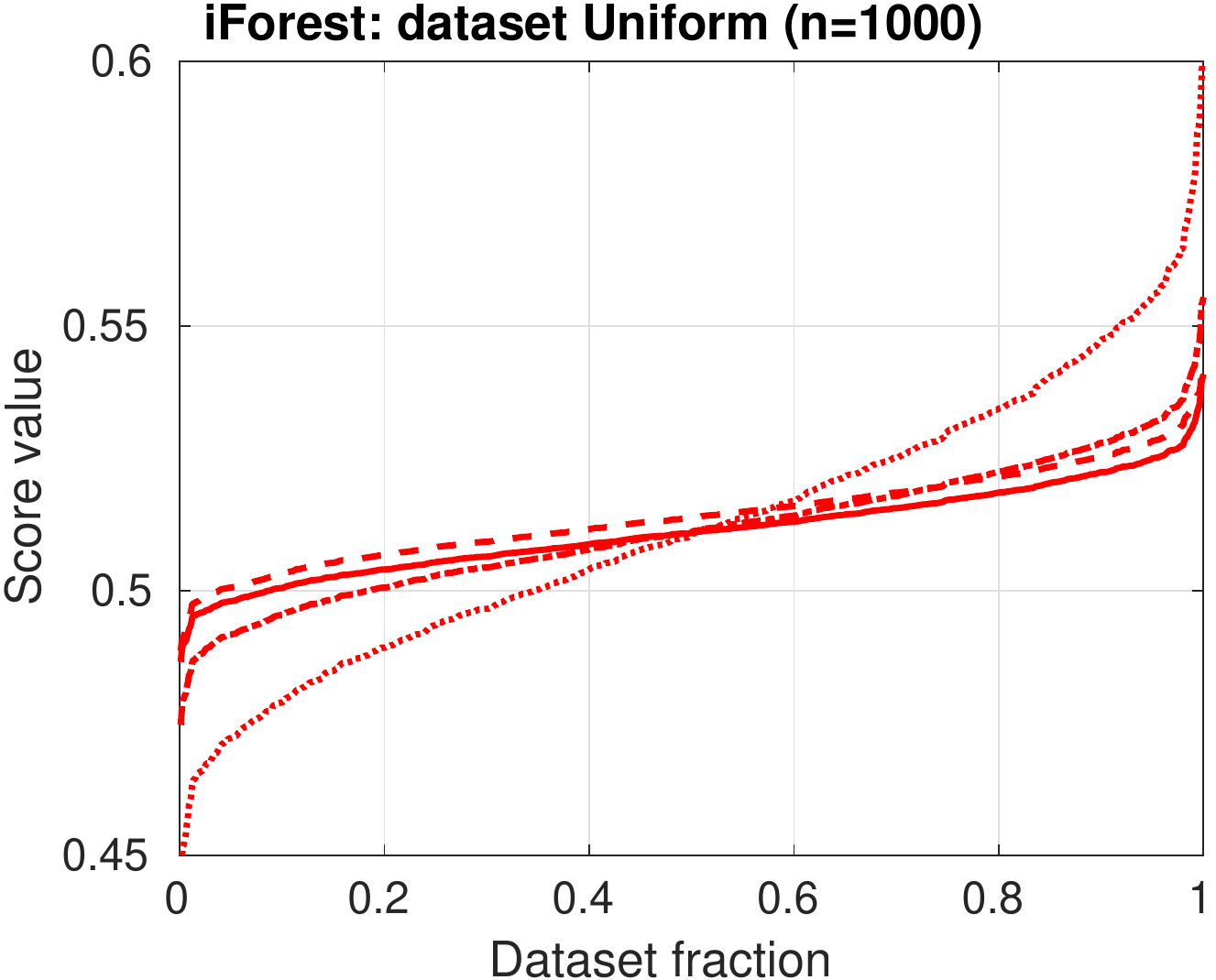}} 
\subfloat[\label{fig:score_distr_cfof}]
{\includegraphics[width=0.28\textwidth]{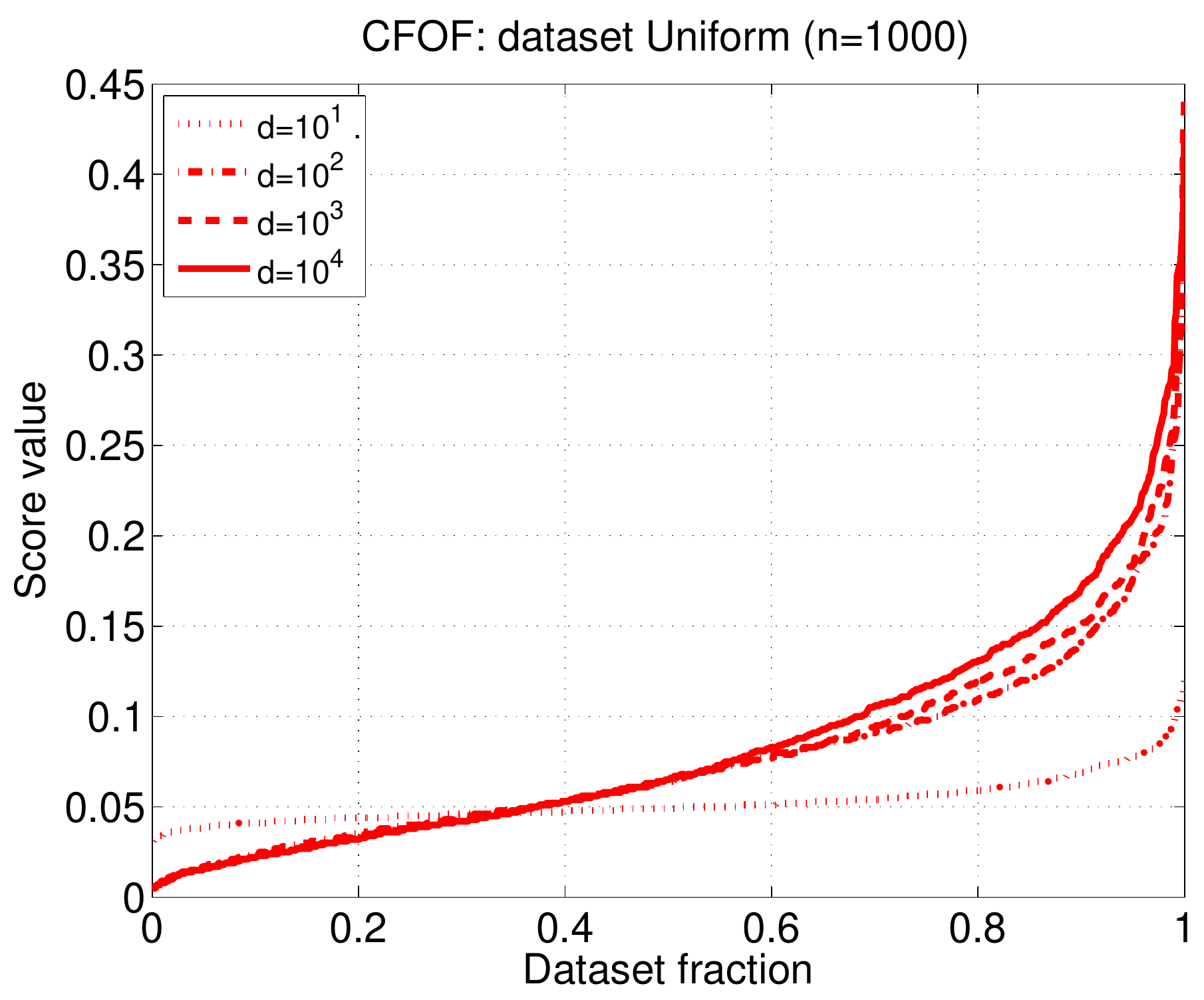}} 
\caption{Sorted outlier scores ($\aKNN$, $\LOF$, ABOF, $\iForest$, and $\CFOF$ methods) 
on uniform data
for different dimensionalities 
$d\in\{10,10^2,10^3,10^4\}$.
}
\label{fig:score_distr}
\end{figure}

\begin{figure}[t]
\centering
\subfloat[\label{fig:ground_truth2A}]
{\includegraphics[width=0.28\columnwidth]{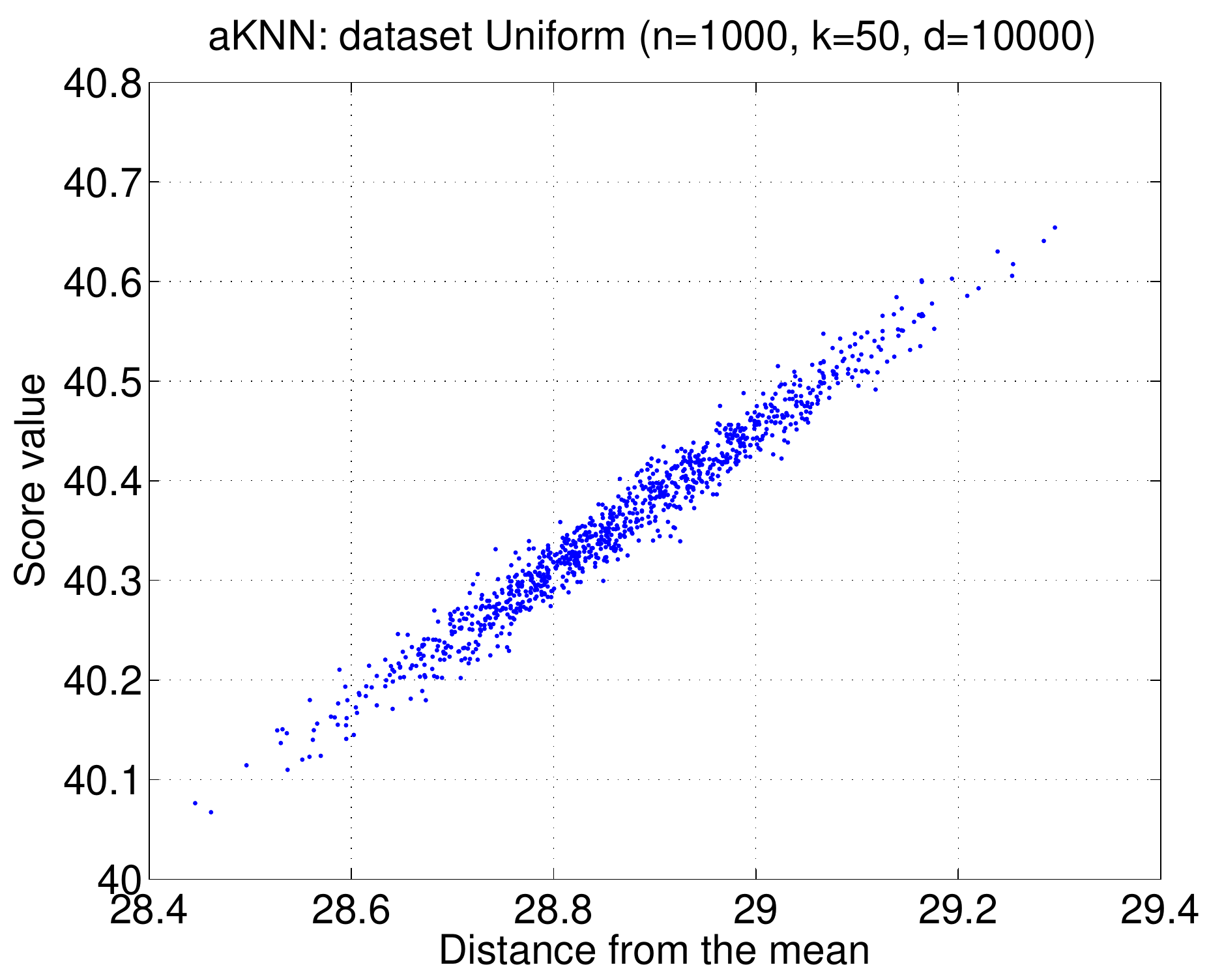}}
\subfloat[\label{fig:ground_truth2B}]
{\includegraphics[width=0.28\columnwidth]{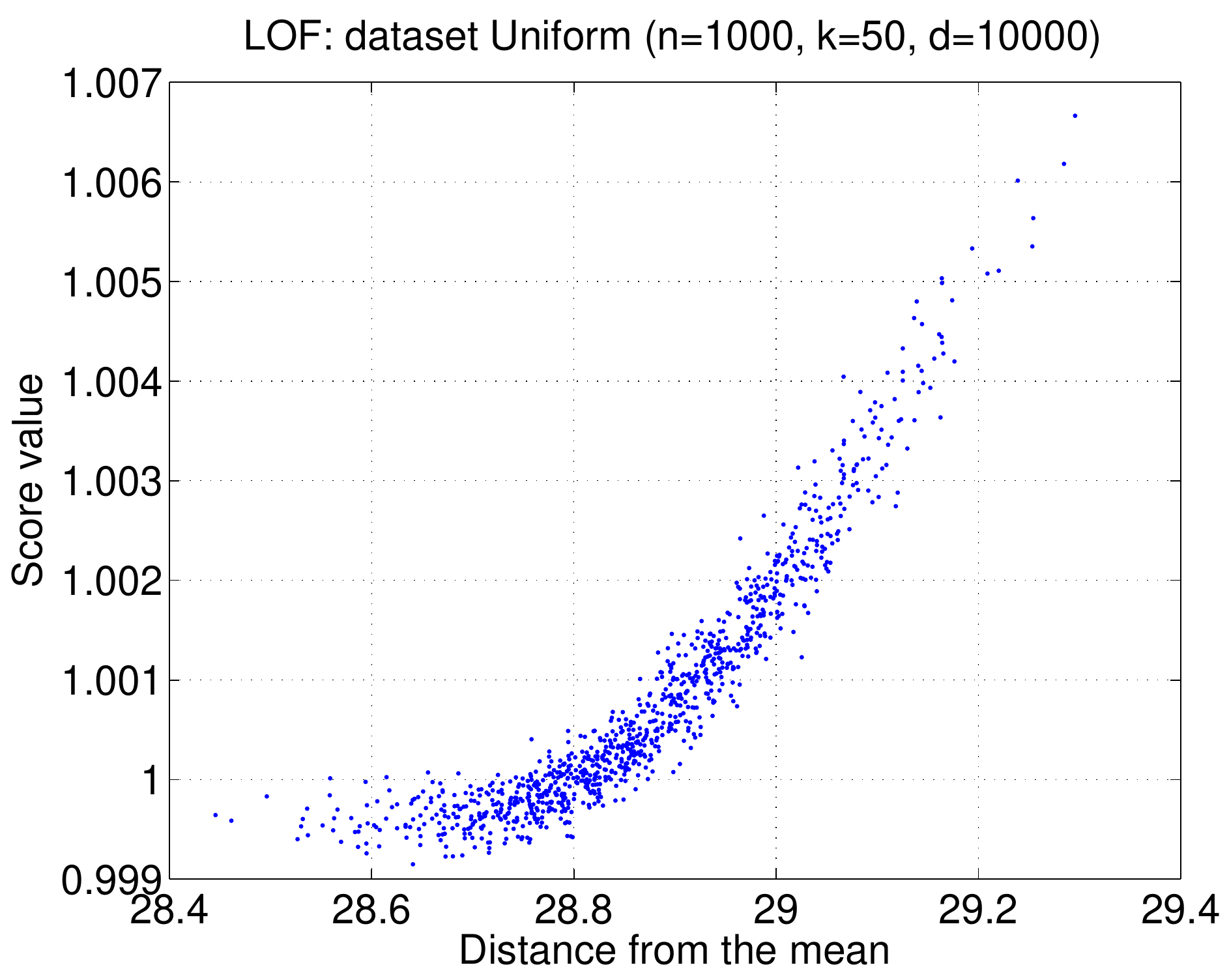}}
\subfloat[\label{fig:ground_truth2C}]
{\includegraphics[width=0.28\columnwidth]{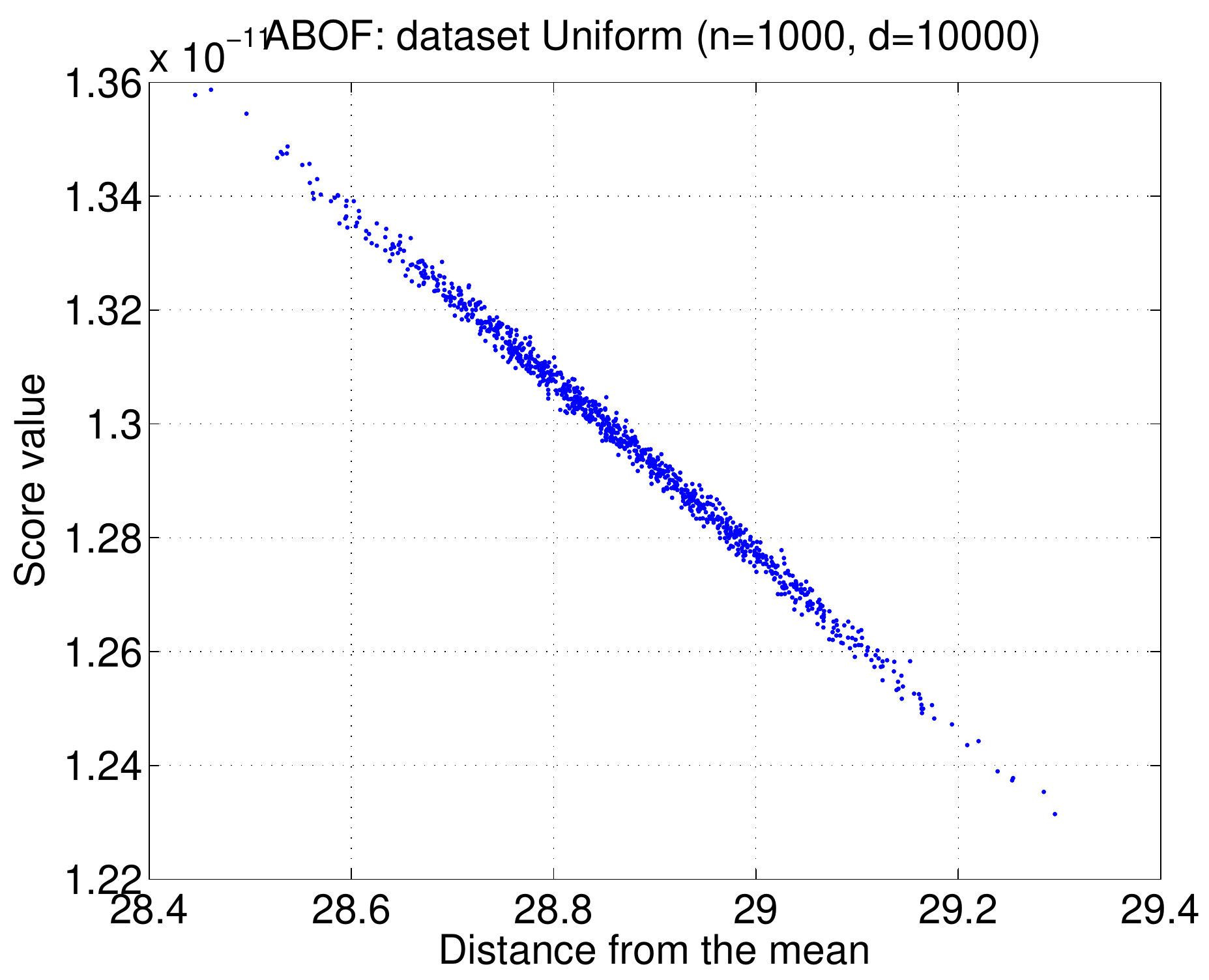}} 
\\
%
\subfloat[\label{fig:ground_truth2Ea}]
{\includegraphics[width=0.28\columnwidth]{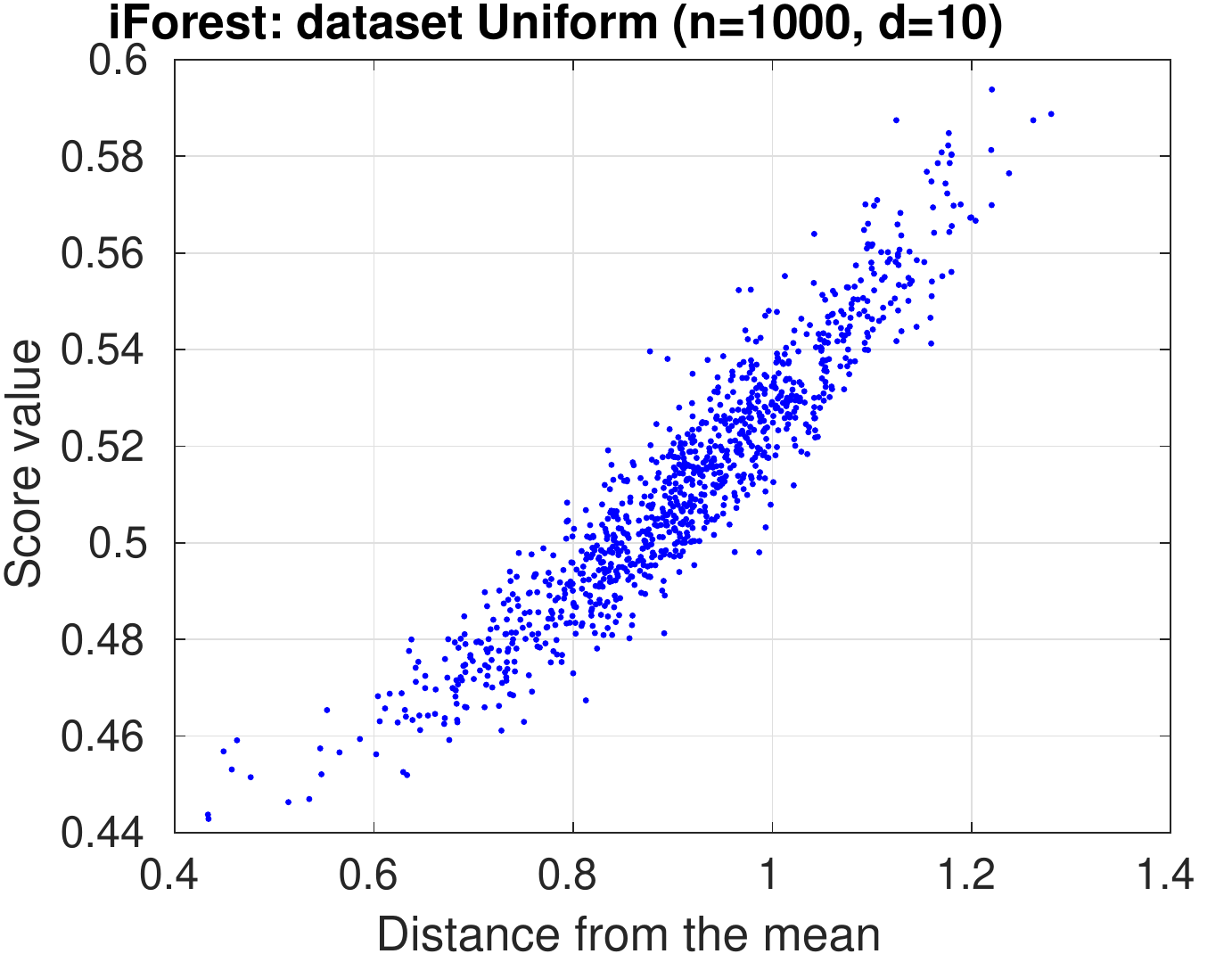}} 
\subfloat[\label{fig:ground_truth2Eb}]
{\includegraphics[width=0.28\columnwidth]{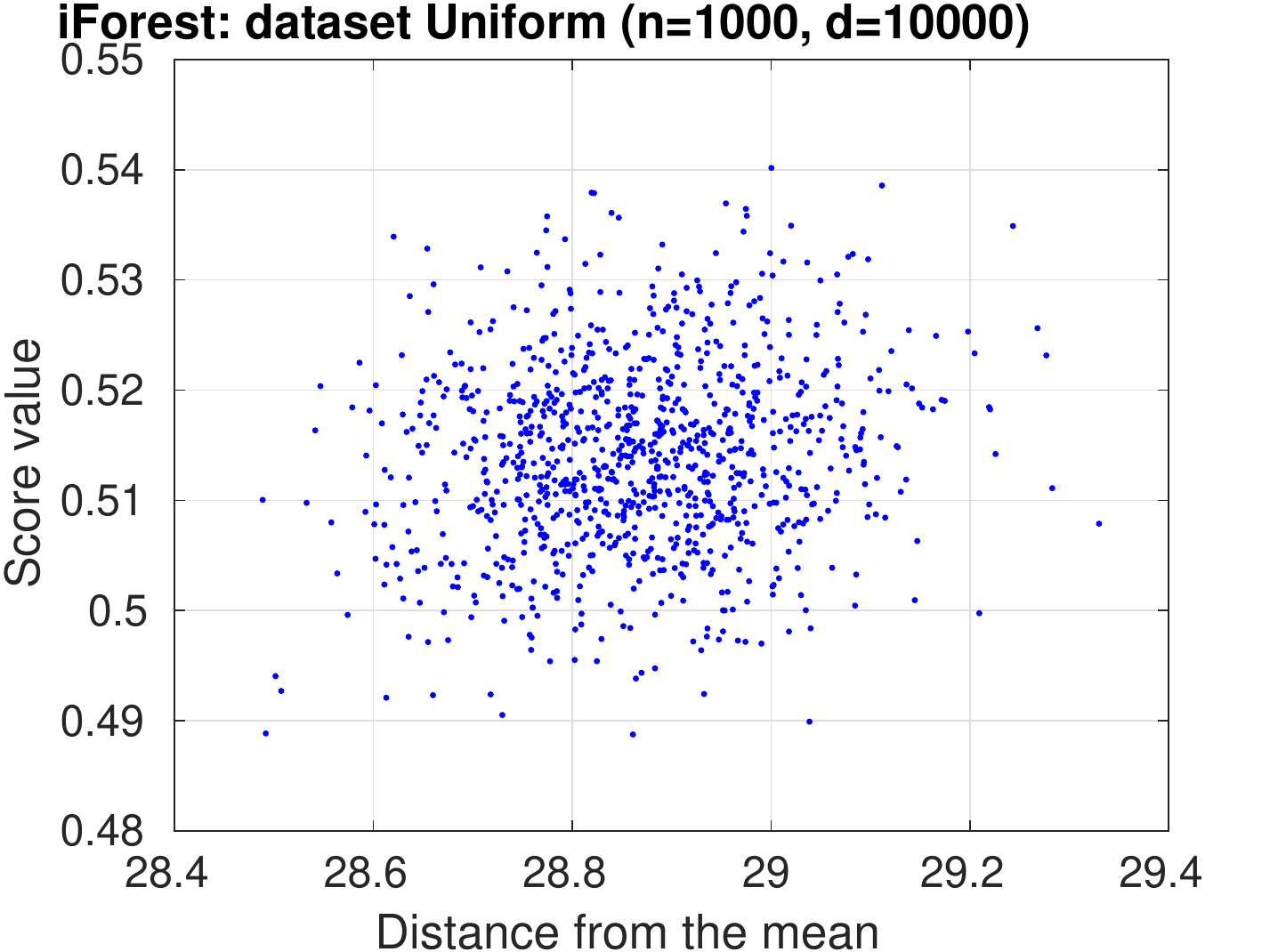}} 
%
\subfloat[\label{fig:ground_truth2D}]
{\includegraphics[width=0.28\columnwidth]{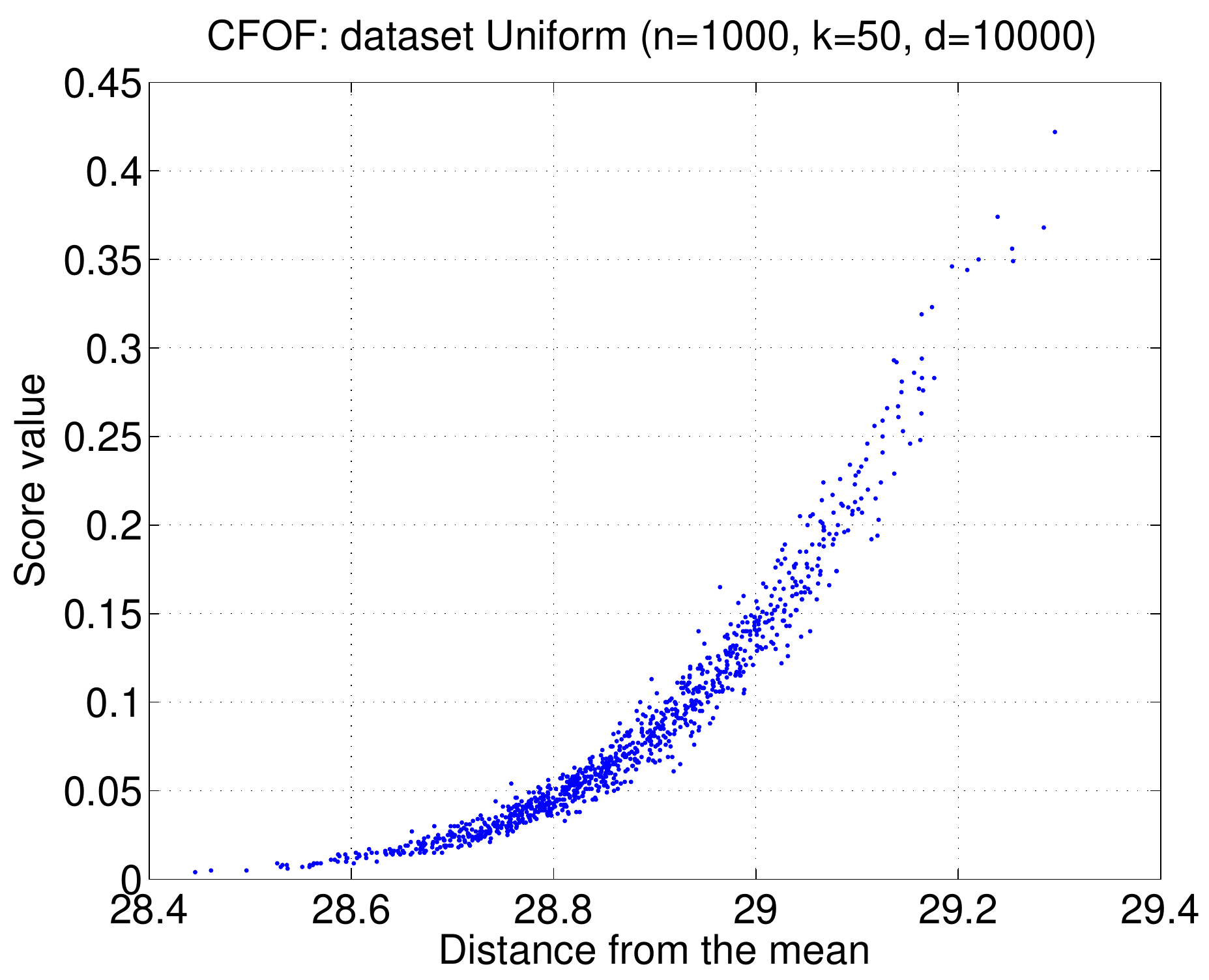}} 
\caption{Score values versus distance
from the mean 
on uniform data
for $d=10^4$ ($\aKNN$, $\LOF$, ABOF, $\iForest$, and $\CFOF$)
and for $d=10$ ($\iForest$).
}
\label{fig:ground_truth2bis}
\end{figure}

In order to visualize score distributions,
Figure \ref{fig:ground_truth2bis}
report the scatter plot of the score values
versus the distance from the mean for $d=10^4$.
As for $\iForest$ is concerned, we reported both 
the scatter plot for $d=10$ (Figure \ref{fig:ground_truth2Ea})
that the scatter plot for $d=10,\!000$ (Figure \ref{fig:ground_truth2Eb}).
Indeed, while $\iForest$ 
associates larger score values to the points maximizing their 
distance from the mean for small dimensionalities (e.g. $d=10$),
when the dimensionality increases
the quality of the scores worsens. 
The scatter plot in Figure \ref{fig:ground_truth2Eb} ($d=10,\!000$)
shows that the $\iForest$ score is almost independent of the distance
from the mean, thus witnessing that its discrimination capability is practically lost.
This can be justified by the kind of strategy pursued by $\iForest$
and by the geometry of intrinsically high-dimensional spaces.
Since the strategy of $\iForest$
consists in recursively partitioning the space by means of 
randomly generated axis-parallel hyperplanes and then using as a score
the recursion depth at which a particular point becomes isolated,
it is conceivable that this procedure 
is destined to output about the same values if the data points
are distributed around
the surface of an hypersphere,
as they tend to do in the case of intrinsically high-dimensional
spaces.

\begin{figure}[t]
\centering
\subfloat[\label{fig:curse_histA}]
{\includegraphics[width=0.48\textwidth]{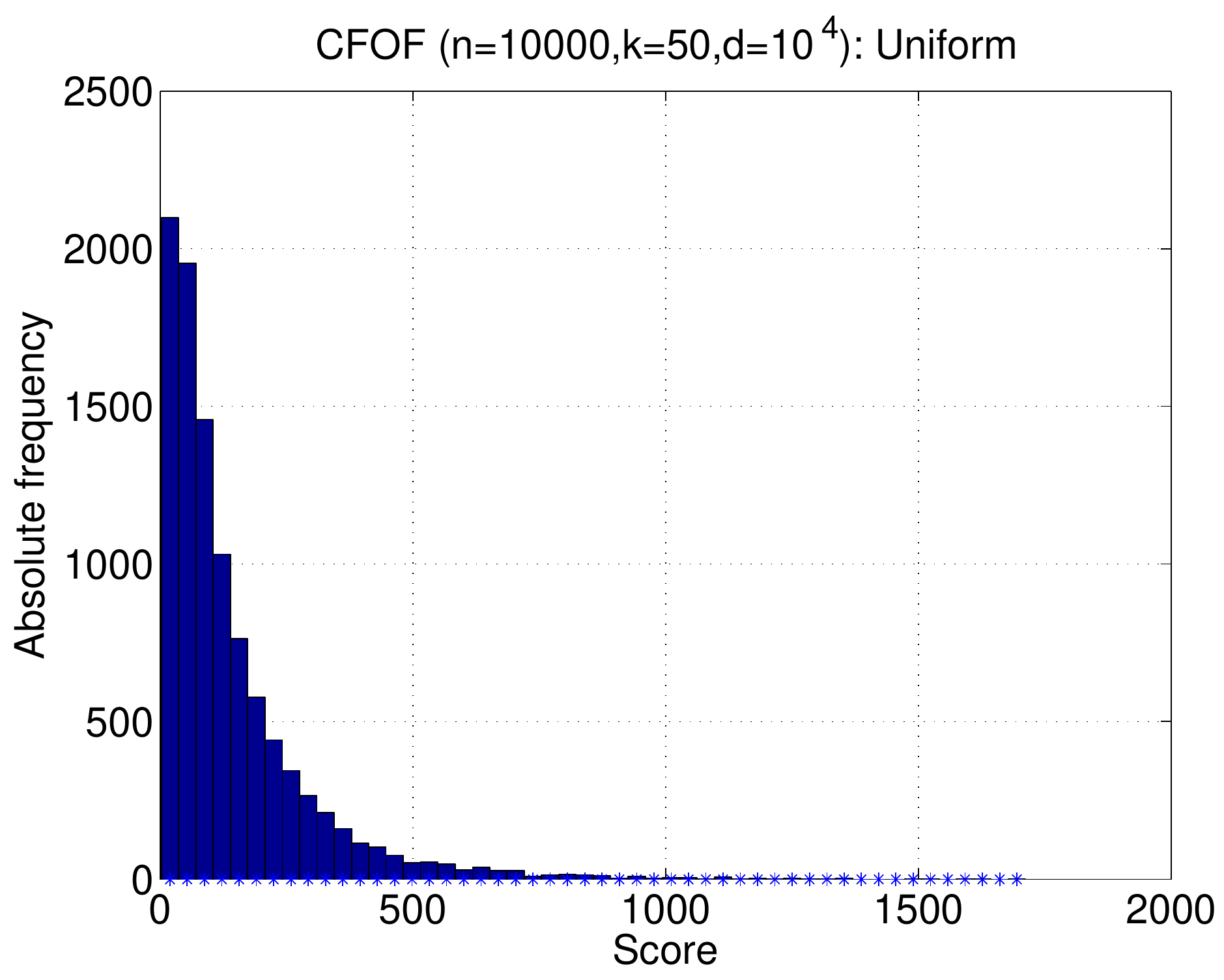}}
\subfloat[\label{fig:curse_histB}]
{\includegraphics[width=0.48\textwidth]{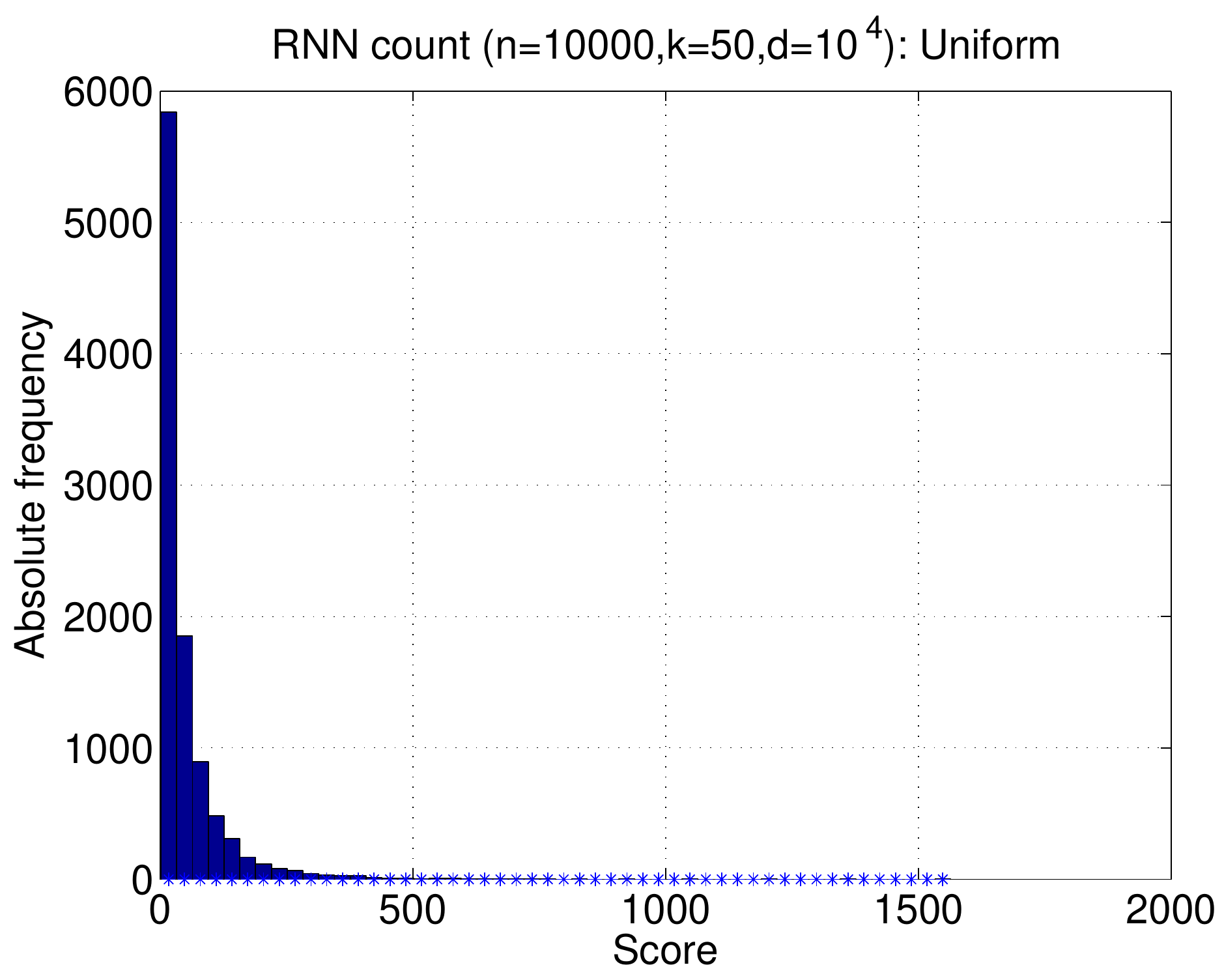}}
\\
\subfloat[\label{fig:curse_cumsumA}]
{\includegraphics[width=0.48\textwidth]{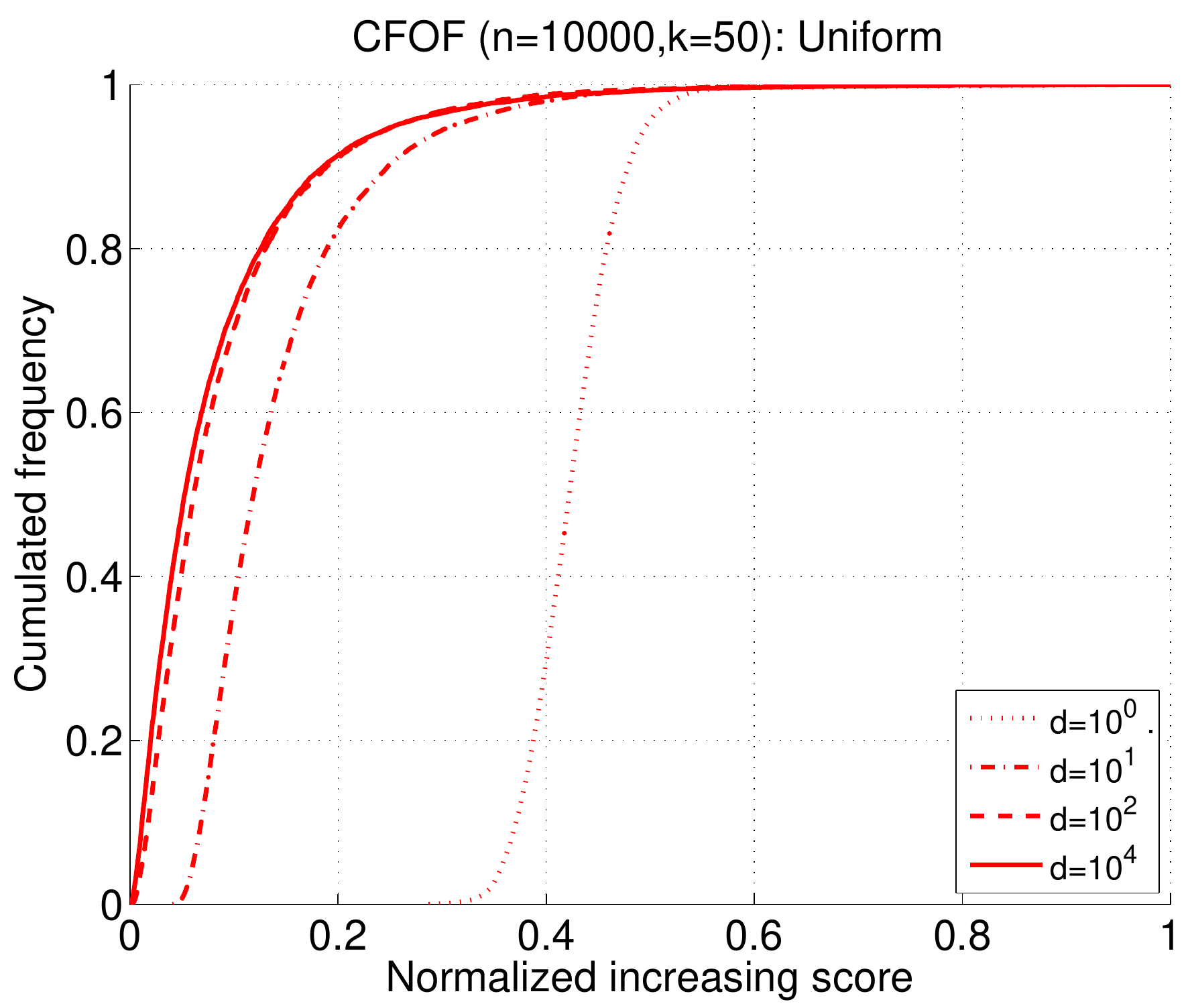}}
\subfloat[\label{fig:curse_cumsumB}]
{\includegraphics[width=0.48\textwidth]{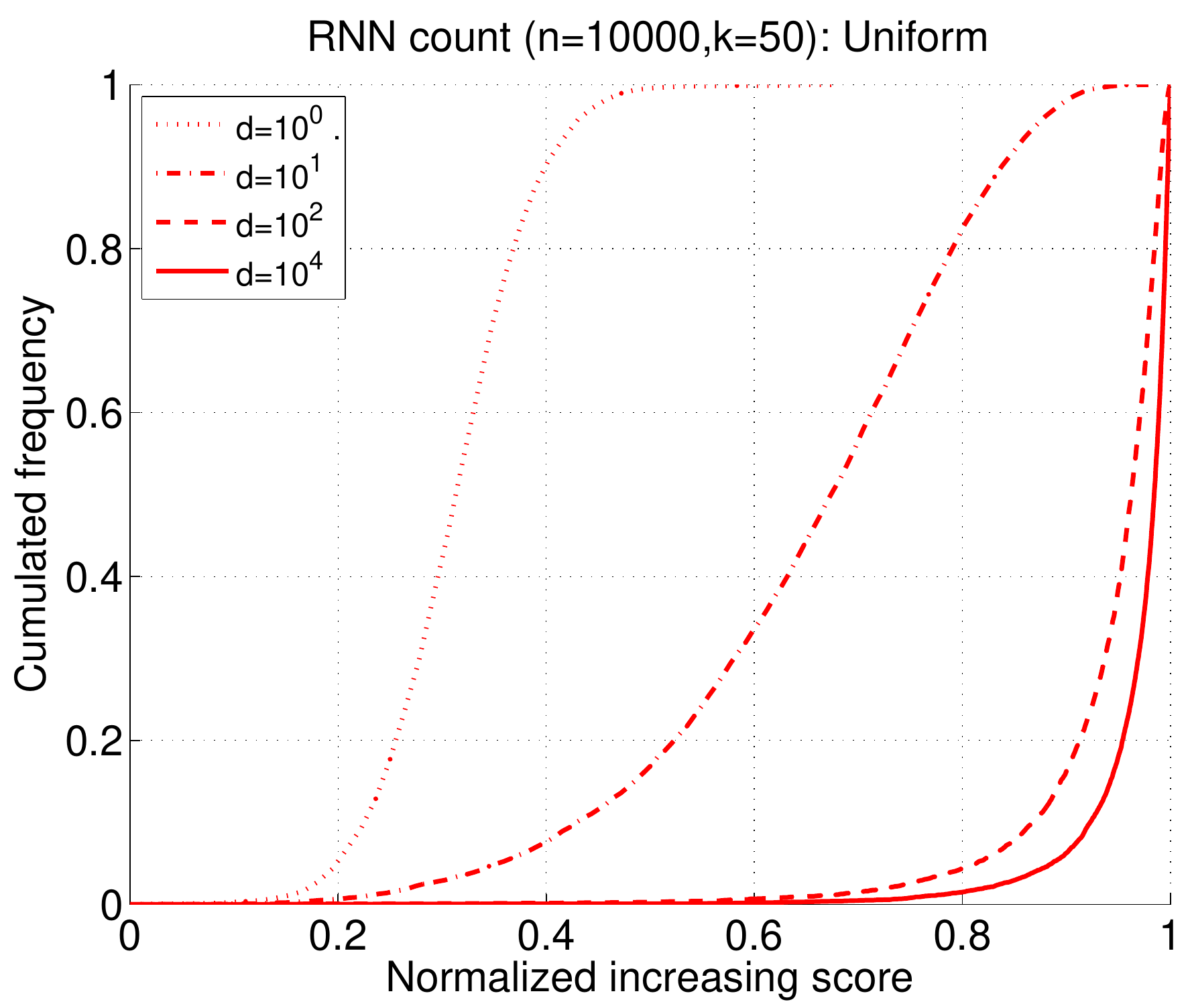}} 
\caption{Distribution of $\CFOF$ (fig. (a))
and $\RNNc$ (fig. (b)) scores on uniformly distributed
$d=10^4$ dimensional data ($n=10,\!000$ and $k=50$)
and cumulated frequency associated with normalized increasing
$\CFOF$ (fig. (c)) and $\RNNc$ (fig. (d)) scores
for different data dimensionalities $d$.
}
\label{fig:curse_cumsum}
\end{figure}

\subsection{Relationship with the hubness phenomenon}
\label{sect:cfof_hubness}

It descends from its definition that
$\CFOF$ has connections with the reverse neighborhood
size, a tool which has been also used for characterizing outliers.
The $\ODIN$ method \cite{HautamakiKF04}
uses of the reverse neighborhood
size $\N_k(\cdot)$ as an outlier score,
which we refer also as RNN count, or $\RNNc$ for short. 
Outliers are those objects 
associated with the smallest RNN counts.
However, it is well-known that the function $\N_k(\cdot)$
suffers of a peculiar problem
known as \textit{hubness} \cite{AucouturierP08,RadovanovicNI09,Angiulli2018}.
As the dimensionality of the 
space increases,
the distribution of $N_k(\cdot)$
becomes skewed to the right with increasing variance,
leading to a very large number of objects
showing very small RNN counts. 
Thus,
the number of 
\textit{antihubs}, that are objects appearing 
in a much smaller number of
$k$ nearest neighbors sets (possibly they are neighbors only
of themselves),
overcomes the number of
\textit{hubs}, that are objects that appear in many
more $k$ nearest neighbor sets than other points,
and, according to the $\RNNc$ score, the vast
majority of the dataset objects become outliers
with identical scores.

Here we provide empirical evidence that $\CFOF$
does not suffer of the hubness problem,
while we refer the reader to Section \ref{sect:concfree}
for the formal demonstration of this property.

Figures \ref{fig:curse_histA} and \ref{fig:curse_histB}
report the distribution of the
$N_k(\cdot)$ value
and of the $\CFOF$ absolute score
for a ten thousand dimensional uniform dataset.
Notice that $\CFOF$ outliers are associated with
the largest score values, hence to the tails of the distribution,
while $\RNNc$ outliers are associated with the smallest score values, hence 
with the largely populated region of the score distribution,
a completely opposite behavior.

To illustrate
the impact of the hubness problem
with the dimensionality,
Figures \ref{fig:curse_cumsumA} and
\ref{fig:curse_cumsumB}
show the 
cumulated frequency associated
with the normalized, between $0$ and $1$, increasing scores.
The normalization has been implemented 
to ease comparison.
As for $\CFOF$,
values have been obtained as
$\frac{{\rm CFOF}(x)}{\max_y {\rm CFOF}(y)}$.
As for $\RNNc$, values have been obtained as
$1-\frac{\N_k(x)}{\max_y\N_k(y)}$.

The curves clarify the 
deep difference between the two approaches.
Here both
$n$ and $k$ are held fixed, while $d$ is increasing
($d\in\{10^0,10^1,10^2,10^4\}$, the curve for $d=10^3$
is omitted for readability, since it is very close to $10^4$).
As for $\RNNc$,
the hubness problem is already evident for $d=10$,
where objects with a normalized score $\ge 0.8$ 
correspond to about the $20\%$ of the dataset,
while the curve for $d=10^2$ closely resembles that for $d=10^4$,
where the vast majority of the dataset objects have
a normalized score close to $1.0$.
As for $\CFOF$, the number of points associated with large 
score values always corresponds to a very small fraction of the
dataset population.

\section{Concentration free property of CFOF}
\label{sect:concfree}

In this section we theoretically ascertain properties
of the $\CFOF$ outlier score.

The rest of the section is organized as follows.
Section \ref{sect:prelim} introduces the notation exploited
throughout the section and some basic definitions.
Section \ref{sect:kurtosis} recalls the concept of kurtosis.
Section \ref{sect:cfofcdf} derives the theoretical cumulative 
distribution function of the $\CFOF$ score.
Section \ref{sect:out_conc} provides the definition of 
concentration of outlier scores together with the 
proof that the $\CFOF$ score does not concentrate 
and Section \ref{sect:cfofcdfvskurt} 
discusses the effect
of the data kurtosis on the distribution of $\CFOF$ outlier scores.
Section \ref{sect:semilocal} studies the behavior $\CFOF$ in
presence of different distributions and establishes its semi--locality property.
Finally, Section \ref{sect:conc_other} studies the concentration
properties of distance-based and density-based outlier scores, and
Section \ref{sect:conc_rnnc}
studies the concentration properties of reverse nearest neighbor-based
outlier scores.

\subsection{Preliminaries}\label{sect:prelim}

The concept of intrinsic dimensionality is related to the analysis of 
independent and identically distributed (i.i.d.) data.
Although variables used to identify each datum could not be statistically
independent, ultimately, the intrinsic dimensionality of the data is identified
as the minimum number $D$ 
of variables needed to represent the data itself \cite{MaatenPH09}.
This corresponds in linear spaces to the number of linearly independent vectors needed 
to describe each point.
Indeed, if random vector components are 
not independent, the concentration phenomenon is still present provided that
the actual number $D$ of ``degrees of freedom'' is sufficiently large \cite{demartines1994}. 
Thus, results 
derived for i.i.d. data continue to be valid provided that the dimensionality $d$
is replaced with $D$.

In the following we use $d$-dimensional i.i.d. random vectors as a 
model of intrinsically high dimensional space.
Specifically, boldface uppercase letters with the symbol ``$(d)$'' 
as a superscript, such as
$\Vect{X}$, $\Vect{Y}$, $\ldots$, denote $d$-dimensional 
random vectors taking values in $\mathbb{R}^d$. The components $X_i$ ($1\le i\le d$) of a 
random vector $\Vect{X}=(X_1,X_2,\ldots,X_d)$ are 
random variables having pdfs $f_{X_i}$ (cdf $F_{X_i}$). 
A random vector is said
independent and identically distributed 
(i.i.d.), if its components are independent
random variables having common pdf $f_X$ (cdf $F_X$).
The generic component $X_i$ of $\Vect{X}$ is also referred to as $X$
when its position does not matter.

Lowercase letters with ``$(d)$'' as a superscript, such as $\vect{x}$, $\vect{y}$, $\ldots$,
denote a specific $d$-dimensional vector taking value in $\mathbb{R}^d$.

Given a random variable $X$, {$\mu_X$ and $\sigma_X$} denote
the mean and standard deviation of $X$.
The symbol $\mu_k(X)$, or simply $\mu_k$ when $X$ is clear from the context,
denotes the $k$th central moment $\E[(X-\mu_X)^k]$ of $X$ ($k\in\mathbb{N}$).
When we used moments, we assume that they exists finite.

A sequence of independent non-identically distributed random variables 
$X_1,X_2,\ldots,X_d$ having non-null variances and finite central moments
$\mu_{i,k}$ ($1\le i\le d$)
is said to have \textit{comparable} central moments if there
exist positive constants $\mu_{\max}\ge\max\{|\mu_{i,k}|\}$
and $\mu_{\min}\le\min\{|\mu_{i,k}|:\mu_{i,k}\neq 0\}$. Intuitively,
this guarantees that the ratio between the greatest and the smallset non-null
moment remains limited.

An independent non-identically distributed random vector $\Vect{X}$
whose components have comparable central moments,
can be treated as an i.i.d. random vector 
whose generic component $X_i$ is such that  
the $h$-th degree ($h\in\mathbb{N}$) of its $k$th central moment $\mu_k(X_i)$
($k\in\mathbb{N}$), that is $\mu_k(X_i)^h$, 
is given by
the average of the $h$-th degree of the central moments 
of its components $X_i$ \cite{Angiulli2018}, defined as follows
\begin{equation}\label{eq:moments}
\mu_k(\Vect{X})^h \equiv
\tilde{\mu}_k^h(\Vect{X}) = \frac{1}{d} \sum_{i=1}^d \E[(X_i-\mu_{X_i})^k]^h 
\end{equation}
This means that all the results given for i.i.d. $d$-dimensional random vectors can be
immediately extended to $d$-dimensional independent non-identically distributed 
random vectors
having comparable central moments and, more in the general, 
to the wider class of real-life data
having $D$ degrees of freedom
with comparable central moments.

$\Phi(\cdot)$ ($\phi(\cdot)$, resp.) denotes the cdf 
(pdf, resp.) of the normal standard distribution.

Results given for distributions in the following,
can be applied to finite set of points by assuming large samples.

\begin{definition}[Squared norm standard score]
Let $\vect{x}$ be a realization of a random vector $\Vect{X}$.
Then, $z_{x,\Vect{X}}$ denotes the \textit{squared norm standard score} 
of $\vect{x}$, that is
\[ z_{x,\Vect{X}} = \frac{\norm{\vect{x}-\vect{\mu}_X}^2 - 
\mu_{\norm{\Vect{X}-\vect{\mu}_X}^2}}{\sigma_{\norm{\Vect{X}-\vect{\mu}_X}^2}}, \]
where $\vect{\mu}_X=(\mu_{X_1},\mu_{X_2},\ldots,\mu_{X_d})$ denotes the mean vector of $\Vect{X}$.
{Clearly, if all the components of the mean vector $\vect{\mu}_X$
assume the same value $\mu_X$, e.g. as in the case of i.i.d. vectors,
then $\vect{\mu}_X$ can be replaced by $\mu_X$.}
The notation $z_x$ is used as a shorthand for $z_{x,\Vect{X}}$ whenever $\Vect{X}$
is clear from the context.
\end{definition}

An outlier score function $sc_{Def}$, or outlier score for simplicity,
according to outlier definition $Def$,
is a function $sc_{Def} : \mathbb{D}\times{\wp}(\mathbb{D}) \mapsto \mathbb{S}$
that, 
given
a set of $n$ objects, or \textit{dataset},
$DS\subseteq\mathbb{D}$ (or $DS\in{\wp}(\mathbb{D})$, where ${\wp}(\mathbb{D})$
denotes the power set of $\mathbb{D}$),
and an object $x\in{DS}$,
returns a real number in 
the interval $\mathbb{S}\subseteq\mathbb{R}$, also said the outlier score (value) of $x$.
The notation $sc_{Def}(\vect{x},\Vect{X})$ is used to denote the outlier score
of $\vect{x}$ in a dataset whose elements are realizations of the random vector
$\Vect{X}$.
The notation $sc_{Def}(x)$ is used 
when $DS$ or $\Vect{X}$ are clear from the context.

\begin{definition}[Outliers]
Given parameter $\alpha\in(0,1)$, the 
\textit{top-$\alpha$ outliers}, or simply \textit{outliers} whenever
$\alpha$ is clear from the context,
in a dataset $DS$ of $n$
points according to outlier definition $Def$, are the $\alpha n$ points
of $DS$ associated with the largest values of score $sc_{Def}$.\footnote{Some definitions
$Def$ associate outliers with the smallest values of score $sc_{Def}$. In these
cases we assume to replace $Def$ with $Def'$ having outlier score $sc_{Def'}=-sc_{Def}$.}
\end{definition}

\subsection{Kurtosis}\label{sect:kurtosis}

\medskip
The kurtosis is a measure of the tailedness of 
the probability distribution of a real-valued random variable,
originating with Karl Pearson \cite{Pearson1905,FioriZ09,Westfall14}. 
Specifically, given random variable $X$, 
the \textit{kurtosis} $\kappa_X$, or simply $\kappa$ whenever $X$ is clear from the context,
is the fourth standardized moment of $X$ 
\begin{equation}\label{eq:kurtosis}
\kappa_X = \E\left[ \left( \frac{X - \mu_X}{\sigma_X} \right)^4 \right] 
= \frac{\mu_4(X)}{\mu_2(X)^2}.
\end{equation}
Higher kurtosis is the result of infrequent extreme deviations or outliers, 
as opposed to frequent modestly sized deviations.
Indeed, since kurtosis is the 
expected value of the standardized data raised to the fourth power. 
data within one standard deviation of the mean
contribute practically nothing to kurtosis 
(note that raising a number that is less than $1$ to the fourth power 
makes it closer to zero),
while the data values that almost totally
contribute to kurtosis are those outside the above region,
that is the outliers.

The lower bound is realized by the Bernoulli distribution with $0.5$ success
probability, having kurtosis $\kappa = 1$. 
\textit{Note that extreme platykurtic distributions, that is
having kurtosis $\kappa=1$, have no outliers.}
However,
there is no upper limit to the kurtosis 
of a general probability distribution, and it may be infinite.

The kurtosis of any univariate normal distribution is $3$, regardless 
of the values of its parameters. It is common to compare 
the kurtosis of a distribution to this value. 
Distributions with kurtosis equal to $3$ are called \textit{mesokurtic}, or 
\textit{mesokurtotic}. 
Distributions with kurtosis less 
than $3$ are said to be \textit{platykurtic}. These distributions 
produce fewer and less extreme outliers than does the normal distribution. An 
example of a platykurtic distribution is the uniform distribution, which has
kurtosis $1.8$. 
Distributions with kurtosis greater than $3$ are said to be 
\textit{leptokurtic}. An example of a leptokurtic distribution is the Laplace 
distribution, which has tails that asymptotically approach zero more slowly than 
a normal distribution, and therefore produces more outliers than the normal distribution. 

\subsection{Theoretical cdf of the $\CFOF$ score}\label{sect:cfofcdf}

Next, we derive the theoretical cdf and pdf of the $\CFOF$ outlier score
together with the expected score associated with a generic realization
of an i.i.d. random vector.

\begin{theorem}\label{theorem:cfofcdf}
Let $\vect{DS}$ be a dataset consisting of realizations of an i.i.d.
random vector $\Vect{X}$ and let $\vect{x}$ be an element of $\vect{DS}$.
For arbitrarily large dimensionalities $d$,
the expected value and 
the cdf of the CFOF score are
\begin{eqnarray}
\label{eq:cfofexpected}
\CFOF(\vect{x}) & \approx & \Phi\left(
\frac{z_{x} \sqrt{\kappa-1} + 2\Phi^{-1}(\varrho) }{ 
\sqrt{\kappa+3} }
\right) \mbox{, and}
\\
\label{eq:cfofcdf}
  \forall \kappa>1, 
 Pr[ \CFOF(\Vect{X}) \le s ] & \approx & \Phi \left(
 \frac{ \Phi^{-1}(s)\sqrt{\kappa+3} - 2\Phi^{-1}(\varrho) }
 {\sqrt{\kappa-1}}
 \right){.}
\end{eqnarray}
{where $\kappa$ denotes the kurtosis of the random variable $X$.}
\end{theorem}

\begin{proof}
It is shown in \cite{Angiulli2018} that, 
for arbitrary large values of $d$,
the expected number of $k$-occurrences of $\vect{x}$ 
in a dataset consisting of $n$ realizations of the random vector $\Vect{X}$ 
is given by (see Theorem 30 of \cite{Angiulli2018})
\begin{equation}\label{eq:nk}
\N_k(\vect{x}) = n\cdot Pr[\vect{x}\in\NN_k(\Vect{X})] \approx
n
\Phi\left(
\frac{\Phi^{-1}(\frac{k}{n})\sqrt{\mu_4+3\mu_2^2} - z_x 
\sqrt{\mu_4-\mu_2^2}}{2\mu_2}
\right).
\end{equation}
Let $\tau(z_x)$ denote 
the smallest integer $k'$ such
that $\N_{k'}(\vect{x})\ge n\varrho$.
By exploiting the equation above it can be concluded that
\begin{equation}\label{eq:cfofz0}
\tau(z_x) \approx
n \Phi\left(
\frac{z_{x} \sqrt{\mu_4-\mu_2^2} + 2\mu_2\Phi^{-1}(\varrho) }{ 
\sqrt{\mu_4+3\mu_2^2} }
\right).
\end{equation}
Since $\CFOF(\vect{x})=k'/n = \tau(z_x)/n$,
the expression of Equation \eqref{eq:cfofexpected} follows
by expressing moments $\mu_2$ and $\mu_4$ in terms of the kurtosis 
$\kappa={\mu_4}/{\mu_2^2}$.

For $\kappa=1$ the $\CFOF$ score is constant
and does not depend on $\vect{x}$ and $z_x$.
However, for $\kappa>1$, since $\Phi(\cdot)$ is a cdf,
the $\CFOF$ score is monotone increasing with the 
standard score $z_x$, thus
$\CFOF(x_1^{(d)}) \le \CFOF(x_2^{(d)})$ if and only if
$z_{x_1} \le z_{x_2}$.

As for the cdf 
$Pr[ \CFOF(\Vect{X}) \le s ]$
of the CFOF score, we need to determine
for which values of $\vect{x}$ the condition
$\CFOF(\vect{x}) \le s$ holds. By levaraging
Equation \eqref{eq:cfofexpected}
\begin{multline*}
\CFOF(\vect{x}\le s) 
~~\Longleftrightarrow~~
\Phi \left( 
\frac{z_x\sqrt{\kappa-1} + 2\Phi^{-1}(\varrho)}{\sqrt{\kappa+3}} 
\right) \le s 
~~\Longleftrightarrow~~ 
\\ ~~\Longleftrightarrow~~
 z_x \le 
 \frac{ \Phi^{-1}(s)\sqrt{\kappa+3} - 2\Phi^{-1}(\varrho) }
 {\sqrt{\kappa-1}}
~~\Longleftrightarrow~~ 
z_x \le t(s).
\end{multline*}
Consider the squared norm $\|\Vect{X}-\mu_X\|^2 = \sum_{i=1}^d (X_i-\mu_X)^2$.
Since the squared norm is the sum of $d$ i.i.d. random variables,
as $d\rightarrow\infty$, by the Central Limit Theorem, the distribution of the 
standard score of $\sum_{i=1}^d (X_i-\mu_X)^2$
tends to a standard normal distribution. This implies that
$\norm{\Vect{X}-\mu_X}^2$ approaches a normal distribution with mean 
$\mu_{\norm{\Vect{X}-\mu_X}^2}$ and standard deviation $\sigma_{\norm{\Vect{X}-\mu_X}^2}$.
Hence, for each $z\ge 0$,
\[ 
Pr\left[ 
\frac{\norm{\Vect{X}-\mu_X}^2-\mu_{\norm{\Vect{X}-\mu_X}^2}}{\sigma_{\norm{\Vect{X}-\mu_X}^2}}
\le z \right] = \Phi(z).
\]
Thus, the probability $Pr[z_x \le t(s)]$ converges to $\Phi(t(s))$,
and the expression of Equation \eqref{eq:cfofcdf} follows.
\end{proof}

In the following, the cdf of the CFOF score will be denoted also as
$F_{\rm CFOF}(s; \varrho, \Vect{X})$, or simply $F_{\rm CFOF}(s)$ when
$\varrho$ and the random vector $\Vect{X}$ {generating the dataset}
are clear from the context.
As for the pdf of the $\CFOF$ score, $\forall \kappa>1$
\begin{multline}\label{eq:cfofpdf}
f_{\rm CFOF}(s) =  \frac{\rm d}{{\rm d} s} F_{\rm CFOF}(s) 
= \sqrt{\frac{\kappa+3}{\kappa-1}} \cdot 
 \frac{1}{\phi(\Phi^{-1}(s))}\cdot 
 \phi \left(
 \frac{ \Phi^{-1}(s)\sqrt{\kappa+3} - 2\Phi^{-1}(\varrho) }
 {\sqrt{\kappa-1}}
 \right).
\end{multline}

\begin{figure}[t]
\centering
\subfloat
{\includegraphics[width=0.48\columnwidth]{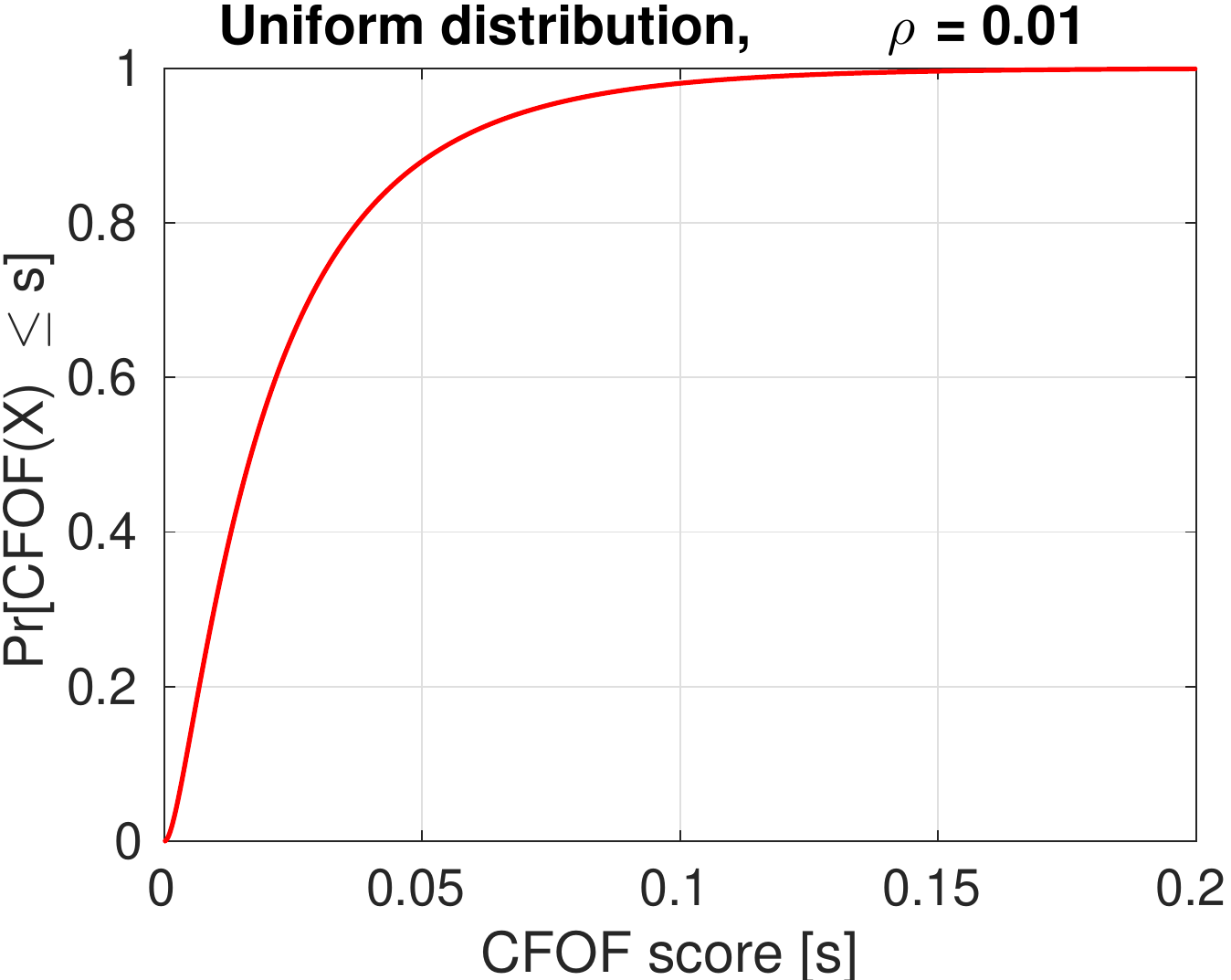}}
~~~~~~~~~~
\subfloat
{\includegraphics[width=0.48\columnwidth]{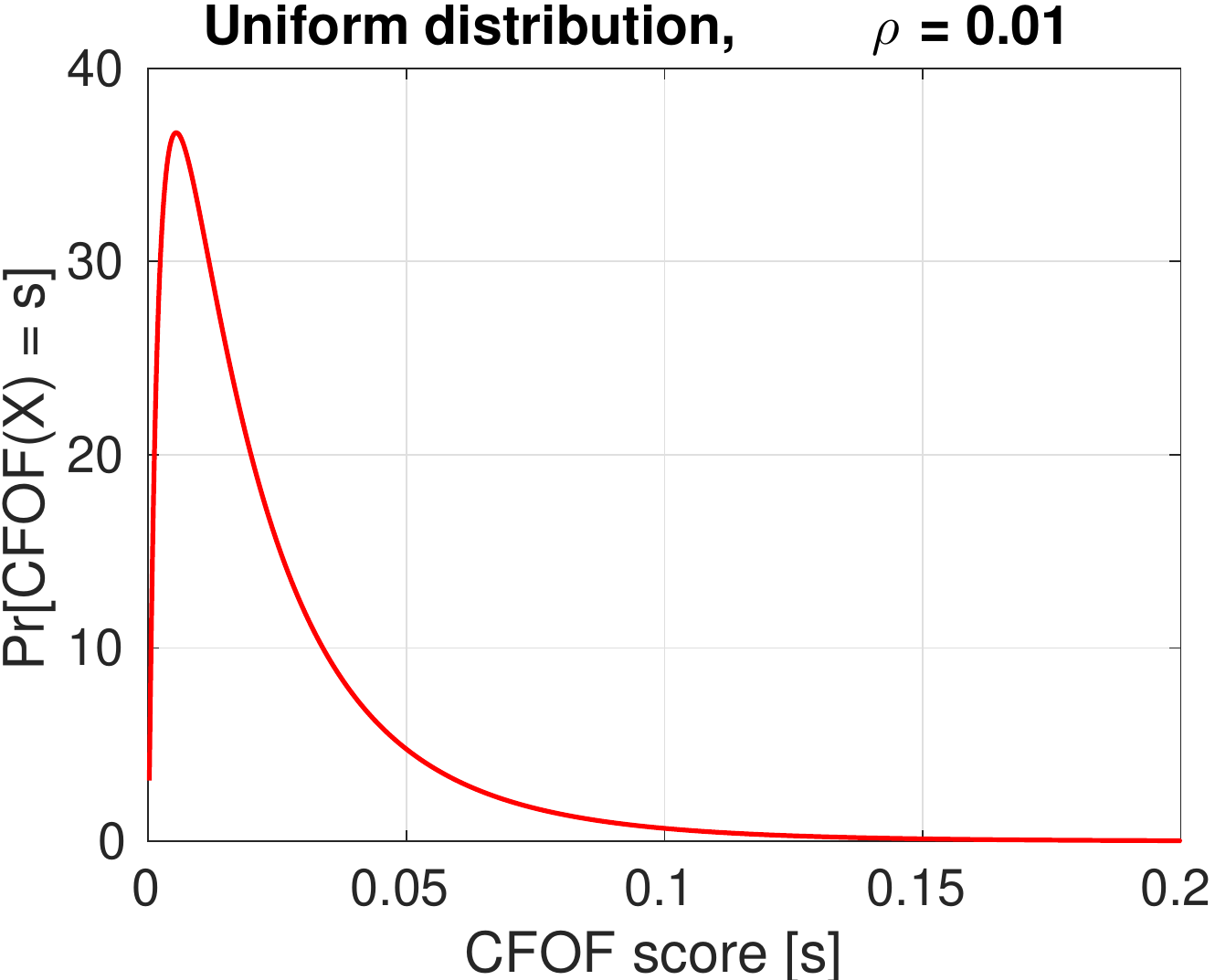}}\\
\subfloat
{\includegraphics[width=0.48\columnwidth]{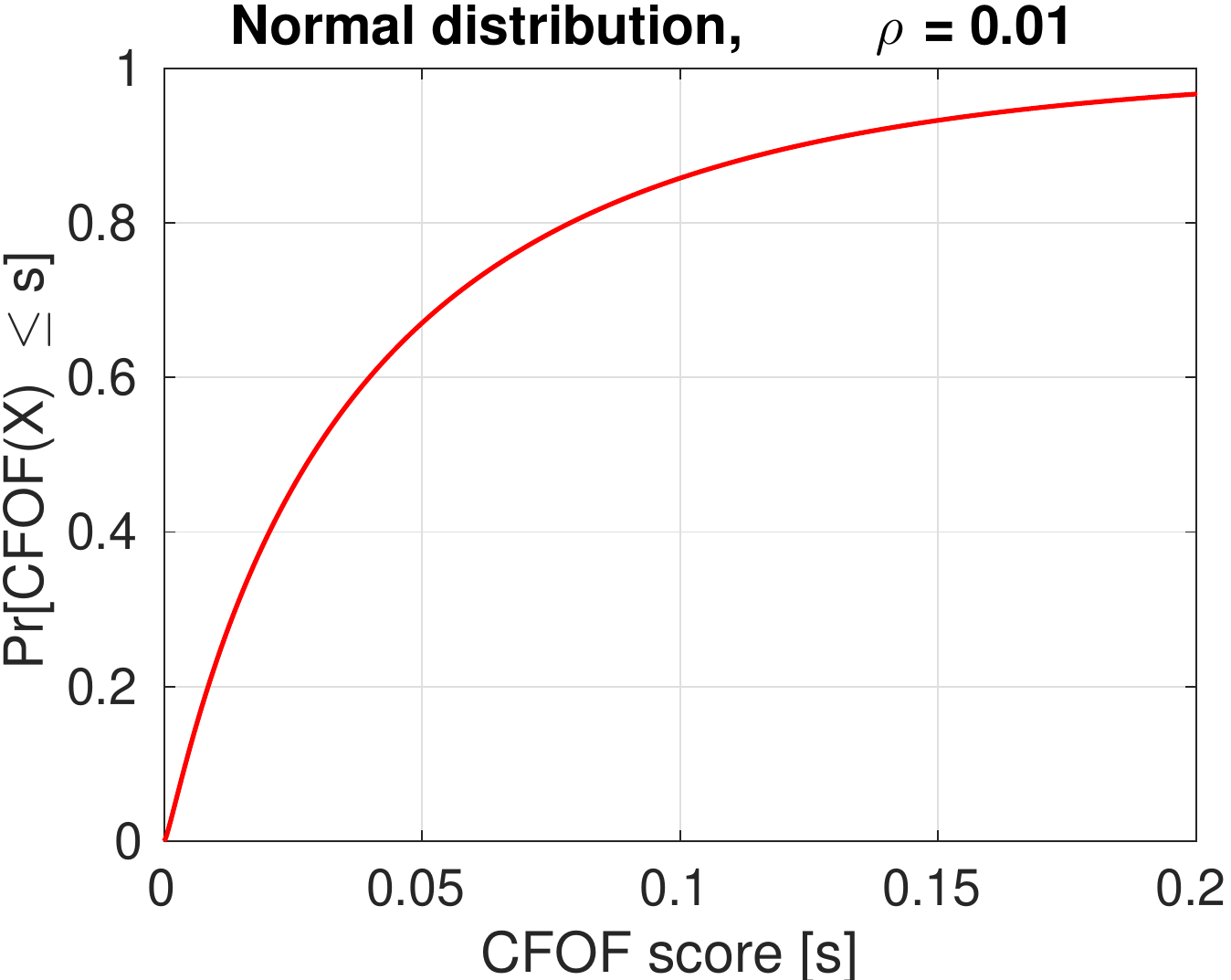}}
~~~~~~~~~~
\subfloat
{\includegraphics[width=0.48\columnwidth]{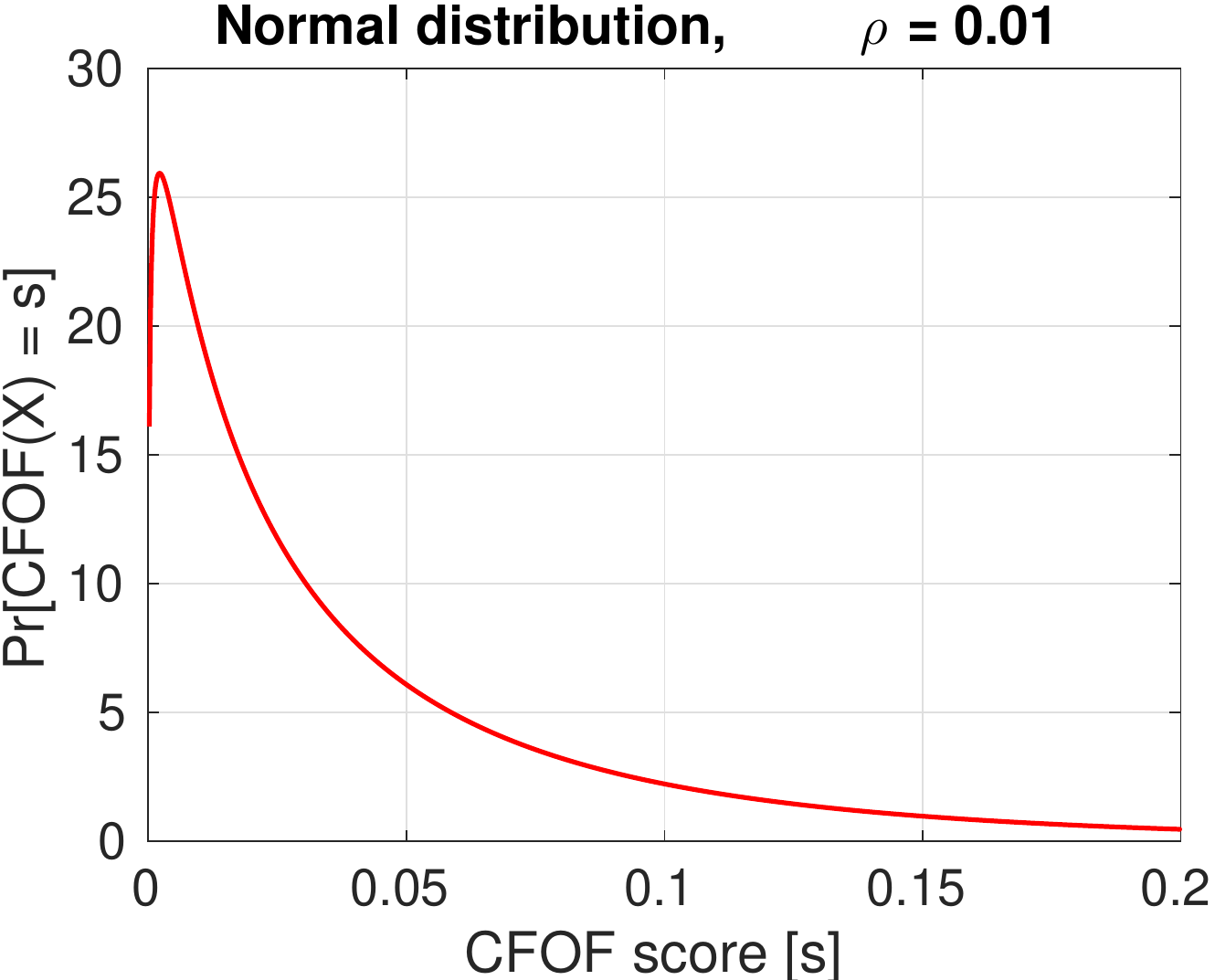}}
\caption{Theoretical cdf (plots on the left) and pdf (plots on the right)
of the $\CFOF$ score associated with an i.i.d. random vector uniformly distributed
(plots on the top) and normally distributed (plots on the bottom).}
\label{fig:cfofcdf}
\end{figure}

Figure \ref{fig:cfofcdf} reports the theoretical cdf according 
to Equation \eqref{eq:cfofcdf} 
and the theoretical pdf according to Equation \eqref{eq:cfofpdf} 
of the $\CFOF$ score associated
with an i.i.d. random vector uniformly distributed 
and normally distributed.

The following result provides a semantical characterization
of $\CFOF$ outliers in large dimensional spaces.

\begin{theorem}
Let $\vect{DS}$ be a dataset consisting of realizations of an i.i.d. random vector
$\Vect{X}$. Then, for arbitrary large dimensionalities $d$,
the $\CFOF$ outliers of $\vect{DS}$ are 
the points associated with the largest 
squared norm standard scores.
\end{theorem}
\begin{proof}
The result descends from the fact that the CFOF score is monotone
increasing with the squared norm standard score, as shown in Theorem \ref{theorem:cfofcdf}.
\end{proof}

\subsection{Concentration of outlier scores}\label{sect:out_conc}

\begin{definition}[Concentration of outlier scores]\label{def:concentration}
Let $Def$ be an outlier definition with outlier
score function $sc_{Def}$. 
We say that the outlier score $sc_{Def}$ concentrates if, for any 
i.i.d. random vector $\Vect{X}$ 
having kurtosis $\kappa>1$,
there exists a family of realizations $\vect{x_0}$ of $\Vect{X}$ ($d\in\mathbb{N}^+$)
such that, for any $\epsilon\in(0,1)$ 
the following property holds\footnote{Note that 
if the above property holds ``for any $\epsilon\in(0,1)$''
it is also the case that it holds ``for any $\epsilon>0$''.
We preferred to use the former condition for 
symmetry with respect to the scores smaller than $sc_{Def}(\vect{x})$. 
Indeed, for any $\epsilon\ge 1$ the absolute value function can 
be removed from Equation \eqref{eq:concentration}.}
\begin{equation}\label{eq:concentration}
\lim_{d \rightarrow \infty} 
Pr \left[ \left| \frac{sc_{Def}(\Vect{X})-sc_{Def}(\vect{x_0})}{sc_{Def}(\vect{x_0})} \right| 
\le \epsilon \right] = 1, 
\end{equation}
that is to say the probability
to observe a score value $sc_{Def}(\Vect{X})$ having relative distance 
not greater than $\epsilon$ from the reference score $sc_{Def}(\vect{x_0})$ 
tends to $1$ as the dimensionality $d$ goes to infinity. 
\end{definition}

\begin{figure}[t]
\centering
\subfloat
{\includegraphics[width=0.48\columnwidth]{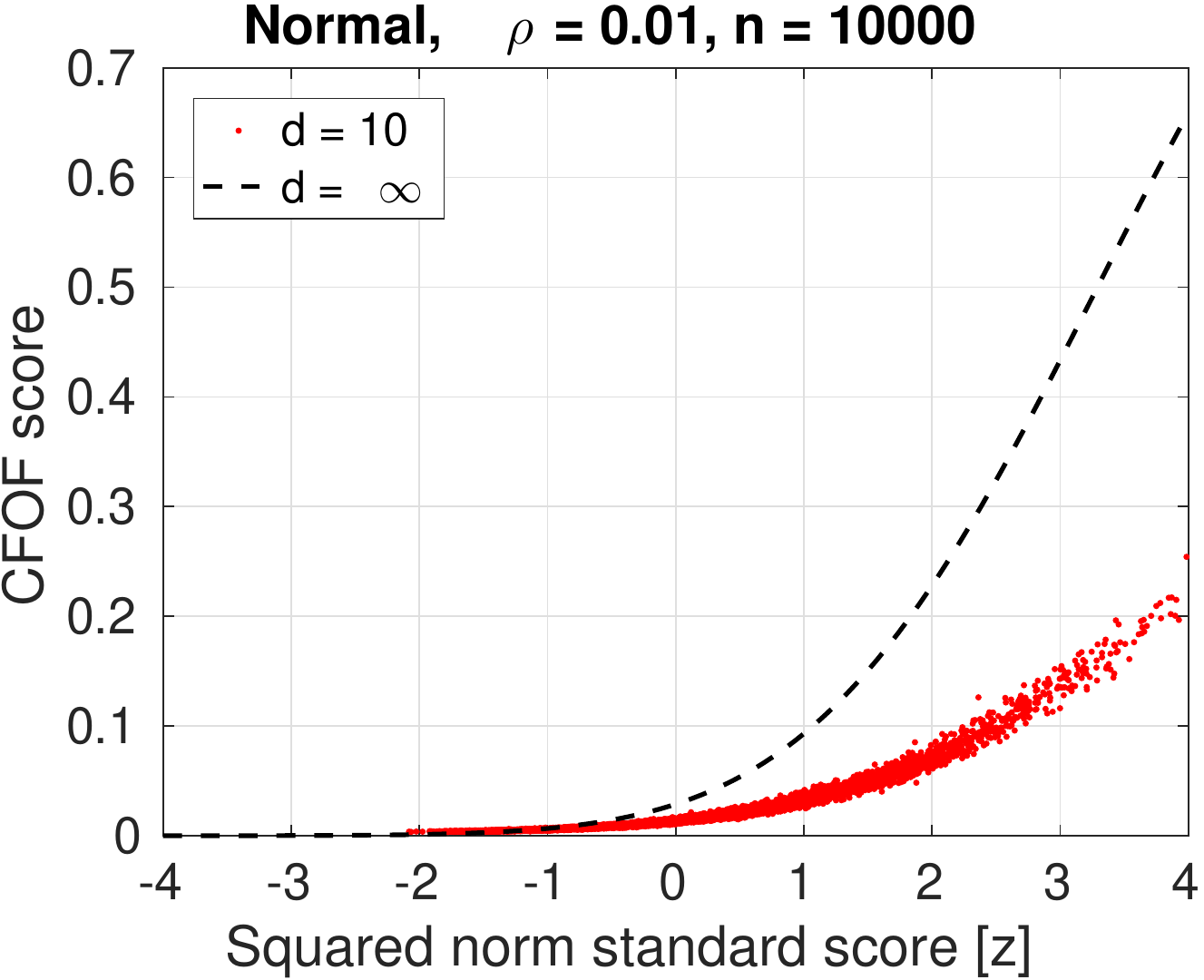}}
~~~~~~~~~~
\subfloat
{\includegraphics[width=0.48\columnwidth]{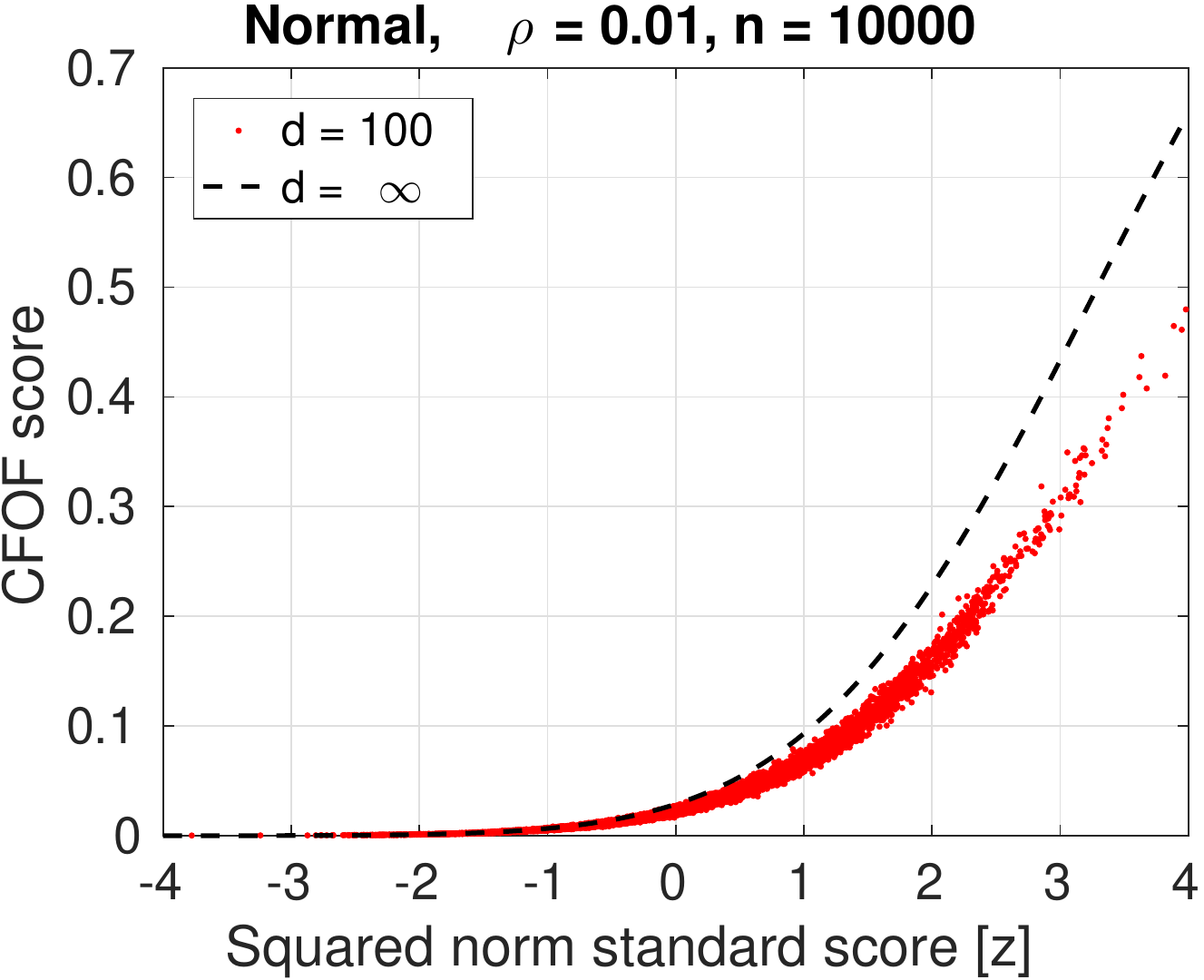}}\\
\subfloat
{\includegraphics[width=0.48\columnwidth]{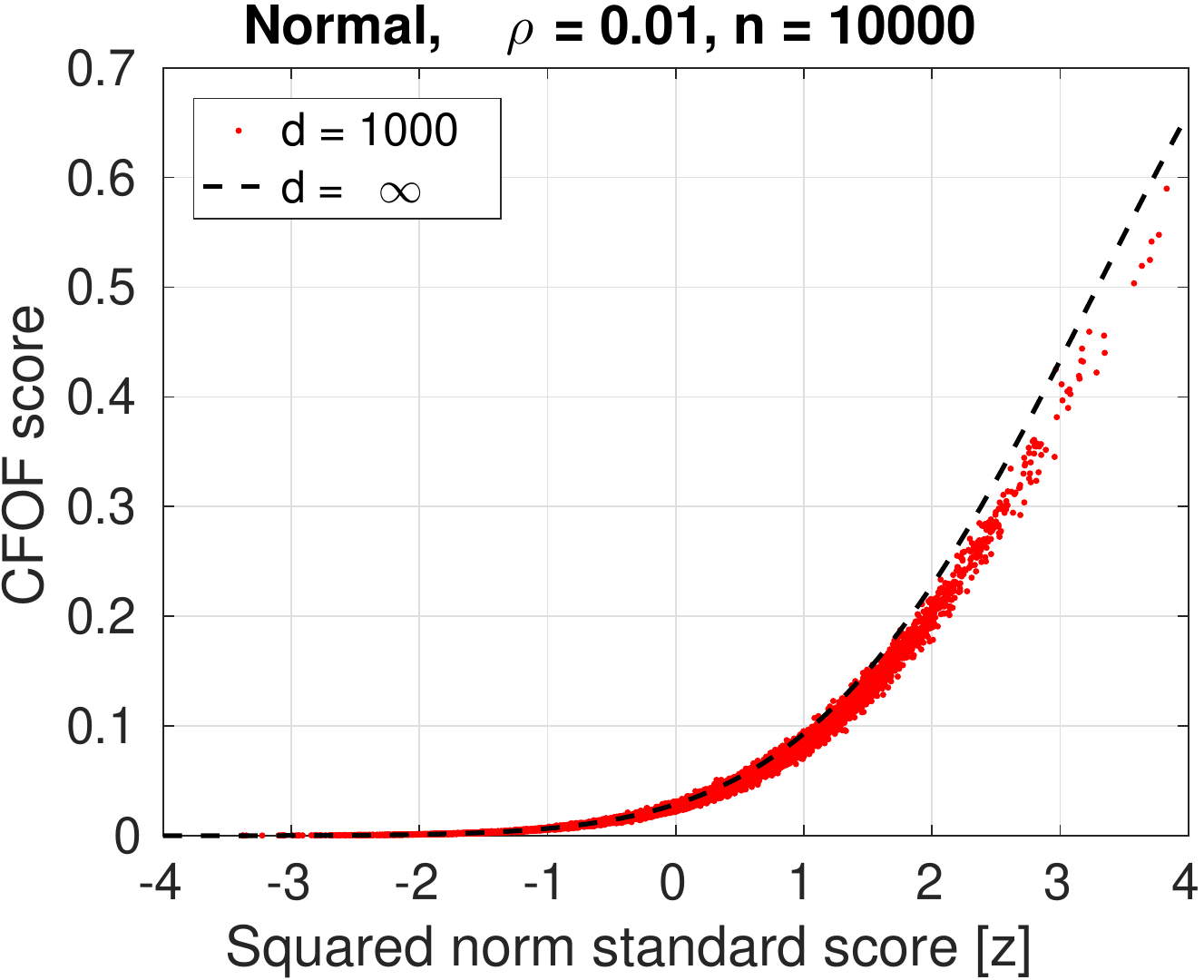}}
~~~~~~~~~~
\subfloat
{\includegraphics[width=0.48\columnwidth]{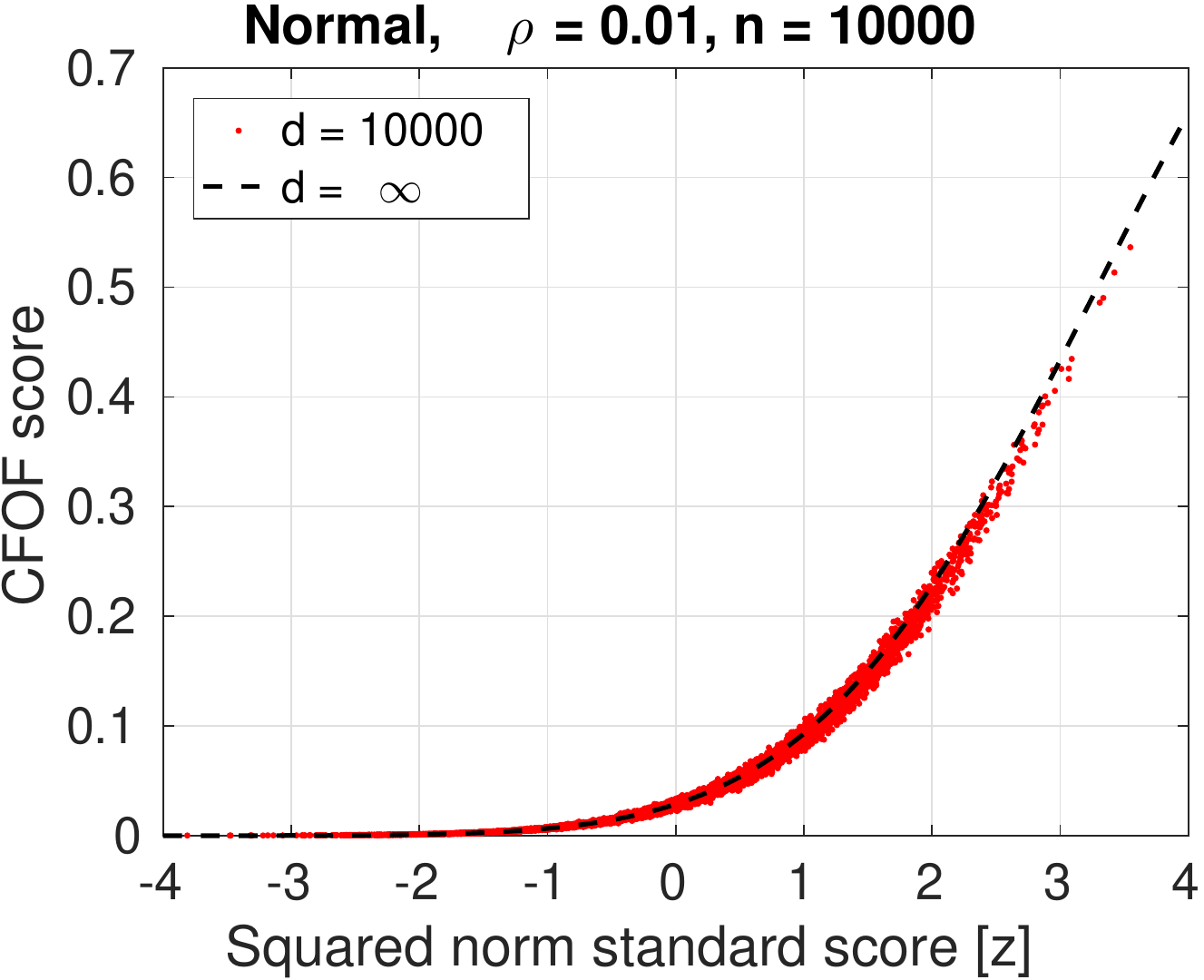}}
\caption{Comparison between the empirical $\CFOF$
scores associated with the squared norm standard score of
a normal dataset for increasing values of $d$ ($d\in\{10,10^2,10^3,10^4\}$)
and the theoretical $\CFOF$ associated
with arbitarily large dimensionalities ($d=\infty$; dashed curve).}
\label{fig:cfoflarged}
\end{figure}

\begin{theorem}\label{lemma:cfof}
Let $Def$ be an outlier definition with outlier score function $sc_{Def}$
which is monotone increasing with respect to the squared norm standard score.
Then, the outlier score $sc_{Def}$ concentrates if and only if, 
for any i.i.d. random vector $\Vect{X}$ having kurtosis $\kappa>1$
and for any $\epsilon\in(0,1)$,
the family of realizations $\vect{x_0}$ of $\Vect{X}$ 
having the property that $z_{x_0} = 0$, is such that 
\[ \lim_{d\rightarrow\infty} \Phi(z_{\vect{x_1}}) \rightarrow 0 ~\mbox{ and }~
\lim_{d\rightarrow\infty} \Phi(z_{\vect{x_2}}) \rightarrow 1,
\]
where $\vect{x_1}$ and $\vect{x_2}$ satisfy the following conditions
\[
sc_{Def}(\vect{x_1}) = (1-\epsilon) sc_{Def}(\vect{x})
~\mbox{ and }~
sc_{Def}(\vect{x_2}) = (1+\epsilon) sc_{Def}(\vect{x}).
\]
\end{theorem}

\begin{proof}
As already pointed out in the proof of Theorem \ref{theorem:cfofcdf},
as $d\rightarrow\infty$,
the distribution of the standard score $Z = (Y-\mu_Y)/\sigma_Y$
of the squared norm $Y = \|\Vect{X}-\mu_X\|^2$
tends to a standard normal distribution
and, hence, for each $z\ge 0$, $Pr[Z\le z] = \Phi(z)$.
Since the $sc_{Def}$ score is by assumption 
monotone increasing with the squared norm standard score,
for arbitrarily large values of $d$ and for any 
realization $\vect{x}$ of $\Vect{X}$,
\[ 
Pr\left[ sc_{Def}(\Vect{X}) \le sc_{Def}(\vect{x}) \right] = \Phi(z_x). 
\]
{If $sc_{Def}$ concentrates, then outlier score values must 
converge towards the outlier score of the points $\vect{x_0}$ 
whose squared norm standard score $z_{x_0}$
corresponds to the expected value $\E[Z] = 0$ of $Z$.}
Given $\epsilon>0$,
in order to hold
\begin{multline*}
\lim_{d\rightarrow\infty} Pr\left[ \left| \frac{sc_{Def}(\Vect{X}) - sc_{Def}(\vect{x_0})}
{sc_{Def}(\vect{x_0})} \right| \le \epsilon \right] = \\ =
\lim_{d\rightarrow\infty} Pr\left[ (1-\epsilon)sc_{Def}(\vect{x_0}) \le sc_{Def}(\Vect{X}) \le
(1+\epsilon)sc_{Def}(\vect{x_0}) \right] = 1,
\end{multline*}
it must be the case that
\[ \lim_{d\rightarrow\infty} \left( \Phi(z_{\vect{x_2}}) - \Phi(z_{\vect{x_1}}) \right) = 1, 
\]
where $\vect{x_1}$ and $\vect{x_2}$ are defined as in the statement of the theorem.
The result then follows from the last condition.
\end{proof}

Now we show that the separation
between the $\CFOF$ scores associated with outliers and
the rest of the $\CFOF$ scores is guaranteed
in any arbitrary large dimensionality.

\begin{theorem}\label{th:cfof}
For any fixed $\varrho\in(0,1)$,
the CFOF outlier score does not concentrate.
\end{theorem}
\begin{proof}
Let $\Vect{X}$ be an i.i.d. random vector having kurtosis $\kappa>1$.
As far as the $\CFOF$ score is concerned,
from Equation \eqref{eq:cfofexpected},
the point $\vect{x_1}$ defined in Theorem \ref{lemma:cfof} is such that
\[ 
z_{x_1} =  
\frac{\Phi^{-1}\left( (1 - \epsilon) \CFOF(x_0^{(d)}) \right)\sqrt{\kappa+3} 
- 2\Phi^{-1}(\varrho)}{\sqrt{\kappa-1}}.
\]
Note that for any fixed $\varrho\in(0,1)$, $\Phi^{-1}(\varrho)$ is finite.
Moreover,
since $z_{x_0}=0$, from Equation \eqref{eq:cfofexpected} 
it holds that $\CFOF(\vect{x_0}) \in (0,0.5)$.
Hence, $\Phi(z_{x_1}) \rightarrow 0$
if and only if 
$z_{x_1} \rightarrow -\infty$ if and only if 
$\Phi^{-1}( (1 - \epsilon) \CFOF(x_0^{(d)}) ) \rightarrow -\infty$
if and only if $(1 - \epsilon) \CFOF(x_0^{(d)}) \rightarrow 0$
if and only if $\epsilon \rightarrow 1$.

Thus, for arbitrary large dimensionalities $d$,
for each realization $\vect{x_0}$ of $\Vect{X}$ 
there exists $\epsilon>0$ 
such that 
$Pr[\left|(\CFOF(\Vect{X})-\CFOF(\vect{x_0}))/\CFOF(\vect{x_0})\right|<\epsilon]<1$ and, hence, 
the $\CFOF$ outlier score does not concentrate.
\end{proof}

Note that, with parameter $\varrho\in(0,1)$,
the $\CFOF$ score does not concentrate
for both bounded and unbounded sample sizes.

Figure \ref{fig:cfoflarged} compares the empirical $\CFOF$
scores associated with the squared norm standard score of
a normal dataset for increasing values of $d$ 
and the theoretical $\CFOF$ associated
with arbitrarily large dimensionalities ($d=\infty$).
It can be seen that for large $d$ values the scores tend
to comply with the value predicted by Equation \eqref{eq:cfofexpected}.
Moreover, as predicted by Theorem \ref{th:cfof},
the concentration of the $\CFOF$ scores is avoided in any dimensionality.

\begin{figure}[t]
\subfloat
{\includegraphics[width=0.48\columnwidth]{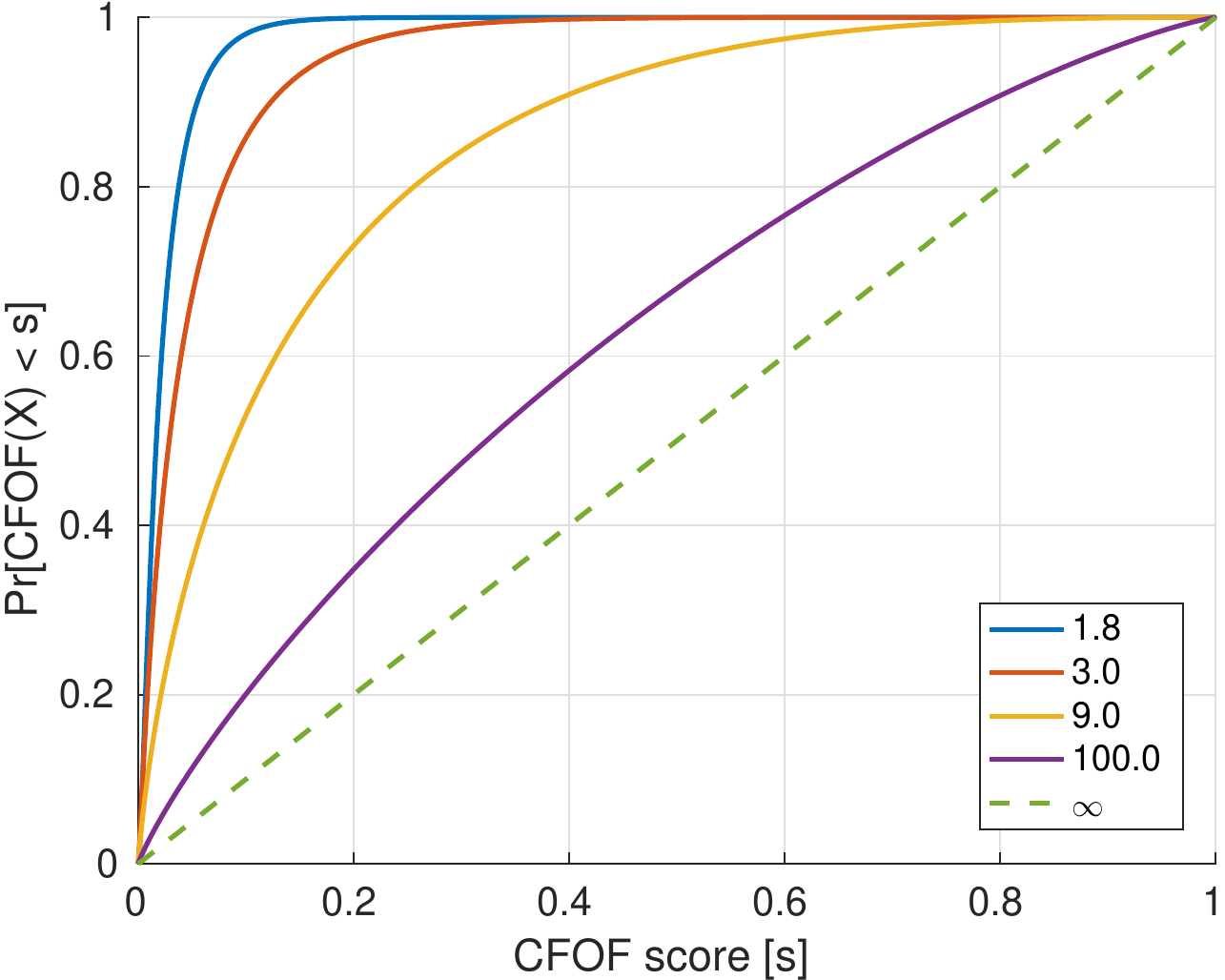}}
~ 
\subfloat
{\includegraphics[width=0.48\columnwidth]{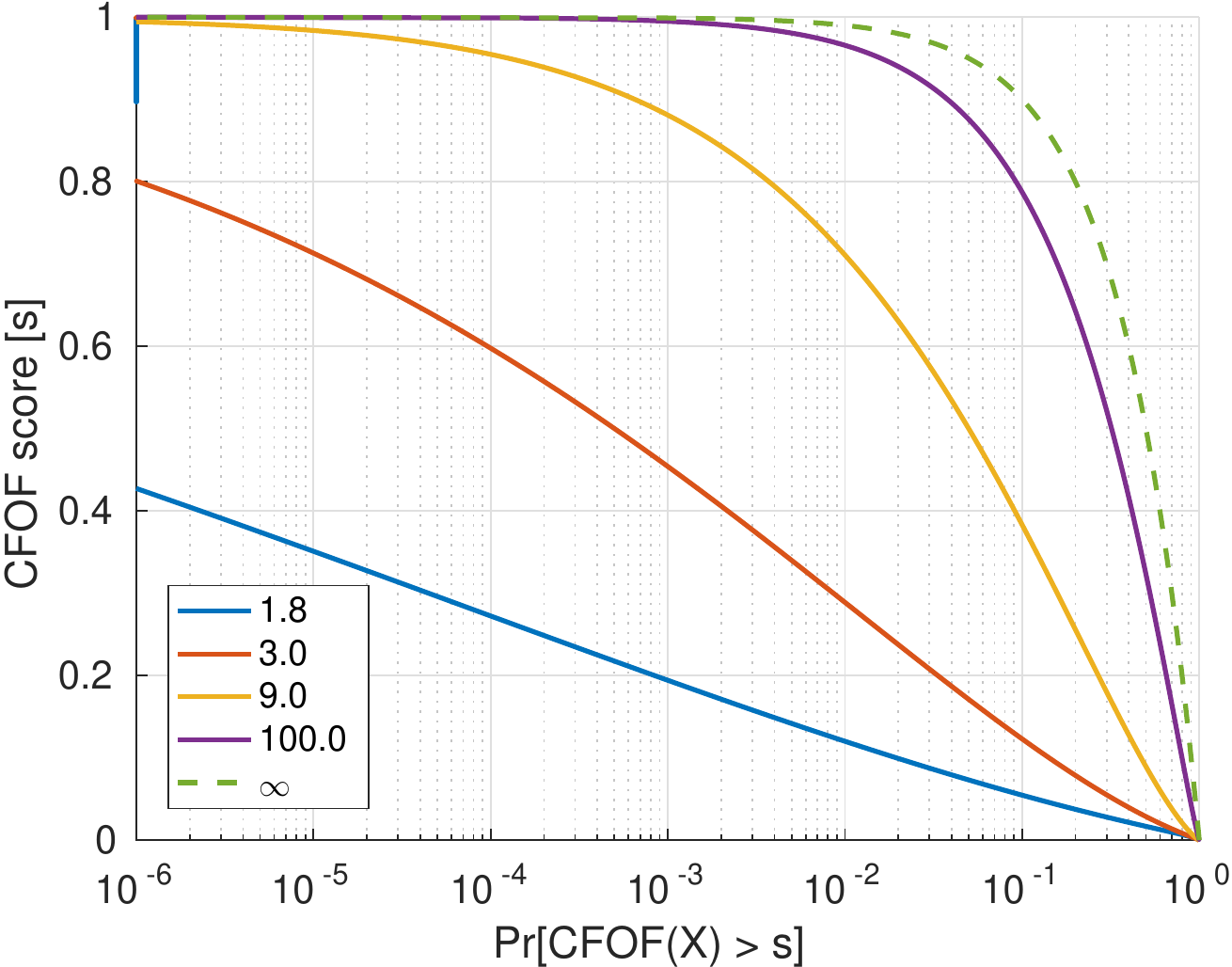}}
\caption{[Best viewed in color.] CFOF scores for different kurtosis values 
$\kappa \in \{1.8, ~3, ~9, ~100, ~\infty\}$ and $\varrho = 0.01$:
cumulative distribution function of the CFOF scores (on the left)
and score values ranked in decreasing order (on the right,
the abscissa reports the fraction of data points  in logarithmic scale).}
\label{fig:cfofcdfvskurt}
\end{figure}

\subsection{The effect of the data kurtosis}\label{sect:cfofcdfvskurt}

\medskip
As recalled above, the kurtosis $\kappa$ is such that
$\kappa\in[1,+\infty]$.
For extreme platykurtic distributions 
(i.e., having $\kappa\rightarrow 1^{+}$)
the $\CFOF$ score tends to $\varrho$, hence to a constant,
while for extreme leptokurtic distributions
(i.e., having $\kappa\rightarrow\infty$)
the $\CFOF$ score tends to the 
cumulative distribution function of the
standard normal distribution.

\begin{theorem}\label{th:cfofcdfvskurt}
Let $\vect{x}$ be a realization of an i.i.d. random vector $\Vect{X}$ having,
w.l.o.g., null mean. Then, for arbitrary large dimensionalities $d$,
\begin{eqnarray*}
\lim_{\kappa_X\rightarrow 1^+} \CFOF(\vect{x}) & \approx & 
\Phi \left( \frac{2}{\sqrt{4}} \Phi^{-1}(\varrho) \right) = \varrho \mbox{, and} \\
\lim_{\kappa_X\rightarrow+\infty} \CFOF(\vect{x}) & \approx & \Phi(z_{x}).
\end{eqnarray*}
\end{theorem}
\begin{proof}
The two expression can be obtained by exploiting the closed form of the 
cdf of the $\CFOF$ score reported in Equation \eqref{eq:cfofexpected}.
\end{proof}

Note that extreme platykurtic distributions have no outliers
and, hence, are excluded by the definition of concentration of outlier score.
For $\kappa = 1$, $\CFOF$ scores are constant,
and this is consistent with the absolute absence of outliers.

Figure \ref{fig:cfofcdfvskurt}
reports the CFOF scores associated with different kurtosis values 
$\kappa \in \{1.8, ~3, ~9, ~100, ~\infty\}$.
Curves are obtained by
leveraging Equation \eqref{eq:cfofcdf}
and setting $\varrho=0.01$.
The curves on the left represent the 
cumulative distribution function of the CFOF scores.
For infinite kurtosis the CFOF scores are 
uniformly distributed between $0$ and $1$, since,
from Equation \eqref{eq:cfofcdf}, for $\kappa\rightarrow\infty$,
$Pr[ \CFOF(\Vect{X}) \le s ] = s$.
Moreover, the curves highlight that the larger the kurtosis of the data
and the larger the probability to observe higher scores.
The curves on the right represent
the score values ranked in decreasing order. 
The abscissa reports the fraction of data points in logarithmic scale.
These curves allow to visualize the fraction of data
points whose score will be above a certain threshold.

\subsection{Behavior in presence of different distributions}\label{sect:semilocal}

Despite the apparent similarity of the $\CFOF$ score with distance-based scores,
these two families of scores are deeply different and, while distance-based outliers 
can be categorized among to the \textit{global outlier} scores,
$\CFOF$ shows adaptivity to different density levels,
a characteristics that makes it more similar to 
\textit{local outlier} scores.

This characteristics depends in part on the fact that
actual distance values are not 
employed in the computation of the score.
Indeed, the $\CFOF$ score is invariant
to all of the transformations 
that do not change
the nearest neighbor ranking, such as
translating the data or scaling the data.

\begin{figure}[t]
\centering
\subfloat[\label{fig:example_localA}The top $25$ outliers (circled points) 
according to the $\CFOF$ 
definition.]{\includegraphics[width=0.48\columnwidth]{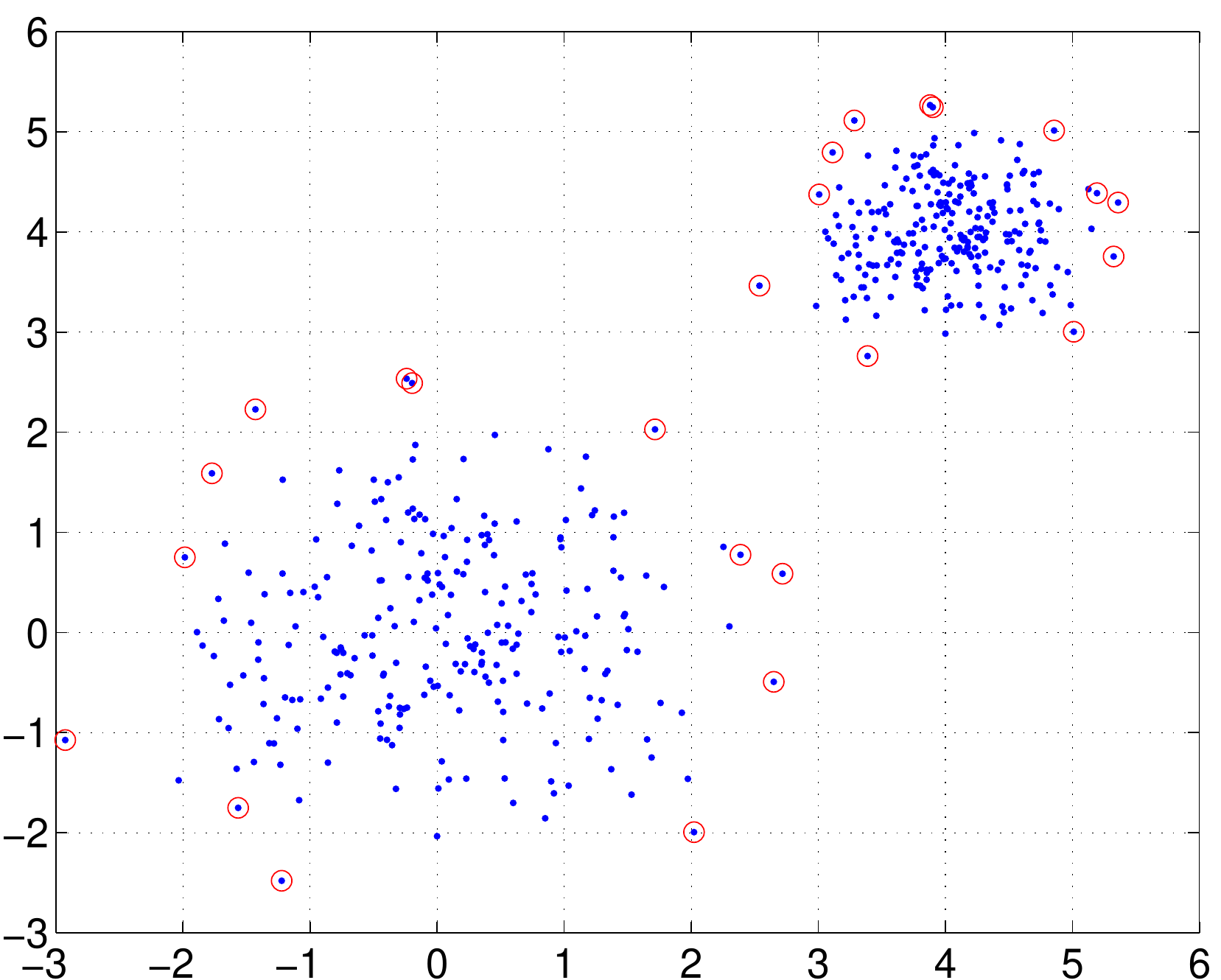}}
~ 
\subfloat[\label{fig:example_localB}
Density estimation performed by means of the $\CFOF$ measure.]
{\includegraphics[width=0.48\columnwidth]{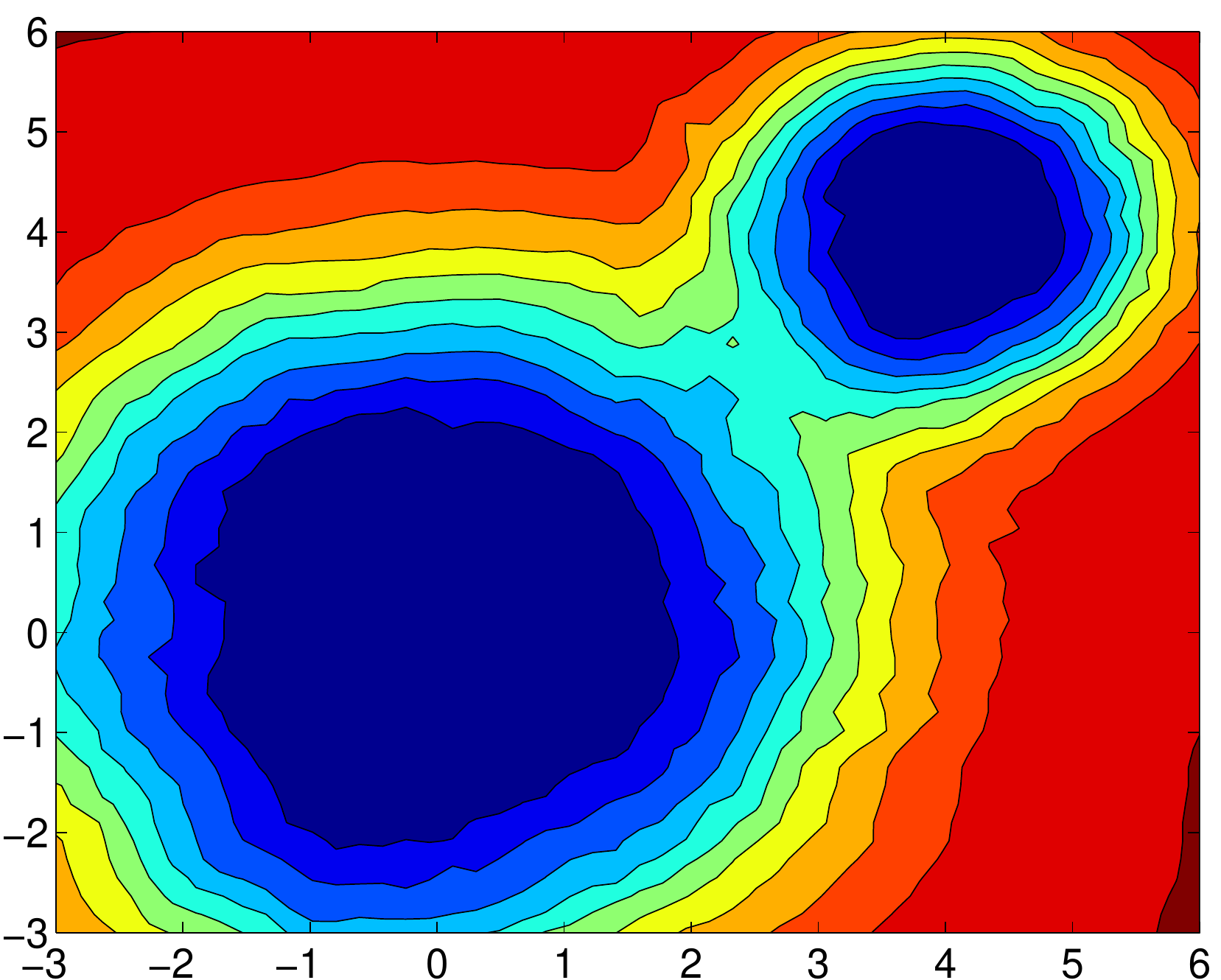}} 
\caption{Outlier scores for two normal clusters with different standard deviations.}
\label{fig:example_local}
\end{figure}

To illustrate, consider Figure \ref{fig:example_local}
showing a dataset consisting of two normally distributed clusters,
each consisting of $250$ points. The cluster centered in $(4,4)$
is obtained by translating and scaling by a factor $0.5$
the cluster centered in the origin.
The top $25$ $\CFOF$ outliers for $k_\varrho=20$ 
are highlighted.  
It can be seen that the outliers are 
the ``same'' objects of the two clusters.
Notice that a similar behavior can be observed,
that is outliers will emerge both in the sparser regions of
the space and along the borders of clusters,
also when the clusters are 
not equally populated,
provided that they contain at least $k_\varrho$ objects.

Now we will provide theoretical evidence that the above
discussed properties of the $\CFOF$ score are valid
in any arbitrary large dimensionality.

\begin{definition}[Translation-invariant and homogeneous outlier score]
Let $\Vect{X}$ an i.i.d. random vector,
let $a\in\mathbb{R}$ and $\vect{b}\in\mathbb{R}^d$, and 
let $\Vect{Y} = a\Vect{X}+\vect{b}$.
An outlier score $sc_{Def}$ is said to be \textit{translation-invariant and homogeneous}
if, for all the realizations $\vect{x}$ of $\Vect{X}$, it holds that 
\[ sc_{Def}(\vect{x},\Vect{X}) = sc_{Def}(a \vect{x}+\vect{b},\Vect{Y}) \]
\end{definition}

\begin{theorem}\label{th:cfofhomogeneous}
For arbitrary large dimensionalities $d$,
the CFOF score is 
translation-invariant and homogeneous.
\end{theorem}
\begin{proof}
We note that 
the squared norm standard score $z_{x,\Vect{X}}$ of $\vect{x}$ is
identical to the squared norm standard score $z_{y,\Vect{Y}}$ of $\vect{y}=a\vect{x}+\vect{b}$.
Indeed,
the mean of $\Vect{Y}$ is the vector $\vect{\mu}_Y=(\mu_{Y_1},\mu_{Y_2},\ldots,\mu_{Y_d})$,
where $\mu_{Y_i} = \E[a X_i + b_i] = a \mu_X + b_i$ ($1\le i\le d$).
As for
$\norm{\vect{y}-\vect{\mu}_Y}^2 = \norm{(a\vect{x}+\vect{b})-(a\mu_X+\vect{b})}^2 = 
\norm{a(\vect{x}-\mu_X)}^2 = a^2\norm{\vect{x}-\mu_X}^2$.
Analogously,
$\norm{\Vect{Y}-\mu_Y}^2 = a^2\norm{\Vect{X}-\mu_X}^2$, and
$\mu_{\norm{\Vect{Y}-\mu_Y}^2} = a^2\mu_{\norm{\Vect{X}-\mu_X}^2}$ and
$\sigma_{\norm{\Vect{Y}-\mu_Y}^2} = a^2\sigma_{\norm{\Vect{X}-\mu_X}^2}$.
Hence,
\[ z_{y,\Vect{Y}} = \frac{\norm{\vect{y}-\mu_Y}^2-\mu_{\norm{\Vect{Y}-\mu_Y}^2}}{\sigma_{\norm{\Vect{Y}-\mu_Y}^2}} =
\frac{a^2\norm{\vect{x}-\mu_X}^2-a^2\mu_{\norm{\Vect{X}-\mu_X}^2}}{a^2\sigma_{\norm{\Vect{X}-\mu_X}^2}} =
z_{x,\Vect{X}}.
\]
Moreover, 
we recall that the $k$th central moment has the following two properties:
$\mu_k(X+b)=\mu_k(X)$ (called translation-invariance), 
and $\mu_k(a X)=a^k\mu_k(X)$ (called homogeneity).
Hence, $\mu_k(Y) = \mu_k(Y_i) = \mu_k(aX+b_i) = a^k\mu_k(X)$.
Since variables $Y_i$ have identical central moments,
by applying the above property to Equation \eqref{eq:cfofcdf}, the statement eventually follows:
\begin{multline*}
\CFOF(a\vect{x}+\vect{b},\Vect{Y}) =
\Phi\left( \frac{z_y\sqrt{\mu_4(Y)-\mu_2(Y)^2}+2\mu_2(Y)\Phi^{-1}(\varrho)}{\sqrt{\mu_4(Y)+3\mu_2(Y)^2}} \right) = \\
= \Phi\left( \frac{z_x\sqrt{a^4\mu_4(X)-\big(a^2\mu_2(X)\big)^2}+2a^2\mu_2(X)\Phi^{-1}(\varrho)}
{\sqrt{a^4\mu_4(X)+3\big(a^2\mu_2(X)\big)^2}} \right) = \\
= \Phi\left( \frac{z_x\sqrt{\mu_4(X)-\mu_2(X)^2}+2\mu_2(X)\Phi^{-1}(\varrho)}
{\sqrt{\mu_4(X)+3\mu_2(X)^2}} \right) = \CFOF(\vect{x},\Vect{X}).
\end{multline*}
\end{proof}

Next, we introduce the concept of i.i.d. mixture random vector
as a tool for modeling an intrinsically high-dimensional dataset
containing data populations having different characteristics.

An \textit{i.i.d. mixture random vector} $\Vect{Y}$ 
is a random vector defined in terms of $K$ i.i.d. random vectors 
$\Vect{Y_1}, \Vect{Y_2}, \ldots, \Vect{Y_K}$
with associated selection probabilities $\pi_1,\pi_2,\ldots,\pi_K$, respectively.
Specifically, for each $i$, with probability $\pi_i$ the random vector $\Vect{Y}$
assumes value $\Vect{Y_i}$.

A dataset $DS$ generated by an i.i.d. mixture random vector $\Vect{Y}$,
consists of $n$ points
partitioned into $K$ clusters
$C_1, C_2, \ldots, C_K$ composed of $n_1,n_2,\ldots,n_K$ points, respectively,
where each $C_i$ is formed by realizations of the random vector $\Vect{Y_i}$
selected with probability $\pi_i$ ($1\le i\le K$).

Given a set of clusters $C_1,C_2,\ldots,C_K$, we
say that they are \textit{non-overlapping} if,
for each cluster $C_i$ and point $x\in C_i$,
$\NN_{n_i}(x,DS)=C_i$.

{The following result clarifies how $\CFOF$ behaves in
presence of clusters having different densities.}

\begin{theorem}\label{th:cfofclusters}
Let $DS$ be a dataset of $K$ non-overlapping clusters $C_1,C_2,\ldots,C_K$
generated by the i.i.d. mixture random vector $\Vect{Y}$,
let $\varrho \le \min_i\{\pi_i\}$,
let $O$ be the top-$\alpha$ $\CFOF$ outliers of $DS$,
and let $O = O_1 \cup O_2 \cup \ldots \cup O_K$
be the partition of $O$ induced by the clusters of $DS$.
Then, for arbitrary large dimensionalities $d$,
each $O_i$ consists of the top-$(\alpha_i/\pi_i)$ 
$(\varrho/\pi_i)$--$\CFOF$ outliers of the dataset
$C_i$ generated by the i.i.d. random vector $\Vect{Y_i}$ {\rm(}$1\le i\le K${\rm)},
where
\[ 
\frac{\alpha_i}{\pi_i} = 
1 - 
F_{\rm CFOF} \left( \frac{s^\ast}{\pi_i}; \frac{\varrho}{\pi_i}, \Vect{Y_i} \right), \]
and $s^\ast$ is such that $F_{\rm CFOF}(s^\ast;\varrho,\Vect{Y})=1-\alpha$.
\end{theorem}
\begin{proof}
Since the clusters are non-overlapping, 
the first $n_i$ direct and reverse neighbors
of each point $\vect{x}\in C_i$ are points belonging to the same cluster $C_i$ of $\vect{x}$.
Hence, being $\varrho n\le n_i$, the $\CFOF$ score of point $\vect{x}$ depends only on
the points within its cluster.
Thus the score $s=k/n$ with respect to the whole dataset
can be transformed into the score $s_i=k/n_i=(k/n)(n/n_i)=s/\pi_i$ 
with respect the cluster $C_i$.
Analogously, the parameter
$\varrho$ with respect to the whole dataset
can be transformed into the parameter $\varrho_i$
with respect to the cluster $C_i$, by requiring that
$\varrho n = \varrho_i n_i$, that is $\varrho_i = \varrho (n/n_i) = \varrho/\pi_i$.
Thus, the cdf of the $\CFOF$ score 
can be formulated as
\[ F_{\rm CFOF}(s;\varrho,\Vect{Y}) = \sum_{i=1}^K \pi_i F_{\rm CFOF} 
\left( \frac{s}{\pi_i}; \frac{\varrho}{\pi_i}, \Vect{Y_i} \right). \]
Consider now the score value $s^\ast$ such that $F_{\rm CFOF}(s^\ast;\varrho,\Vect{Y}) = 1-\alpha$.
Then the expected number of outliers from cluster $C_i$
is $\alpha_i n = n \pi_i \left( 1 - F_{\rm CFOF} ( s^\ast/\pi_i; \varrho/\pi_i, \Vect{Y_i} ) \right)$,
from which the result follows.
\end{proof}

Thus, from the above result it can be concluded that the number
of outliers coming from each cluster $C_i$ is related both to 
its generating distribution $\Vect{Y_i}$ 
and to
its relative size $\pi_i$.

As for the generating distribution,
if we consider points having positive squared
norm standard score (which form the most extreme half of the population),
then
at the same squared norm standard score value,
the $\CFOF$ score is higher for points whose 
generating distribution has larger kurtosis.

\begin{theorem}\label{th:cfof_kurt}
Let $\Vect{X}$ and $\Vect{Y}$ be two i.i.d. random vectors, and let
$\vect{x}$ and $\vect{y}$ be two realizations of $\Vect{X}$ and $\Vect{Y}$, respectively,
such that $z_{x,\Vect{X}} = z_{y,\Vect{Y}} \ge 0$. Then,
for arbitrary large dimensionalities $d$,
$\CFOF(\vect{x},\Vect{X}) \ge \CFOF(\vect{y},\Vect{Y})$
if and only if $\kappa_X \ge \kappa_Y$.
\end{theorem}
\begin{proof}
The result descends from the fact that 
for non-negative $z\ge 0$ values,
the $\CFOF$ score
is monotone increasing with respect to the kurtosis parameter.
\end{proof}

\begin{figure}
\centering
\includegraphics[width=0.48\columnwidth]{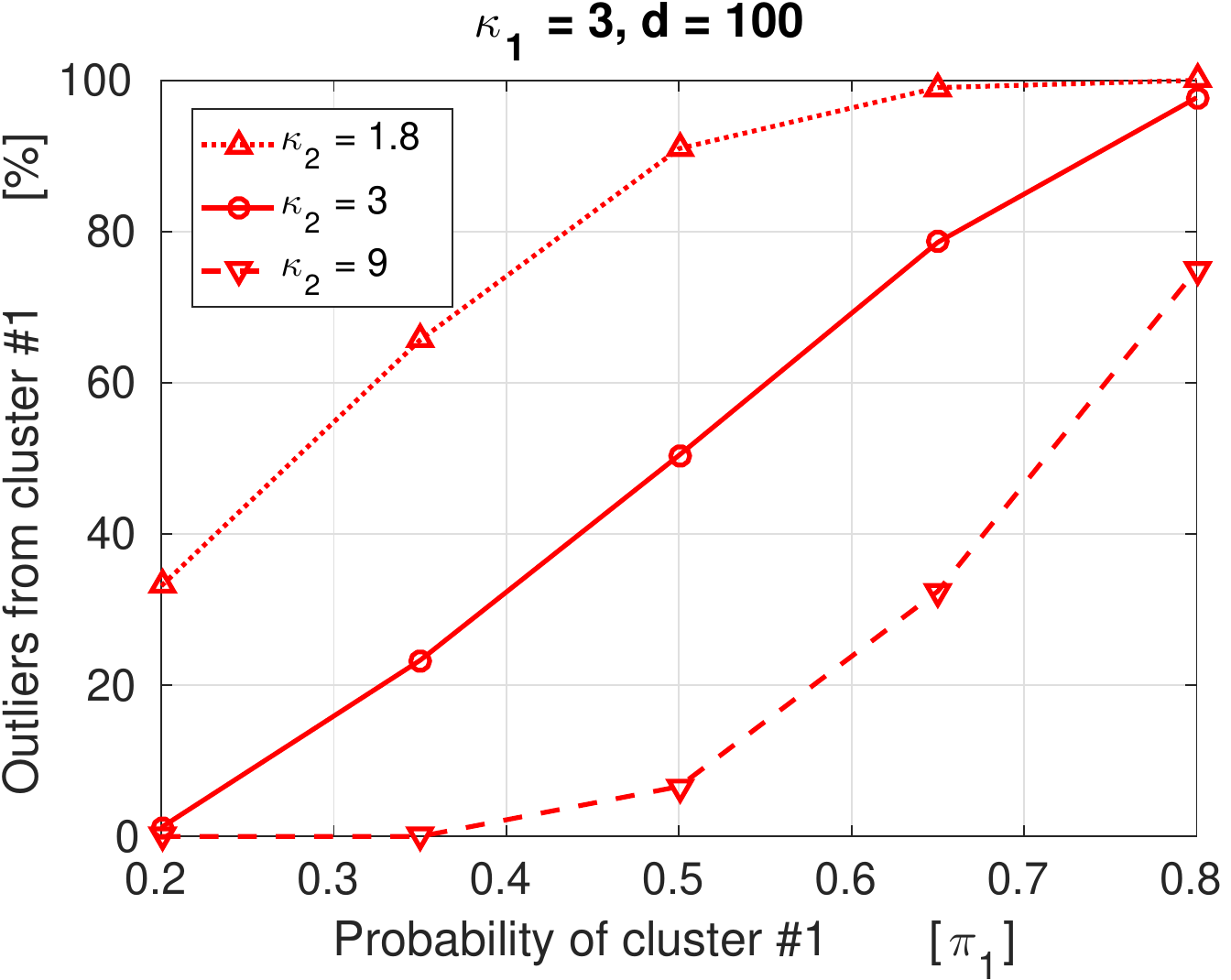} 
\caption{Percentage of $\CFOF$ outliers from the first cluster,
having kurtosis $\kappa_1 = 3$,
versus the fraction of points belonging to the first cluster 
of a $d=100$ dimensional
dataset containing a second cluster having kurtosis $\kappa_2$:
curves concern the cases $\kappa_2\in\{1.8,3,9\}$.}
\label{fig:cfofmixture}
\end{figure}

Thus, the larger the cluster kurtosis, the larger the fraction of 
outliers coming from that cluster.
To clarify
the relationship between kurtosis and number of outliers
we considered a dataset consisting of two clusters: the first
(second, resp.)
cluster consists of points generated according to an i.i.d.
random vector having kurtosis $\kappa_1 = 3$ ($\kappa_2 \in \{1.8, 3, 9\}$, resp.).
We then varied the occurrence probability $\pi_1$ of the first cluster
in the interval $[0.2,0.8]$, while the occurrence probability 
of the second cluster is $\pi_2 = 1-\pi_1$,
and set $\varrho$ to $0.01$ and
the fraction $\alpha$ of outliers to be selected to $0.05$. 
Figure \ref{fig:cfofmixture} reports 
the {percentage} of the top-$\alpha$ outliers that come from the 
first cluster (that is $\alpha_1' = \frac{\alpha_1}{\alpha}\cdot 100$)
as a function of $\pi_1$ for different $\kappa_2$ values.
It can be seen that, for $\kappa_1=\kappa_2$ the fraction $\alpha_1'$ is more closely
related to $\pi_1$ (with $\alpha_1'=0.5$ for $\pi_1=0.5$),
while in the general $\alpha_1'$ is directly proportional to the ratio 
$\frac{\kappa_1}{\kappa_2}$.

\smallskip
Intuitively, 
for datasets consisting of multiple shifted and scaled copies
of a given cluster,
a score able to retrieve local outliers
should return the same outliers from each cluster, 
both in terms of their number
and of their location within the cluster.
Next we show that this is indeed the case of the $\CFOF$ score
in any arbitrary large dimensionality.

We say that a collection of sets $S_1,S_2,\ldots,S_K$ 
is \textit{balanced}, if $|S_1|=|S_2|=\ldots=|S_K|$.

We say that an i.i.d. mixture random vector $\Vect{Y}$ is
\textit{homogeneous} if there exist random vector $\Vect{Y_0}$, 
real numbers $a_i\in\mathbb{R}$, and vectors $\vect{b_i}\in\mathbb{R}^d$,
such that $\Vect{Y_i}=a_i\Vect{Y_0}+\vect{b_i}$ ($1\le i\le K$).

\begin{theorem}\label{th:cfofclustbalanced}
Let $\vect{DS}$ be a dataset composed of balanced non-overlapping clusters
generated by an homogeneous i.i.d. mixture random vector $\Vect{Y}$.
Then, for arbitrary large dimensionalities $d$, 
the partition of the $\CFOF$ outliers of $\vect{DS}$ induced by the clusters of $\vect{DS}$
is balanced.
\end{theorem}
\begin{proof}
Since clusters are balanced, $\pi_i = 1/K$.
Moreover, 
by Theorems \ref{th:cfofhomogeneous} and \ref{th:cfofclusters}, 
\[ F_{\rm CFOF}(s; \varrho, \Vect{Y}) = K \cdot F_{\rm CFOF}\left( s K; \varrho K, \Vect{Y_0} \right), \] 
which means that, for any score value $s$, the number of expected
outliers from any cluster is the same.
\end{proof}

We call \textit{semi--local} any outlier definition 
for which the property stated in the statement
of Theorem \ref{th:cfofclustbalanced}
holds for homogeneous i.i.d. mixture random vectors,
but not for non-homogeneous i.i.d. mixture random vectors.

\subsection{On the concentration properties of distance-based and density-based
outlier scores}
\label{sect:conc_other}

We have already discussed the \textit{distance concentration} phenomenon,
that is the tendency of distances to become almost indiscernible as dimensionality
increases, and empirically shown that the concentration effect also affects 
different families of outlier scores (see Section \ref{sect:cfof_distconc}).
In this section we theoretically assess the concentration properties of
some of these scores,
by formally proving that distance-based and density outlier scores concentrate
according to Definition \ref{def:concentration}.
Specifically, we consider the KNN \cite{RRS00},
$\aKNN$ \cite{AP02}, and $\LOF$ \cite{BKNS00} outlier scores.

Intuitively, the above results descend by the specific role played by
distances in these definitions.
Indeed, distance-based outlier scores detect outliers on the basis
of the distance separating points from their nearest neighbors, a measure
whose discrimination capability, due to the distance concetration effect,
is expected to become more and more feeble as 
the dimensionality increases.
As for density-based outliers, in order to capture a notion of \textit{local}
(or normalized) density, they compute the ratio between 
the distance separating a point from its nearest neighbors
and the average of the same measure associated with their nearest neighbors.
Being the two above terms subject to concentration of distances, their ratio 
is expected to become closer and closer to $1$ as the dimensionality increases.

Formal details, including the role played by parameters,
are provided in subsequent Theorems \ref{th:knn_conc}, \ref{th:aknn_conc}, 
and \ref{th:lof_conc}.

\begin{theorem}\label{th:knn_conc}
For any fixed integer number $k$ such that $0<k<n$ (for any fixed 
real number $\varrho\in(0,1)$, resp.), 
the $\KNN$ outlier score with parameter $k$ 
($k = \varrho n$, resp.)
concentrates.
\end{theorem}
\begin{proof}
Let $\Vect{X}$ be an i.i.d. random vector having kurtosis $\kappa>1$.
W.l.o.g., assume that $\mu_{X}=0$, then from \cite{Angiulli2018} (see Lemma 23)
\begin{multline*}
\KNN(\vect{x}) =
\dist\left(\vect{x},\nn_k(\vect{x})\right) \approx \sqrt{\norm{\vect{x}}^2 + 
\mu_{\norm{\Vect{X}}^2} + \Phi^{-1}\left(k/n\right)
\sigma_{\norm{\vect{x}-\Vect{X}}^2} } = \\
= \sqrt{\norm{\vect{x}}^2 + 
d\mu_2 + \Phi^{-1}\left(k/n\right)
\sqrt{d(\mu_4-\mu_2^2) + 4\mu_2\norm{\vect{x}}^2} }.
\end{multline*}
For any fixed $k$ such that $0<k<n$ ($\varrho\in(0,1)$, resp.),
it is the case that
$0<k/n<1$ and $\Phi^{-1}(k/n)$ is finite.
Moreover, 
for arbitrary large dimensionalities $d$, 
$\sqrt{O(\norm{\vect{x}}^2,d)+\sqrt{O(\norm{\vect{x}}^2,d)}} \approx
\sqrt{O(\norm{\vect{x}}^2,d)}$ and, then
\[
\KNN(\vect{x})
\approx 
\sqrt{\norm{\vect{x}}^2 + d\mu_2}.
\]
Since
$\norm{\vect{x}}^2 = z_x\sigma_{\norm{\vect{X}}^2} + \mu_{\norm{\vect{X}}^2} =
z_x \sqrt{d(\mu_4-\mu_2^2)} + d\mu_2$, by substituting 
\[
\KNN(\vect{x}) \approx 
\sqrt{z_x\sqrt{d(\mu_4-\mu_2^2)} + 2d\mu_2}.
\]
From the above expressions, the $\KNN$ score is monotone increasing
with the squared norm standard score. Then, Theorem \ref{lemma:cfof}
can be applied.
Consider the family $\vect{x_0}$ ($d\in\mathbb{N}^+$)
of realizations of $\Vect{X}$ 
such that $\norm{\vect{x_0}}^2 = \mu_{\norm{\vect{X}}^2} = d\mu_2$.
Then, the point $\vect{x_1}$ defined in Theorem \ref{lemma:cfof} is such that
\begin{eqnarray*}
\KNN(\vect{x_1}) =
& (1-\epsilon)\KNN(\vect{x_0}) 
& \Longrightarrow \\
\sqrt{\norm{\vect{x_1}}^2 + 2d\mu_2} =
& (1-\epsilon) \sqrt{\norm{\vect{x_0}}^2 + d\mu_2} 
& \Longrightarrow \\
\sqrt{z_{x_1}\sqrt{d(\mu_4-\mu_2^2)} + 2d\mu_2} =
& (1-\epsilon) \sqrt{2 d\mu_2} 
& \Longrightarrow \\
{z_{x_1}\sqrt{d(\mu_4-\mu_2^2)} + 2d\mu_2} =
& (1-\epsilon)^2 {2 d\mu_2} 
& \Longrightarrow \\
z_{x_1} \sqrt{d(\mu_4-\mu_2^2)} = 
& (\epsilon^2- 2\epsilon) 2d\mu_2 
& \Longrightarrow \\ 
z_{x_1} = & \displaystyle \frac{2\mu_2}{\sqrt{\mu_4-\mu_2^2}} (\epsilon^2-2\epsilon) \sqrt{d}
& \Longrightarrow \\
z_{x_1} = & \displaystyle - \frac{2\epsilon(2-\epsilon) }{\sqrt{\kappa-1}} \sqrt{d}.
\end{eqnarray*}
As for the point $\vect{x_2}$ defined in Theorem \ref{lemma:cfof},
for symmetry it is such that
\[ z_{x_2} = \frac{2\mu_2}{\sqrt{\mu_4-\mu_2^2}} (\epsilon^2+2\epsilon) \sqrt{d} =
\frac{2 \epsilon (2 +\epsilon)}{\sqrt{\kappa-1}} \sqrt{d}. \]
Now, consider any $\epsilon\in(0,1)$. Then, $z_{x_1} = O(-\sqrt{d})$ is negative 
and $z_{x_2} = O(\sqrt{d})$ is positive.
Moreover, for $d\rightarrow\infty$, $z_{x_1}$ diverges to $-\infty$ and 
$z_{x_2}$ diverges to $+\infty$ and, hence
\[ \lim_{d\rightarrow\infty}\Phi(z_{x_1}) = \Phi(-\infty)=0 
~~~\mbox{ and }~~~
\lim_{d\rightarrow\infty}\Phi(z_{x_2}) = \Phi(+\infty)=1,
\]
and the statement follows.
\end{proof}
As for the effect of the kurtosis on 
the concentration of distance-based
and density-based scores scores,
from the expressions of $z_{x_1}$ and $z_{x_2}$ in 
Theorem \ref{th:knn_conc}, we can conclude
the convergence rate towards concentration 
of these scores is inversely 
proportional to the 
square root of the data kurtosis.

\begin{theorem}\label{th:aknn_conc}
For any fixed integer number $k$ such that $0<k<n$ 
(for any fixed real number $\varrho\in(0,1)$, resp.), 
the $\aKNN$ outlier score with parameter $k$ 
($k = \varrho n$, resp.)
concentrates.
\end{theorem}
\begin{proof}
Recall that 
\[ \aKNN(\vect{x}) = \sum_{i=1}^k \dist(\vect{x},\nn_i(\vect{x})) 
= \sum_{i=1}^k \KNN_i(\vect{x}), \]
and, hence, $\aKNN$ can be regarded as the sum of $k$ 
$\KNN$ scores,
each one associated with a different parameter ranging in $\{1,2,\ldots,k\}$.
Since, by Theorem \ref{th:knn_conc}, the $\KNN$ score concentrates,
it must be the case that also the sum of $k$ of these $\KNN$ scores
must concentrate and, hence, also the
$\aKNN$ score concentrates.
\end{proof}

\begin{theorem}\label{th:lof_conc}
For any fixed integer number $k$ such that $0<k<n$ 
(for any fixed real number $\varrho\in(0,1)$, resp.), 
the $\LOF$ 
outlier score with parameter $k$ 
($k = \varrho n$, resp.)
concentrates.
\end{theorem}
\begin{proof}
The $\LOF$ score is defined as
\[ 
\LOF(\vect{x}) \approx \frac{ \textit{lr--dist}(\vect{x}) }{\frac{1}{k} \sum_{i=1}^k \textit{lr--dist}(\nn_i(\vect{x}))},
\]
where
\begin{eqnarray*}
\textit{k--dist}(\vect{x}) & = & \dist(\vect{x},\nn_k(\vect{x})), \\
\textit{r--dist}(\vect{x},\vect{y}) & = & \max \{ \dist(\vect{x},\vect{y}), \textit{k--dist}(\vect{y}) \}, \mbox{ and} \\
\textit{lr--dist}(\vect{x}) & = & \sum_{i=1}^k \textit{r-dist}(\vect{x},\nn_i(\vect{x})), \\
\end{eqnarray*}
are the $k$-distance (\textit{k--dist}), 
the reachability-distance (\textit{r--dist}), 
and the local reachability distance (\textit{lr--dist}), respectively.
This score assigns value $1$ to points inside clusters and score significantly 
larger than $1$ to outliers.

Consider outlier points $\vect{x}$. Since outliers are not neighbors of their nearest neighbors,
that is $\dist(\vect{x},\nn_i(\vect{x})) > \textit{k--dist}(\nn_i(\vect{x}))$ for $i\le k$,
then the numerator of $\LOF$ must coincide with $\aKNN(\vect{x})$.
As for the denominator, note that for any point $\vect{x}$, $\aKNN(\nn_i(\vect{x})) \le \textit{lr--dist}(\nn_i(\vect{x}))$.
Hence, for $\vect{x}$ an outlier point
\[ 
\LOF(\vect{x}) \le \frac{ \aKNN(\vect{x}) }{\frac{1}{k} \sum_{i=1}^k \aKNN(\nn_i(\vect{x}))}.
\]
The concentration of $\LOF$ then follows from the concentration of the $\aKNN$ score.
\end{proof}

Summarizing, from the above analysis it follows that distance-based and density-based scores concentrate 
in the following cases:
\begin{itemize}
 \item[---] \textit{sample size $n$}: both bounded ($n>k$ finite) and unbounded ($n\rightarrow\infty$);
 \item[---] \textit{parameter $k$}: both fixed ($k={\rm const.}$) and variable ($k=\varrho n$ with $\varrho \in(0,1)$).
\end{itemize}
Moreover,
note that for finite $n$, these score concentrate even for 
$k=n$ (or, equivalently, $\varrho=1)$.\footnote{
Consider the proof of Theorem \ref{th:knn_conc}.
Even if for $k=n$ the term $\Phi^{-1}(k/n)$  evaluates to $+\infty$, 
this can be considered valid only in the case of infinite sample sizes.
For finite samples, the contribution to the $\KNN$ score associated with the 
standard deviation $\sigma_{\norm{\vect{x}-\Vect{X}}^2}$ of the
squared distance must be finite and, hence, can be ignored, as already done
in the proof of Theorem \ref{th:knn_conc} for $k<n$.}
Thus,
the only way to avoid concentration of distance-based and density-based
outlier scores is to consider infinite samples, that is $n=\infty$, and to employ the
parameter $k=\infty$, or equivalently $\varrho=1$.\footnote{In this case 
distance-outliers coincide with the points located on the
boundaries of the data distribution support. If the support is infinite, 
outliers are located at infinity.}

\subsection{On the concentration properties of
reverse nearest neighbor-based outliers}\label{sect:conc_rnnc}

Consider the following result from \cite{Newman1983,Angiulli2018}: Let $k>0$ be a fixed natural
number, then
\[ \lim_{n\rightarrow\infty} \lim_{d\rightarrow\infty} \N_k \xrightarrow{\,D\,} 0, \]
where the convergence is in distribution.

Intuitively, the above result states that if $k$ is held fixed, while
both the dimensionality $d$ and the sample size $n$ tend to infinity,
then the $k$-occurrences function will tend to be identically equal to zero.
This is a consequence of the hubness phenomenon, since only a few points, the hubs, 
located in proximity of the mean, will be selected as $k$ nearest neighbors
by any other point.
Thus, the $\RNNc$ score for fixed parameter $k$
and unbounded sample sizes $n$ is subject to concentration.\footnote{This also means that
we can enforce $\CFOF$ to concentrate
only by allowing inconsequentially small reverse neighborhood sizes
$\varrho\rightarrow 0$ in presence of unboundedly large samples $n\rightarrow\infty$.}

Moreover, by Equation \eqref{eq:nk}, the
concentration of the $\RNNc$ scores can be avoided by 
relating the parameter $k$ to the sample size.
This behavior of $\RNNc$ scores has been already
observed in the literature \cite{RadovanovicNI15},
by noticing that for $\RNNc$
the discrimination of scores represents a notable weakness
of the methods, with two contributing factors: hubness and inherent discreteness. 
Thus, 
in order to add more discrimination to $\RNNc$ 
they suggested to
raise $k$, possibly to some
value comparable with $n$, but with two concerns: ($i$)
with increasing $k$ the notion of outlier moves from local to global,
thus if local outliers are of interest they can be missed; ($ii$) $k$ values
comparable with $n$ raise issues with computational complexity.

It is important to highlight that the behavior of the $\CFOF$ score 
is deeply different from that of the $\RNNc$ scores,
since $\CFOF$ is not affected at all by the two above mentioned problems.
Indeed, $\CFOF$ outliers are well separated from inliers also for 
relatively small
values of the parameter $\varrho$.

\begin{figure}[t]
\centering
\subfloat[\label{fig:cfofvsnk1}]
{\includegraphics[width=0.48\columnwidth]{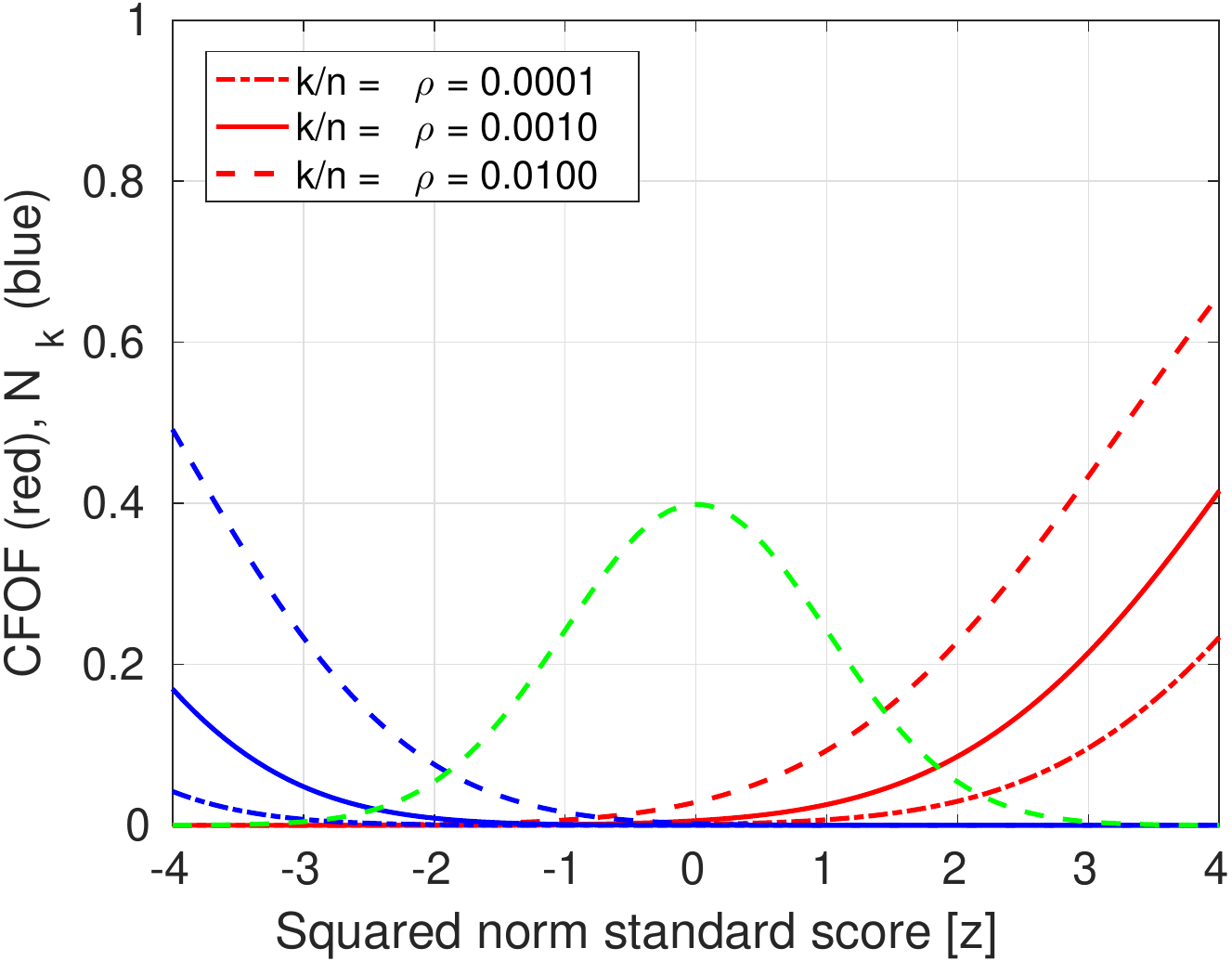}}
~
\subfloat[\label{fig:cfofvsnk2}]
{\includegraphics[width=0.48\columnwidth]{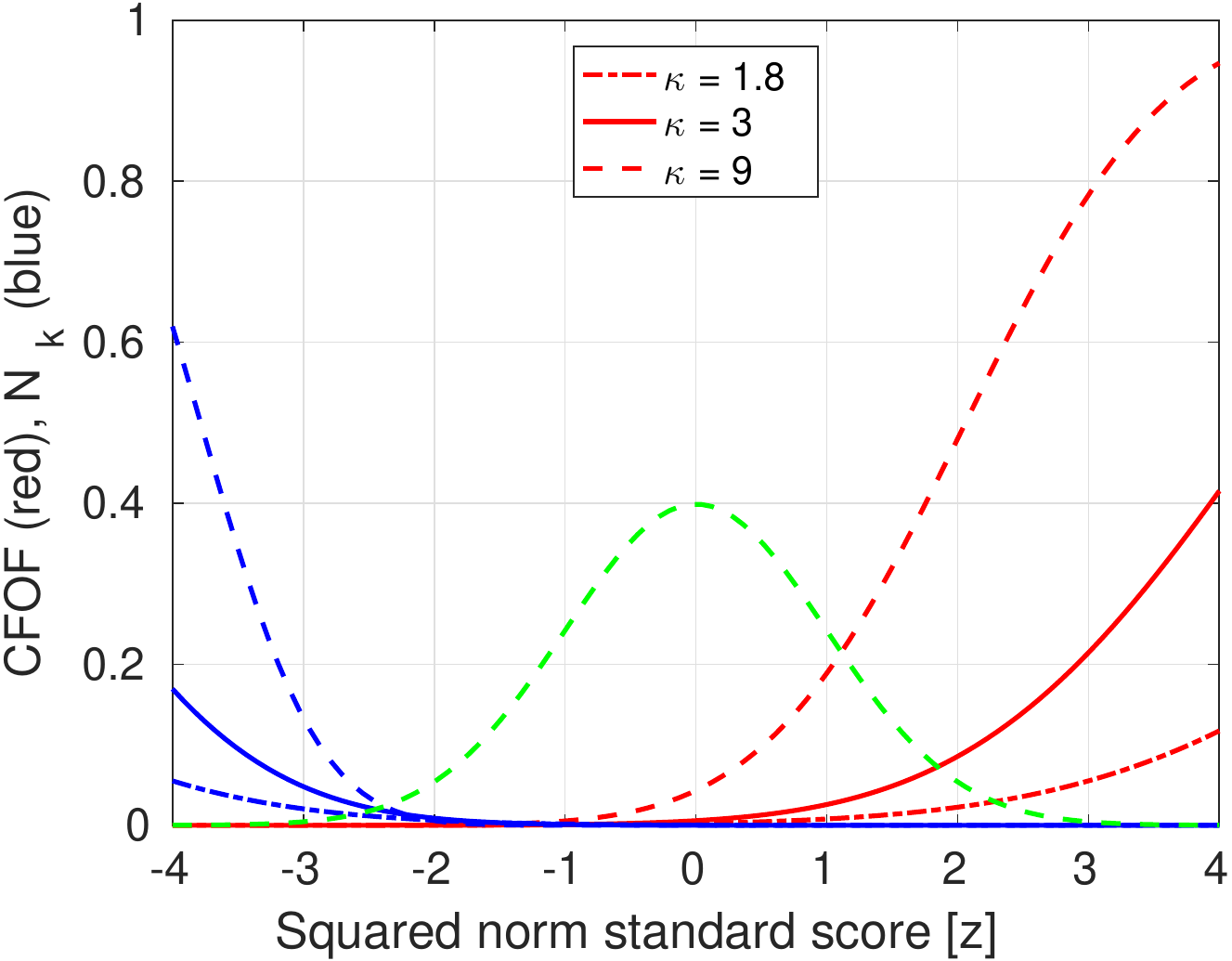}}
\\
\subfloat[\label{fig:cfofvsnk1bis}]
{\includegraphics[width=0.48\columnwidth]{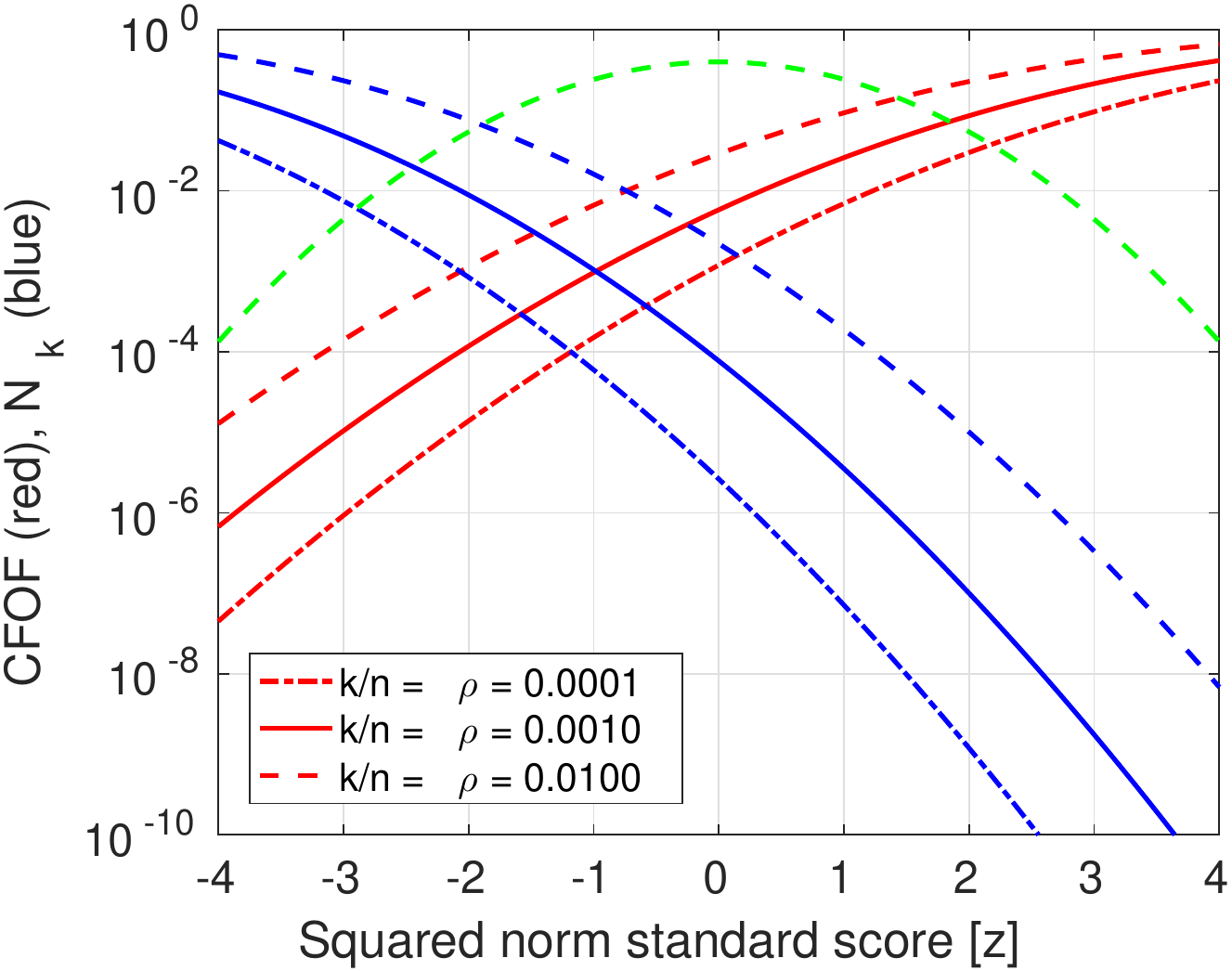}}
~
\subfloat[\label{fig:cfofvsnk2bis}]
{\includegraphics[width=0.48\columnwidth]{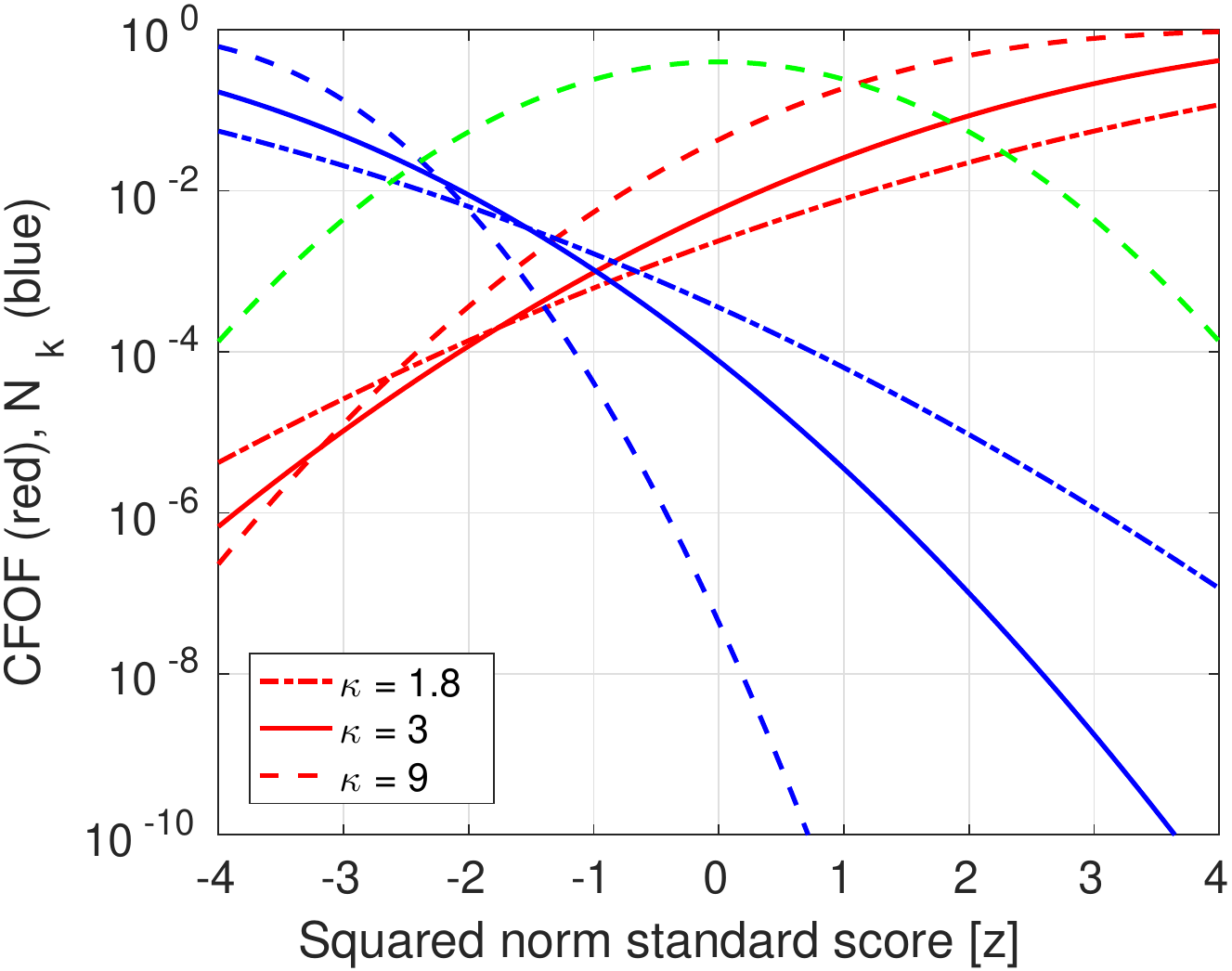}}
\caption{[Best viewed in color.] $\CFOF$ (red curves) and 
normalized $\RNNc$ (blue curves) scores versus the squared norm standard
score $z$, for different combinations of 
$\varrho = n/k\in\{ 0.0001, 0.001, 0.01 \}$ and kurtosis $\kappa=3$
(figs. \ref{fig:cfofvsnk1} and \ref{fig:cfofvsnk1bis},
the latter reporting scores in logarithmic scale),
and different kurtoses $\kappa\in\{1.8,3,9\}$ and $\varrho=n/k=0.001$
(figs. \ref{fig:cfofvsnk2} and \ref{fig:cfofvsnk2bis},
the latter reporting scores in logarithmic scale).}
\label{fig:cfofvsnk}
\end{figure}

To illustrate,
Figure \ref{fig:cfofvsnk} reports
for arbitrary large dimensionalities,
the $\CFOF$ (red increasing curves) 
and $\RNNc$ scores (blue decreasing curves)
as a function of the squared norm standard score $z$.
As for the green dashed (normally distributed) curve, it is proportional to the
fraction of the data population having the
squared norm standard score reported on the abscissa.

Figure \ref{fig:cfofvsnk1} is
concerns i.i.d. random vectors having distribution with
kurtosis $\kappa=3$ (as the normal distribution)
for different values of $\varrho = k/n$,
namely $\varrho=0.0001$ (dash-dotted curves),
$\varrho=0.001$ (solid curves), 
and $\varrho = 0.01$ (dashed curves).
Since outliers are the points associated with largest values
of $z$, it is clear that while $\CFOF$ is able to establish 
a clear separation between outliers and inliers for any value
of the parameter $\varrho$,  
the $\N_k$ function is close to zero for $z\ge 0$,
thus presenting a large false positive rate.

Figure \ref{fig:cfofvsnk2} reports
the $\CFOF$ and $\RNNc$ scores
associated with i.i.d. random vectors having cdfs with
different kurtosis values, namely $\kappa = 1.8$ (as the uniform distribution),
$\kappa = 3$ (as the normal distribution), 
and $\kappa = 9$ (as the exponential distribution)
when {$\varrho = k/n = 0.001$}.

Figures \ref{fig:cfofvsnk1bis} and \ref{fig:cfofvsnk2bis}
report the scores in logarithmic scale,
to appreciate differences between $\N_k$ values.
According to Equation \eqref{eq:nk}, for $z\ge 0$ the $\RNNc$ 
score is monotonically decreasing with the kurtosis.
However, the figures point out that in this case the $\RNNc$ outlier scores 
are infinitesimally small and even decreasing of orders of magnitude
with the kurtosis.
By applying the same line of reasoning of Theorem \ref{th:cfofhomogeneous}
and Theorem \ref{th:cfofclustbalanced} to Equation \eqref{eq:nk}, 
we can show that $\RNNc$ is both translation-invariant and homogeneous
and semi--local.
Unfortunately, it must be also pointed out that 
the hubness phenomenon affecting the $\RNNc$ score
and the discussed behavior of $\RNNc$ for different kurtosis values,
make the above properties immaterial for many combinations
of $n$ and $k$, and for many combinations of clusters kurtoses.

Consider the standard deviation $\sigma^{out}_{z_0}(sc)$
of the outlier score $sc$ distribution starting from $z_0$:
\[ 
\sigma^{out}_{z_0}(sc) = \int_{z_0}^{+\infty} (sc(z)-\mu^{out}_{z_0}(sc))^2 \cdot \phi(z) ~{\rm d}z 
 ~~~\mbox{ with }~~~
\mu^{out}_{z_0}(sc) = \int_{z_0}^{+\infty} sc(z)\cdot \phi(z) ~{\rm d}z,
\]
where,
$\mu^{out}_{z_0}(sc)$ represents the mean of the above distribution.
Intuitively, $\sigma^{out}_{z_0}(sc)$ measures the amount of variability
of the score values $sc$ associated with the most extreme observations of the 
data population.
Thus, the ratio 
\[ \sqrt{\frac{\sigma^{out}_{z_0}(\CFOF)}{\sigma^{out}_{z_0}(\ODIN)}}, \]
called \textit{separation} for short in the following,
measures how larger is the variability of the scores of the $\CFOF$ outliers
with respect to that of $\ODIN$.

\begin{figure}[t]
\centering
\subfloat[\label{fig:sep1}]
{\includegraphics[width=0.48\columnwidth]{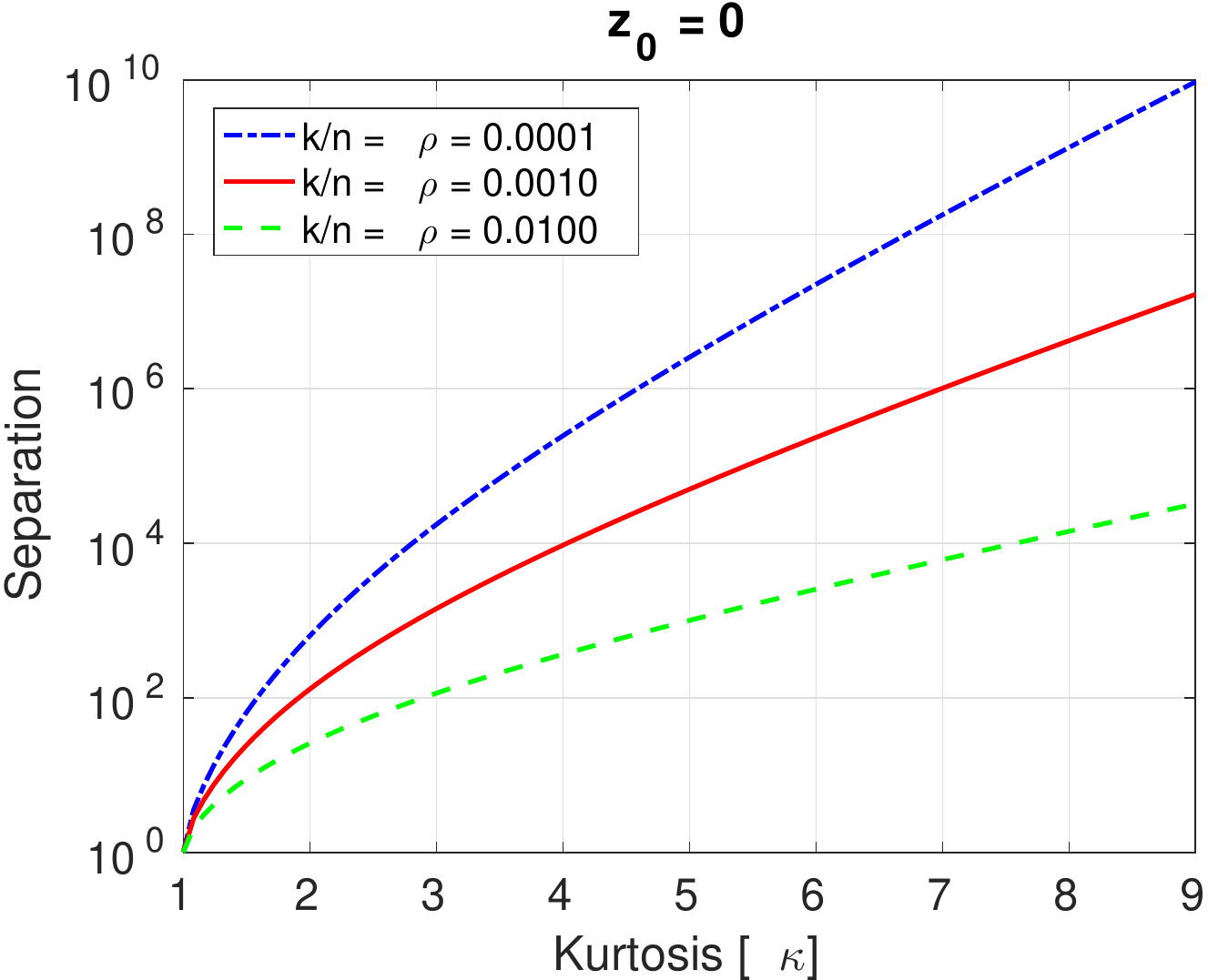}}
\subfloat[\label{fig:sep2}]
{\includegraphics[width=0.48\columnwidth]{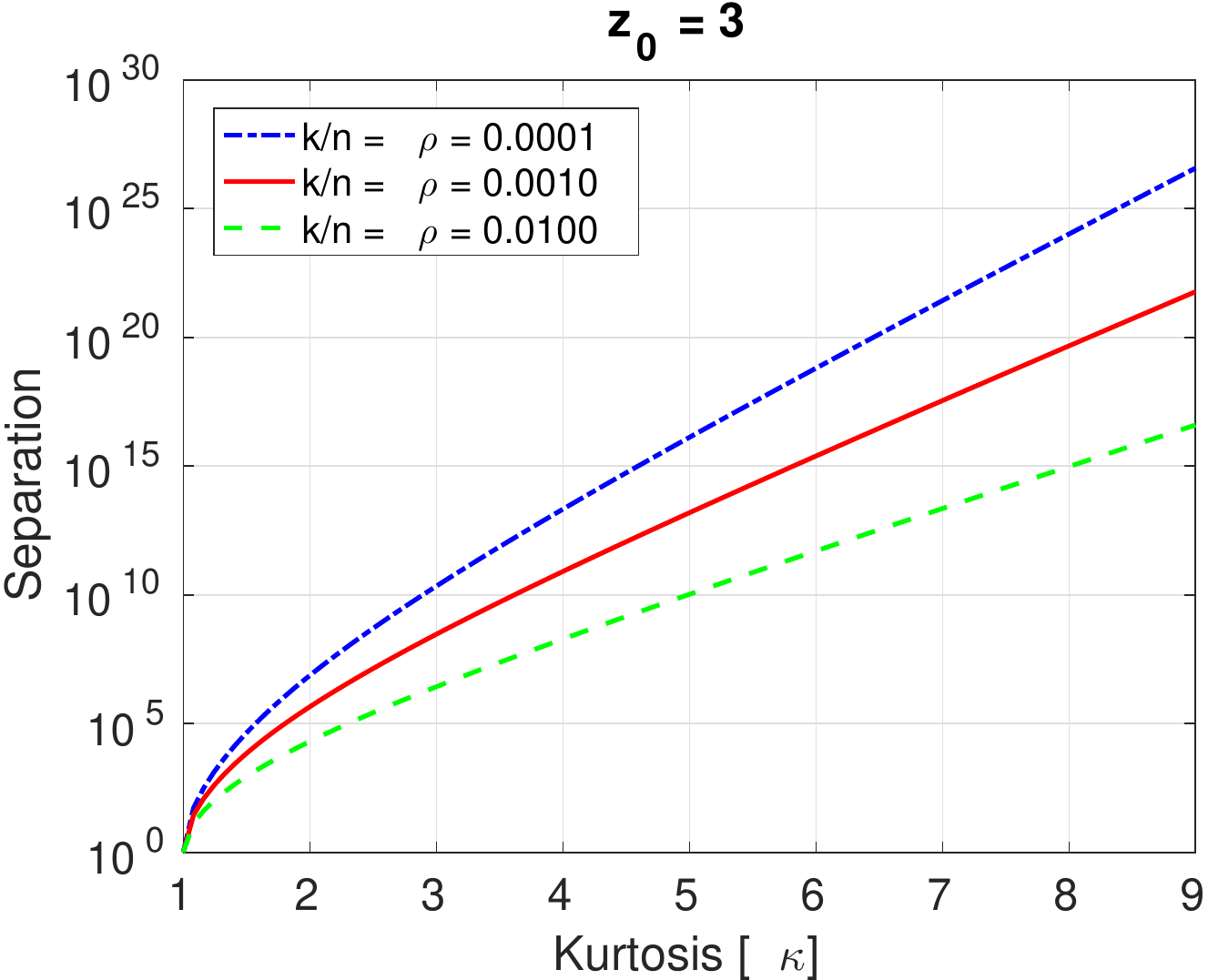}}
\caption{[Best viewed in color.] 
Ratio between the standard deviation of $\CFOF$ and normalized $\RNNc$ scores
associated with points having squared norm standard score greater than
$z_0$ ($z_0 = 0$ on the left and $z_0=3$ on the right).
}
\label{fig:separation}
\end{figure}

Figure \ref{fig:separation} shows the separation 
for arbitrary large dimensionalities, 
for kurtosis $\kappa\in[1,9]$,
and for $z_0=0$ (Figure \ref{fig:sep1}) and $z_0=3$ (Figure \ref{fig:sep2}).
According to the values reported in the plots, the separation 
is of several orders of magnitude and,
moreover, the smaller the parameter $\varrho$, 
the larger its value.

\section{The $\FastCFOF$ algorithm}
\label{sect:algorithm}

In general,
$\CFOF$ scores can be determined
in time $O(n^2 d)$, 
where $d$ denotes is the dimensionality
of the feature space (or the cost of computing a distance),
after computing all pairwise dataset distances.

In this section,
we introduce the $\FastCFOF$ technique for dealing with 
very large and high-dimensional datasets.
As its main peculiarity,
$\FastCFOF$ does not require the computation of the exact
nearest neighbor sets and, from the computational point of
view, does not suffer of 
the curse of dimensionality affecting 
nearest neighbor search techniques.

The technique builds on the following probabilistic formulation of the $\CFOF$
score. 
Let $\Vect{X}$ be a random vector 
from an unknown probability law $p(\cdot)$. Given parameter $\varrho\in(0,1)$,
the (\textit{Probabilistic}) \textit{Concentration Outlier Factor} $\CFOF$ of
the realization $\vect{x}$ of $\Vect{X}$ 
is defined as follows:
\begin{equation}\label{eq:cfof_soft}
\CFOF(\vect{x}) = \min \left\{ 
\frac{k'}{n} 
\mbox{ s.t. } Pr[\vect{x}\in\NN_{k'}(\Vect{X})] \ge \varrho
\right\},
\end{equation}
where the probability has to be considered conditional on $\vect{x}$ being
drawn in a sample of $n$ realizations of the random vector $\Vect{X}$.
The above probabiliy is estimated by using
the empirical distribution of the data as $p(\cdot)$.
To differentiate the two definitions reported in Equations \eqref{eq:cfof_hard}
and \eqref{eq:cfof_soft}, we also refer to the former as \textit{hard}-$\CFOF$ and
to the latter as \textit{soft}-$\CFOF$.

Intuitively, the \textit{soft}-$\CFOF$ score measures how many neighbors have to be
taken into account in order for the expected number of dataset objects having it
among their neighbors to correspond to a fraction $\varrho$ of the overall population.

The rationale of Equation \eqref{eq:cfof_soft} is that of providing 
a definition
tailored to leverage sampling for determining $\CFOF$ outliers.
Indeed, as discussed in the following section describing the 
$\FastCFOF$ algorithm, the probability $Pr[\vect{x}\in\NN_{k'}(\Vect{X})]$ 
is estimated by looking at the number of observations having $\vect{x}$
among their $k'$ nearest neighors in a sample of size $s$ of the whole dataset.
However, we point out that if $s$ is equal to the sample size $n$, then 
Equation \eqref{eq:cfof_soft} 
reduces to Equation \eqref{eq:cfof_hard},
that is \textit{soft}-$\CFOF$ and \textit{hard}-$\CFOF$ coincide.
Indeed, in this case the above probability is smaller $\varrho$
for $k' < n\cdot\CFOF(\vect{x})$ and not smaller than $\varrho$ for
$k' \ge n\cdot\CFOF(\vect{x})$.
This also means that the (exact) $\CFOF$ algorithm coincides with $\FastCFOF$ for $s=n$.

\subsection{The $\FastCFOF$ technique}

\newcommand{\stepF}[1]{H\big(#1\big)}

The
$\FastCFOF$ algorithm aims at improving efficiency by
estimating $\SoftCFOF$ scores.

Given a dataset $\DS$ and two objects $x$ and $y$
from $\DS$, 
the building block of the algorithm is the computation of 
the probability $Pr[x\in\NN_k(y)]$.
Consider the boolean function $B_{x,y}(z)$ defined on
instances $z$ of $\DS$ such that
$B_{x,y}(z)=1$ 
if $z$ lies within the hyper-ball of center $y$ and radius $\dist(x,y)$,
and $0$ otherwise.
We want to estimate the average value $\overline{B}_{x,y}$ of $B_{x,y}$ in 
$\DS$,
that is the ratio of instances $z\in\DS$ such that $B_{x,y}(z)=1$ holds,
which corresponds to the probability $p(x,y)$
that a randomly picked dataset object $z$ is at distance
not grater than $\dist(x,y)$ from $y$, that is
$\dist(y,z) \le \dist(x,y)$.

In particular, for our purposes it is enough to compute
$\overline{B}_{x,y}$ within a certain error bound.
Thus, we resort to the \textit{batch sampling} technique,
which consists in picking up $s$
elements of $\DS$ randomly and estimating
$p(x,y) = \overline{B}_{x,y}$ as the fraction of the elements of the sample
satisfying $B_{x,y}$ \cite{Watanabe2005}.

Let $\hat{p}(x,y)$ the value of
${p}(y,x)$ estimated by means of the sampling procedure.
Given $\delta>0$, an \textit{error probability},
and $\epsilon$, $0<\epsilon<1$, an \textit{absolute error},
if the size $s$ of the sample satisfies certain
conditions \cite{Watanabe2005}
the following relationship holds
\begin{equation}\label{eq:pxy_bound}
Pr[|\hat{p}(x,y)-p(x,y)|\le\epsilon] > 1-\delta. 
\end{equation}
For large values of $n$, it holds that
the variance of the Binomial distribution
becomes negligible with respect to the mean
and the cumulative distribution function
$\textit{binocdf}(k;p,n)$
tends to the step function $\stepF{k-np}$,
where $\stepF{k}=0$ for $k<0$ and $\stepF{k}=1$ for $k>0$.
Thus, we can approximate the value 
$Pr[x\in\NN_k(y)] = \textit{binocdf}(k;p(x,y),n)$ 
with the boolean function $\stepF{k-k_{up}(x,y)}$,
with $k_{up}(x,y)=n \widehat{p}(x,y)$.\footnote{Alternatively,
by exploiting the Normal approximation
of the Binomial distribution,
a suitable value for $k_{up}(x,y)$
is given by
$k_{up}(x,y) = n \widehat{p}(x,y) + c \sqrt{n \widehat{p}(x,y) 
(1-\widehat{p}(x,y))}$
with $c\in[0,3]$.}
It then follows that 
we can obtain $Pr[x\in\NN_k(\textbf{X})]$
as the average value of
the boolean function $\stepF{k-n \widehat{p}(x,y)}$,
whose estimate can be again obtained by exploiting
batch sampling.
Specifically, $\FastCFOF$ exploits one single sample
in order to perform the two estimates above described.

\begin{algorithm}[t]\small 
	\KwIn{Dataset $\DS = x_1,\ldots,x_n$ of size $n$,
	parameters $\bm{\varrho}=\varrho_1,\ldots,\varrho_\ell\in(0,1)$,
	parameters $\epsilon,\delta \in(0,1)$,
	number of histogram bins $B$
	}
	\KwOut{$\CFOF$ scores $sc_{1,\bm{\varrho}},\ldots,sc_{n,\bm{\varrho}}$}
	\SetAlgorithmName{Algorithm}{}{}
	$s$ = 
$\left\lceil\frac{1}{2\epsilon^2}\log\left(\frac{2}{\delta}
\right)\right\rceil$\;
	$i$ = $0$\;
	\While{$i < n$}{
	    \If{$i+s < n$}{
		$a$ = $i+1$\;
	    }\Else{
		$a$ = $n-s+1$\;
	    }
	    $b$ = $a+s-1$\;
	    $part$ = $\langle x_{a},\ldots,x_{b} \rangle$\;
	    $\langle sc_{a,\bm{\varrho}},\ldots,sc_{b,\bm{\varrho}} \rangle$ = 
			$\CFOFprocpart(part,\bm{\varrho},B,n)$\;
	    $i$ = $i+s$\;
	}
	\caption{\FastCFOF}
	\label{algo:fast_cfof}
\end{algorithm}

The pseudo-code of $\FastCFOF$ 
is reported in Algorithm \ref{algo:fast_cfof}.
The algorithm receives in input a list
$\bm{\varrho} = \varrho_1,\ldots,\varrho_\ell$
of values for the parameter $\varrho$, since it is able to
perform a \textit{multi-resolution analysis}, that is 
to simultaneously compute dataset scores associated with different values
of the parameter $\varrho$, with the same asymptotic temporal cost.
We assume 
$\bm{\varrho}=0.001, 0.005, 0.01, 0.05, 0.1$ and $\epsilon,\delta=0.01$
as default values for the parameters.
See later for details on the effect of the parameters.

First, 
the algorithm determines the size $s$
of the \textit{sample} (or \textit{partition}) of the dataset
needed in order to guarantee the
bound reported in Eq. \eqref{eq:pxy_bound}.
We notice that
the algorithm does not require the dataset to be 
entirely loaded in main memory, since only one partition at
a time is needed to carry out the computation. Thus, the technique
is suitable also for disk resident datasets.
Moreover, 
we assume that the position of the objects within
the dataset has no connection with their spatial relationship,
for otherwise a preprocessing step consisting in
randomizing the dataset is required \cite{Sanders98}.

Then, each partition consisting of a group of $s$ 
consecutive objects, is processed by
$\CFOFprocpart$, whose pseudo-code is reported in Procedure 
\ref{algo:process_partition}.
This subroutine 
estimates $\CFOF$ scores of the objects within
the partition
through batch sampling.

The matrix $hst$, consisting of $s\times B$
counters, is employed by $\CFOFprocpart$.
The entry $hst(i,k)$ of $hst$ is used to estimate
how many times the sample object $x'_i$ is 
the $k$th nearest neighbor of a generic object
dataset.
Values of $k$, ranging from $1$ to $n$, are partitioned into $B$
log-spaced bins. The function $\textit{k\_bin}$ maps original
$k$ values to the corresponding bin, while $\textit{k\_bin}^{-1}$
implements the reverse mapping by returning a certain value
within the corresponding bin.

\begin{algorithm}[t]\small 
	\KwIn{Dataset sample $\langle x'_1,\ldots,x'_s \rangle$ of size $s$, 
	parameters $\varrho_1,\ldots,\varrho_\ell\in(0,1)$,
	histogram bins $B$,
	dataset size $n$}
	\KwOut{$\CFOF$ scores $\langle 
sc'_{1,\bm{\varrho}},\ldots,sc'_{s,\bm{\varrho}} \rangle$}
	\SetAlgorithmName{Procedure}{}{}
	initialize matrix $hst$ of $s\times B$ elements to $0$\; 
	\tcp{Nearest neighbor count estimation}
	\ForEach{$i=1$ to $s$}{
	    \tcp{Distances computation}
	    \ForEach{$j=1$ to $s$}{
		$dst(j)$ = $\dist(x'_i,x'_j)$\;
	    }
	    \tcp{Count update}
	    $ord$ = $\textit{sort}(dst)$\;
	    \ForEach{$j=1$ to $s$}{
		$p$ = $j/s$\;
		$k_{up}$ = $\lfloor n p + c\sqrt{np(1-p)} + 0.5 \rfloor$\;
		$k_{pos}$ = $\textit{k\_bin}(k_{up})$\;
		$hst(ord(j),k_{pos}) = hst(ord(j),k_{pos}) + 1$\;
	    }
	}
	\tcp{Scores computation}
	\ForEach{$i=1$ to $s$}{
	    $count$ = $0$\;
	    $k_{pos}$ = $0$\;
	    $l$ = $1$\;
	    \While{$l \le \ell$}{
		\While{$count <s \varrho_l$}{
		    $k_{pos}$ = $k_{pos} + 1$\;
		    $count$ = $count + hst(i,k_{pos})$\;
		}
		$sc'_{i,\varrho_l} = \textit{k\_bin}^{-1}(k_{pos})/n$\;
		$l$ = $l + 1$\;
	    }
	}
	\caption{\CFOFprocpart}
	\label{algo:process_partition}
\end{algorithm}

For each sample object $x'_i$
the distance $dst(j)$ from any other sample object $x'_j$
is computed (lines 3-4) and, then,
distances are ordered (line 5)
obtaining 
the list $ord$ of sample identifiers
such that
$dst(ord(1))\le dst(ord(2)) \le\ldots\le dst(ord(s))$.
Notice that, since we are interested only in ranking distances,
as far as distance computations is concerned, 
in the Euclidean space the squared distance can
be employed to save time.

Moreover, for each element $ord(j)$ of $ord$,
the variable $p$ is set to $j/s$ (line 7),
representing the probability $p(x'_{ord(j)},x'_i)$, 
estimated through the sample,
that a randomly picked 
dataset object is located within the region
of radius $dst(ord(j)) = dist(x'_i,x'_{ord(j)})$
centered in $x'_i$.
The value $k_{up}$ (line 8)
represents the point of transition 
from $0$ to $1$ of the step function $\stepF{k-k_{up}}$
employed to approximate the probability 
$Pr[x'_{ord(j)}\in\NN_k(y)]$ when $y=x'_i$.
Thus, before concluding each cycle of the inner loop (lines 6-10),
the $\textit{k\_bin}(k_{up})$-th 
entry of $hst$ associated with the sample
$x'_{ord(j)}$ is incremented.

The last step of the procedure consists
in the computation of the $\CFOF$ scores.
For each sample $x'_i$, the associated counts
are accumulated, by using the variable $count$, 
until their sum exceeds the value
$s \varrho$ and, then, the associated value of $k_{pos}$
is employed to obtain the score.
We notice that the parameter $\varrho$ is
used only to determine when to stop the above
computation. By using this strategy, $\FastCFOF$
supports multi-resolution outlier analysis
with no additional cost (lines 12-20).

\subsection{Temporal and spatial cost}\label{sect:fastcfof_cost}

Now we consider
the temporal cost of the algorithm $\FastCFOF$.
The procedure $\CFOFprocpart$
is executed $\left\lceil\frac{n}{s}\right\rceil$ times.
During this procedure,
the first part, concerning nearest neighbor counts estimation,
costs $O(s \cdot (sd + s\log s + s))$, since
computing distances costs $O(s d)$, 
sorting distances costs $O(s\log s)$,
and updating counts costs $O(s)$ ($s$ iterations, each of which has cost 
$O(1)$).
As for the second part, concerning scores computation,
in the worst case all the entries of the matrix $hst$ are
visited, and the cost is $O(s B)$, where $B$,
the number of bins within $hst$, is a constant.
Summarizing, the temporal cost of $\FastCFOF$ is
\begin{equation}\label{eq:fastcfof_cost}
O\left( \left\lceil \frac{n}{s} \right\rceil \cdot
\big( s \cdot \left( s d + s\log s + s \right) + s B \big) \right),
\end{equation}
that is $O(n\cdot s\cdot \max\{d, \log s\})$.
If we assume that $d$ dominates $\log s$, we can conclude that
the asymptotic cost of the technique is
\begin{equation}\label{eq:fastcfof_cost_asympt}
O\left( s\cdot n\cdot d \right), 
\end{equation}
where $s$ is independent of the number $n$
of dataset objects, and can be considered a constant
depending only on $\epsilon$ and $\delta$,
and $n\cdot d$ represents precisely the size of the input.
We can conclude that technique is linear in the 
size of the input data.

Notice that the method can be easily adapted
to return the top $m$ outliers, by introducing $\ell$
heaps of $m$ elements each, to be updated 
after each partition
computation. The additional $O(n\log m)$ temporal cost 
has no impact on the overall cost.

As for the spatial cost, $\CFOFprocpart$
needs $O(s B)$ space for storing counters $hst$,
space $O(2 s)$ for storing distances $dst$ and
the ordering $ord$, and space $\ell s$ for storing scores
to be returned in output, hence $O(s \big( B + 2 + \ell) \big)$.
Hence, the cost is linear in the sample size.
If the dataset is disk-resident,
the a buffer maintaining the sample
is needed requiring additional $O(s d)$ space.

\subsection{Parallel $\FastCFOF$}
\label{sect:alg_parcfof}

The $\FastCFOF$ algorithm
can take advantage 
of parallelization techniques.
First of all, partitions can be processed 
independently by different processors or nodes
of a multi-processor/computer system.
As for the computations pertaining to single
partitions, it follows from the cost analysis
that the heaviest task is the computation of 
pairwise distances.
This task can 
take fully advantage of 
MIMD parallelism
by partitioning distance computations on the 
cores of a multi-core processor, and 
of SIMD parallelization
through the use of vector primitives
that are part of instruction sets
for computing distances.
All of these computations 
are embarrassingly parallel, since they
do not need communication
of intermediate results.

Let $P$ denote the number of distinct processors/nodes
available (in the multicomputer scenario the overhead due to need of 
partitioning the dataset on the different
nodes has to be considered),
let $C$ denote the number of cores (or hardware threads) 
per multi-core processor,
and $V$ denote the width of the vector registers 
within each core, then the cost of the parallel version
of $\FastCFOF$ is
\begin{equation}\label{eq:fastcfof_cost_par}
O\left( \frac{s\cdot n\cdot d}{P\cdot C\cdot V} \right).
\end{equation}
We implemented a parallel version
for multi-core processors 
working on disk-resident datasets
that elaborates partitions sequentially, 
but simultaneously employs all cores and the vector registers to elaborate
each single partition.
Specifically, this version has been implemented
by using the C Programming Language ({\tt gcc} compiler),
with the Advanced Vector Extensions (AVX) 
intrinsics functions to perform SIMD computations
(vector instructions process in parallel $V=8$
pairs of {\tt\small float}),
and with the Open Multiprogramming (OpenMP)
API to perform MIMD (multi-core) computations.

The implementation exploits column-major ordering of the 
sample points, to make the code independent of the number of 
attributes, and the {loop unrolling} and 
{cache blocking} techniques \cite{PetersenA04}:
\textit{loop unrolling} allows
to reduce loop control instructions and branch penalties,
to hide memory latencies, and to increase exploitation
of replicated vectorial arithmetic logical units 
typical of super-scalar processors, while
\textit{cache blocking} improves the locality of memory accesses.

\section{Experimental results}
\label{sect:experiments}

In this section, we present experimental results concerning $\CFOF$.

{Experiments are organized as follows.}
We discuss experimental
results involving the $\FastCFOF$ algorithm,
including \textit{scalability} and \textit{approximation accuracy} of the approach
(Section \ref{sect:exp_fastcfof}).
We illustrate experiments 
designed to study the \textit{concentration properties} of the $\CFOF$
definition \textit{on real-life data}
(Section \ref{sect:exp_conc}).
We investigate the behavior of
the $\CFOF$ definition on 
\textit{synthetically generated multivariate
data}, and compare it with existing reverse nearest neighbor-based, 
distance-based, density-based,
and angle-based outlier definitions
(Section \ref{sect:exp_synth}).
Finally,
we compare
$\CFOF$ with 
other outlier definitions
by using \textit{labelled data} as \textit{ground truth}
(Section \ref{sect:exp_real}).

Since in some experiments, results of $\CFOF$ are compared
with those obtained by other outlier detection methods,
namely $\ODIN$, $\antiHub$, $\aKNN$, $\LOF$, and $\FastABOD$, 
next we briefly recall these methods.

The $\aKNN$ method, for average KNN, is a distance-based approach 
that rank points on the basis of 
the average distance from their $k$ nearest neighbors, 
also called \textit{weight} in the literature \cite{AP02,AP05,AP06}.
The Local Outlier Factor method \cite{BKNS00}, $\LOF$ for short, 
is a density-based method which
measures the degree of an object to be an outlier by comparing 
the density in its neighborhood 
with the average density in the neighborhood of its neighbors,
where the density of a point is related to the distance
to its $k$-th nearest neighbor.\footnote{$\LOF$ uses the concept of 
\textit{reachability distance} to model the density of a point,
which, roughly speaking, corresponds to the distance from $k$-th nearest neighbor
with some minor modifications
aiming to mitigate the effect of statistical fluctuations of the distance.
}
Differently from distance-based definitions, which declare as outliers 
the points where the estimated data density is low, density-based definitions 
score points on the basis of the degree of disagreement between the estimated 
density of the point and the estimated density of its surrounding or neighboring 
points.
Thus, density-based outlier definitions are better characterized as a notion of local 
outlier, as opposite to distance-based definitions, representing a notion of global outlier.
The Angle-Based Outlier Detection method \cite{KriegelSZ08}, ABOD for short,
scores data points on the basis of the variability of the angles formed 
by a point with each other pair of points (the score is also called
ABOF, for Angle-Based Outlier Factor).
The intuition is that for a point within a cluster, 
the angles between difference vectors to pairs of other points 
differ widely, while the variance of these angles will become smaller 
for points at the border of a cluster and for isolated points.
The ABOF measures the variance 
over the angles between the difference vectors of a point to all pairs 
of other points in the dataset, 
weighted by the distance from these points.
The $\FastABOD$ method approximates the ABOF score 
by considering only the pairs of points with the 
strongest weight in the variance, that are the $k$ nearest neighbors.
{In some cases, this method is not included in the comparison
because of its slowness (the method has {cubic temporal cost}).}
$\ODIN$ \cite{HautamakiKF04}, also referred to as $\RNNc$,
is a {reverse nearest neighbor-based}
approach, which uses $\N_k(x)$ as outlier score of $x$.
Since $\ODIN$ is prone to the hubness phenomenon,
the $\antiHub$ method \cite{RadovanovicNI15}
refines $\ODIN$ 
by returning the weighted mean of the sum of the $\N_k$ scores of 
the neighbors of the point and of the $\N_k$ score of the point itself.

\subsection{Experiments with $\FastCFOF$}
\label{sect:exp_fastcfof}

Experiments of this section are designed to study 
the scalability (see Section \ref{sect:exp_fastcfof_scal})
and the accuracy (see Section \ref{sect:exp_fastcfof_acc})
of the $\FastCFOF$ algorithm.

\medskip
Experiments were performed on a PC equipped with an
Intel Core i7-3635QM $2.40$GHz CPU
having $4$ physical cores and $8$ hardware 
threads.\footnote{Ivy Bridge 
microarchitecture, launched in Q3'12, $6$MB of cache L3.}
The PC is equipped with $12$GB of main memory 
and the software runs under the Linux operating system.

The number $B$ of $hst$ bins was set to $1,\!000$ and
the constant $c$ used to compute $k_{up}$ was set to $0$.
The implementation of $\FastCFOF$ is that described in
Section \ref{sect:alg_parcfof} with
number of processors $P=1$, 
number of hardware threads $C=8$, 
and SIMD register width $V=8$ single-precision
floating point numbers ($64$ bit code
having $16$ vectorial $256$ bit registers).
If not otherwise stated, we assume $0.01$ as 
the default value for the parameters
$\varrho$, $\epsilon$, and $\delta$.

\medskip
The dataset employed are described next.
\textit{Clust2} 
is a synthetic dataset family, with $n\in[10^4,10^6]$ and $d\in[2,10^3]$,
consisting of two normally distributed clusters,
the first
centered in the origin and having standard deviation $1$,
and the second centered in $(4,\ldots,4)$ and having
standard deviation $0.5$.
The \textit{MNIST} 
dataset,\footnote{See {\tt http://yann.lecun.com/exdb/mnist/}} 
consists of $n=60,\!000$ vectors
having $d=784$ dimensions,
representing handwritten digits.
Digits have been
size-normalized and centered in a fixed-size 
$28\times 28$ gray level image.
The \textit{YearPredictionMSD} 
dataset,\footnote{See {\tt https://archive.ics.uci.edu/ml/datasets/yearpredictionmsd}}
or \textit{MSD} for short in the following,
is a subset of the Million Song Dataset,
a freely-available collection of audio features 
and metadata for a million contemporary popular music 
tracks.\footnote{See {\tt http://labrosa.ee.columbia.edu/millionsong/}}
The dataset consists mostly of western songs, 
commercial tracks ranging from $1922$ to $2011$.
There are $n=515,\!345$ instances
of $90$ attributes:
$12$ encode timbre averages and $78$ encode timbre covariances.
The \textit{SIFT10M} 
dataset,\footnote{See {\tt https://archive.ics.uci.edu/ml/datasets/SIFT10M}}
consists of $11,\!164,\!866$ data points
representing SIFT features ($d=128$)
extracted from the 
Caltech-256 object category 
dataset.\footnote{See {\tt http://resolver.caltech.edu/CaltechAUTHORS:CNS-TR-2007-001}}
Each SIFT feature 
is extracted from a $41\times 41$ image patch.
This data set has been used for evaluating 
approximate nearest neighbor search methods.

Datasets are encoded by using single-precision floating point values
(requiring $32$ bits)
and are stored in secondary memory, where they occupy
$3.8$GB ($n=10^6$ and $d=10^3$) for \textit{Clust2},
$180$MB for \textit{MNIST},
$177$MB for \textit{MSD},
and $5.5$GB for \textit{SIFT10M}.

\subsubsection{Scalability of $\FastCFOF$}
\label{sect:exp_fastcfof_scal}

\begin{figure}[t]
\centering
\subfloat[\label{fig:exp_scal_nd}]
{\includegraphics[width=0.48\textwidth]{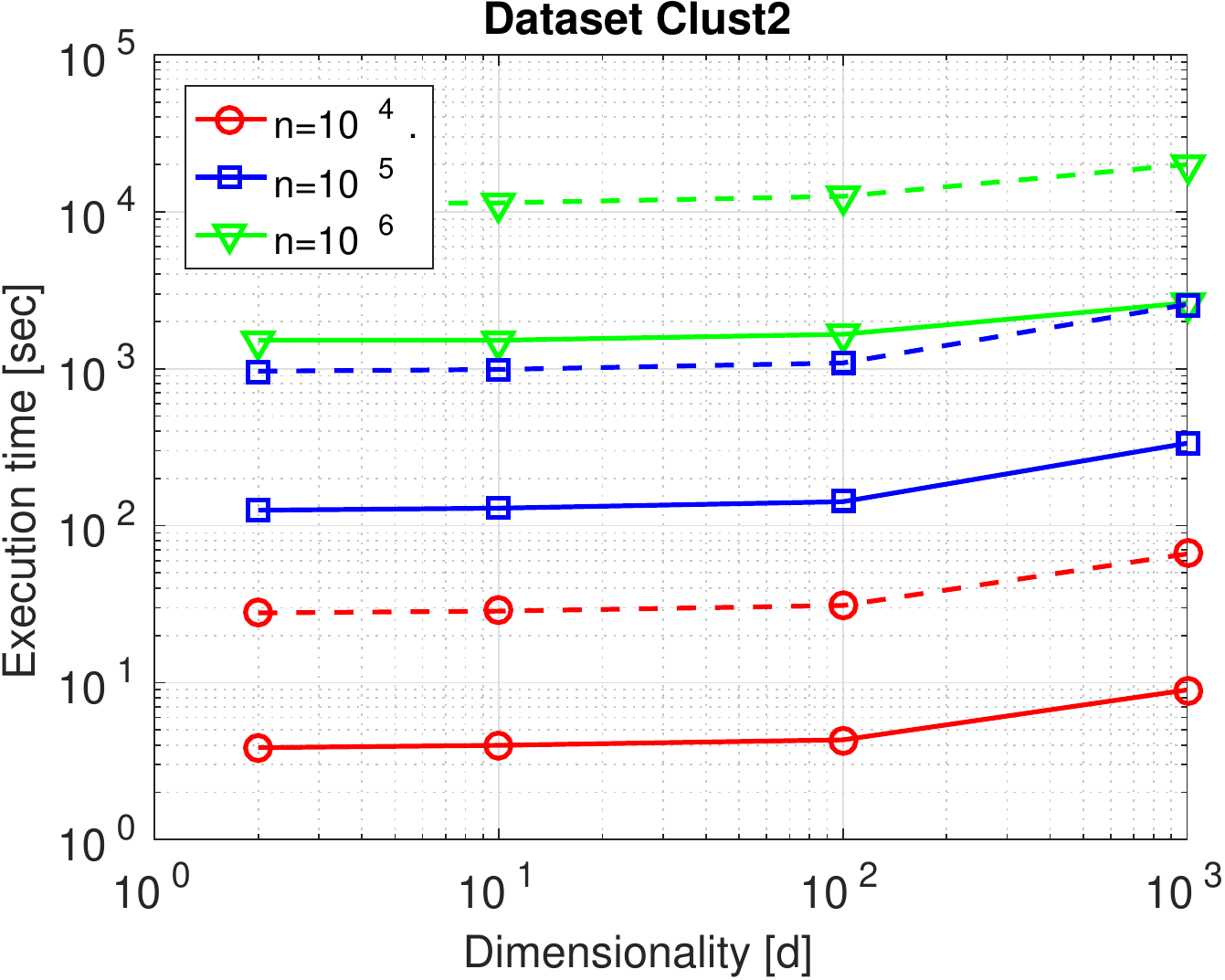}}
\caption{[Best viewed in color.] Scalability analysis of $\FastCFOF$
with respect to the dataset dimensionality $d$ and the dataset size $n$.}
\label{fig:exp_scal}
\end{figure}

In this section, we study the scalability of $\FastCFOF$
w.r.t. the dataset size $n$ and the dataset dimensionality $d$.
For this experiment we considered the \textit{Clust2} synthetic
dataset family.
Figure \ref{fig:exp_scal_nd}
shows the execution time on \textit{Clust2}
for the default sample size $s=26624$ ($\epsilon=0.01$ and $\delta=0.01$),
$n\in[10^4,10^6]$ and $d\in[2,10^3]$. 
The largest dataset considered is that
for $n=1,\!000,\!000$ and $d=1,\!000$
whose running time was about $44$ minutes.

As for the scalability with respect to $d$,
we can notice that
there are no appreciable differences
between the cases $d=2$ and $d=10$ and
only about the ten percent of increment between the cases
$d=10$ and $d=100$.
Moreover,
the execution time for the case $d=1,\!000$
is about a factor of two larger than that for the case $d=100$.

To understand this behavior, consider the asymptotic execution time
of the parallel $\FastCFOF$.
Since both distance computation and sorting are distributed among
the CPU cores, but only distance computations are 
vectorized, 
for relatively small dimensionality values $d$,
the cost of sorting distances (specifically the term $\log s$) dominates the
cost of computing distances by exploiting vectorization (that is the term $\frac{d}{V}$).
This confirms that the algorithm takes full advantage of  
the SIMD parallelism,
for otherwise the cost would soon be dependent on the 
dimensionality $d$.\footnote{The 
temporal dependence on the dimensionality
could be emphasized for small $d$ values
by including a vectorized sorting algorithm
in the implementation of the $\FastCFOF$ algorithm,
see e.g. \cite{InoueT15}, though for large $d$ values this 
modification should not modify the temporal trend.
}
Moreover, we note that 
the larger the dimensionality, the greater the exploitation of 
SIMD primitives associated with highly regular code
which is efficiently pipelined
(distances computation code 
presents higher instruction level parallelism efficiency due to
reduced stalls, predictable branches, and loop unrolled and
cache blocked code).

The dashed curves represent the execution times 
obtained by disabling the MIMD
parallelism (only one thread hardware, i.e. $C=1$).
The ratio between the execution time of the algorithms
for $C=8$ and $C=1$
is about $7.6$, thus confirming the 
effectiveness of the MIMD parallelization schema.
Summarizing, this experiment confirms that the parallel
$\FastCFOF$ takes full advantage of both MIMD and SIMD
parallelism.

As for the scalability w.r.t. the dataset size $n$,
the execution time for $n=10^5$
is roughly one order of magnitude larger than that for $n=10^4$
and the
the execution time for $n=10^6$ is 
roughly two orders of magnitude larger than that for $n=10^4$.
Moreover,
the execution time for $n=10^6$ 
is about a factor of $10$ larger than that for $n=10^5$,
which appears to be consistent with the asymptotic cost analysis.

\begin{figure}[t]
\centering
\subfloat[\label{fig:exp_scal_time}]
{\includegraphics[width=0.48\textwidth]{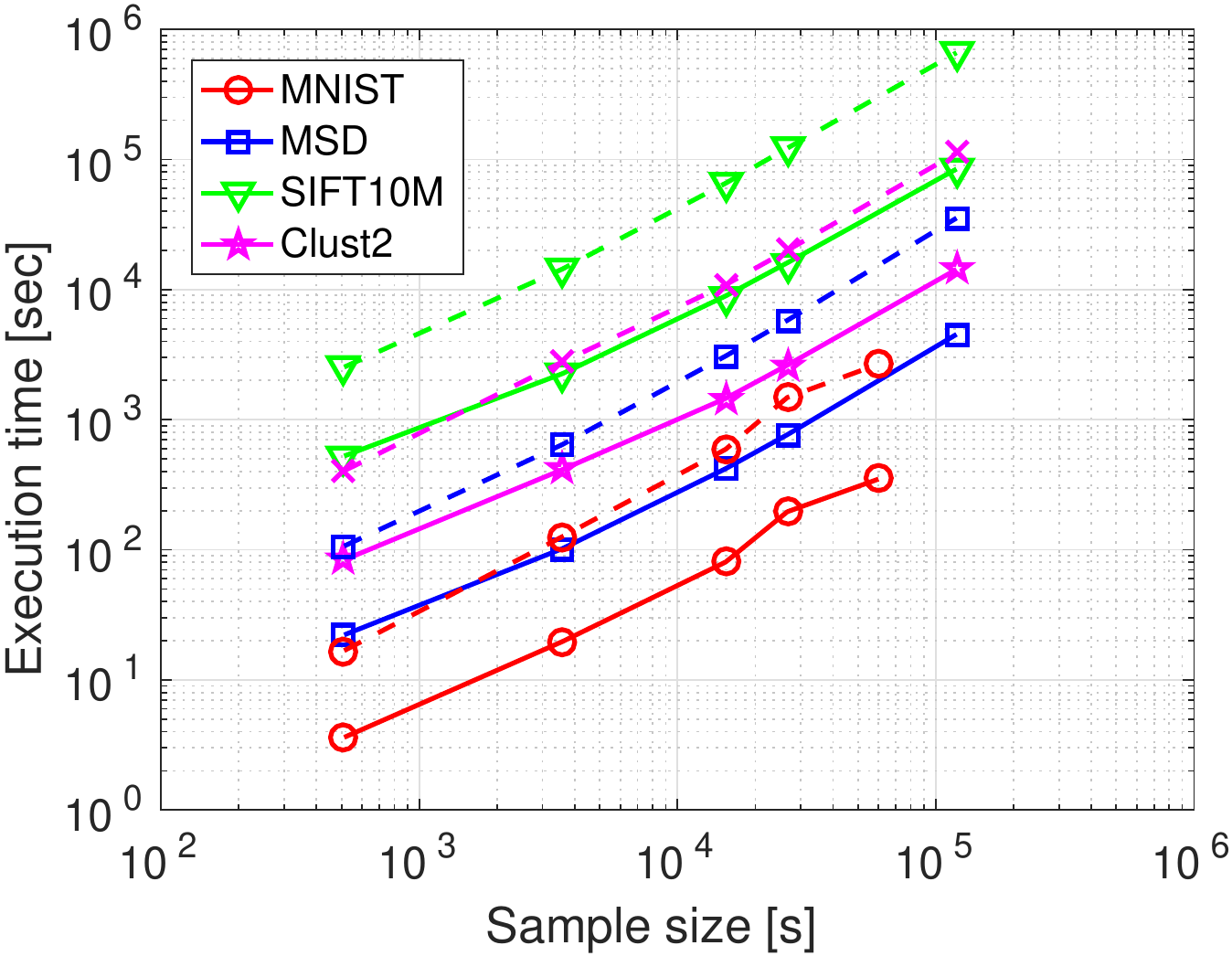}}
~
\subfloat[\label{fig:exp_scal_norm_time}]
{\includegraphics[width=0.48\textwidth]{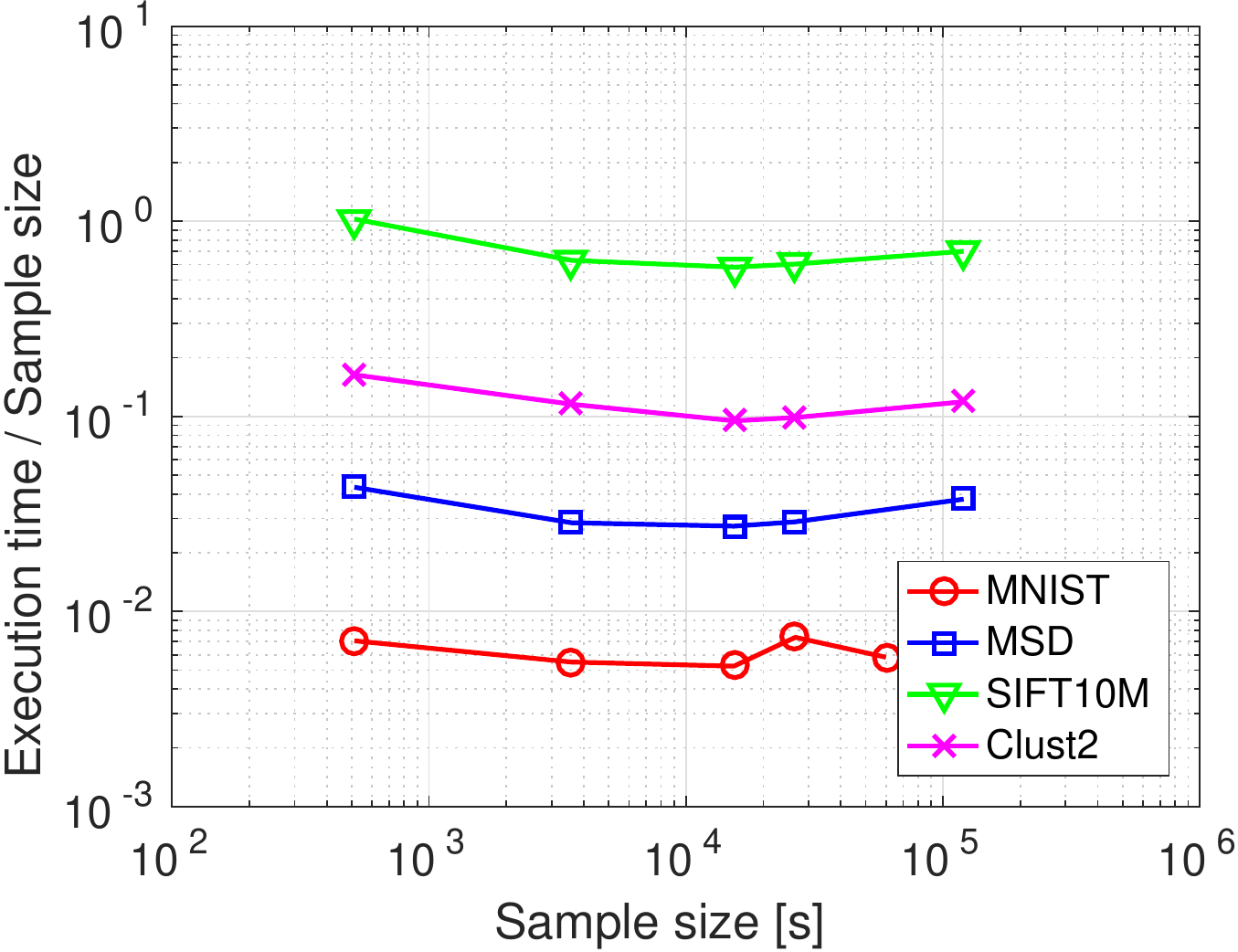}}
\caption{[Best viewed in color.] Scalability analysis of $\FastCFOF$
with respect to the sample size $s$.}
\label{fig:exp_scal_s}
\end{figure}

\medskip
Figure \ref{fig:exp_scal_time}
shows the execution time of $\FastCFOF$
on the datasets \textit{MNIST}, 
\textit{MSD}, \textit{MNIST10M}, and \textit{Clust2}
($n=10^6$, $d=10^3$)
for different sample sizes $s$, that are 
$s=512$ ($\epsilon=0.1$, $\delta=0.1$),
$s=3,\!584$ ($\epsilon=0.025$, $\delta=0.025$),
$s=15,\!360$ ($\epsilon=0.01$, $\delta=0.1$),
$s=26,\!624$ ($\epsilon=0.01$, $\delta=0.01$), and
$s=120,\!320$ ($\epsilon=0.005$, $\delta=0.005$).
Solid lines concern the parallelized version,
while dashed lines are the execution times obtained by disabling
MIMD parallelism.
Table \ref{table:speedup} reports the MIMD speedup, that is ratio between the execution time 
of the latter version over the former version.

\begin{table}[t]
\begin{center}
\begin{tabular}{|l||r|r|r|r|r|}
\hline
Dataset / $s$  &  \multicolumn{1}{|c|}{$512$}    &     $3,\!584$   &     $15,\!360$    &    $26,\!624$  & $120,\!320$ \\
\hline\hline
\textit{MNIST}   &        $4.6407$  &     $6.4068$   &    $7.3548$   &    $7.5545$   &    $~^\ast7.6404$ \\
\textit{MSD}     &        $4.8205$  &     $6.3137$   &    $7.3426$   &    $7.5671$   &    $7.8296$ \\
\textit{SIFT10M} &        $4.8265$  &     $6.3677$   &    $7.3545$   &    $7.5998$   &    $7.8325$ \\
\textit{Clust2}  &        $4.8381$  &     $6.7250$   &    $7.4681$   &    $7.6165$   &    $7.8330$ \\
\hline
\textit{Average} & $4.7815$    &   $6.4533$     &    $7.3800$   &    $7.5845$    &   $7.8317$ \\
\hline
\end{tabular}
\end{center}
\caption{MIMD speedup of parallel $\FastCFOF$
for different sample sizes $s$
($^\ast$the last speedup of \textit{MNIST} is relative to whole dataset,
i.e. to the sample size $s=60,\!000$, and is not considered
in the average speedup).
}
\label{table:speedup}
\end{table}

Except for very small sample sizes, that is up a few thousands of points,
the ratio rapidly approaches the number $V=8$ of hardware threads employed.
We can conclude that the overhead associated with the management 
of the threads is almost satisfactorily amortized 
starting from samples of ten thousand elements.

To understand the dependency of the execution time from
the sample size, Figure \ref{fig:exp_scal_norm_time} reports
the ratio between the algorithm run time and the sample size $s$,
as a function of $s$.
In the range of samples considered, the above ratio
is quite insensitive to the sample size
and its trend can be considered approximatively constant.
This observation agrees with the linear dependence on
the sample size in the asymptotic cost.
However,
by taking a closer look, the curves are slightly deceasing 
till to the intermediate sample size and, then,
slightly increasing. We explain this trend by noticing that,
during the descending course, the MIMD speedup increases rapidly and then,
during the ascending course, the cost of sorting becomes more appreciable.

\subsubsection{Accuracy of $\FastCFOF$}
\label{sect:exp_fastcfof_acc}

\begin{figure}[t]
\centering
\subfloat[\label{fig:exp_accB}]
{\includegraphics[width=0.48\textwidth]{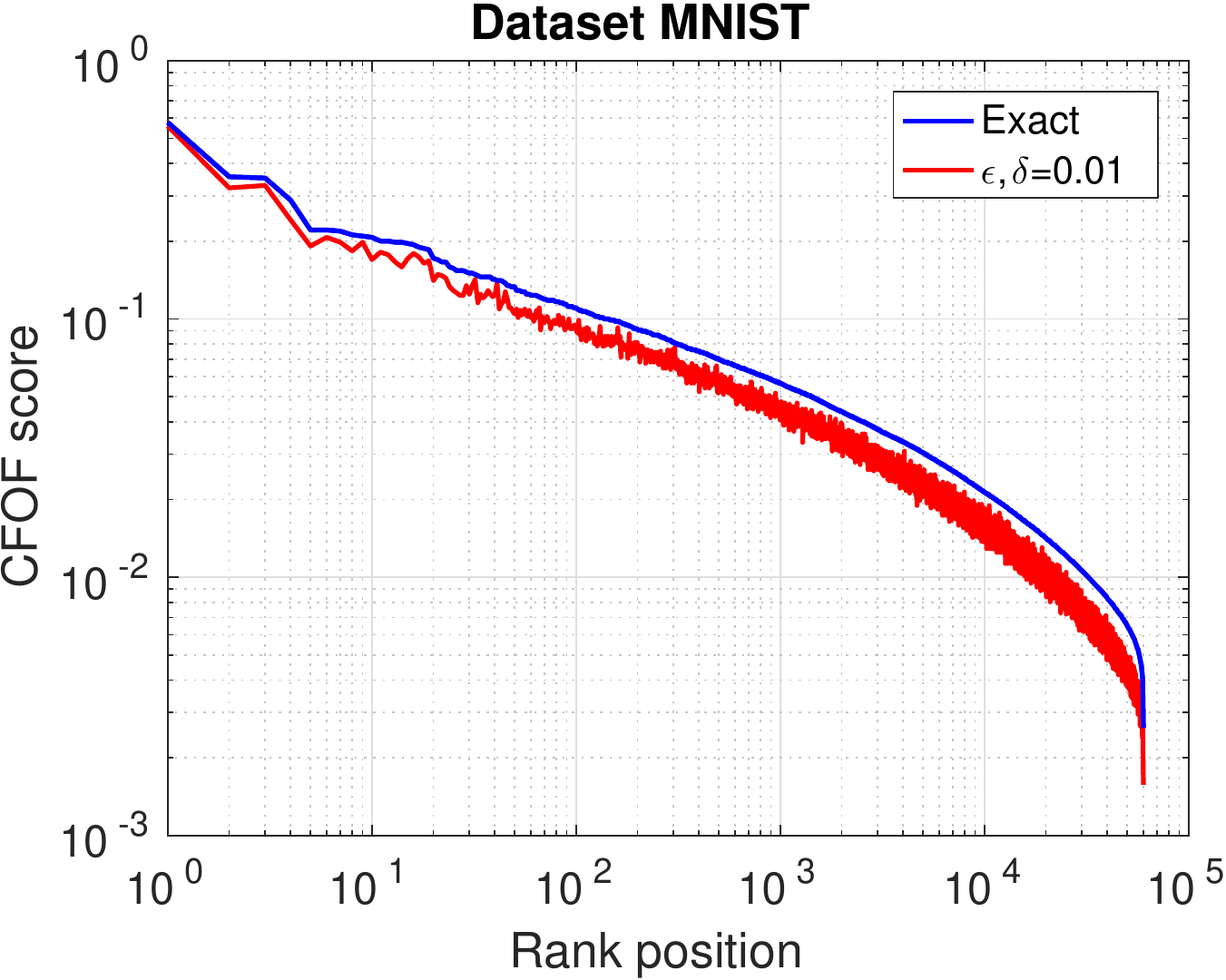}}
~
\subfloat[\label{fig:exp_accC}]
{\includegraphics[width=0.48\textwidth]{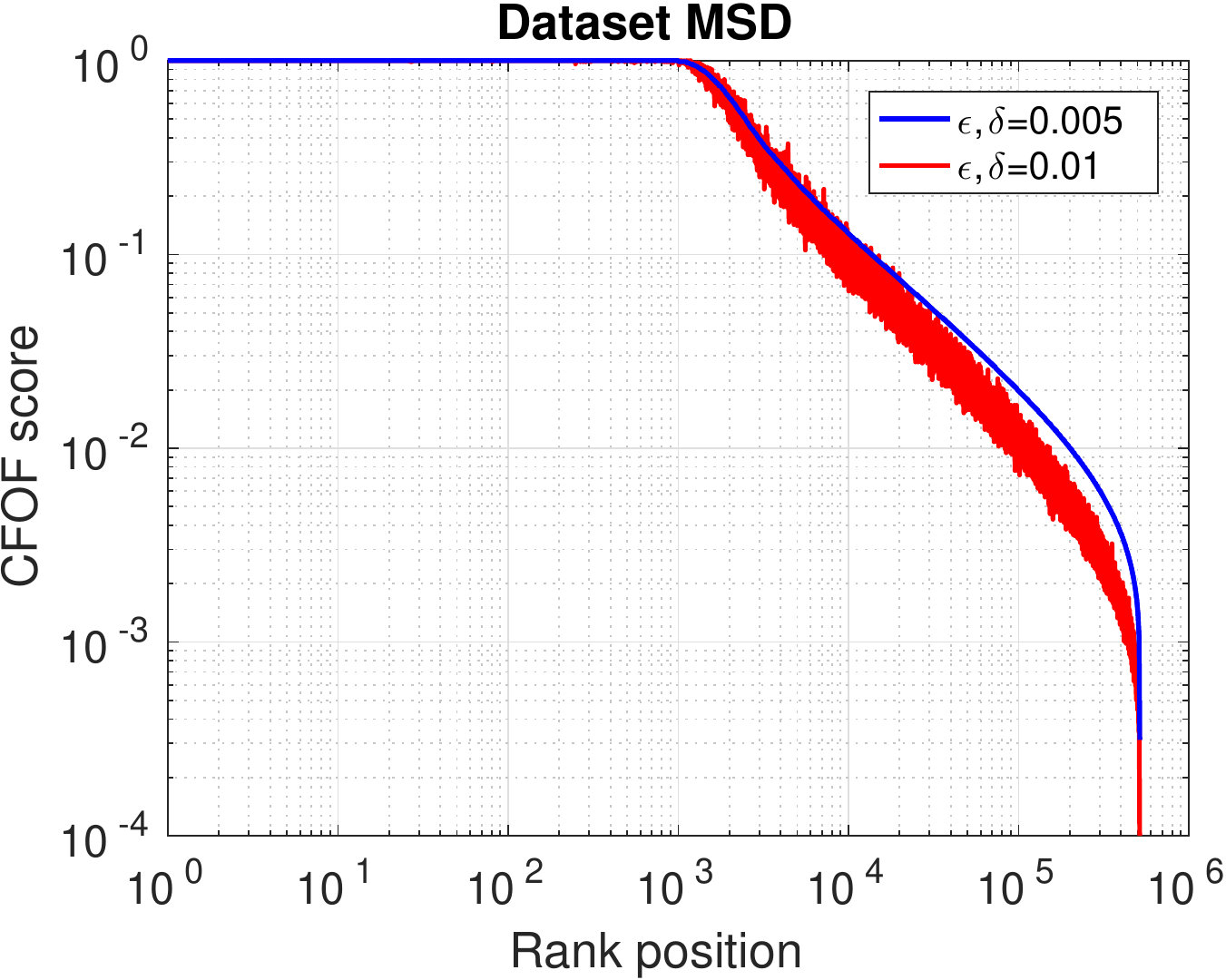}}
\\
\subfloat[\label{fig:exp_accC}]
{\includegraphics[width=0.48\textwidth]{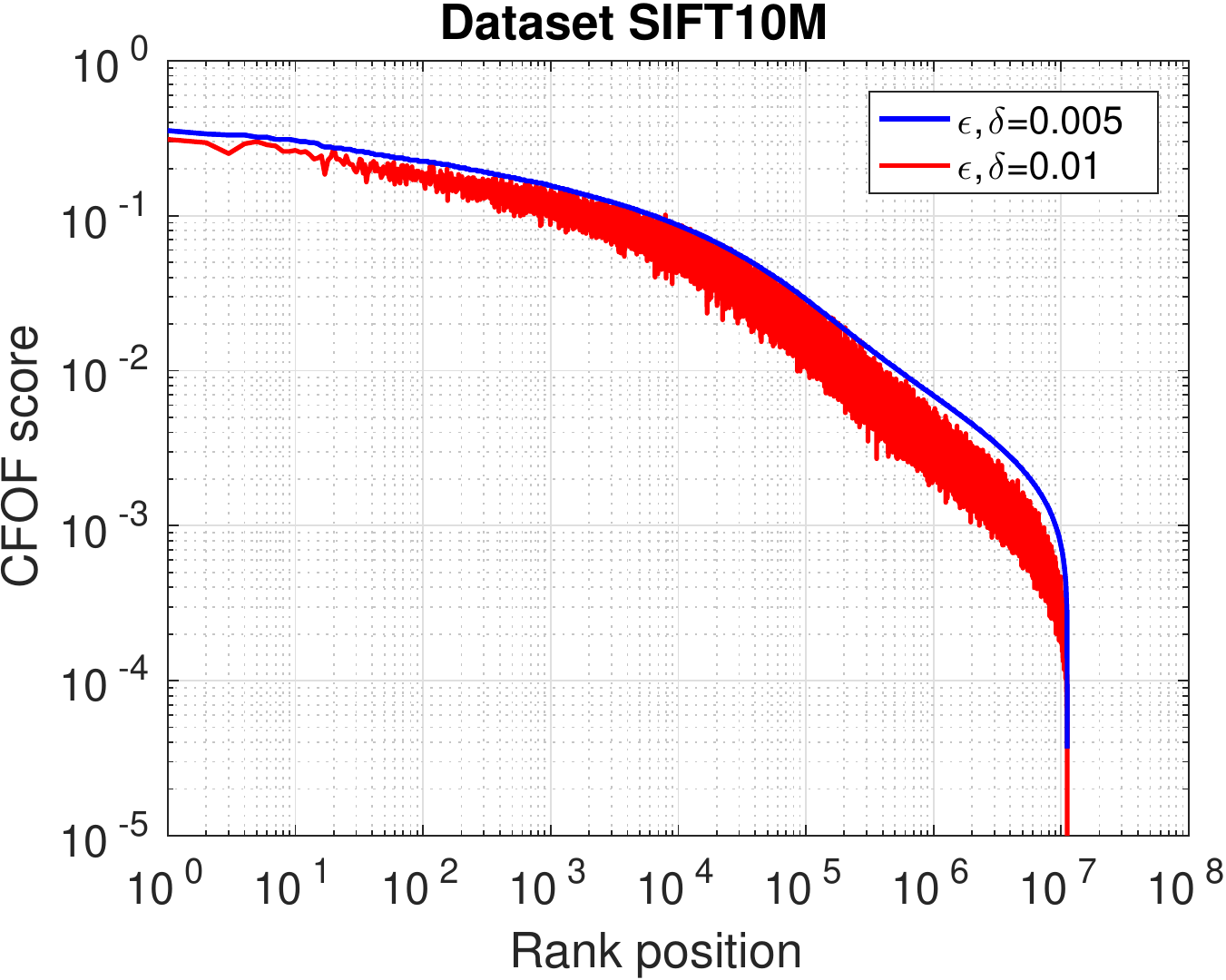}}
~
\subfloat[\label{fig:exp_accA}]
{\includegraphics[width=0.48\textwidth]{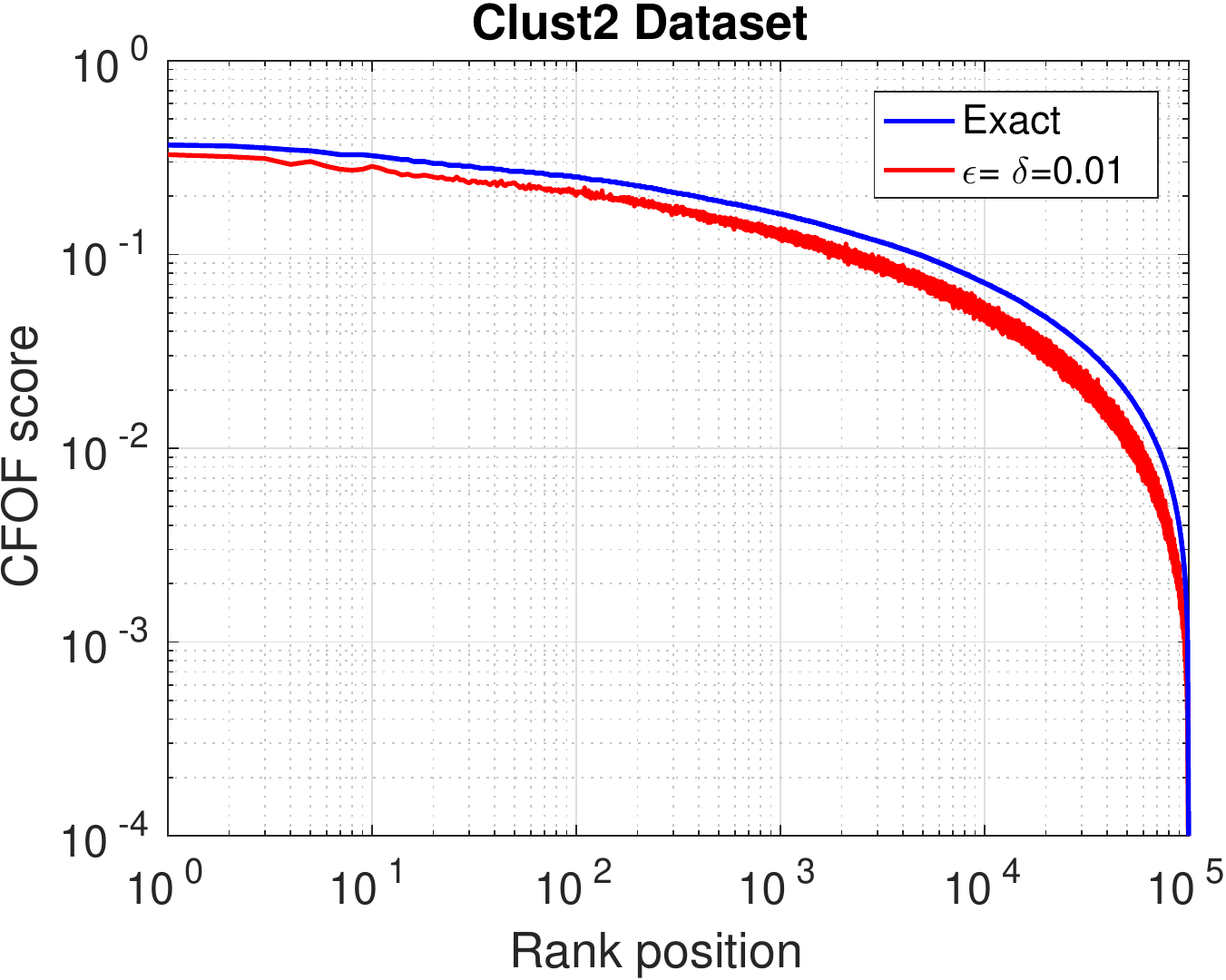}}
\caption{[Best viewed in color.]
Comparison between the exact $\CFOF$ scores (blue smooth curve)
and the corresponding approximate $\CFOF$ score (red noisy curve) 
computed by $\FastCFOF$ for $s=26,\!624$ and $\varrho=0.01$.}
\label{fig:exp_acc}
\end{figure}

The goal of this experiment is to assess the
quality of the result of $\FastCFOF$
for different sample sizes,
that is different combinations of the parameters
$\epsilon$ and $\delta$.
With this aim we first computed the scores
for a very large sample size $s_{max}=120,\!320$ associated
with parameters $\epsilon=0.005$ and $\delta=0.005$.
For \textit{MNIST} and \textit{Clust2} ($n=100,\!000$, $d=100$)
we determined the exact scores, since in these cases the sample size exceeds the 
dataset size. 
For the other two datasets, determining the exact scores
is prohibitive, thus we assumed as exact scores those associated with
the sample of size $s_{max}$.

Figure \ref{fig:exp_acc} compares the scores for $s=s_{max}$
with those obtained for the standard sample size $s=26,\!624$.
The blue (smooth) curve represents the exact scores
sorted in descending order and the $x$-axis represents
the outlier rank position of the dataset points.
As for the red (noisy) curve, it shows the approximate scores
associated with the points at each rank position.
The curves highlight that the ranking position
tends to be preserved and that
in both cases
top outliers
are associated with the largest scores.

To measure the accuracy of $\FastCFOF$, we
used the \textit{Precision}, or \textit{Prec} for short,
also referred to as $P@n$ in the literature \cite{Craswell2016}.
The \textit{Prec} measure is the precision associated with
the top outliers,
defined as the fraction of true outliers among
the top-$\alpha$ outliers reported by the technique
($Prec@\alpha$ for short in the following).
The latter measure is important in context
of outlier detection evaluation,
since outlier detection techniques 
aim at detecting the top-$\alpha$
most deviating objects of the input dataset.
We also employed the
\textit{Spearman's rank correlation coefficient}
to assesses relationship between the two rankings \cite{CorderF14}.
This coefficient is high (close to $1$)
when observations have a similar rank.

Tables \ref{table:mnist}, \ref{table:msd}, \ref{table:sift10m},
and \ref{table:clust2}
report 
the precision for $\alpha=0.001$ ($Prec@0.001$)
and for $\alpha=0.01$ ($Prec@0.01$)
in correspondence of different combinations 
of the sample size $s\in\{512, 3,\!584, 15,\!360, 26,\!624 \}$ 
and of the parameter $\varrho\in\{\varrho_1,\ldots,\varrho_5\}$
($\varrho_1=0.001$, $\varrho_2=0.005$, $\varrho_3=0.01$,
$\varrho_4=0.05$, $\varrho_5=0.1$).
Moreover,
Table \ref{table:spear} reports the Spearman ranking coefficient
for the same combinations of the parameters.

Both the precision and the correlation improve 
for increasing sample sizes $s$
and $\varrho$ values.
Interestingly,
excellent values can be achieved even with small samples
especially for larger $\varrho$ values.
Even the worst results, for $s=512$, are surprisingly good, 
considering also the modest cost at which they were obtained.

This behavior is remarkable. 
On one side, the ability to achieve very good accuracy with reduced sample
sizes $s$, enables $\CFOF$ to efficiently process huge datasets.
On the other side, 
the ability to improve accuracy by increasing $\varrho$
with a given sample size $s$, that is to say at the same temporal cost,
enables $\CFOF$ to efficiently deal even large values of the parameter
$\varrho$. 

We point out that 
the latter kind of behavior has been posed
as a challenge in \cite{RadovanovicNI15} for the method $\antiHub$.
Indeed, since $\antiHub$ reaches peak performance
for values of $k\approx n$, they concluded that the development 
of an approximate version of this method 
represents a significant challenge since, to the best of their knowledge, 
current approximate
$k$-nearest neighbors algorithms assume small constant $k$.

We can justify the very good accuracy of the method by noticing 
that the larger the $\CFOF$ score of point $x$ and, 
for any other point $y$, 
the larger the probability $p(x,y)$ that a generic dataset point
will be located between $x$ and $y$ and, moreover,
the smaller the impact of the error $\epsilon$
on the estimated value $\widehat{p}(x,y)$.
Intuitively, the points we are interested in, that is the outliers,
are precisely those less prone to bad estimations.

\begin{table}
\scriptsize
\begin{center}
\begin{tabular}{cc}
\begin{tabular}{|c||c|c|c|c|c|}
\multicolumn{6}{c}{$Prec@0.001$} \\
\hline
$s$ & $\varrho_1$ & $\varrho_2$ & $\varrho_3$  & $\varrho_4$  & $\varrho_5$ \\
 & $0.001$ & $0.005$ &  $0.01$ &  $0.05$ & $0.1$ \\
\hline\hline
$   512$ & --- & $0.2333$ & $0.4500$ & $0.7167$ & $0.7167$ \\
$  3584$ & $0.6333$ & $0.6667$ & $0.7833$ & $0.9333$ & $0.9000$ \\
$ 15360$ & $0.8333$ & $0.9167$ & $0.9333$ & $0.9833$ & $0.9833$ \\
$ 26624$ & $0.9000$ & $0.9500$ & $0.9667$ & $0.9667$ & $0.9833$ \\
\hline
\end{tabular}
&
\begin{tabular}{|c||c|c|c|c|c|}
\multicolumn{6}{c}{$Prec@0.005$} \\
\hline
$s$ & $\varrho_1$ & $\varrho_2$ & $\varrho_3$  & $\varrho_4$  & $\varrho_5$ \\
 & $0.001$ & $0.005$ &  $0.01$ &  $0.05$ & $0.1$ \\
\hline\hline
$   512$ & --- & $0.4400$ & $0.5700$ & $0.7133$ & $0.7567$ \\
$  3584$ & $0.5733$ & $0.7533$ & $0.8033$ & $0.9367$ & $0.9400$ \\
$ 15360$ & $0.7733$ & $0.9233$ & $0.9367$ & $0.9733$ & $0.9867$ \\
$ 26624$ & $0.8900$ & $0.9467$ & $0.9633$ & $0.9900$ & $0.9867$ \\
\hline
\end{tabular}
\\
\begin{tabular}{|c||c|c|c|c|c|}
\multicolumn{6}{c}{$Prec@0.01$} \\
\hline
$s$ & $\varrho_1$ & $\varrho_2$ & $\varrho_3$  & $\varrho_4$  & $\varrho_5$ \\
 & $0.001$ & $0.005$ &  $0.01$ &  $0.05$ & $0.1$ \\
\hline\hline
$   512$ & --- & $0.4433$ & $0.5667$ & $0.7467$ & $0.7583$ \\
$  3584$ & $0.5517$ & $0.7500$ & $0.8200$ & $0.9250$ & $0.9167$ \\
$ 15360$ & $0.8117$ & $0.9050$ & $0.9300$ & $0.9633$ & $0.9650$ \\
$ 26624$ & $0.8633$ & $0.9350$ & $0.9600$ & $0.9633$ & $0.9717$ \\
\hline
\end{tabular}
&
\begin{tabular}{|c||c|c|c|c|c|}
\multicolumn{6}{c}{$Prec@0.05$} \\
\hline
$s$ & $\varrho_1$ & $\varrho_2$ & $\varrho_3$  & $\varrho_4$  & $\varrho_5$ \\
 & $0.001$ & $0.005$ &  $0.01$ &  $0.05$ & $0.1$ \\
\hline\hline
$   512$ & --- & $0.5353$ & $0.6427$ & $0.7827$ & $0.8430$ \\
$  3584$ & $0.6490$ & $0.8057$ & $0.8470$ & $0.9277$ & $0.9490$ \\
$ 15360$ & $0.8443$ & $0.9260$ & $0.9497$ & $0.9683$ & $0.9787$ \\
$ 26624$ & $0.8890$ & $0.9553$ & $0.9643$ & $0.9777$ & $0.9817$ \\
\hline
\end{tabular}
\end{tabular}
\end{center}
\caption{\textit{MNIST} dataset: $Prec@n$.}
\label{table:mnist}
\end{table}

\begin{table}
\scriptsize
\begin{center}
\begin{tabular}{cc}
\begin{tabular}{|c||c|c|c|c|c|}
\multicolumn{6}{c}{$Prec@0.001$} \\
\hline
$s$ & $\varrho_1$ & $\varrho_2$ & $\varrho_3$  & $\varrho_4$  & $\varrho_5$ \\
 & $0.001$ & $0.005$ &  $0.01$ &  $0.05$ & $0.1$ \\
\hline\hline
$   512$ & --- & $0.7771$ & $0.8411$ & $0.8643$ & $0.8411$ \\
$  3584$ & $0.7829$ & $0.9574$ & $0.9690$ & $1.0000$ & $1.0000$ \\
$ 15360$ & $0.9128$ & $0.9729$ & $0.9787$ & $1.0000$ & $1.0000$ \\
$ 26624$ & $0.9341$ & $0.9787$ & $0.9826$ & $1.0000$ & $1.0000$ \\
\hline
\end{tabular}
&
\begin{tabular}{|c||c|c|c|c|c|}
\multicolumn{6}{c}{$Prec@0.005$} \\
\hline
$s$ & $\varrho_1$ & $\varrho_2$ & $\varrho_3$  & $\varrho_4$  & $\varrho_5$ \\
 & $0.001$ & $0.005$ &  $0.01$ &  $0.05$ & $0.1$ \\
\hline\hline
$   512$ & $0.0062$ & $0.6923$ & $0.8196$ & $0.8514$ & $0.8157$ \\
$  3584$ & $0.7439$ & $0.8956$ & $0.9274$ & $0.9988$ & $0.9992$ \\
$ 15360$ & $0.8867$ & $0.9519$ & $0.9697$ & $0.9988$ & $0.9992$ \\
$ 26624$ & $0.9208$ & $0.9674$ & $0.9794$ & $0.9984$ & $0.9988$ \\
\hline
\end{tabular}
\\
\begin{tabular}{|c||c|c|c|c|c|}
\multicolumn{6}{c}{$Prec@0.01$} \\
\hline
$s$ & $\varrho_1$ & $\varrho_2$ & $\varrho_3$  & $\varrho_4$  & $\varrho_5$ \\
 & $0.001$ & $0.005$ &  $0.01$ &  $0.05$ & $0.1$ \\
\hline\hline
$   512$ & --- & $0.6962$ & $0.7965$ & $0.9523$ & $0.9816$ \\
$  3584$ & $0.7396$ & $0.8948$ & $0.9265$ & $0.9957$ & $1.0000$ \\
$ 15360$ & $0.8840$ & $0.9542$ & $0.9643$ & $0.9986$ & $1.0000$ \\
$ 26624$ & $0.9191$ & $0.9664$ & $0.9785$ & $0.9983$ & $1.0000$ \\
\hline
\end{tabular}
&
\begin{tabular}{|c||c|c|c|c|c|}
\multicolumn{6}{c}{$Prec@0.05$} \\
\hline
$s$ & $\varrho_1$ & $\varrho_2$ & $\varrho_3$  & $\varrho_4$  & $\varrho_5$ \\
 & $0.001$ & $0.005$ &  $0.01$ &  $0.05$ & $0.1$ \\
\hline\hline
$   512$ & $0.0512$ & $0.6891$ & $0.7988$ & $0.9089$ & $0.9401$ \\
$  3584$ & $0.7418$ & $0.8920$ & $0.9272$ & $0.9680$ & $0.9809$ \\
$ 15360$ & $0.8852$ & $0.9527$ & $0.9700$ & $0.9837$ & $0.9907$ \\
$ 26624$ & $0.9173$ & $0.9675$ & $0.9806$ & $0.9876$ & $0.9934$ \\
\hline
\end{tabular}
\end{tabular}
\end{center}
\caption{\textit{MSD} dataset: $Prec@n$.}
\label{table:msd}
\end{table}

\begin{table}
\scriptsize
\begin{center}
\begin{tabular}{cc}
\begin{tabular}{|c||c|c|c|c|c|}
\multicolumn{6}{c}{$Prec@0.001$} \\
\hline
$s$ & $\varrho_1$ & $\varrho_2$ & $\varrho_3$  & $\varrho_4$  & $\varrho_5$ \\
 & $0.001$ & $0.005$ &  $0.01$ &  $0.05$ & $0.1$ \\
\hline\hline
$   512$ & --- & $0.2959$ & $0.4381$ & $0.7260$ & $0.8462$ \\
$  3584$ & $0.4495$ & $0.6248$ & $0.7247$ & $0.8970$ & $0.9653$ \\
$ 15360$ & $0.7126$ & $0.8191$ & $0.8738$ & $0.9577$ & $0.9927$ \\
$ 26624$ & $0.7897$ & $0.8699$ & $0.9135$ & $0.9684$ & $0.9964$ \\
\hline
\end{tabular}
&
\begin{tabular}{|c||c|c|c|c|c|}
\multicolumn{6}{c}{$Prec@0.005$} \\
\hline
$s$ & $\varrho_1$ & $\varrho_2$ & $\varrho_3$  & $\varrho_4$  & $\varrho_5$ \\
 & $0.001$ & $0.005$ &  $0.01$ &  $0.05$ & $0.1$ \\
\hline\hline
$   512$ & $0.0053$ & $0.3861$ & $0.5693$ & $0.8418$ & $0.8941$ \\
$  3584$ & $0.4920$ & $0.7085$ & $0.8099$ & $0.9483$ & $0.9705$ \\
$ 15360$ & $0.7448$ & $0.8622$ & $0.9166$ & $0.9840$ & $0.9912$ \\
$ 26624$ & $0.8124$ & $0.9016$ & $0.9417$ & $0.9903$ & $0.9942$ \\
\hline
\end{tabular}
\\
\begin{tabular}{|c||c|c|c|c|c|}
\multicolumn{6}{c}{$Prec@0.01$} \\
\hline
$s$ & $\varrho_1$ & $\varrho_2$ & $\varrho_3$  & $\varrho_4$  & $\varrho_5$ \\
 & $0.001$ & $0.005$ &  $0.01$ &  $0.05$ & $0.1$ \\
\hline\hline
$   512$ & --- & $0.4221$ & $0.6116$ & $0.8603$ & $0.9086$ \\
$  3584$ & $0.5096$ & $0.7431$ & $0.8359$ & $0.9490$ & $0.9703$ \\
$ 15360$ & $0.7558$ & $0.8859$ & $0.9257$ & $0.9781$ & $0.9888$ \\
$ 26624$ & $0.8209$ & $0.9206$ & $0.9468$ & $0.9850$ & $0.9930$ \\
\hline
\end{tabular}
&
\begin{tabular}{|c||c|c|c|c|c|}
\multicolumn{6}{c}{$Prec@0.05$} \\
\hline
$s$ & $\varrho_1$ & $\varrho_2$ & $\varrho_3$  & $\varrho_4$  & $\varrho_5$ \\
 & $0.001$ & $0.005$ &  $0.01$ &  $0.05$ & $0.1$ \\
\hline\hline
$   512$ & $0.0510$ & $0.4887$ & $0.6527$ & $0.8754$ & $0.9173$ \\
$  3584$ & $0.5784$ & $0.7926$ & $0.8659$ & $0.9585$ & $0.9730$ \\
$ 15360$ & $0.7989$ & $0.9099$ & $0.9437$ & $0.9840$ & $0.9899$ \\
$ 26624$ & $0.8542$ & $0.9377$ & $0.9611$ & $0.9896$ & $0.9935$ \\
\hline
\end{tabular}
\end{tabular}
\end{center}
\caption{\textit{SIFT10M} dataset: $Prec@n$.}
\label{table:sift10m}
\end{table}

\begin{table}
\scriptsize
\begin{center}
\begin{tabular}{cc}
\begin{tabular}{|c||c|c|c|c|c|}
\multicolumn{6}{c}{$Prec@0.001$} \\
\hline
$s$ & $\varrho_1$ & $\varrho_2$ & $\varrho_3$  & $\varrho_4$  & $\varrho_5$ \\
 & $0.001$ & $0.005$ &  $0.01$ &  $0.05$ & $0.1$ \\
\hline\hline
$   512$ & --- & $0.5000$ & $0.5500$ & $0.7600$ & $0.6800$ \\
$  3584$ & $0.6200$ & $0.7900$ & $0.8600$ & $0.9300$ & $0.9200$ \\
$ 15360$ & $0.8100$ & $0.8900$ & $0.9100$ & $0.9600$ & $0.9700$ \\
$ 26624$ & $0.8700$ & $0.9200$ & $0.9500$ & $0.9800$ & $1.0000$ \\
\hline
\end{tabular}
&
\begin{tabular}{|c||c|c|c|c|c|}
\multicolumn{6}{c}{$Prec@0.005$} \\
\hline
$s$ & $\varrho_1$ & $\varrho_2$ & $\varrho_3$  & $\varrho_4$  & $\varrho_5$ \\
 & $0.001$ & $0.005$ &  $0.01$ &  $0.05$ & $0.1$ \\
\hline\hline
$   512$ & $0.0060$ & $0.5840$ & $0.6820$ & $0.7960$ & $0.7760$ \\
$  3584$ & $0.6900$ & $0.8420$ & $0.8900$ & $0.9240$ & $0.9200$ \\
$ 15360$ & $0.8580$ & $0.9500$ & $0.9620$ & $0.9740$ & $0.9580$ \\
$ 26624$ & $0.8940$ & $0.9560$ & $0.9700$ & $0.9760$ & $0.9700$ \\
\hline
\end{tabular}
\\
\begin{tabular}{|c||c|c|c|c|c|}
\multicolumn{6}{c}{$Prec@0.01$} \\
\hline
$s$ & $\varrho_1$ & $\varrho_2$ & $\varrho_3$  & $\varrho_4$  & $\varrho_5$ \\
 & $0.001$ & $0.005$ &  $0.01$ &  $0.05$ & $0.1$ \\
\hline\hline
$   512$ & --- & $0.6010$ & $0.7120$ & $0.8330$ & $0.8090$ \\
$  3584$ & $0.6960$ & $0.8520$ & $0.9000$ & $0.9500$ & $0.9340$ \\
$ 15360$ & $0.8720$ & $0.9350$ & $0.9600$ & $0.9810$ & $0.9540$ \\
$ 26624$ & $0.8980$ & $0.9620$ & $0.9700$ & $0.9920$ & $0.9790$ \\
\hline
\end{tabular}
&
\begin{tabular}{|c||c|c|c|c|c|}
\multicolumn{6}{c}{$Prec@0.05$} \\
\hline
$s$ & $\varrho_1$ & $\varrho_2$ & $\varrho_3$  & $\varrho_4$  & $\varrho_5$ \\
 & $0.001$ & $0.005$ &  $0.01$ &  $0.05$ & $0.1$ \\
\hline\hline
$   512$ & $0.0526$ & $0.6788$ & $0.7612$ & $0.8680$ & $0.8794$ \\
$  3584$ & $0.7676$ & $0.8864$ & $0.9110$ & $0.9622$ & $0.9568$ \\
$ 15360$ & $0.9050$ & $0.9544$ & $0.9642$ & $0.9868$ & $0.9820$ \\
$ 26624$ & $0.9318$ & $0.9690$ & $0.9762$ & $0.9960$ & $0.9858$ \\
\hline
\end{tabular}
\end{tabular}
\end{center}
\caption{\textit{Clust2} dataset: $Prec@n$.}
\label{table:clust2}
\end{table}

\begin{table}
\scriptsize
\begin{center}
\begin{tabular}{cc}
\begin{tabular}{|c||c|c|c|c|c|}
\multicolumn{6}{c}{MNIST} \\
\hline
$s$ & $\varrho_1$ & $\varrho_2$ & $\varrho_3$  & $\varrho_4$  & $\varrho_5$ \\
 & $0.001$ & $0.005$ &  $0.01$ &  $0.05$ & $0.1$ \\
\hline\hline
$   512$ & --- & $0.6265$ & $0.8022$ & $0.9526$ & $0.9793$ \\
$  3584$ & $0.7429$ & $0.9375$ & $0.9669$ & $0.9934$ & $0.9971$ \\
$ 15360$ & $0.9447$ & $0.9880$ & $0.9938$ & $0.9987$ & $0.9994$ \\
$ 26624$ & $0.9737$ & $0.9946$ & $0.9972$ & $0.9995$ & $0.9997$ \\
\hline
\end{tabular}
&
\begin{tabular}{|c||c|c|c|c|c|}
\multicolumn{6}{c}{MSD} \\
\hline
$s$ & $\varrho_1$ & $\varrho_2$ & $\varrho_3$  & $\varrho_4$  & $\varrho_5$ \\
 & $0.001$ & $0.005$ &  $0.01$ &  $0.05$ & $0.1$ \\
\hline\hline
$   512$ & --- & $0.7757$ & $0.8976$ & $0.9741$ & $0.9852$ \\
$  3584$ & $0.8542$ & $0.9712$ & $0.9853$ & $0.9965$ & $0.9979$ \\
$ 15360$ & $0.9714$ & $0.9941$ & $0.9969$ & $0.9992$ & $0.9995$ \\
$ 26624$ & $0.9847$ & $0.9968$ & $0.9984$ & $0.9996$ & $0.9998$ \\
\hline
\end{tabular}
\\
\begin{tabular}{|c||c|c|c|c|c|}
\multicolumn{6}{c}{SIF10M} \\
\hline
$s$ & $\varrho_1$ & $\varrho_2$ & $\varrho_3$  & $\varrho_4$  & $\varrho_5$ \\
 & $0.001$ & $0.005$ &  $0.01$ &  $0.05$ & $0.1$ \\
\hline\hline
$   512$ & --- & $0.6315$ & $0.8022$ & $0.9482$ & $0.9725$ \\
$  3584$ & $0.7495$ & $0.9369$ & $0.9665$ & $0.9925$ & $0.9960$ \\
$ 15360$ & $0.9424$ & $0.9863$ & $0.9927$ & $0.9984$ & $0.9991$ \\
$ 26624$ & $0.9683$ & $0.9927$ & $0.9961$ & $0.9991$ & $0.9995$ \\
\hline
\end{tabular}
&
\begin{tabular}{|c||c|c|c|c|c|}
\multicolumn{6}{c}{Clust2} \\
\hline
$s$ & $\varrho_1$ & $\varrho_2$ & $\varrho_3$  & $\varrho_4$  & $\varrho_5$ \\
 & $0.001$ & $0.005$ &  $0.01$ &  $0.05$ & $0.1$ \\
\hline\hline
$   512$ & --- & $0.8725$ & $0.9425$ & $0.9822$ & $0.9881$ \\
$  3584$ & $0.9333$ & $0.9860$ & $0.9922$ & $0.9975$ & $0.9983$ \\
$ 15360$ & $0.9884$ & $0.9972$ & $0.9984$ & $0.9995$ & $0.9996$ \\
$ 26624$ & $0.9943$ & $0.9986$ & $0.9992$ & $0.9997$ & $0.9998$ \\
\hline
\end{tabular}
\end{tabular}
\end{center}
\caption{Spearman rank correlation.}
\label{table:spear}
\end{table}

\medskip
We point out
that the trade-off between the 
reduction in computation time obtained by $\FastCFOF$ 
and its approximation accuracy
can be understood by looking both at
Figures \ref{fig:exp_scal}-\ref{fig:exp_scal_s} 
and Tables \ref{table:mnist}-\ref{table:clust2}.
Before concluding,
in order
to make more immediately intelligible the above
trade-off, we report in Figure \ref{fig:exp_scal_tradeoff} 
the {relative execution time} 
versus the sample size $s$
(see Figure \ref{fig:exp_scal_sample_rtime})
and the $Prec@\alpha$ measure versus the relative execution time 
(see Figure \ref{fig:exp_scal_rtime_prec}).
The \textit{relative execution time} is the ratio
between the execution time of $\FastCFOF$ for a given
sample size $s$ and its execution time for the case $s=n$,
{corresponding to the temporal cost of $\CFOF$}.
Since, in the case of \textit{MSD} and \textit{SIFT10M},
using a sample of the same size of the whole dataset
resulted practically infeasible, we estimated the associated running time 
by exploiting interpolation.\footnote{
We employed quadratic interpolation in order to take into account the 
effect of the term $\log s$ associated with sorting distances, 
as already discussed in Sections \ref{sect:fastcfof_cost}
and \ref{sect:exp_fastcfof_scal}, obtaining an infinitesimal 
second order coefficient.}

Figure \ref{fig:exp_scal_sample_rtime} shows that the 
execution time of $\FastCFOF$
corresponds to a small fraction of that
required by the exact $\CFOF$ computation,
and that this fraction becomes smaller and smaller 
as the dataset size increases, the 
time savings corresponding to different orders of magnitude.
In Figure \ref{fig:exp_scal_rtime_prec}, 
the solid line represents the value of 
$Prec@0.01$ for $\varrho = 0.05$ and the dashed line the
same value for $\varrho = 0.01$, while the dotted line
represents the ranking correlation for $\varrho = 0.05$
(see also Tables \ref{table:mnist}-\ref{table:spear}).
Specifically, the figure 
provides a picture of how
the accuracy level of $\FastCFOF$ varies
with time savings.
E.g., consider the solid line, the accuracy is at least close to $0.95$ 
even for the second sample size, 
corresponding to a relative execution time of a few percent
on the smallest datasets and
to a relative execution time smaller than the one percent 
on the largest ones.

\begin{figure}[t]
\centering
\subfloat[\label{fig:exp_scal_sample_rtime}]
{\includegraphics[width=0.48\textwidth]{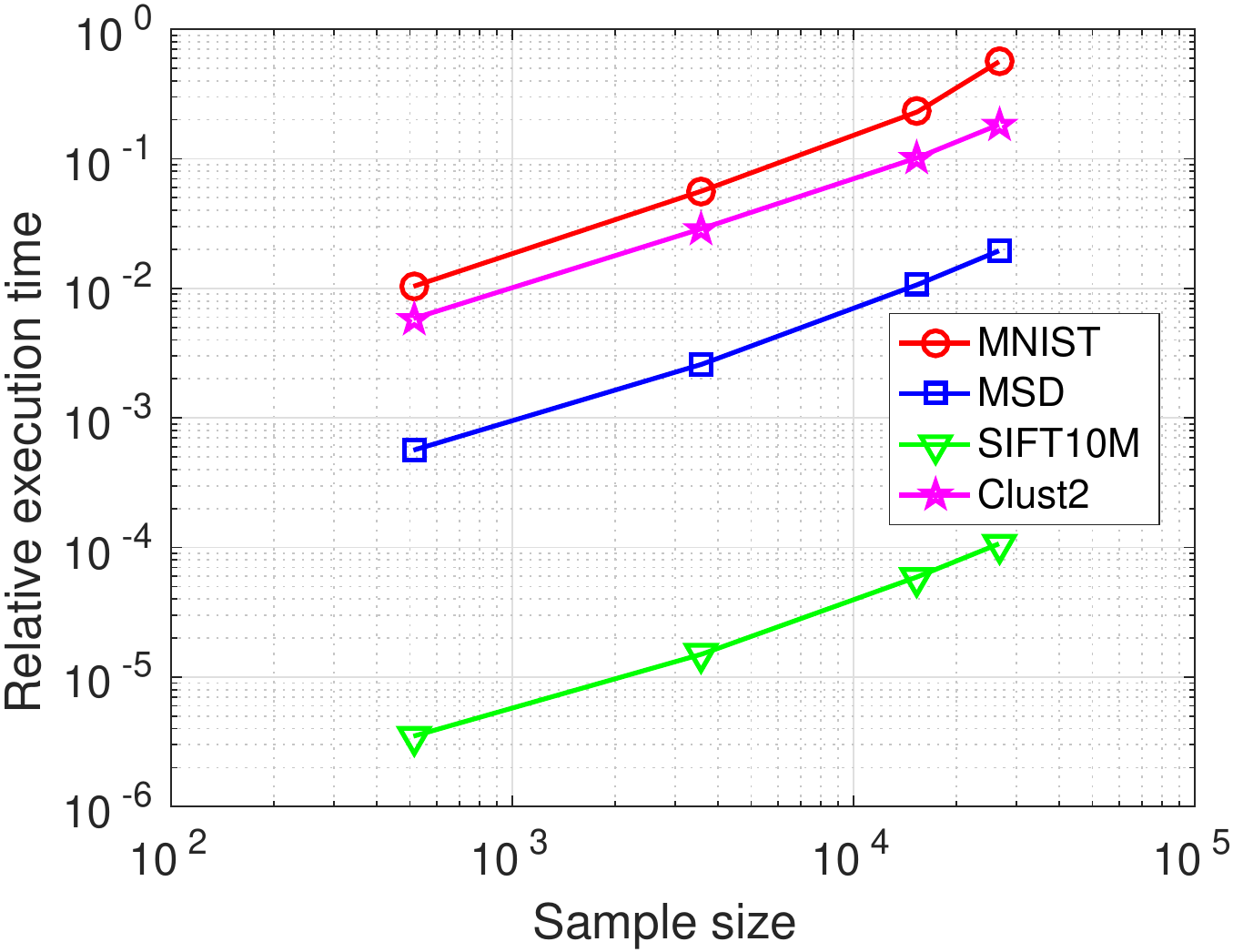}}
~
\subfloat[\label{fig:exp_scal_rtime_prec}]
{\includegraphics[width=0.48\textwidth]{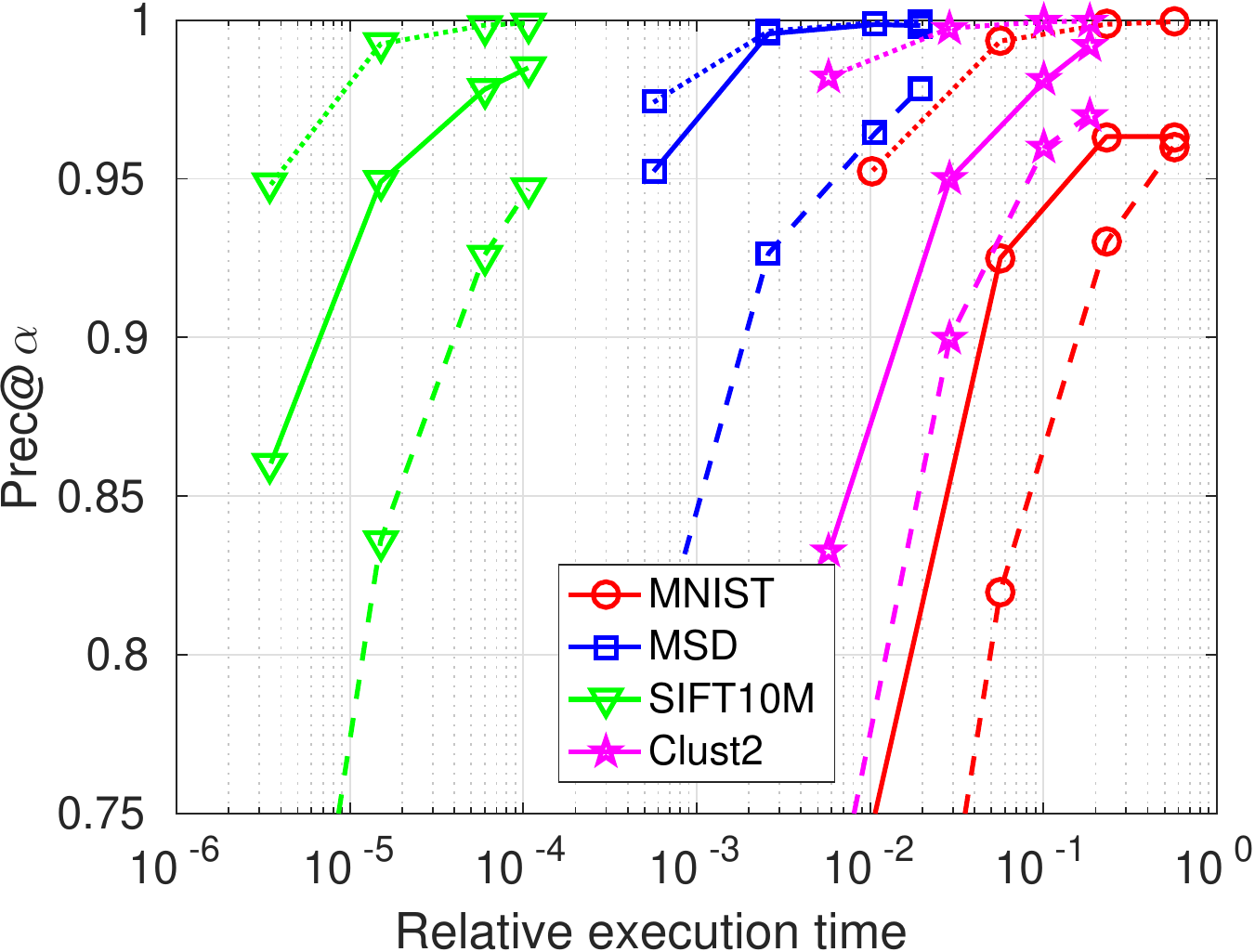}}
\caption{[Best viewed in color.] 
Reduction in computation time obtained by $\FastCFOF$ and trade-off versus
approximation accuracy.
}
\label{fig:exp_scal_tradeoff}
\end{figure}

\begin{figure}[t]
\centering
\subfloat[]
{\includegraphics[width=0.33\textwidth]{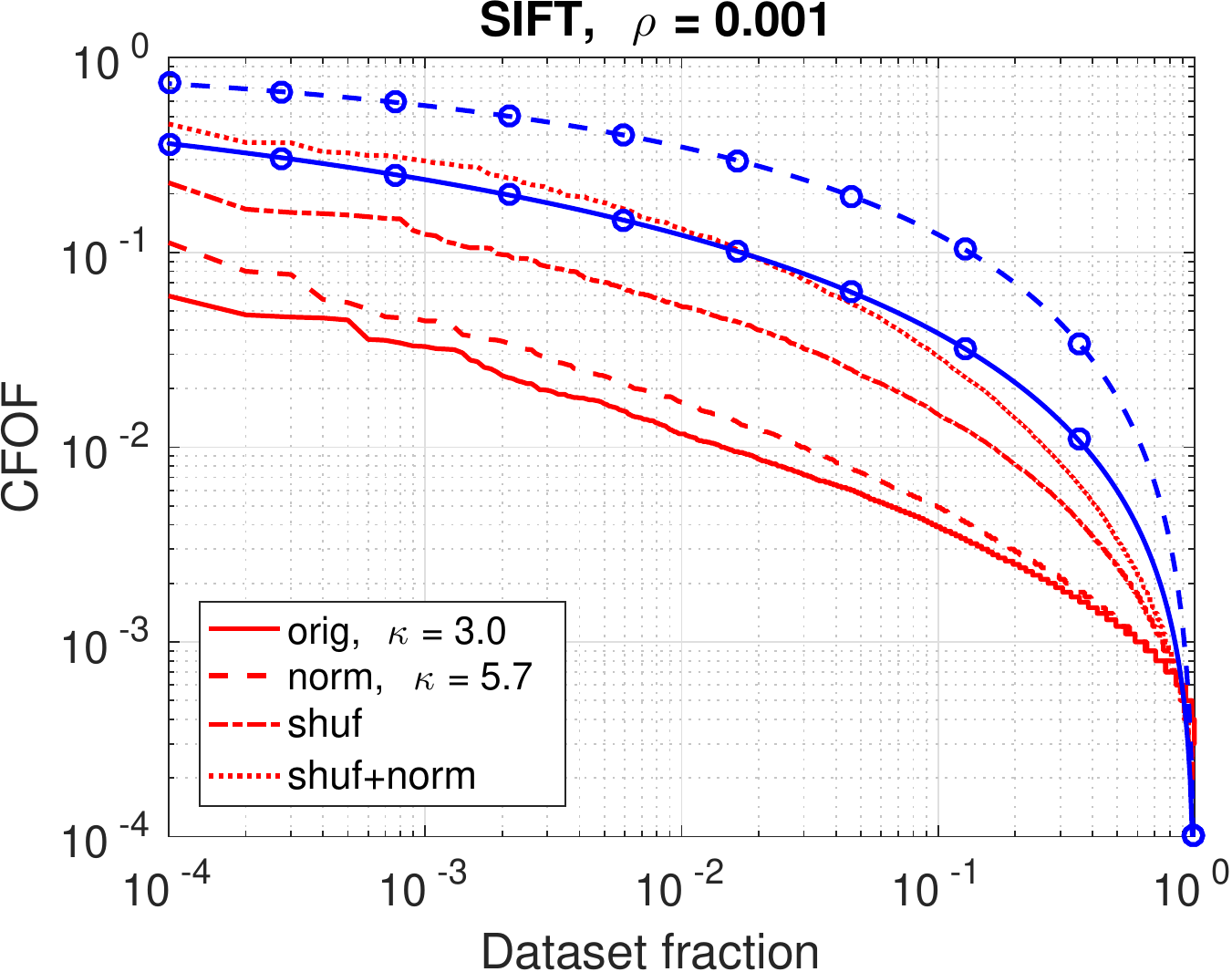}}
~
\subfloat[]
{\includegraphics[width=0.33\textwidth]{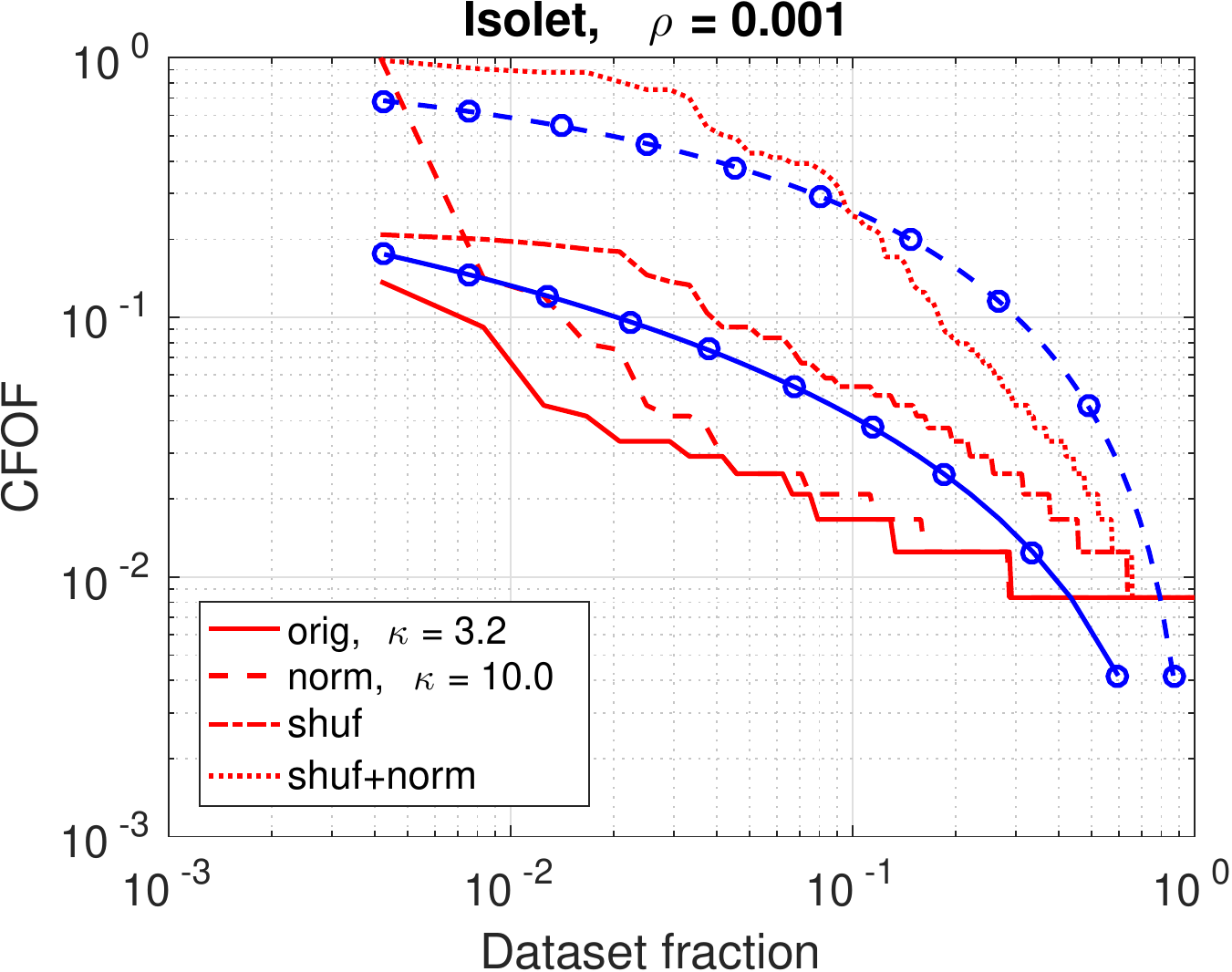}}
~
\subfloat[]
{\includegraphics[width=0.33\textwidth]{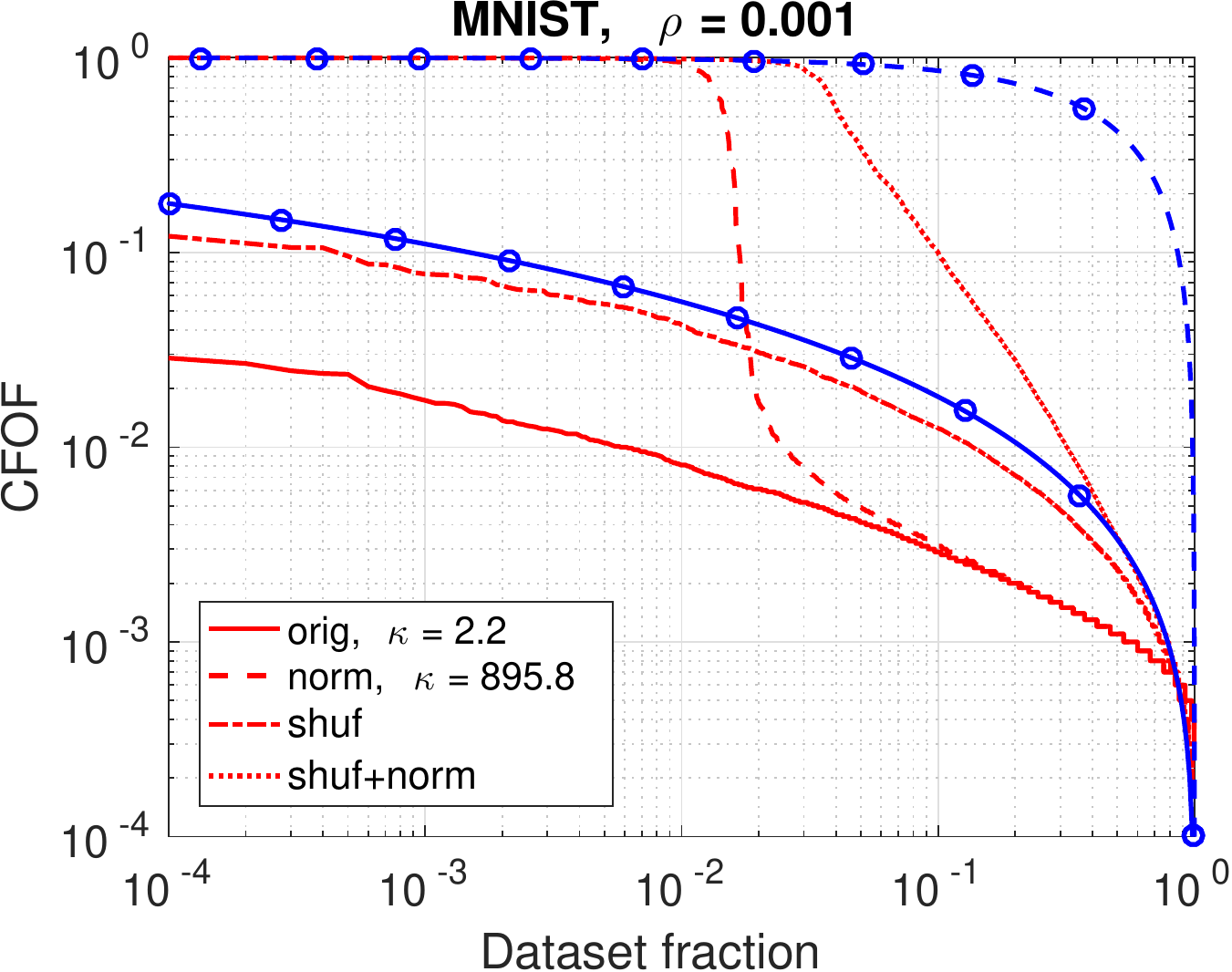}}
\\
\subfloat[]
{\includegraphics[width=0.33\textwidth]{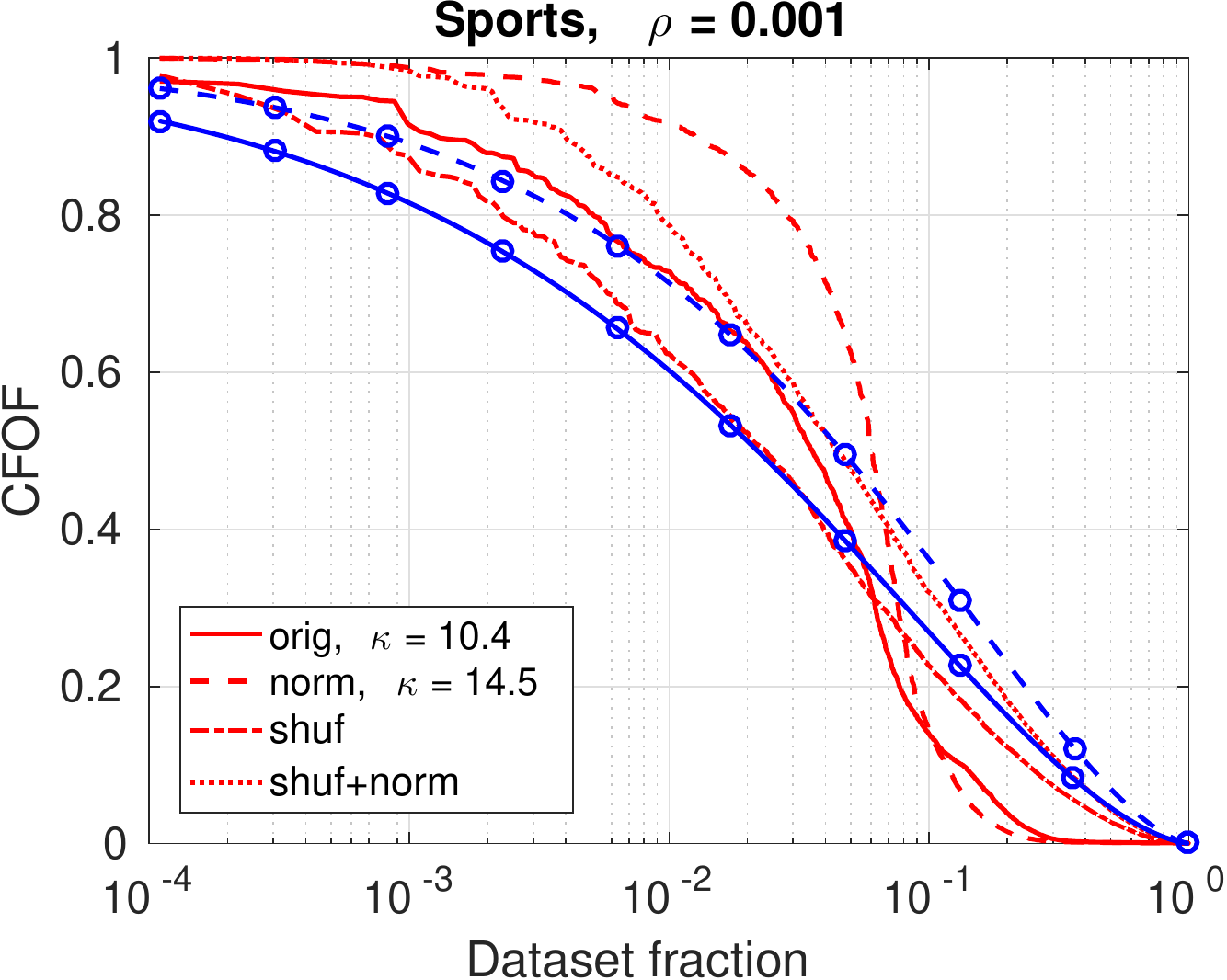}}
~
\subfloat[]
{\includegraphics[width=0.33\textwidth]{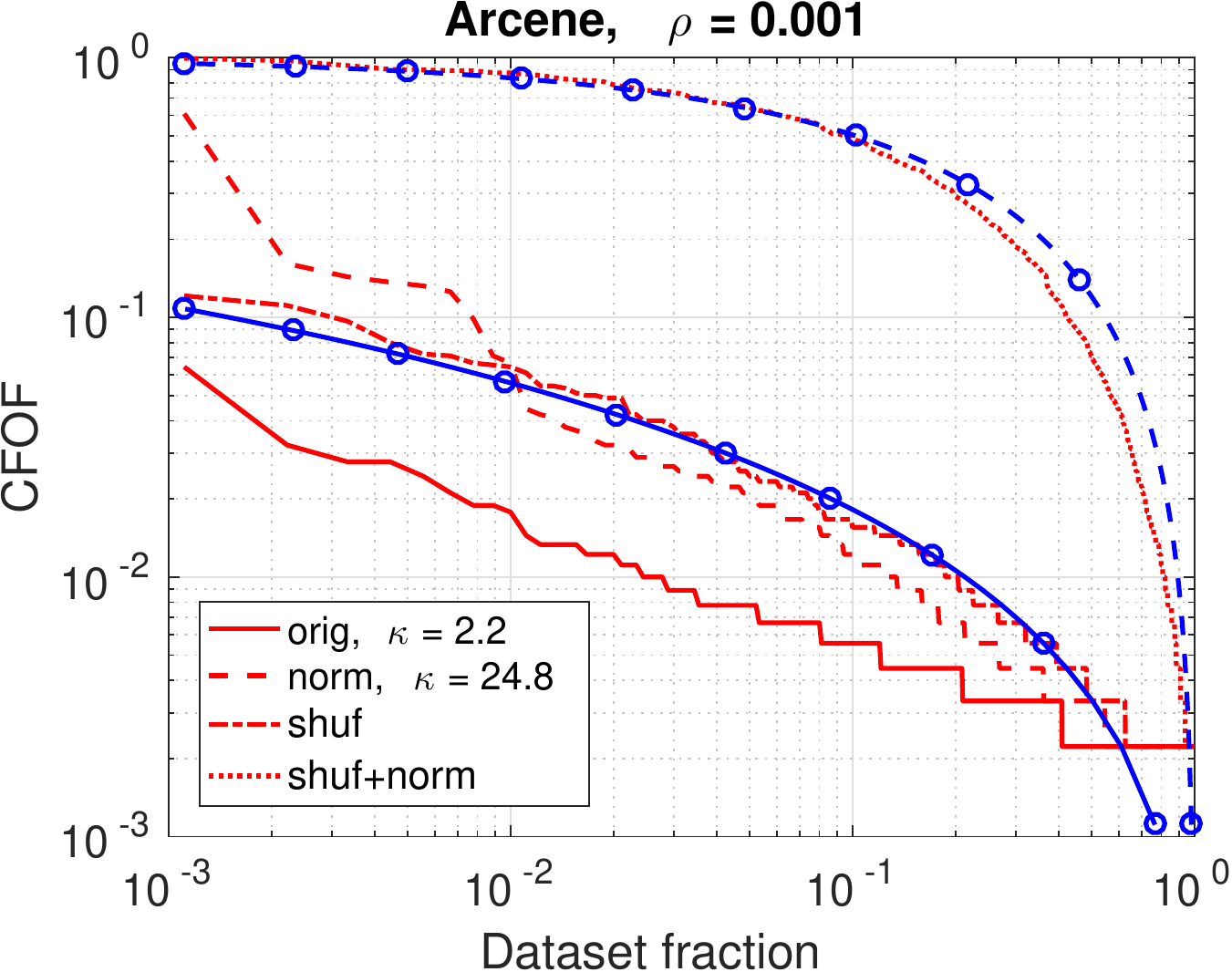}}
~
\subfloat[]
{\includegraphics[width=0.33\textwidth]{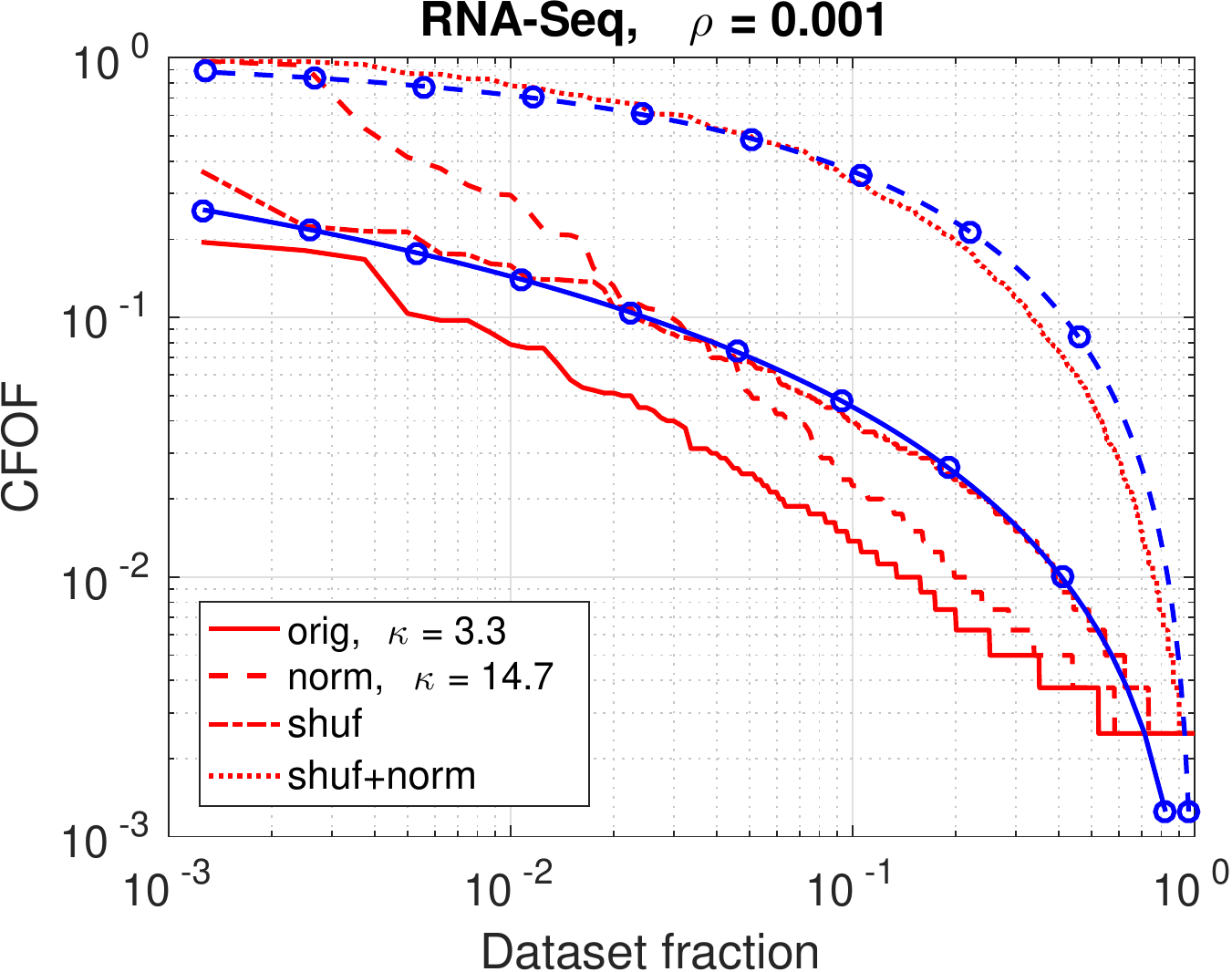}}
\caption{[Best viewed in color.]
Distribution of $\CFOF$ scores for $\varrho = 0.001$:
original data (red solid line), 
normalized data (red dashed line),
shuffled data (red dashed-dotted line), 
shuffled and normalized data (red dotted line),
theoretical distribution with kurtosis $\kappa=\kappa_{orig}$,
(blue solid line), and 
theoretical distribution with kurtosis $\kappa=\kappa_{norm}$
(blue dashed line).
}
\label{fig:realconc001}
\end{figure}

\begin{figure}[t]
\centering
\subfloat[]
{\includegraphics[width=0.33\textwidth]{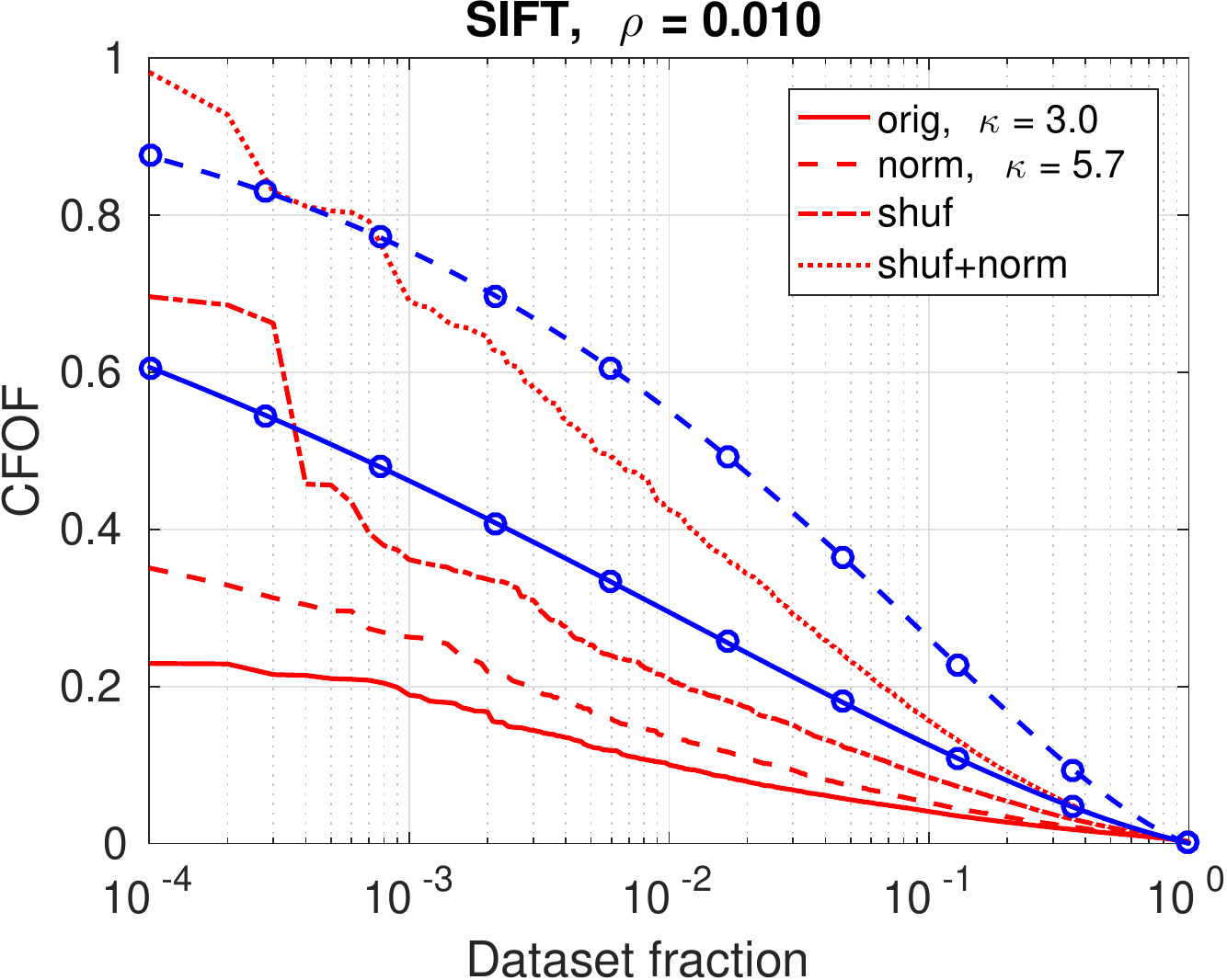}}
~
\subfloat[]
{\includegraphics[width=0.33\textwidth]{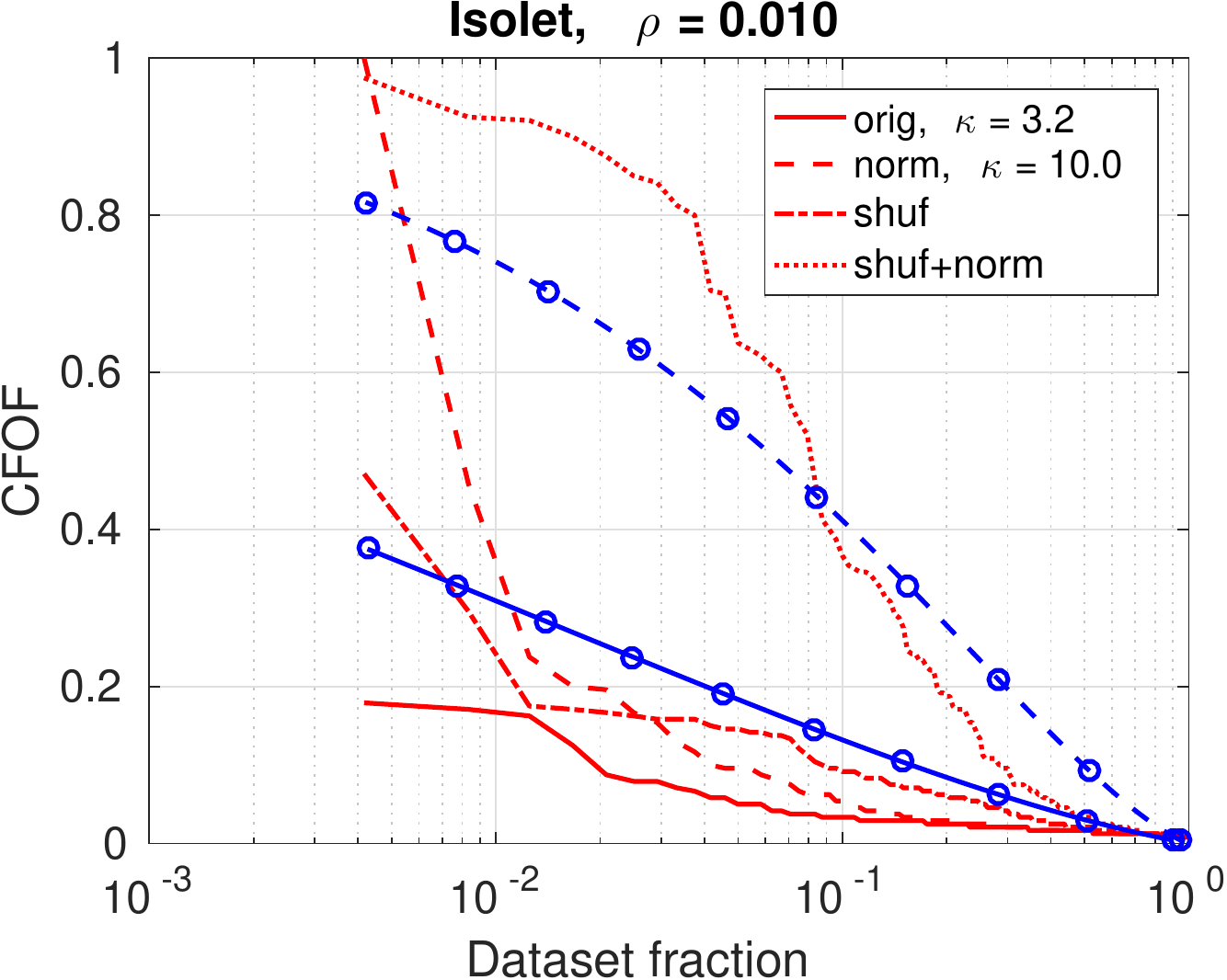}}
~
\subfloat[]
{\includegraphics[width=0.33\textwidth]{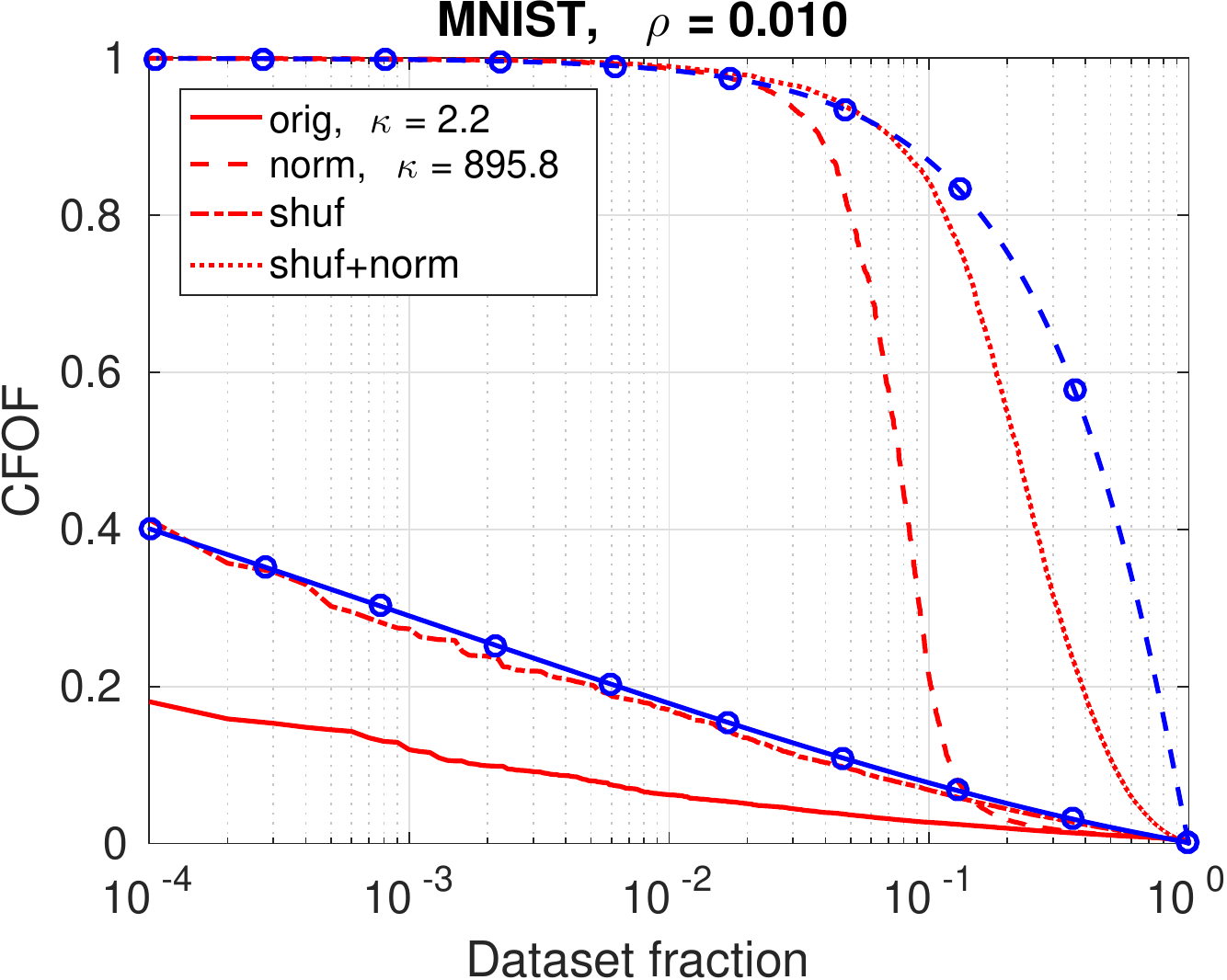}}
\\
\subfloat[]
{\includegraphics[width=0.33\textwidth]{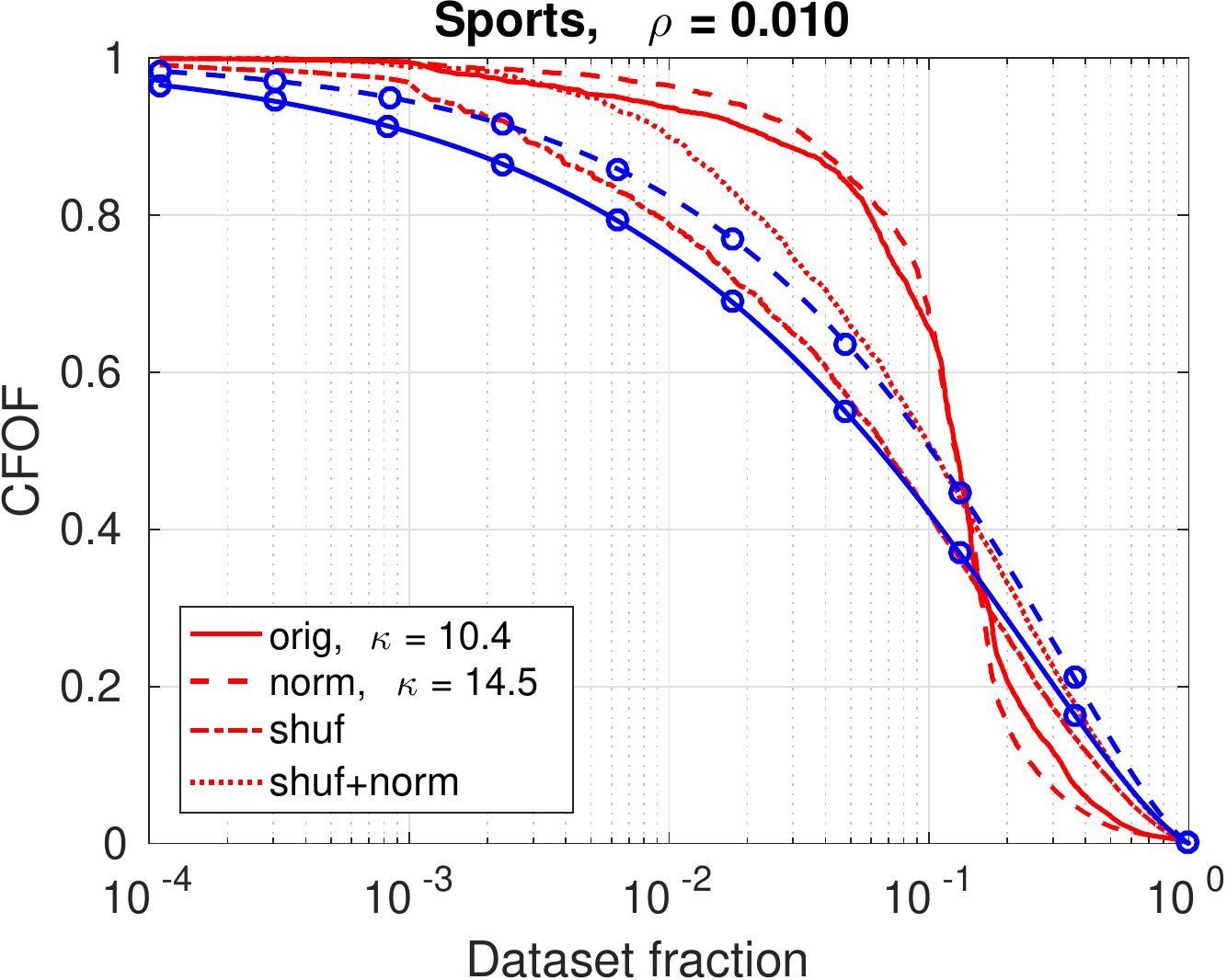}}
~
\subfloat[]
{\includegraphics[width=0.33\textwidth]{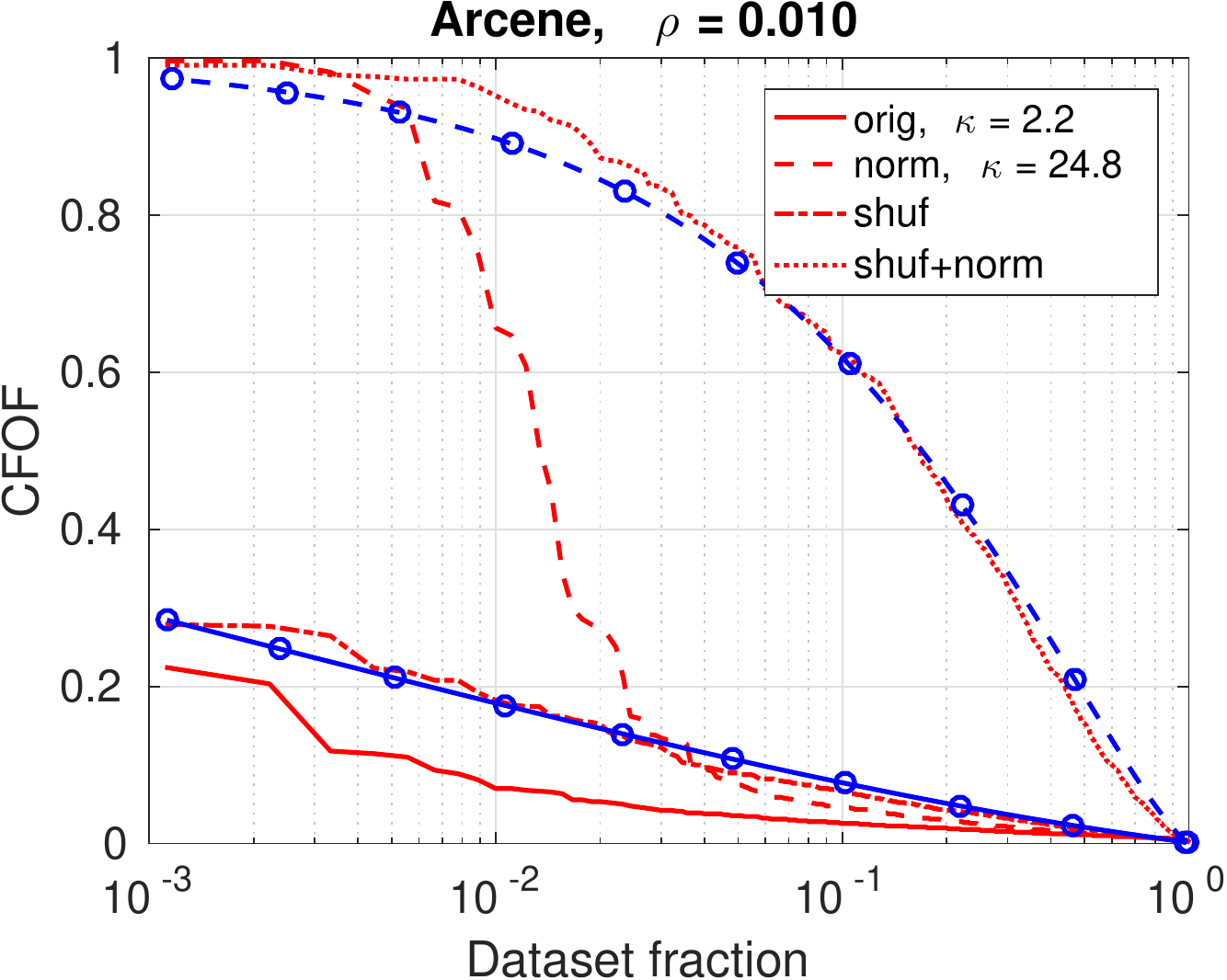}}
~
\subfloat[]
{\includegraphics[width=0.33\textwidth]{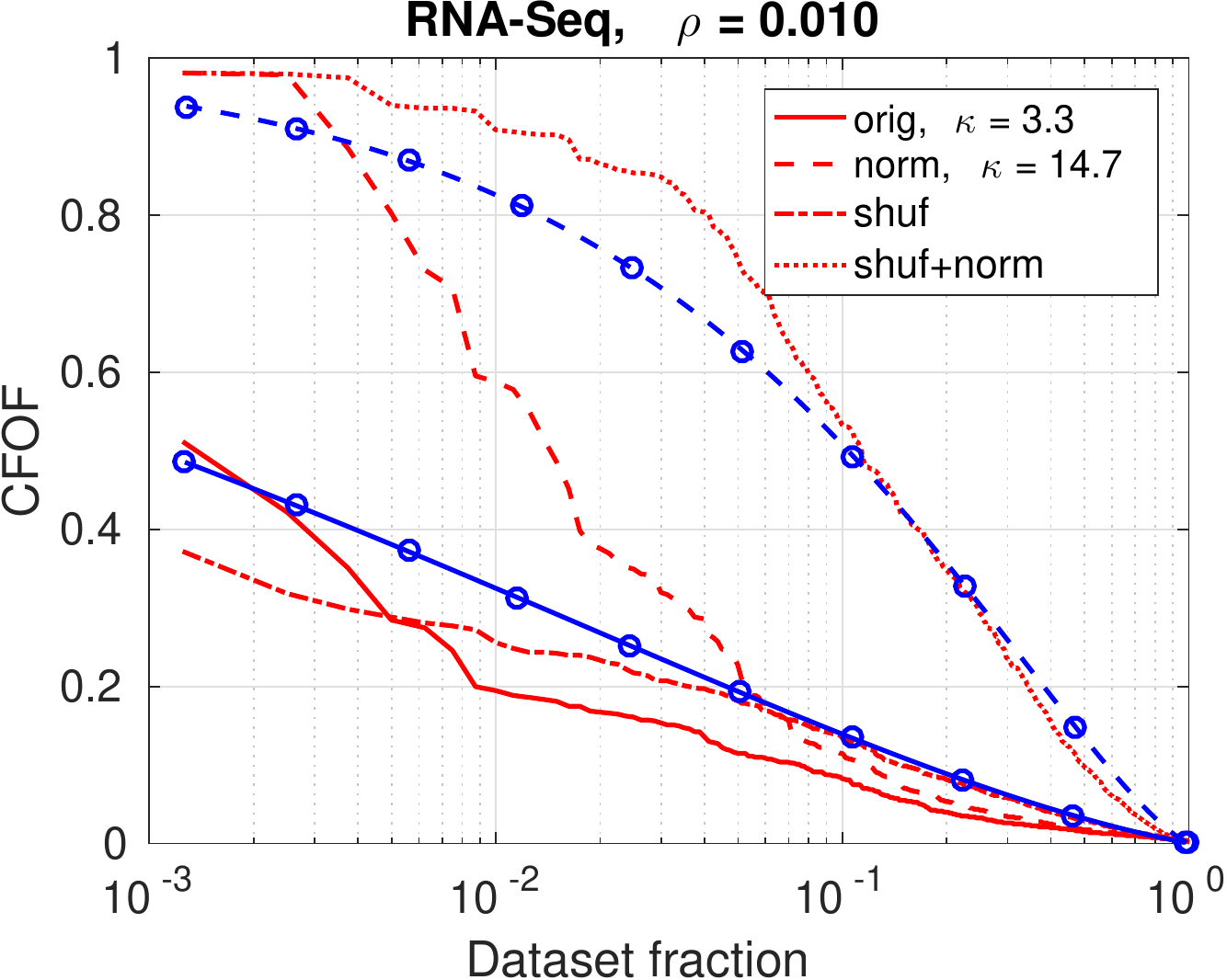}}
\caption{[Best viewed in color.]
Distribution of $\CFOF$ scores for $\varrho = 0.01$:
original data (red solid line), 
normalized data (red dashed line),
shuffled data (red dashed-dotted line), 
shuffled and normalized data (red dotted line),
theoretical distribution with kurtosis $\kappa=\kappa_{orig}$,
(blue solid line with circles markers), and 
theoretical distribution with kurtosis $\kappa=\kappa_{norm}$
(blue dashed line with circles markers).}
\label{fig:realconc01}
\end{figure}

\subsection{Concentration on real data}
\label{sect:exp_conc}

In this section and in the following ones,
we employed the $\HardCFOF$ definition.

To study the concentration behavior of $\CFOF$ on real
data we considered some datasets
having dimensionality
varying from hundreds to thousands,
namely
\textit{SIFT} ($d=128$, $n=10,\!000$),\footnote{The 
dataset consist of the base vectors of the {\tt ANN\_SIFT10K} vector set
available at {\tt http://corpus-texmex.irisa.fr/}.},
\textit{Isolet} ($d=617$, $n=240$),\footnote{See {\tt https://archive.ics.uci.edu/ml/datasets/isolet}.}
\textit{MNIST} test ($d=698$, $n=10,\!000$),\footnote{See {\tt http://yann.lecun.com/exdb/mnist/}.}
\textit{Sports} ($d=5,\!625$, $n=9,\!120$),\footnote{See {\tt https://archive.ics.uci.edu/ml/datasets/daily+and+sports+activities}.}
\textit{Arcene} ($d=10,\!000$, $n=900$),\footnote{See {\tt https://archive.ics.uci.edu/ml/datasets/Arcene}.}
and
\textit{RNA-Seq} ($d=20,\!531$, $n=801$).\footnote{See {\tt https://archive.ics.uci.edu/ml/datasets/gene+expression+cancer+RNA-Seq}.}
Attributes having null variance have been removed from
the original datasets.

To help result interpretation,
in addition to the original data, we considered also normalized and shuffled data.
The normalized data is obtained by mean centering the data and then
dividing by the standard deviation. 
The shuffled data is obtained by randomly permuting the elements within 
every attribute. As already noted in the literature \cite{Francois2007},
the shuffled dataset is marginally distributed as the original one,
but because all the relationships between variables are destroyed, its 
component are independent, and its intrinsic dimension is equal
to its embedding dimension.

\medskip
Figures \ref{fig:realconc001} and \ref{fig:realconc01} show the distribution of the
$\CFOF$ scores for $\varrho = 0.001$ and $\varrho = 0.01$, respectively,
on the above described datasets.
The abscissa reports score values, either in linear or in logarithmic scale, 
while the ordinate reports
the dataset fraction, always in logarithmic scale. Scores are ranked in descending
order, thus the scores on the left are associated with the 
outliers. Note that in the plots the dataset fraction increases exponentially
from left to right.

Specifically, red curves are relative to the scores associated with
the original data (solid line), the normalized data (dashed line),
the shuffled data (dashed-dotted line), and the 
shuffled and normalized data (dotted line).
To compare the empirical score distribution with the theoretical
distribution described in Sections \ref{sect:cfofcdf} and \ref{sect:cfofcdfvskurt}, we
measured the kurtosis $\kappa_{orig}$ of the original data
and the kurtosis $\kappa_{norm}$ of the normalized data
(see Equations \eqref{eq:moments} and \eqref{eq:kurtosis} and
also Equations \eqref{eq:kurt_real} and \eqref{eq:kurt_normreal} in the following). 
Note that shuffling the
data has no effect on the kurtosis.
The blue solid line with circles markers is the theoretical distribution
of the $\CFOF$ scores 
having kurtosis $\kappa_{orig}$,
and the blue dashed line with circles markers is the theoretical distribution
of the $\CFOF$ scores
having kurtosis $\kappa_{norm}$. These curves have been
obtained by leveraging the expression of Equation \eqref{eq:cfofcdf}.

In general, the curves show that the distribution of the 
scores is unbalanced, 
with a small fraction of points associated with the 
largest scores and the majority of the points associated
with score values that may vary over orders of magnitude.
These curve witness for the absence of concentration on real life data.

By comparing curves for $\varrho = 0.001$ and
$\varrho = 0.01$, it can be seen that  
scores increase with the parameter $\varrho$,
as expected from the theoretical analysis.

Moreover, normalizing data 
enlarges the probability to observe larger scores.
More specifically,
this effect can be explained 
in the light of the analysis of Section \ref{sect:cfofcdfvskurt},
by noting that often the kurtosis of the normalized data increases 
sensibly. 
E.g. consider 
$\kappa_{orig}=2.2$ versus
$\kappa_{norm}=24.8$ for \textit{Arcene}
or $\kappa_{orig}=2.2$ versus
$\kappa_{norm}=895.8$ for \textit{MIST}.

Consider an independent non-identically distributed random vector
$\Vect{X}$ having, w.l.o.g., null mean.\footnote{The vector $\Vect{X}$
can be always be replaced by the vector $\Vect{X}-\mu_{\Vect{X}}$.}
Then, according to Equations \eqref{eq:kurtosis} and \eqref{eq:moments},
it can be considered equivalent to an i.i.d. random vector 
having kurtosis
\begin{equation}\label{eq:kurt_real}
\kappa_{\Vect{X}} = \frac{ \tilde{\mu}_4(\Vect{X})}{\tilde{\mu}_2^2(\Vect{X})} =
\frac{\E[{\mu}_4(X_i)]}{\E[{\mu}_2(X_i)^2]},
\end{equation}
and the kurtosis $\kappa_{orig} = \kappa_{\Vect{X}}$ of the original vector
is given by the ratio between the average fourth central moment
and the average squared second central moment of the coordinates of $\Vect{X}$.

Consider now the normalized random vector $\Vect{Z}$ such that,
for $i=1,\ldots,d$,
$Z_i = X_i/\sigma(X_i)$. Then, $\mu_2(Z_i)=\sigma^2(Z_i)=1$, and
\begin{align}\label{eq:kurt_normreal}
\kappa_{\Vect{Z}} = 
\frac{ \tilde{\mu}_4(\Vect{Z})}{\tilde{\mu}_2^2(\Vect{Z})} =
\frac{ \E[{\mu}_4(Z_i)]}{\E[{\mu}_2^2(Z_i)]} =
\frac{1}{d} \sum_{i=1}^d \E\left[\left(\frac{X_i}{\sqrt{\E[X_i^2]}}\right)^4\right] = 
\nonumber
\\
= \frac{1}{d} \sum_{i=1}^d \frac{\E[X_i^4]}{\E[X_i^2]^2}  =
\frac{1}{d} \sum_{i=1}^d \frac{\mu_4(X_i)}{\mu_2^2(X_i)}  =
\frac{1}{d} \sum_{i=1}^d \kappa_{X_i} = 
\E[\kappa_{X_i}],
\end{align}
from which it can be concluded that 
the kurtosis $\kappa_{norm}$ of the normalized vector $\Vect{Z}$
is given by the average kurtosis 
of the coordinates of the original vector $\Vect{X}$

Thus, the kurtosis of the normalized data is larger
than that of the original data
when attributes have non-homogeneous kurtosis $\kappa_{X_i}$,
with some attributes having extreme deviations,
while the fourth central moments $\mu_4(X_i)$ 
and the squared second central moments $\mu_2(X_i)^2$
of the attributes are, on the average, much more similar.

For example, consider the dataset \textit{MNIST}
having $\kappa_{orig}=2.2$ and $\kappa_{norm}=895.8$.
Despite its moments $\tilde{\mu}_4$ and $\tilde{\mu}_2^2$ are comparable,
the kurtosis of the single attributes vary of four orders
of magnitude (from about $1$ to $10^4$).
This is related to the presence of a large number of zeros ($77.4\%$
of the whole dataset, after removing attributes consisting 
only of zeros) which are not uniformly distributed among attributes.
Indeed, the attributes with very high kurtosis consist almost
entirely of zeros: the attribute having kurtosis $10^4$
contains only one non-zero entry, while the attributes
having small kurtosis contain at least the fifty percent of non-zero values.
Note that is not our intention to argue that normalizing makes always sense, 
but only to observe and explain the behavior of the score on normalized data.

To understand the effect of normalization 
on the result of the $\CFOF$ technique, we 
computed the Spearman rank correlation
between the $\CFOF$ scores associated with the original data
and the $\CFOF$ scores associated with normalized data.
The correlations are:
$0.9511$ for \textit{SIFT}, $0.7454$ for \textit{Isolet}, 
$0.8338$ for \textit{MNIST}, $0.8934$ for \textit{Sports}, 
$0.6090$ for \textit{Arcene}, and $0.8898$ for \textit{RNA-Seq}.
These correlations suggest that in different cases there is some agreement
between the ranking induced by the method on original and normalized data.

\medskip
As already noticed,
the shuffled data is marginally distributed as the original data,
but its component are independent.
Thus, we expect that shuffled data 
complies better with the theoretical prediction.
Indeed, the $\CFOF$ score distribution of the shuffled data
follows quite closely the theoretical distribution of 
Equation \eqref{eq:cfofcdf}. This experiment 
provides strong evidence 
for the validity of the analysis accomplished
in Section \ref{sect:cfofcdf}.
This is true both for the original shuffled data,
compare the red dash-dotted curve with the solid blue line,
and for the normalized shuffled data, compare the red 
dotted line with the blue dashed line.
E.g., for \textit{Arcene} the shuffled data curves
appear to be superimposed with the theoretical ones.
In any case, the general trend of the theoretical $\CFOF$
distribution is always followed by both the original and the 
normalized data.

\begin{figure}[t]
\centering
\subfloat[]
{\includegraphics[width=0.33\textwidth]{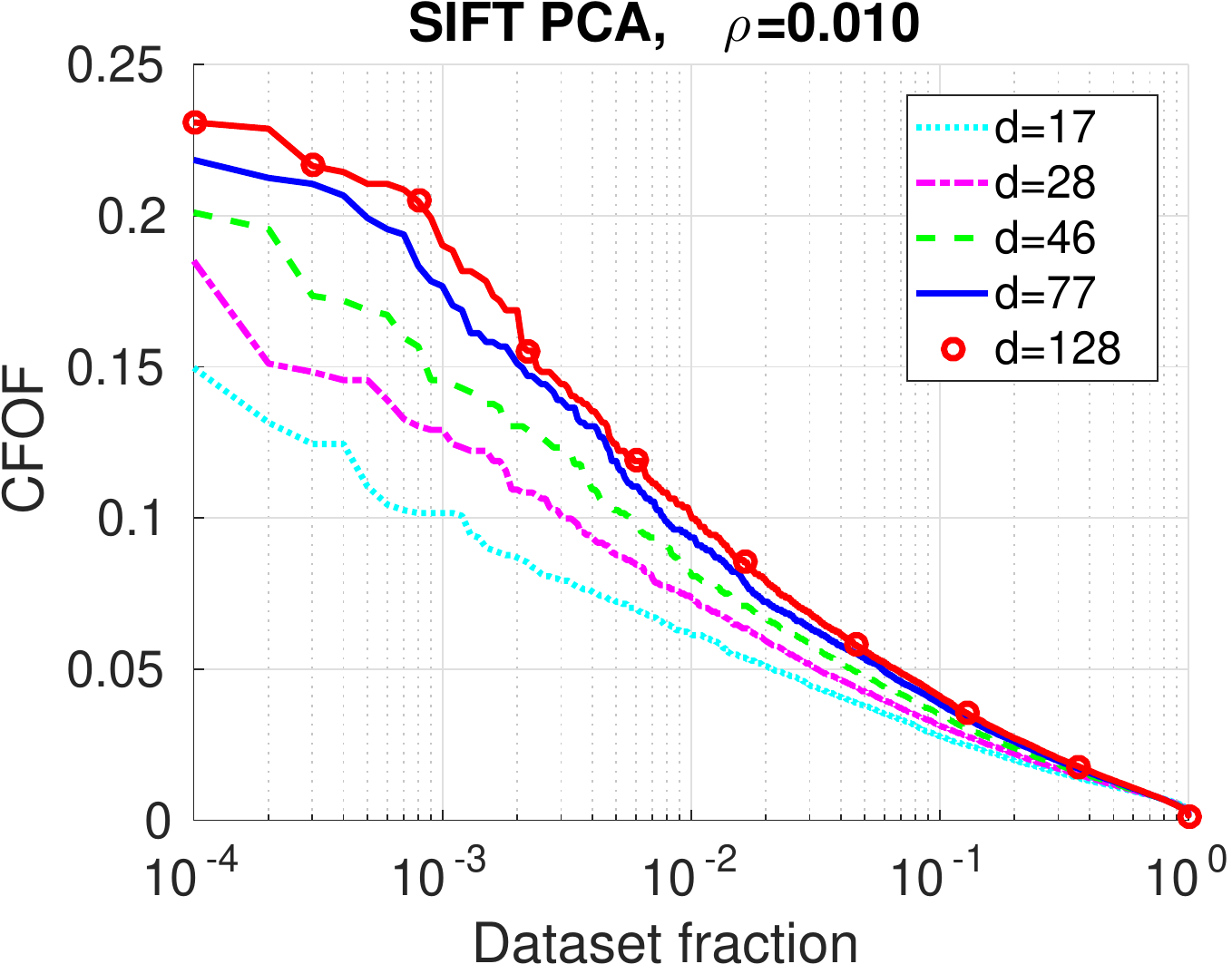}}
~
\subfloat[]
{\includegraphics[width=0.33\textwidth]{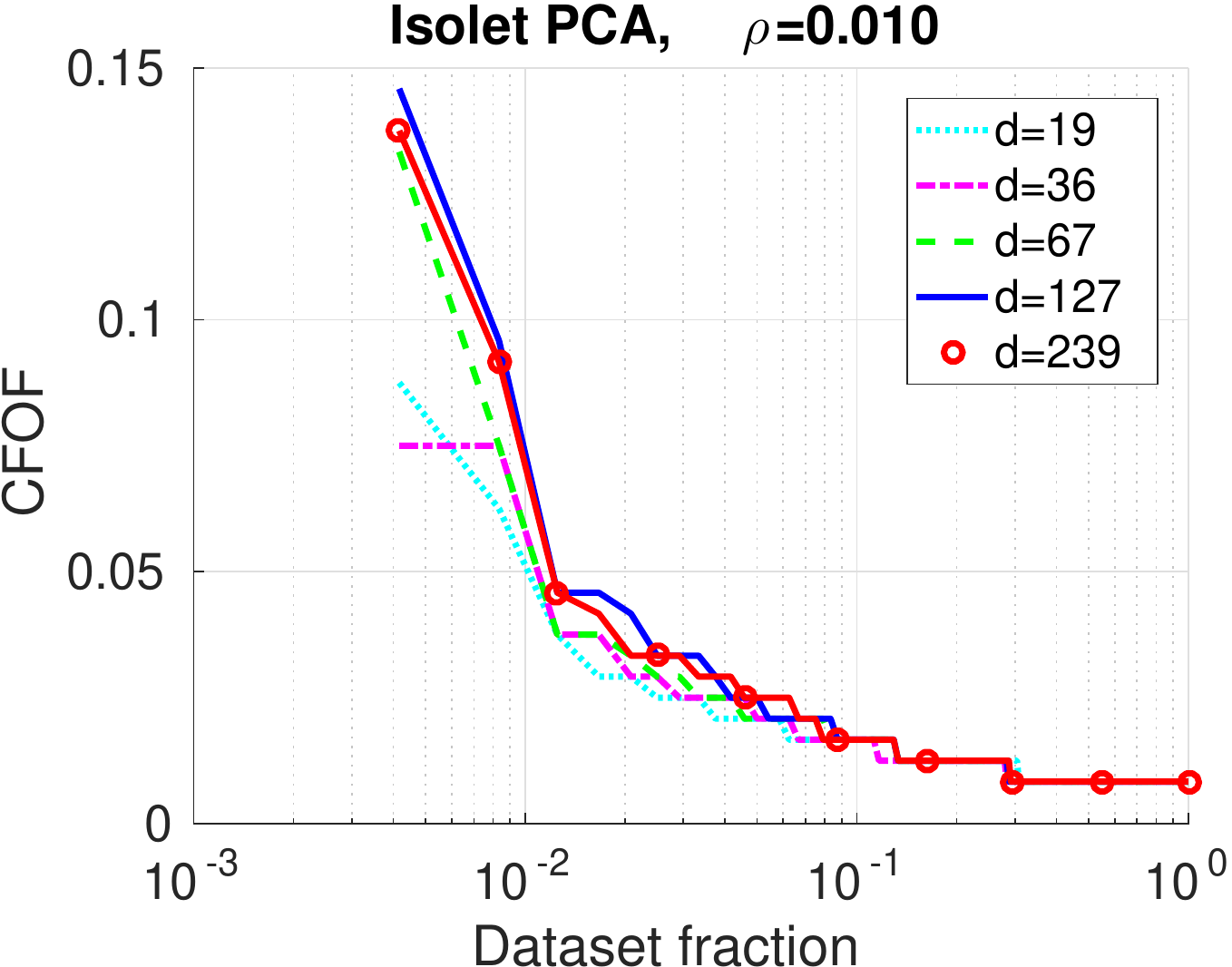}}
~
\subfloat[]
{\includegraphics[width=0.33\textwidth]{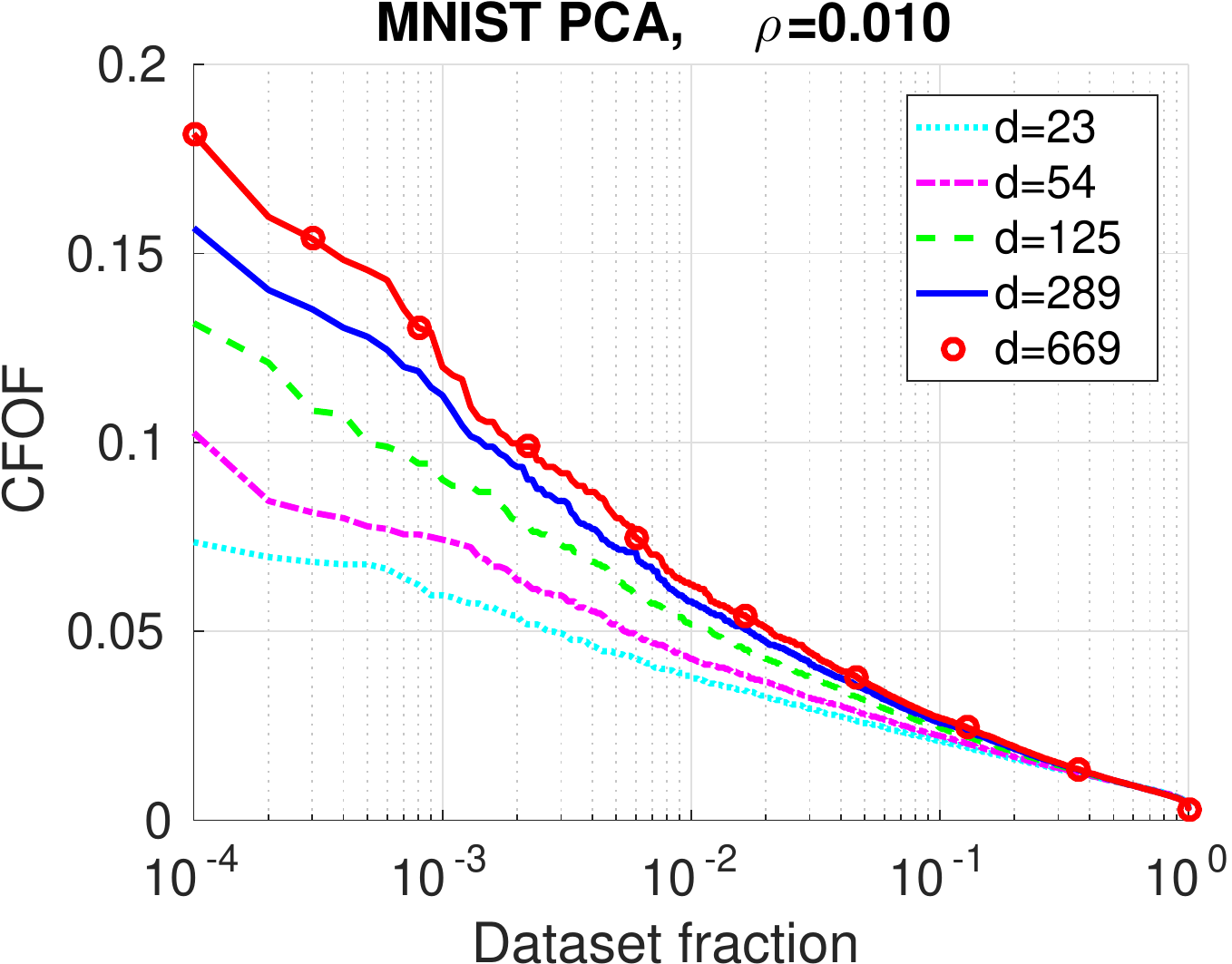}}
\\
\subfloat[]
{\includegraphics[width=0.33\textwidth]{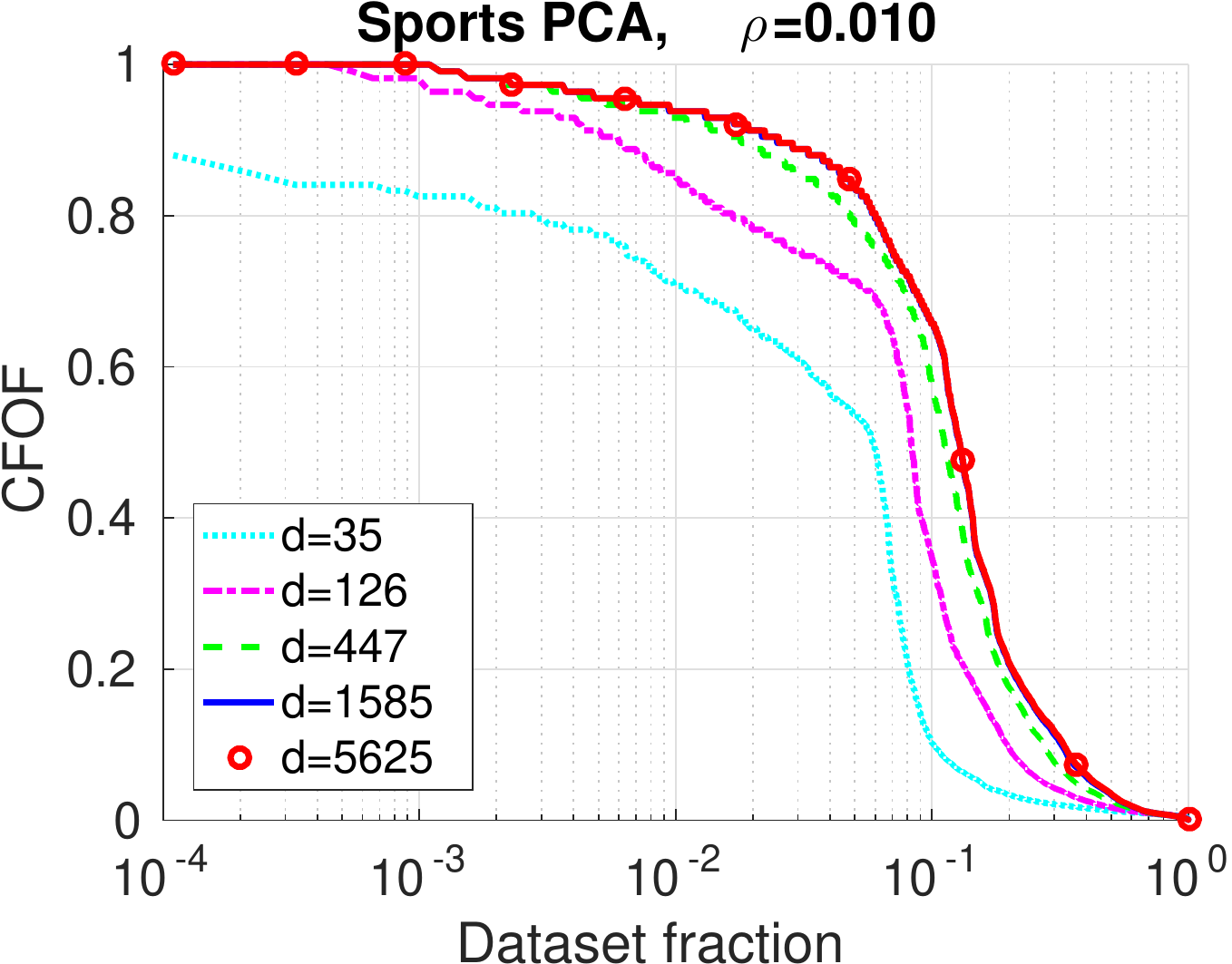}}
~
\subfloat[]
{\includegraphics[width=0.33\textwidth]{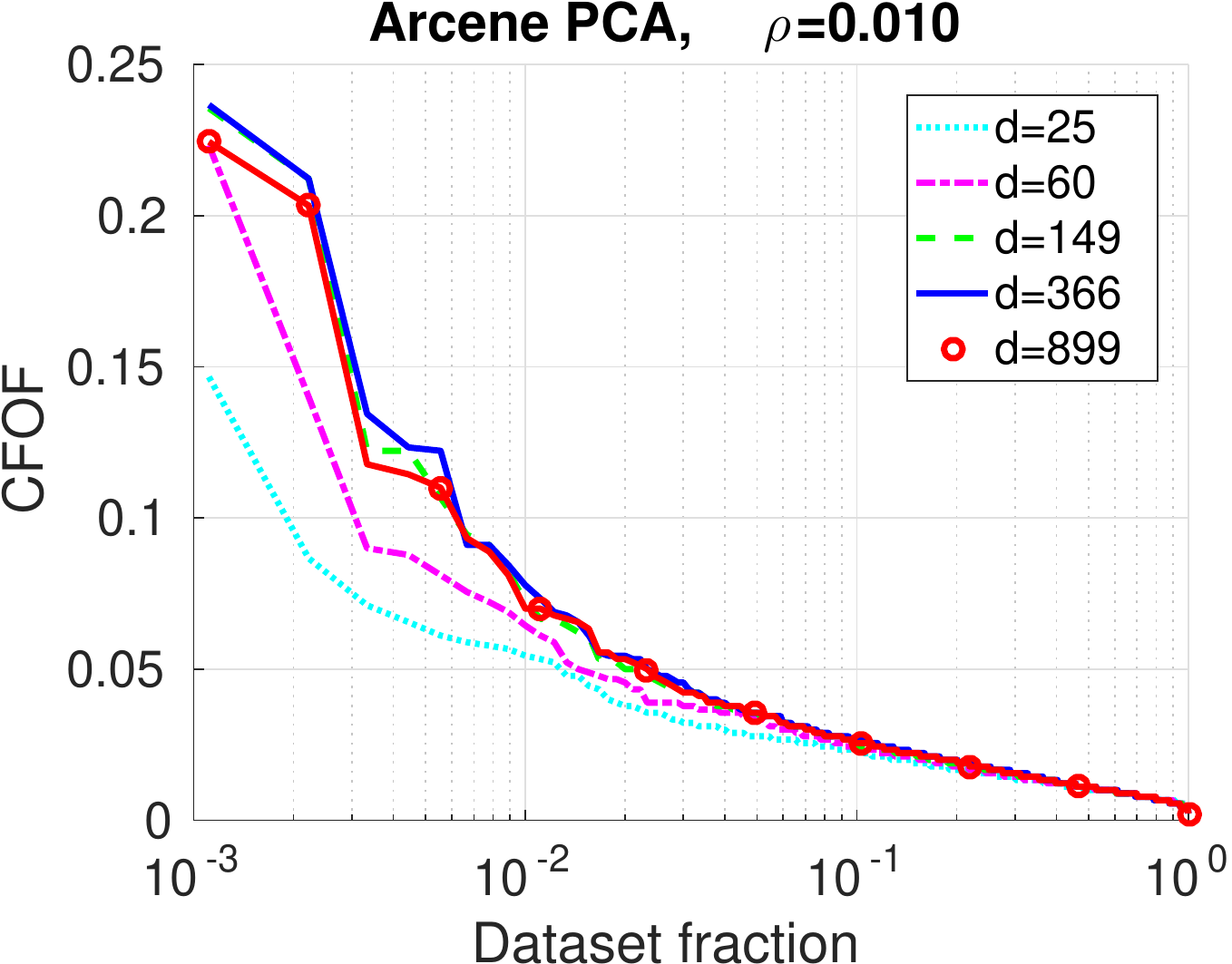}}
~
\subfloat[]
{\includegraphics[width=0.33\textwidth]{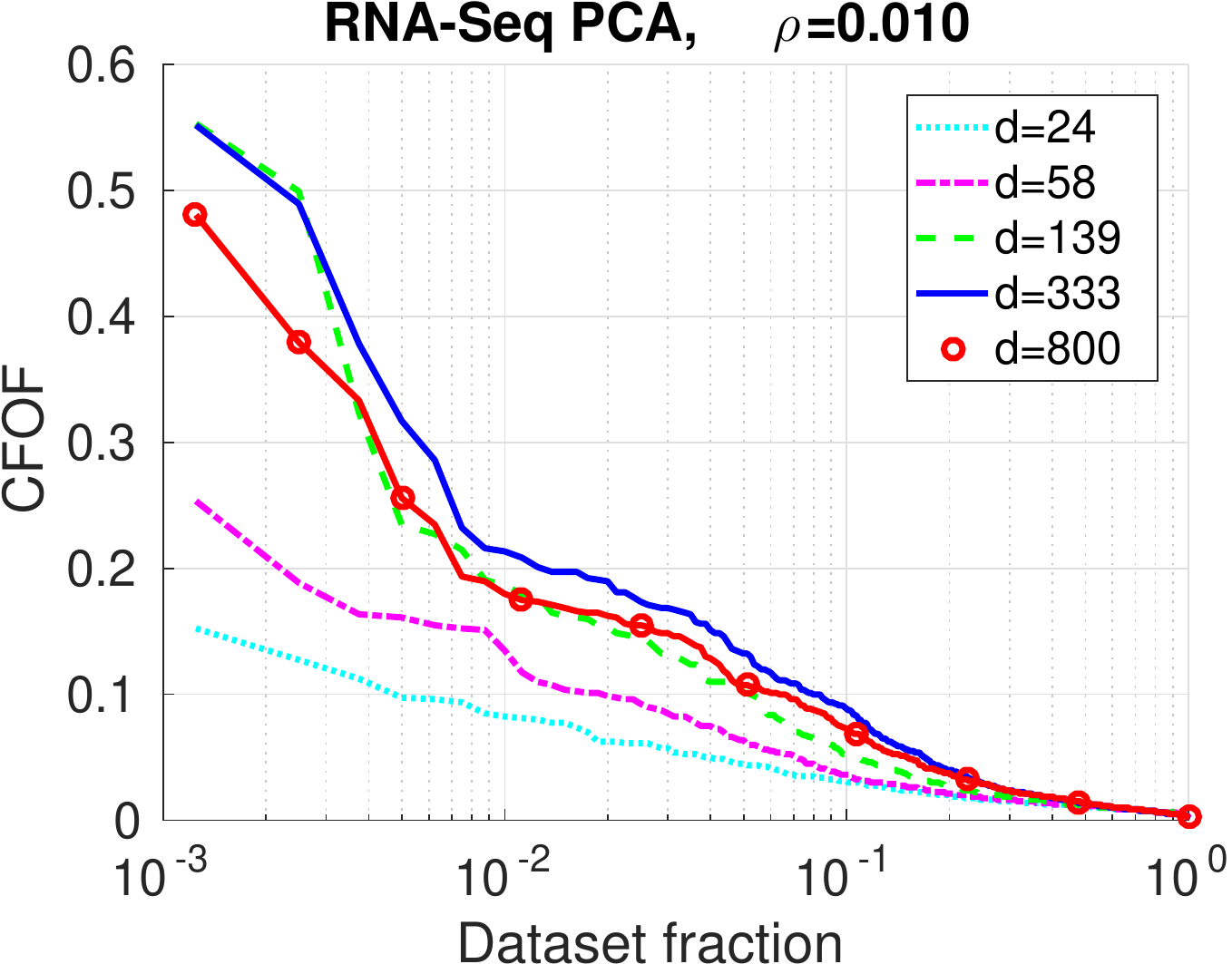}}
\caption{[Best viewed in color.] $\CFOF$ score distribution for $\varrho=0.01$
on real data projected
on the first principal components. The curve marked with
small circles concerns the full feature space.}
\label{fig:dupont1}
\end{figure}

\medskip
To study the effect of varying the dimensionality,
we used Principal Component Analysis to determine 
the principal components of the original data and, then,
considered datasets of increasing dimensionality
obtained by projecting the original dataset on its
first $d$ principal components.
Here $d$ is log-spaced between $1$ and
the number of distinct principal components
$d_{max} = \min \{ D, n-1 \}$,
where $D$ denotes the number of attributes 
of the original dataset.

Figure \ref{fig:dupont1} shows the distribution of $\CFOF$
scores for $\varrho = 0.01$. 
The dotted (red) curve concerns the full feature space, while the other
curves are relative to the projection on the first principal
components.
It can be seen that the general form of the score distribution holding
in the full feature space is still valid for the subspace projected
datasets.

\begin{figure}[t]
\centering
\subfloat[]
{\includegraphics[width=0.33\textwidth]{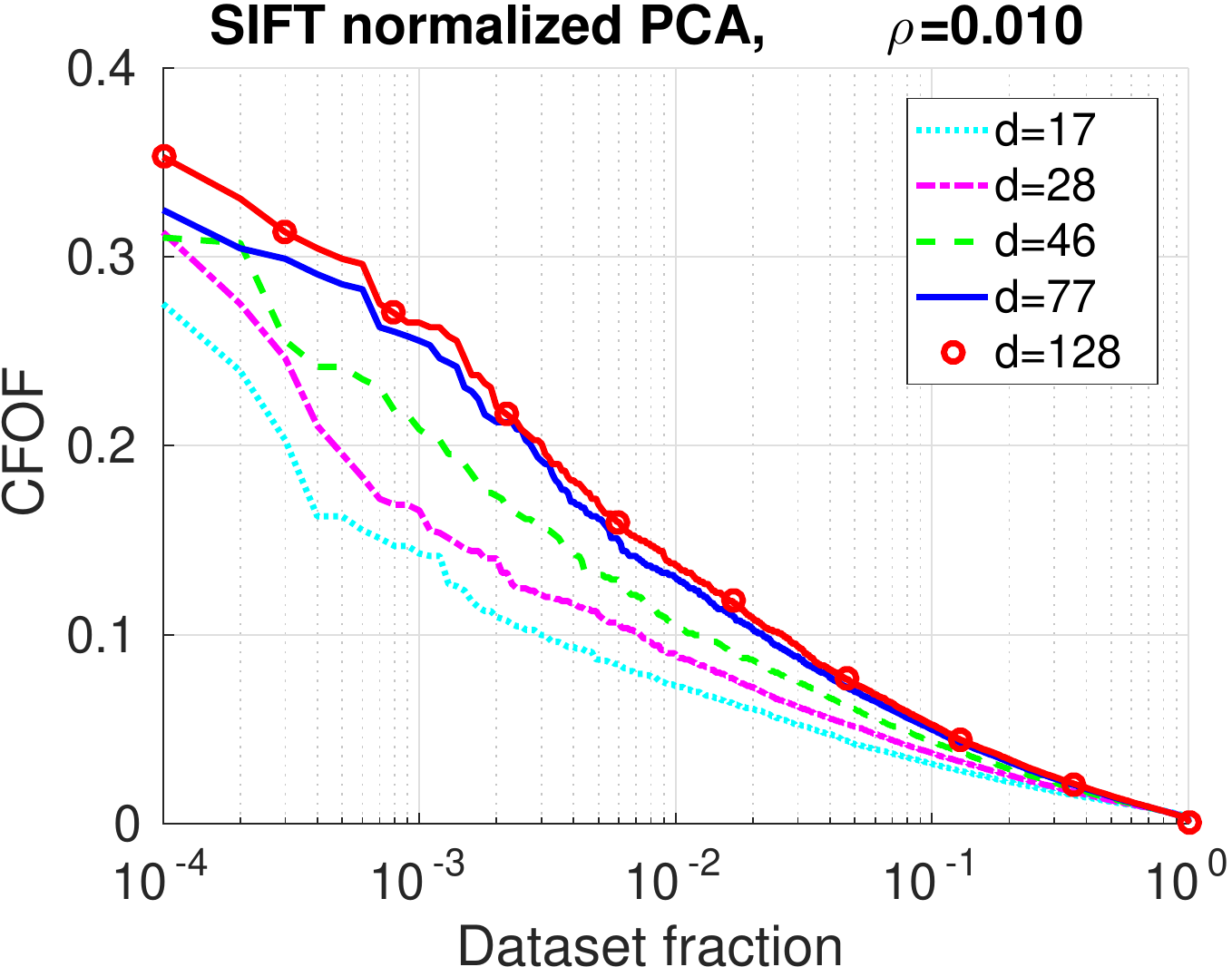}}
~
\subfloat[]
{\includegraphics[width=0.33\textwidth]{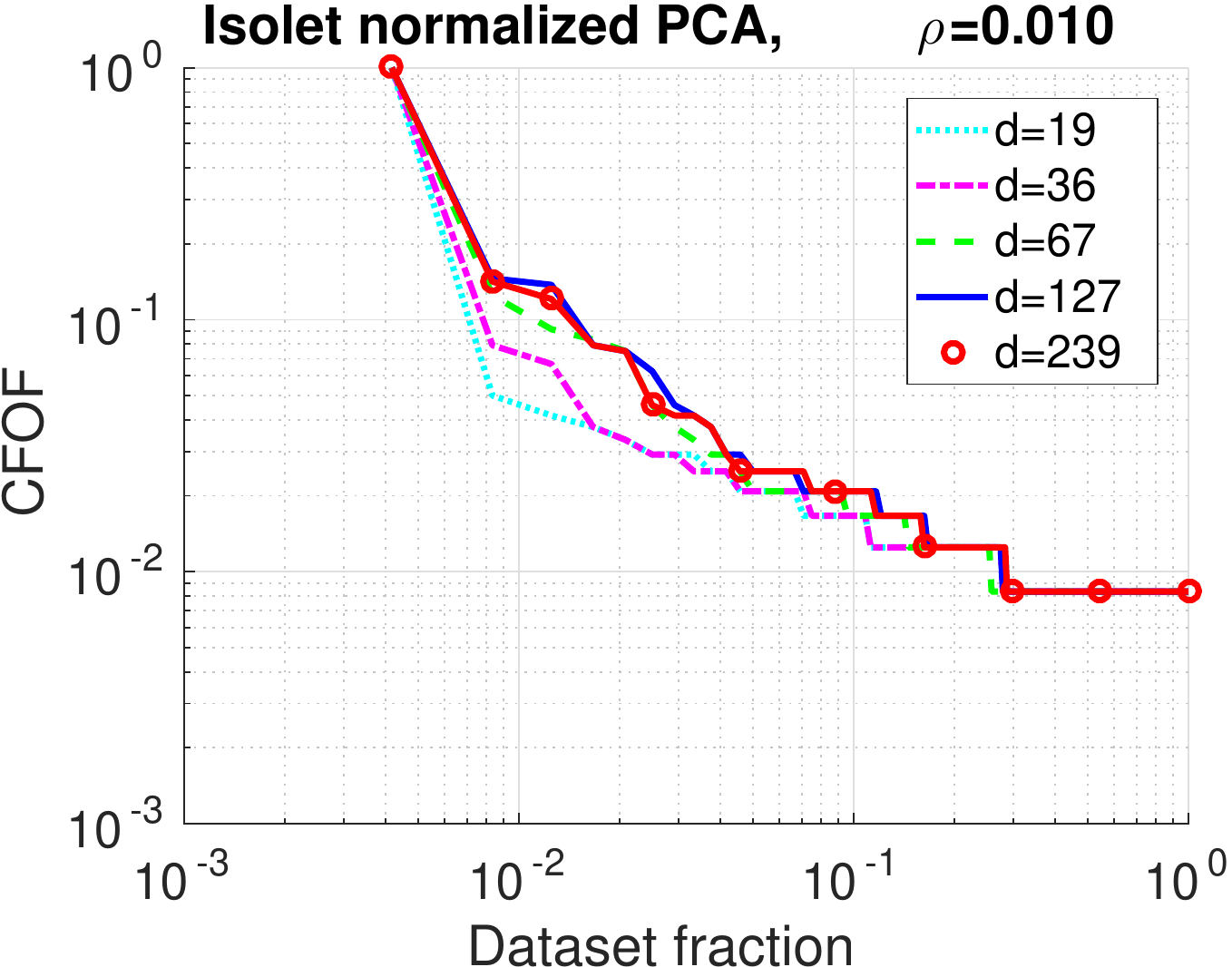}}
~
\subfloat[]
{\includegraphics[width=0.33\textwidth]{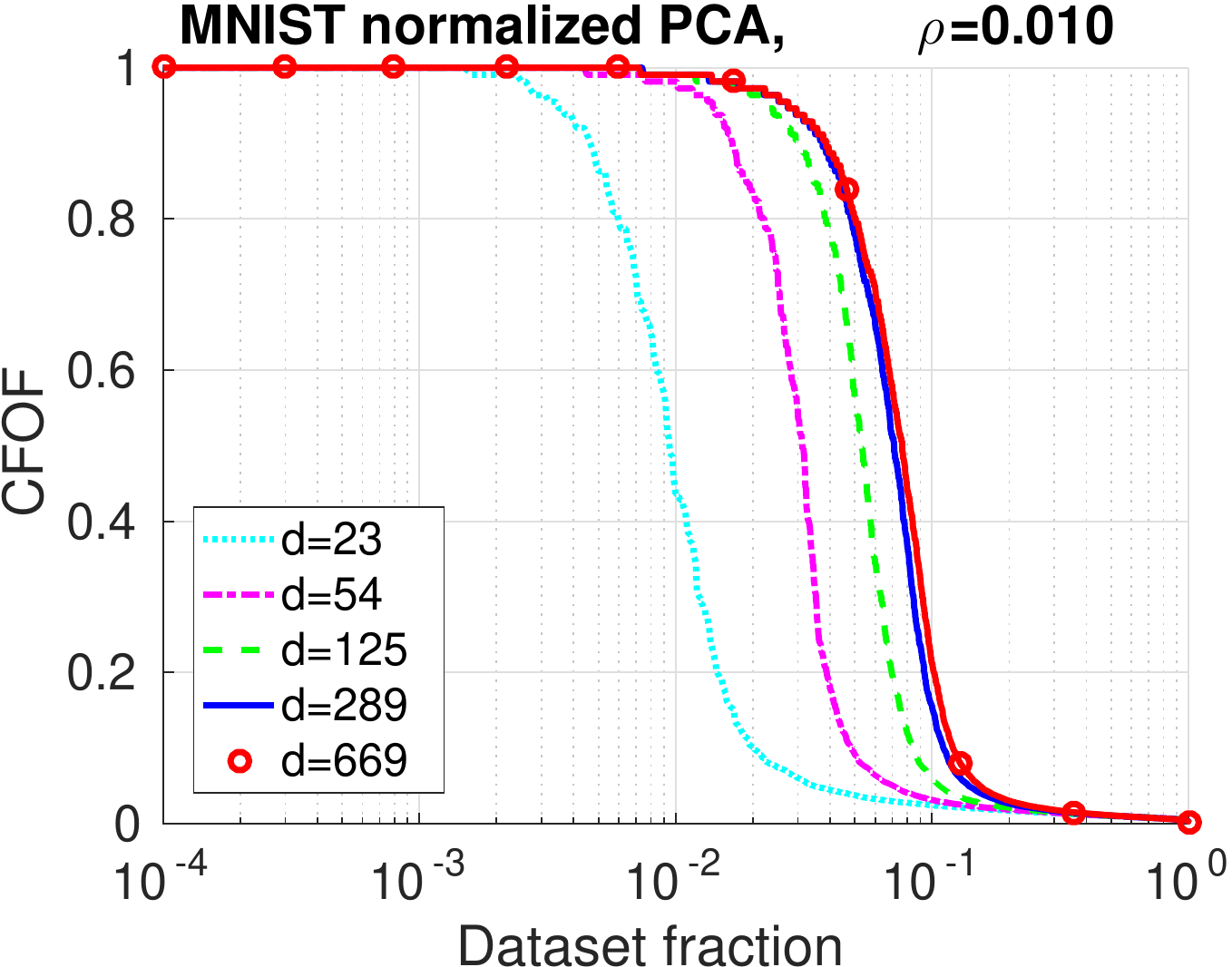}}
\\
\subfloat[]
{\includegraphics[width=0.33\textwidth]{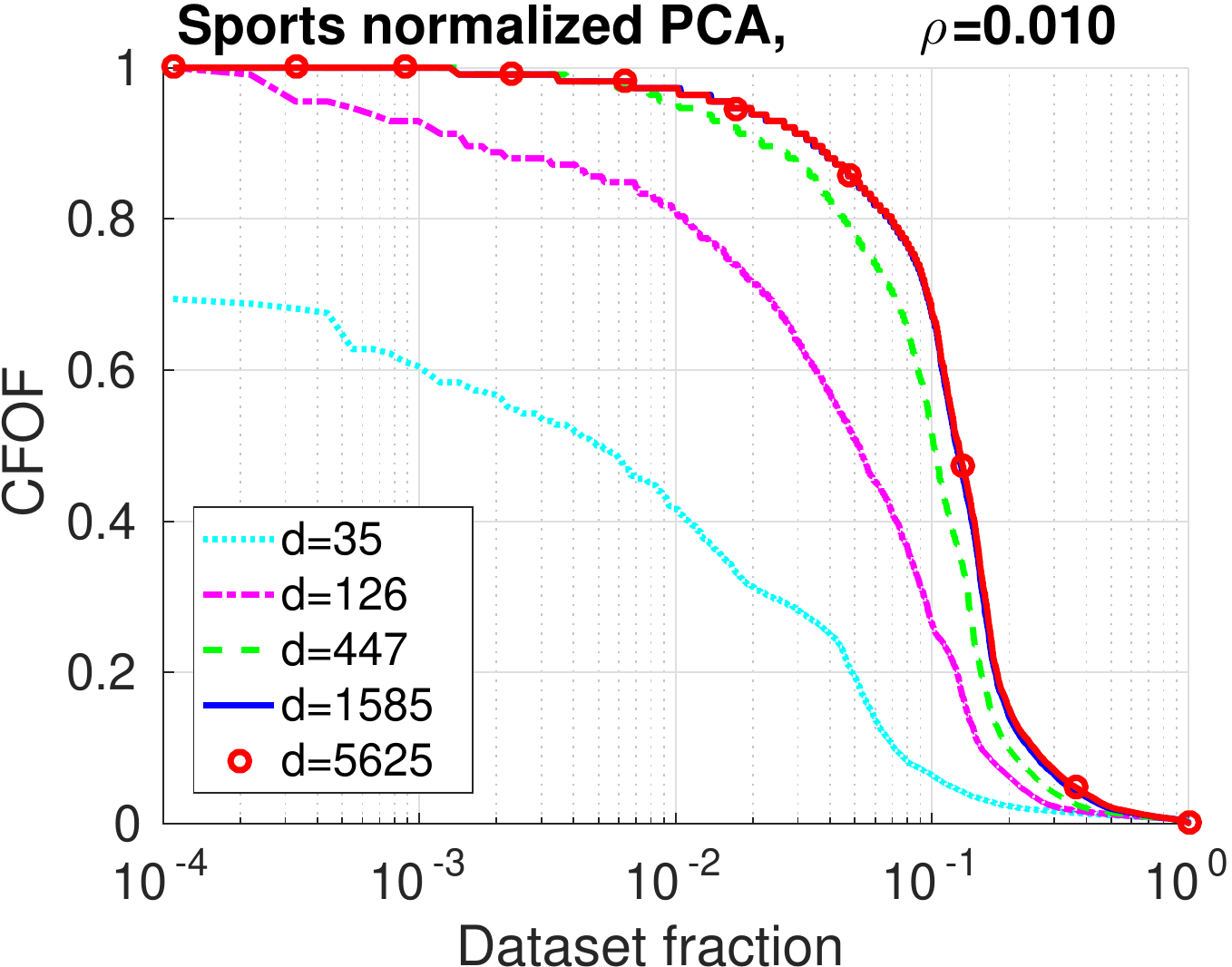}}
~
\subfloat[]
{\includegraphics[width=0.33\textwidth]{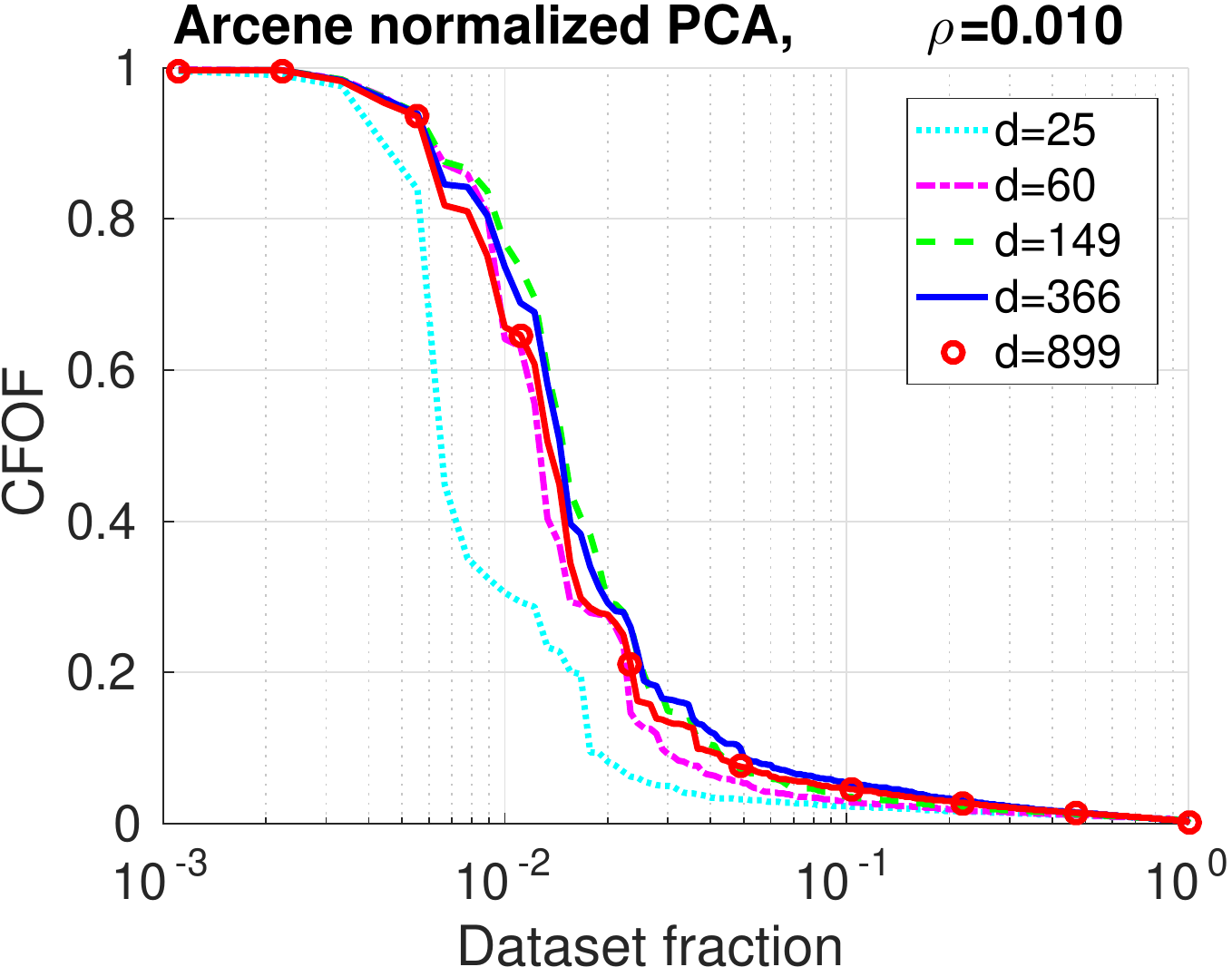}}
~
\subfloat[]
{\includegraphics[width=0.33\textwidth]{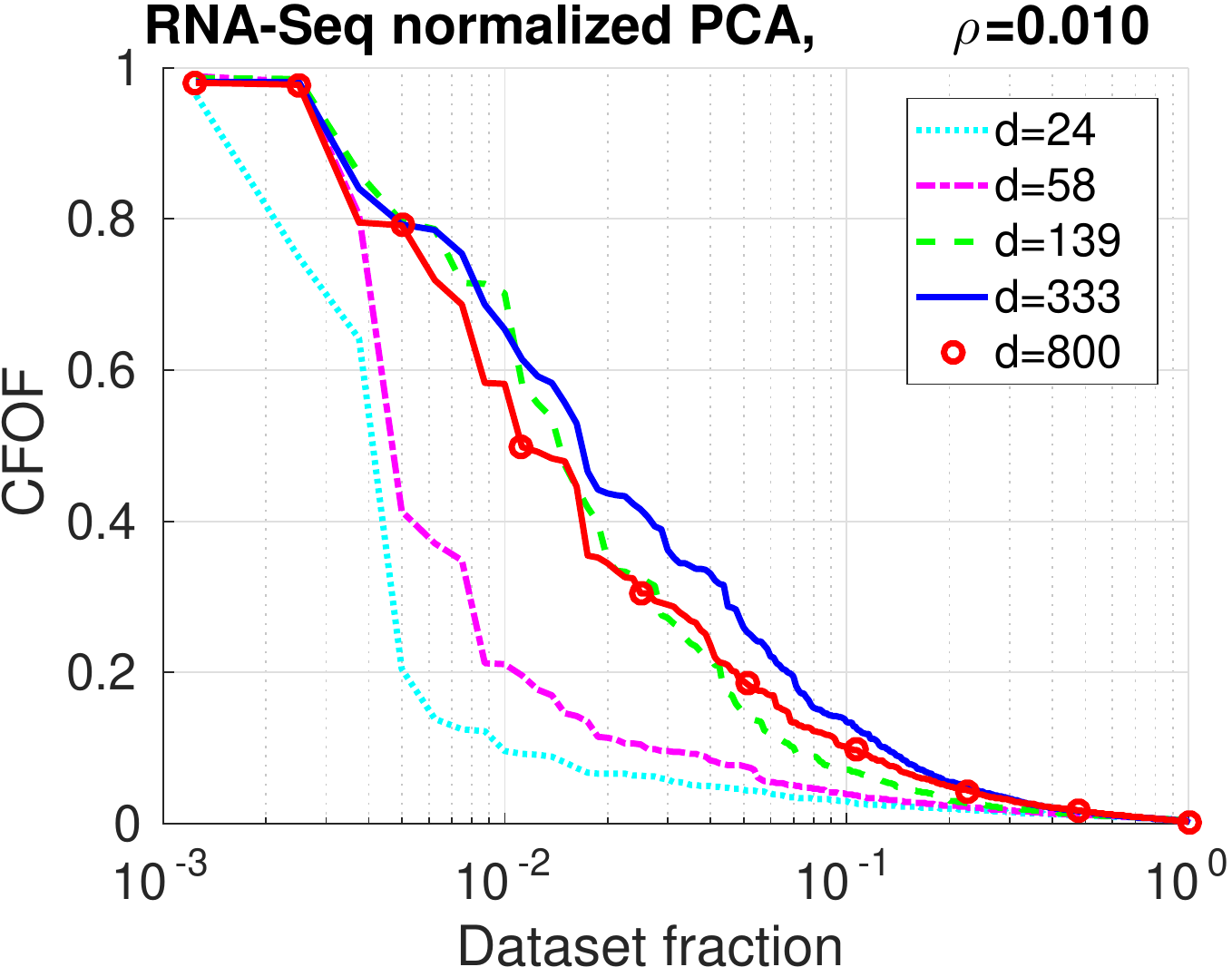}}
\caption{[Best viewed in color.] $\CFOF$ score distribution for $\varrho=0.01$
on real data normalized and projected
on the first principal components. The curve marked with small circles 
concerns the full
feature space.}
\label{fig:dupont2}
\end{figure}

In order to compare the different curves,
we define a measure of abruptness
of an outlier score distribution.
Let $Sc$ be a set of outlier scores, let
$\mbox{top}_\alpha(Sc)$ denote the scores associated with the 
top $\alpha$ outliers, and let 
$\mbox{med}(Sc)$ denote the median score in $Sc$.
We call \textit{concentration ratio} $\mbox{CR}$ the ratio
\begin{equation}\label{eq:conc_ratio}
\mbox{CR}_\alpha(Sc) = \frac{\sigma\big(\mbox{top}_{\alpha}(Sc)\big)}{\mbox{med}({Sc})} 
\end{equation}
between the standard deviation associated with the 
top $\alpha$ outlier scores and the median outlier score.
The numerator term above 
measures how well the score distribution 
separates the most deviating
observations of the population.
Since the standard deviations are relative to scores having
different distributions,
the term in the denominator serves the purpose of normalizing the
above separation, expressing it as a quantity
relative to the magnitude of the most central score.
We note that the mean is not suitable for
normalizing the top outlier scores,
in that scores may vary of different orders of magnitude,
with possibly a few large score values and a lot of 
negligible score values, and thus the mean would be mostly influenced by
the most extreme scores.

From the analysis of Section \ref{sect:cfofcdfvskurt},
extremely high kurtosis distributions are 
those maximizing the probability to observe large scores
(close to $1$), but also the ones 
associated with the less steep $\CFOF$ score distribution.
The lower bound to the relative concentration
for high kurtosis distributions 
can be obtained by exploiting Theorem \ref{th:cfofcdfvskurt}.
In this case,
$\lim_{\kappa\rightarrow\infty} Pr[\CFOF(\Vect{X})\le s] = s$
and the $\CFOF$ scores are uniformly distributed between $0$ and $1$.
Then, $\sigma(\mbox{top}_\alpha(Sc))$ is the standard deviation
of an uniform distribution in the interval $[1-\alpha,\alpha]$,
that is $\frac{\alpha}{2\sqrt{3}}$, and $\mbox{med}(Sc)$ is the
median value of an uniform distribution in the interval $[0,1]$,
that is $0.5$. Hence, 
$\lim_{\kappa\rightarrow\infty} \mbox{CR}_\alpha(Sc) = \frac{\alpha}{4\sqrt{3}}$.

\begin{table}[t]
\centering
\begin{tabular}{|l||r|r|r|r|r|}
\multicolumn{6}{c}{original data, $\alpha = 0.1$, $\varrho = 0.01$} \\
\hline
& $d_1$ & $d_2$ & $d_3$ & $d_4$ & $d_{max}$ \\
\hline\hline
\textit{SIFT}    & $ 1.334$ & $ 1.661$ & $1.891$ & $2.100$ & $2.213$ \\
\textit{Isolet}  & $ 1.985$ & $ 1.933$ & $3.018$ & $3.496$ & $3.263$ \\
\textit{MNIST}   & $ 0.762$ & $ 0.974$ & $1.300$ & $1.568$ & $1.747$ \\
\textit{Sports}  & $17.014$ & $ 9.161$ & $4.246$ & $2.847$ & $2.732$ \\
\textit{Arcene}  & $ 1.712$ & $ 2.387$ & $3.032$ & $3.100$ & $2.900$ \\
\textit{RNA-Seq} & $ 2.063$ & $ 3.549$ & $7.354$ & $5.781$ & $4.834$ \\
\hline
\multicolumn{6}{c}{normalized data, $\alpha = 0.1$, $\varrho = 0.01$} \\
\hline
& $d_1$ & $d_2$ & $d_3$ & $d_4$ & $d_{max}$ \\
\hline\hline
\textit{SIFT}    & $ 1.880$ & $ 2.146$ & $ 2.445$ & $ 2.755$ & $ 2.810$ \\
\textit{Isolet}  & $23.924$ & $23.914$ & $23.845$ & $23.830$ & $23.818$ \\
\textit{MNIST}   & $23.795$ & $36.031$ & $34.974$ & $26.449$ & $22.963$ \\
\textit{Sports}  & $12.145$ & $14.678$ & $ 8.780$ & $ 4.543$ & $ 4.062$ \\
\textit{Arcene}  & $21.883$ & $24.090$ & $22.639$ & $18.310$ & $19.433$ \\
\textit{RNA-Seq} & $12.858$ & $15.504$ & $18.314$ & $12.606$ & $12.297$ \\
\hline
\end{tabular}
\caption{Concentration ratio (see Equation \eqref{eq:conc_ratio}) 
obtained by $\CFOF$ on real data projected
on its $d_i$ first principal components for $\alpha=0.1$ and $\varrho = 0.01$.}
\label{table:RC_PCA}
\end{table}

Figure \ref{table:RC_PCA} reports the concentration ratio
as the dimensionality increases for both original and normalized data.
For the five values of the dimensions $d_1, d_2, d_3, d_4, d_{max}$, please
refer to the legends of Figures \ref{fig:dupont1} and \ref{fig:dupont2}.

The table shows that the trend of the concentration ratio
depends on the data at hand,
in some cases it is monotone increasing
and and some others it
reaches a maximum at some intermediate dimensionality.

Unexpectedly, normalizing raises the value of the concentration ratio.
Probably this can be explained by the fact that,
despite the kurtosis is increased,
correlations between variables are not lost,
and, hence, the final effect 
seems to be that of emphasizing scores of the most
deviating points.
We note that 
a secondary effect of normalization is to
prevent variables
from having too little effect on the distance values.
As pointed out in Section \ref{sect:cfofcdf}, Equation \eqref{eq:cfofcdf} holds
also in the case of independent non-identically distributed
random variables provided, that they have comparable central moments.
However,
since the shuffled original data 
is in good agreement with the theoretical prediction,
normalization do not seem central to achieve 
comparable moments for these datasets.

Table \ref{table:relconc_comp} reports the concentration ratios
obtained by $\CFOF$, ODIN, $\antiHub$ ($\aHub$, for short), $\aKNN$, $\LOF$, and $\iForest$
on the above datasets.
$\ODIN$ presents in different cases extremely small concentration ratios.
Indeed, from Equation \eqref{eq:nk}, the standard deviation 
of the top CFOF scores tend to $0$, while the median score tends to $\varrho$.
In any case,
$\CFOF$ has concentration ratios that are different orders of magnitude larger than
that of the other scores.
$\antiHub$ mitigates the concentration problems of
$\ODIN$.
As for the concentration ratios of
distance-based and density-based methods, usually they
get larger on normalized data with increased kurtosis.
For example, on the normalized \textit{MNIST} dataset,
having a very large kurtosis,
$\aKNN$ and $\LOF$ show the greatest concentration ratio.
Indeed, according the analysis of Section \ref{sect:conc_other},
concentration of these scores can be avoided only for data having
infinite kurtosis.
$\iForest$ has  very small concentration ratios on normalized
data, even for large kurtosis values. The same holds for the original data, 
except that now the ratio appears to benefit of a larger kurtosis.

\begin{table}[t]
\centering
\begin{tabular}{|l|r||r|r|r|r|r|r|}
\multicolumn{8}{c}{original data, $\alpha = 0.1$, $\varrho = 0.01$}\\
\hline
 & $\kappa_{orig}$ & $\CFOF$ & $\ODIN$ & $\aHub$ & $\aKNN$ & $\LOF$ & $\iForest$ \\
\hline\hline
\textit{SIFT}     &  $3.0$ & $2.213$ & $0.071$ & $0.264$ & $0.052$ & $0.038$ & $0.015$ \\ 
\textit{Isolet}   &  $3.2$ & $3.263$ & $<10^{-6}$ & $0.050$ & $0.055$ & $0.055$ & $0.011$ \\ 
\textit{MNIST}    &  $2.2$ & $1.747$ & $0.085$ & $0.168$ & $0.056$ & $0.141$ & $0.009$ \\ 
\textit{Sports}   & $10.4$ & $2.732$ & $0.013$ & $0.368$ & $0.411$ & $1.492$ & $0.154$ \\ 
\textit{Arcene}   &  $2.2$ & $2.900$ & $0.098$ & $0.163$ & $0.169$ & $0.064$ & $0.008$ \\ 
\textit{RNA-Seq}  &  $3.3$ & $4.834$ & $<10^{-6}$ & $0.548$ & $0.107$ & $0.093$ & $0.011$ \\ 
\hline
\multicolumn{7}{c}{normalized data, $\alpha$ = $0.1$, $\varrho$ = $0.01$}\\
\hline
  & $\kappa_{norm}$ & $\CFOF$ & $\ODIN$ & $\aHub$ & $\aKNN$ & $\LOF$ & $\iForest$ \\
\hline\hline
\textit{SIFT}        &   $5.7$ &  $2.810$ & $0.066$ & $0.285$ & $0.054$ & $0.036$ & $0.012$ \\ 
\textit{Isolet}      &  $10.0$ & $23.817$ & $<10^{-6}$ & $0.164$ & $0.270$ & $0.073$ & $0.016$ \\ 
\textit{MNIST}       & $895.8$ & $22.963$ & $0.077$ & $0.389$ & $1.178$ & $1.030$ & $0.009$ \\ 
\textit{Sports}      &  $14.5$ & $4.0617$ & $0.005$ & $0.186$ & $0.318$ & $0.807$ & $0.013$ \\ 
\textit{Arcene}      &  $24.8$ & $19.433$ & $<10^{-6}$ & $0.457$ & $0.284$ & $0.143$ & $0.009$ \\ 
\textit{RNA-Seq}     &  $14.7$ & $12.297$ & $<10^{-6}$ & $0.514$ & $0.141$ & $0.153$ & $0.011$ \\ 
\hline
\end{tabular}
\caption{Concentration ratio (see Equation \eqref{eq:conc_ratio}) 
obtained by different outlier
detection definitions on real data 
in the full feature space
for $\alpha=0.1$ and $\varrho=0.01$.}
\label{table:relconc_comp}
\end{table}

\subsection{Comparison on synthetic data}
\label{sect:exp_synth}

This section concerns the 
scenario in which all the coordinates are independent.

With this aim,
we generated multivariate data
having the following characteristics:
($i$) independent identically distributed (i.i.d.) coordinates, ($ii$) independent
non-identically distributed (i.non-i.d.) coordinates, and ($iii$)
normal multivariate with non-diagonal covariance matrix.\footnote{
In the last case, data consists of 
$d$-dimensional points coming from a multivariate
normal distribution with covariance matrix $\Sigma = S^\top S$,
where $S$ is a matrix of standard normally distributed random values.}
Since we obtained very similar results for all the above kinds of data,
for the sake of shortness, 
in the following we report results concerning i.i.d. data.

\begin{table}[t]
\centering
\scriptsize
\begin{tabular}{c}
\textit{Unimodal} dataset \\
\begin{tabular}{|l||c|c|c|c||c|c|c|c|}
\cline{2-9}
\multicolumn{1}{c}{} & \multicolumn{4}{|c||}{$AUC_{mean}$} & \multicolumn{4}{|c|}{$AUC_{max}$} \\
\hline 
\multicolumn{1}{|r||}{$Method$ / $d$} & $10$ & $100$ & $1000$ & $10000$ & $10$ & $100$ & $1000$ & $10000$ \\
\hline\hline
$\CFOF$ & $\it 0.9886$ & $\bf 0.9945$ & $\bf 0.9957$ & $\bf 0.9962$ & $\bf 0.9999$ & $\bf 0.9999$ & $\bf 0.9998$ & $\bf 0.9999$ \\
$\ODIN$ & $0.9523$ & $0.8777$ & $0.8331$ & $0.8194$ & $0.9995$ & $0.9997$ & $\it 0.9997$ & $0.9996$ \\
$\antiHub$ & $0.9361$ & $0.7802$ & $0.6560$ & $0.6261$ & $0.9868$ & $0.9639$ & $0.9154$ & $0.8646$ \\
$\aKNN$ & $\bf 0.9939$ & $\it 0.9934$ & $\it 0.9935$ & $\it 0.9925$ & $\it 0.9998$ & $\bf 0.9999$ & $\bf 0.9998$ & $\it 0.9998$ \\
$\LOF$ & $0.9725$ & $0.9850$ & $0.9877$ & $0.9879$ & $\bf 0.9999$ & $\bf 0.9999$ & $\bf 0.9998$ & $\it 0.9998$ \\
$\FastABOD$ & $0.9833$ & $0.9317$ & $0.8466$ & $0.7161$ & $0.9988$ & $0.9990$ & $0.9937$ & $0.9485$ \\
$\iForest$ & --- & --- & --- & --- & $0.9790$ & $0.9045$ & $0.6925$ & $0.5781$ \\
\hline
\end{tabular}
\\
\begin{tabular}{|l||c|c|c|c||c|c|c|c|}
\cline{2-9}
\multicolumn{1}{c}{} & \multicolumn{4}{|c||}{$Prec_{mean}$} & \multicolumn{4}{|c|}{$Prec_{max}$} \\
\hline 
\multicolumn{1}{|r||}{$Method$ / $d$} & $10$ & $100$ & $1000$ & $10000$ & $10$ & $100$ & $1000$ & $10000$ \\
\hline\hline
$\CFOF$ & $\it 0.8321$ & $\bf 0.8758$ & $\bf 0.8840$ & $\bf 0.8780$ & $\bf 0.9760$ & $\bf 0.9790$ & $\bf 0.9673$ & $\bf 0.9730$ \\
$\ODIN$ & $0.6292$ & $0.5387$ & $0.4891$ & $0.4750$ & $0.9380$ & $0.9560$ & $0.9533$ & $0.9500$ \\
$\antiHub$ & $0.5080$ & $0.2602$ & $0.1271$ & $0.1038$ & $0.7120$ & $0.5780$ & $0.4000$ & $0.3180$ \\
$\aKNN$ & $\bf 0.8582$ & $\it 0.8606$ & $\it 0.8572$ & $\it 0.8424$ & $\it 0.9700$ & $\it 0.9740$ & $\it 0.9680$ & $0.9660$ \\
$\LOF$ & $0.7614$ & $0.8171$ & $0.8283$ & $0.8195$ & $\it 0.9700$ & $0.9720$ & $0.9620$ & $\it 0.9680$ \\
$\FastABOD$ & $0.7750$ & $0.6119$ & $0.3941$ & $0.1907$ & $0.9080$ & $0.9200$ & $0.7980$ & $0.5140$ \\
$\iForest$ & --- & --- & --- & --- & $0.7160$ & $0.4140$ & $0.1480$ & $0.0800$ \\
\hline
\end{tabular}
~\\
\textit{Multimodal} dataset \\
\begin{tabular}{|l||c|c|c|c||c|c|c|c|}
\cline{2-9}
\multicolumn{1}{c}{} & \multicolumn{4}{|c||}{$AUC_{mean}$} & \multicolumn{4}{|c|}{$AUC_{max}$} \\
\hline 
\multicolumn{1}{|r||}{$Method$ / $d$} & $10$ & $100$ & $1000$ & $10000$ & $10$ & $100$ & $1000$ & $10000$ \\
\hline\hline
$\CFOF$ & $\bf 0.9730$ & $\bf 0.9851$ & $\bf 0.9837$ & $\bf 0.9825$ & $\bf 0.9988$ & $\bf 0.9989$ & $\bf 0.9989$ & $\bf 0.9989$ \\
$\ODIN$ & $0.9317$ & $0.8346$ & $0.8089$ & $0.7972$ & $0.9987$ & $\it 0.9988$ & $\it 0.9987$ & $\bf 0.9989$ \\
$\antiHub$ & $0.8981$ & $0.7218$ & $0.6192$ & $0.5802$ & $0.9727$ & $0.9649$ & $0.8931$ & $\it 0.8284$ \\
$\aKNN$ & $0.7582$ & $0.7584$ & $0.7584$ & $0.7582$ & $0.7614$ & $0.7622$ & $0.7621$ & $0.7621$ \\
$\LOF$ & $\it 0.9512$ & $\it 0.9605$ & $\it 0.9642$ & $\it 0.9632$ & $\it 0.9961$ & $\it 0.9988$ & $\it 0.9987$ & $\bf 0.9989$ \\
$\FastABOD$ & $0.7513$ & $0.7286$ & $0.6785$ & $0.6204$ & $0.7588$ & $0.7619$ & $0.7620$ & $0.7621$ \\
$\iForest$ & --- & --- & --- & --- & $0.7457$ & $0.6889$ & $0.6005$ & $0.5461$ \\
\hline
\end{tabular}
\\
\begin{tabular}{|l||c|c|c|c||c|c|c|c|}
\cline{2-9}
\multicolumn{1}{c}{} & \multicolumn{4}{|c||}{$Prec_{mean}$} & \multicolumn{4}{|c|}{$Prec_{max}$} \\
\hline 
\multicolumn{1}{|r||}{$Method$ / $d$} & $10$ & $100$ & $1000$ & $10000$ & $10$ & $100$ & $1000$ & $10000$ \\
\hline\hline
$\CFOF$ & $\bf 0.7566$ & $\bf 0.7979$ & $\bf 0.8125$ & $\bf 0.8047$ & $\bf 0.9097$ & $\bf 0.9130$ & $\bf 0.9229$ & $\bf 0.9183$ \\
$\ODIN$ & $0.5891$ & $0.4919$ & $0.4725$ & $0.4576$ & $\it 0.8930$ & $\it 0.9113$ & $0.9130$ & $0.9128$ \\
$\antiHub$ & $0.4358$ & $0.2154$ & $0.1334$ & $0.1066$ & $0.6220$ & $0.6580$ & $0.5000$ & $0.4040$ \\
$\aKNN$ & $0.4814$ & $0.4819$ & $0.4843$ & $0.4858$ & $0.4980$ & $0.5000$ & $0.5000$ & $0.5000$ \\
$\LOF$ & $\it 0.6715$ & $\it 0.7320$ & $\it 0.7530$ & $\it 0.7473$ & $0.8540$ & $0.9060$ & $\it 0.9140$ & $\it 0.9160$ \\
$\FastABOD$ & $0.4490$ & $0.3847$ & $0.2722$ & $0.1656$ & $0.4860$ & $0.5000$ & $0.5000$ & $0.5000$ \\
$\iForest$ & --- & --- & --- & --- & $0.4100$ & $0.2920$ & $0.1140$ & $0.0620$ \\
\hline
\end{tabular}
~\\
\textit{Multimodal artificial} dataset
\\
\begin{tabular}{|l||c|c|c|c||c|c|c|c|}
\cline{2-9}
\multicolumn{1}{c}{} & \multicolumn{4}{|c||}{$AUC_{mean}$} & \multicolumn{4}{|c|}{$AUC_{max}$} \\
\hline 
\multicolumn{1}{|r||}{$Method$ / $d$} & $10$ & $100$ & $1000$ & $10000$ & $10$ & $100$ & $1000$ & $10000$ \\
\hline\hline
$\CFOF$ & $\bf 0.9834$ & $\bf 0.9999$ & $\bf 1.0000$ & $\bf 1.0000$ & $\bf 1.0000$ & $\bf 1.0000$ & $\bf 1.0000$ & $\bf 1.0000$ \\
$\ODIN$ & $0.9408$ & $0.8631$ & $0.8147$ & $0.8048$ & $\bf 1.0000$ & $\bf 1.0000$ & $\bf 1.0000$ & $\bf 1.0000$ \\
antiHub2 & $0.9365$ & $0.8425$ & $0.7554$ & $0.7401$ & $\it 0.9970$ & $\it 0.9985$ & $\it 0.9956$ & $\it 0.9788$ \\
$\aKNN$ & $0.7624$ & $0.7625$ & $0.7625$ & $0.7625$ & $0.7625$ & $0.7625$ & $0.7625$ & $0.7625$ \\
$\LOF$ & $\it 0.9719$ & $\it 0.9773$ & $\it 0.9777$ & $\it 0.9779$ & $\bf 1.0000$ & $\bf 1.0000$ & $\bf 1.0000$ & $\bf 1.0000$ \\
$\FastABOD$ & $0.7604$ & $0.7548$ & $0.7495$ & $0.7449$ & $0.7624$ & $0.7625$ & $0.7625$ & $0.7625$ \\
$\iForest$ & --- & --- & --- & --- & $0.7603$ & $0.7575$ & $0.7450$ & $0.7368$ \\
\hline
\end{tabular}
\\
\begin{tabular}{|l||c|c|c|c||c|c|c|c|}
\cline{2-9}
\multicolumn{1}{c}{} & \multicolumn{4}{|c||}{$Prec_{mean}$} & \multicolumn{4}{|c|}{$Prec_{max}$} \\
\hline 
\multicolumn{1}{|r||}{$Method$ / $d$} & $10$ & $100$ & $1000$ & $10000$ & $10$ & $100$ & $1000$ & $10000$ \\
\hline\hline
$\CFOF$ & $\bf 0.9324$ & $\bf 0.9991$ & $\bf 1.0000$ & $\bf 1.0000$ & $\bf 1.0000$ & $\bf 1.0000$ & $\bf 1.0000$ & $\bf 1.0000$ \\
$\ODIN$ & $0.7349$ & $0.6012$ & $0.5564$ & $0.5434$ & $\bf 1.0000$ & $\bf 1.0000$ & $\bf 1.0000$ & $\bf 1.0000$ \\
antiHub2 & $0.6317$ & $0.4255$ & $0.3346$ & $0.3028$ & $\it 0.8880$ & $\it 0.9560$ & $\it 0.9120$ & $\it 0.8080$ \\
$\aKNN$ & $0.4996$ & $0.5000$ & $0.5000$ & $0.5000$ & $0.5000$ & $0.5000$ & $0.5000$ & $0.5000$ \\
$\LOF$ & $\it 0.8743$ & $\it 0.9463$ & $\it 0.9521$ & $\it 0.9521$ & $\bf 1.0000$ & $\bf 1.0000$ & $\bf 1.0000$ & $\bf 1.0000$ \\
$\FastABOD$ & $0.4874$ & $0.4660$ & $0.4506$ & $0.4409$ & $0.5000$ & $0.5000$ & $0.5000$ & $0.5000$ \\
$\iForest$ & --- & --- & --- & --- & $0.4880$ & $0.4940$ & $0.4680$ & $0.4540$ \\
\hline
\end{tabular}
\end{tabular}
\caption{\textit{AUC} and \text{Precision} for the synthetic datasets.}
\label{table:synthetic_data}
\end{table}

We considered three dataset families.
The first family, called \textit{Unimodal}, consists
of multivariate i.i.d. standard normal data.

The two other dataset families 
are described in \cite{RadovanovicNI15}
and concern the scenario involving clusters of different
densities. Each dataset of this family is composed of 
points partitioned into two equally-sized clusters.
The first (second, resp.) cluster consists of points coming from a $d$-dimensional
normal distribution with independent components having mean $-1$ ($+1$, resp.) and 
standard deviation $0.1$ ($1$, resp.).
Moreover, a variant of each dataset containing artificial outliers
is obtained by moving the $\alpha$ fraction of the points maximizing the distance from their
cluster center 
even farther from
the center by $20\%$ of the distance.
We refer to the family of datasets above described with (without, resp.) 
artificial outliers as to \textit{Multimodal artificial}
(\textit{Multimodal}, resp.).\footnote{We note that 
\cite{RadovanovicNI15} experimented
only the \textit{Multimodal artificial} family.}

We considered datasets of $n=1,\!000$ points and
varied the dimensionality $d$ from $10$ to $10^4$.
Points maximizing the distance from their cluster center were 
marked as outliers, and the fraction of outliers was set to $\alpha=0.05$.
For each dataset we considered $N=20$ log-spaced values for
the parameter $k$ ranging from $2$ to $n/2$.
Results are averaged over ten runs.

\begin{figure}[t]
\centering
\subfloat[\textit{Unimodal}]
{\includegraphics[width=0.33\columnwidth]{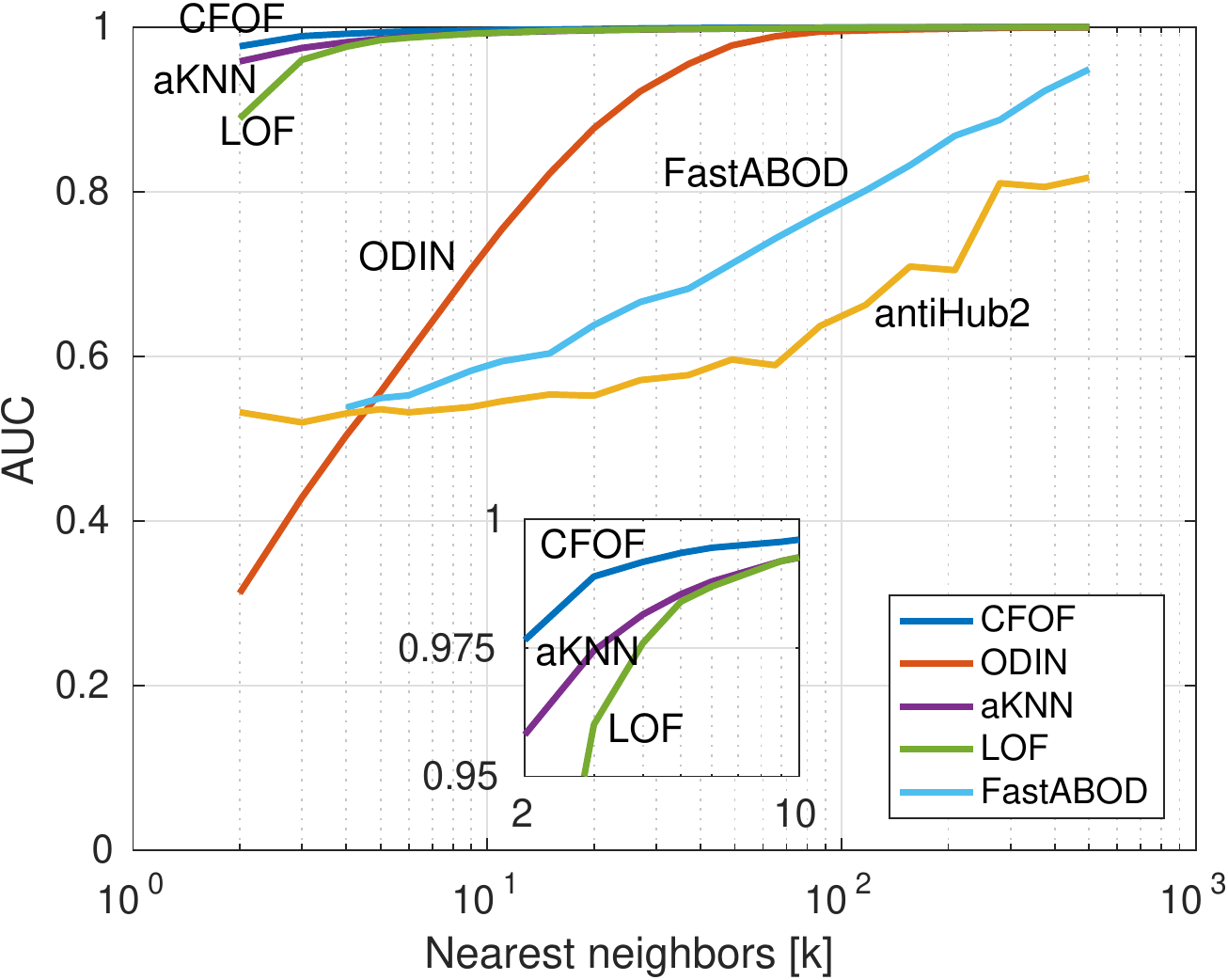}}
~
\subfloat[\textit{Multimodal}]
{\includegraphics[width=0.33\columnwidth]{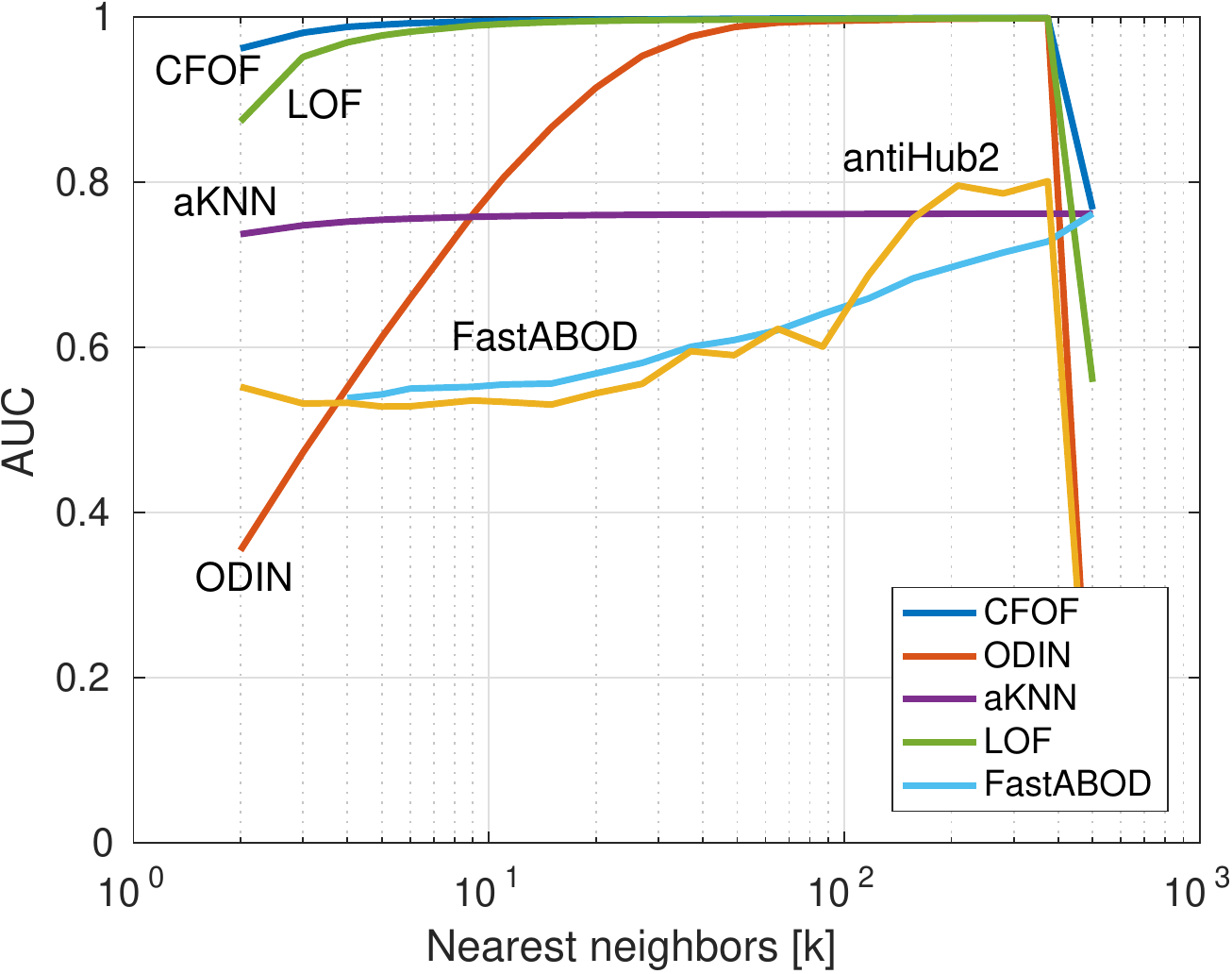}}
~
\subfloat[\label{fig:multimodalart_plot1}\textit{Multimodal artificial}]
{\includegraphics[width=0.33\columnwidth]{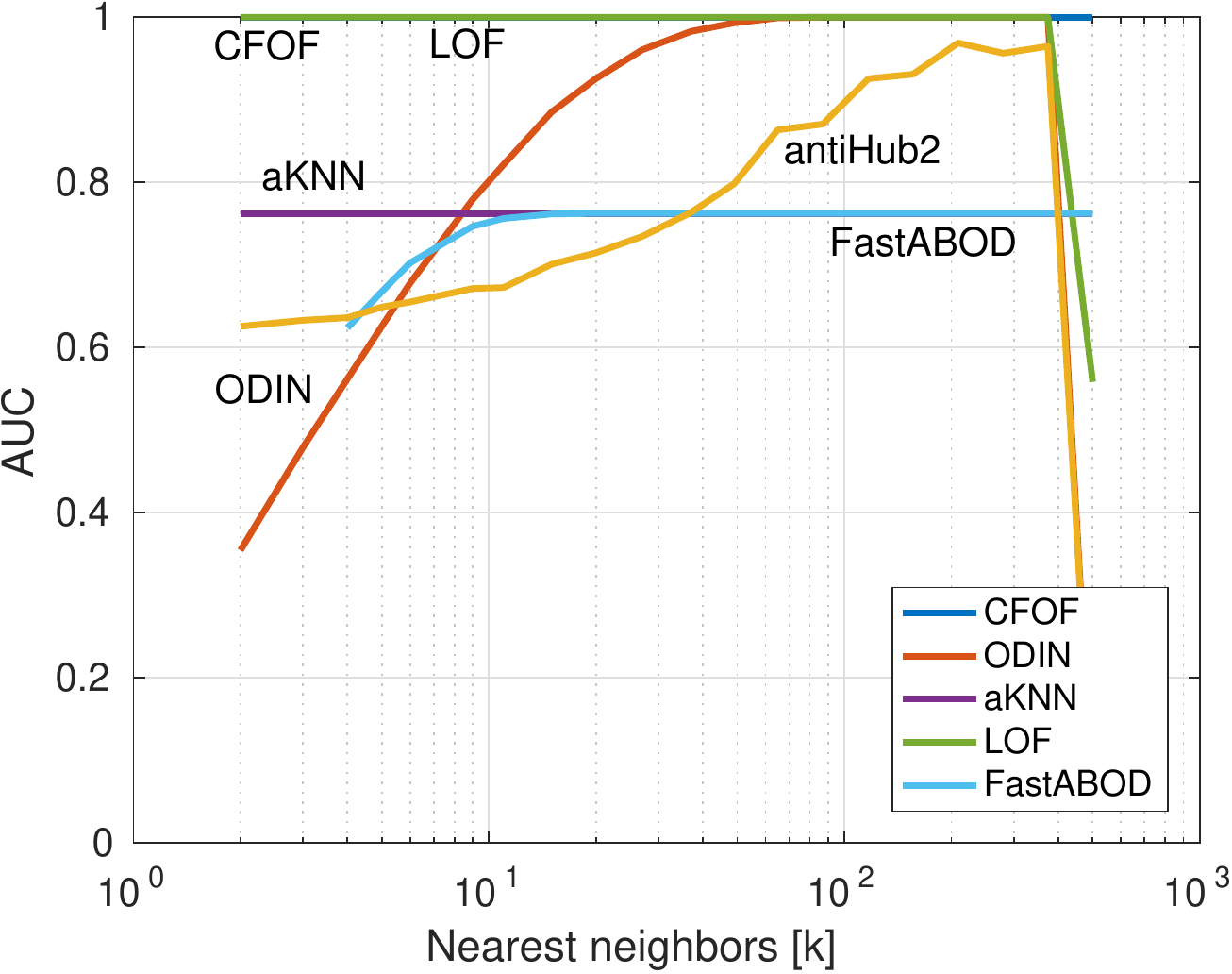}}
\\
\subfloat[\textit{Unimodal}]
{\includegraphics[width=0.33\columnwidth]{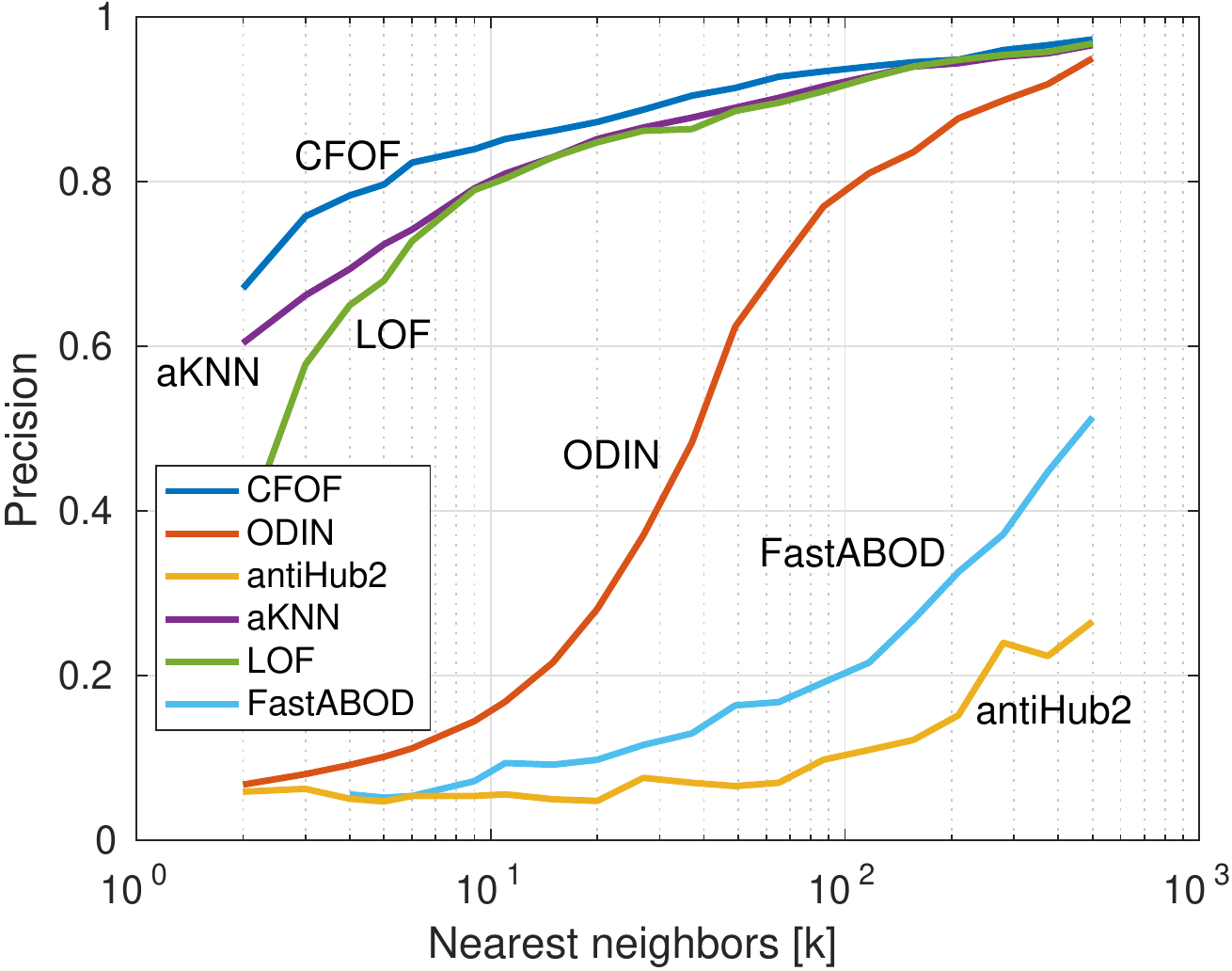}} 
~
\subfloat[\textit{Multimodal}]
{\includegraphics[width=0.33\columnwidth]{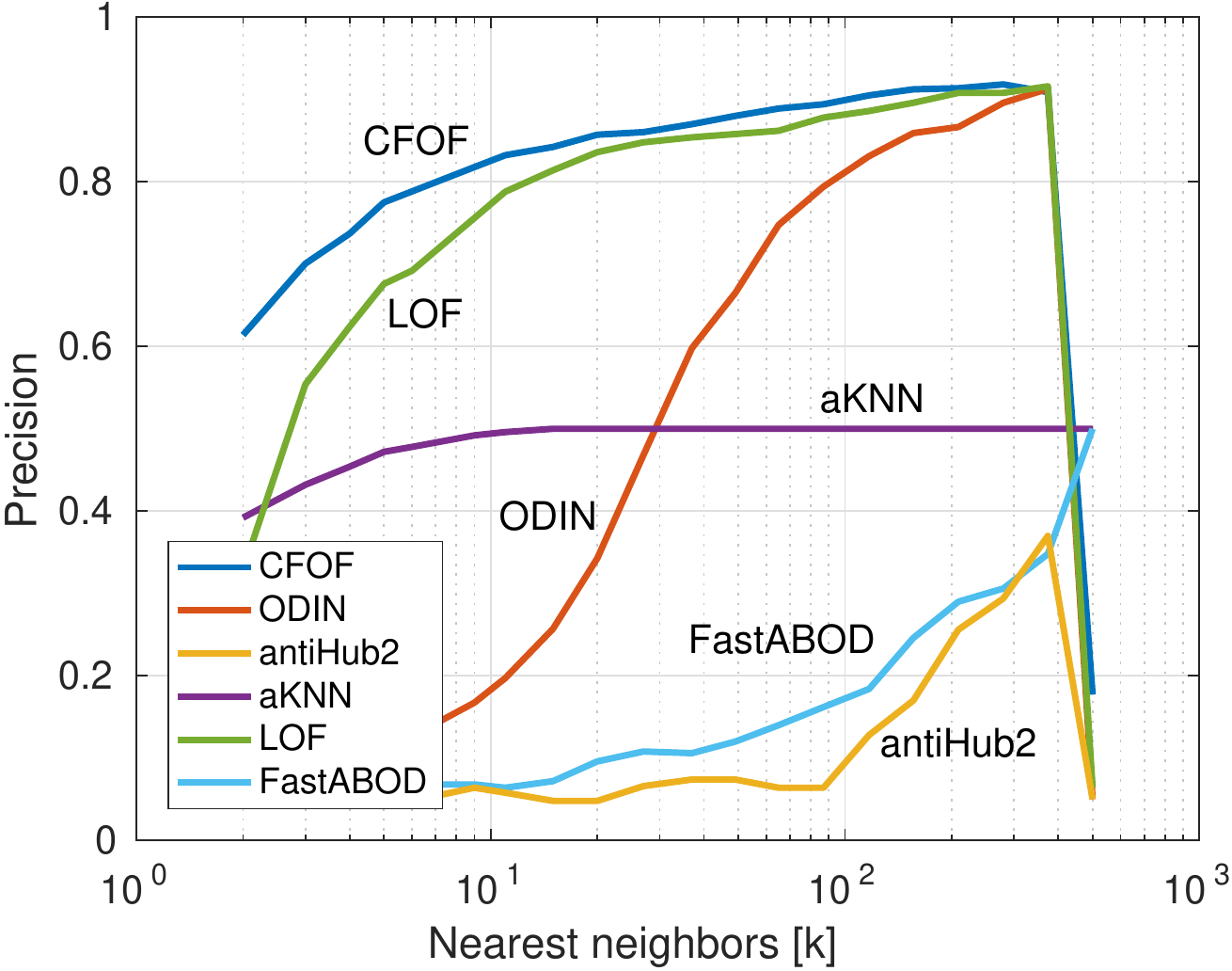}} 
~
\subfloat[\label{fig:multimodalart_plot2}\textit{Multimodal artificial}]
{\includegraphics[width=0.33\columnwidth]{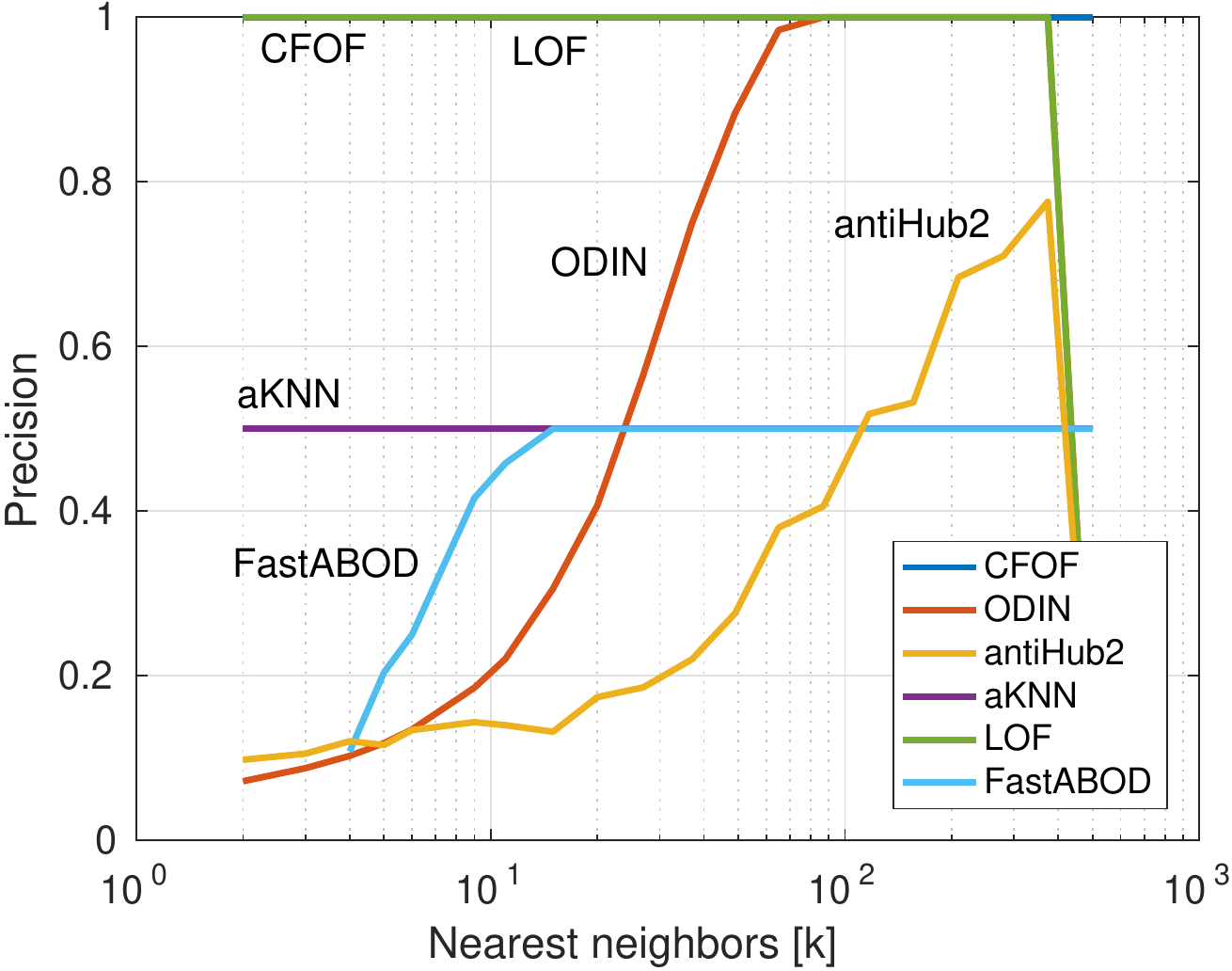}} 
\caption{[Best viewed in color.] $AUC$ and \textit{Precision} for $d=10,\!000$.}
\label{fig:aucprec}
\end{figure}

\begin{figure}[t]
\centering
\subfloat[\textit{Unimodal}]
{\includegraphics[width=0.33\columnwidth]{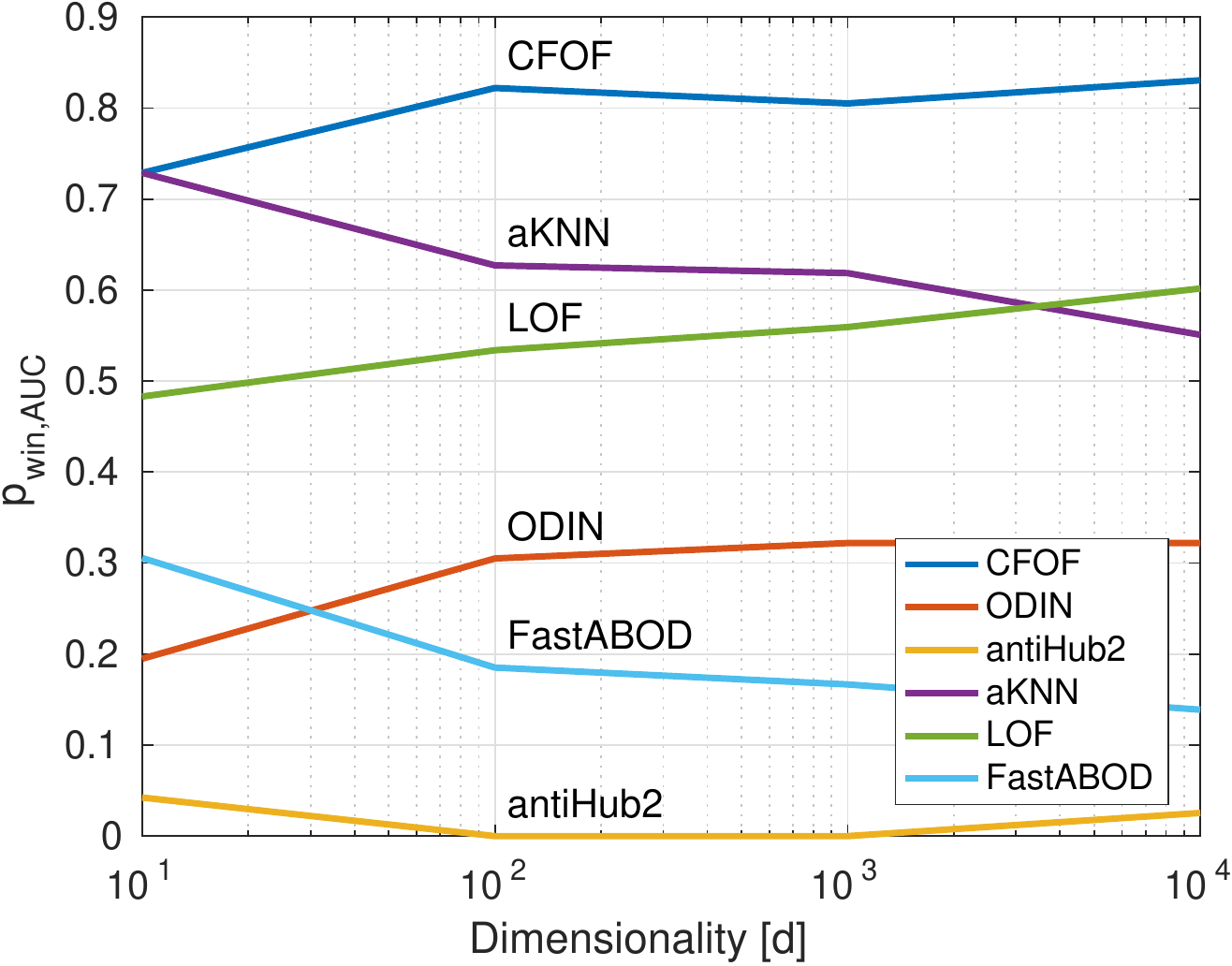}}
~
\subfloat[\textit{Multimodal}]
{\includegraphics[width=0.33\columnwidth]{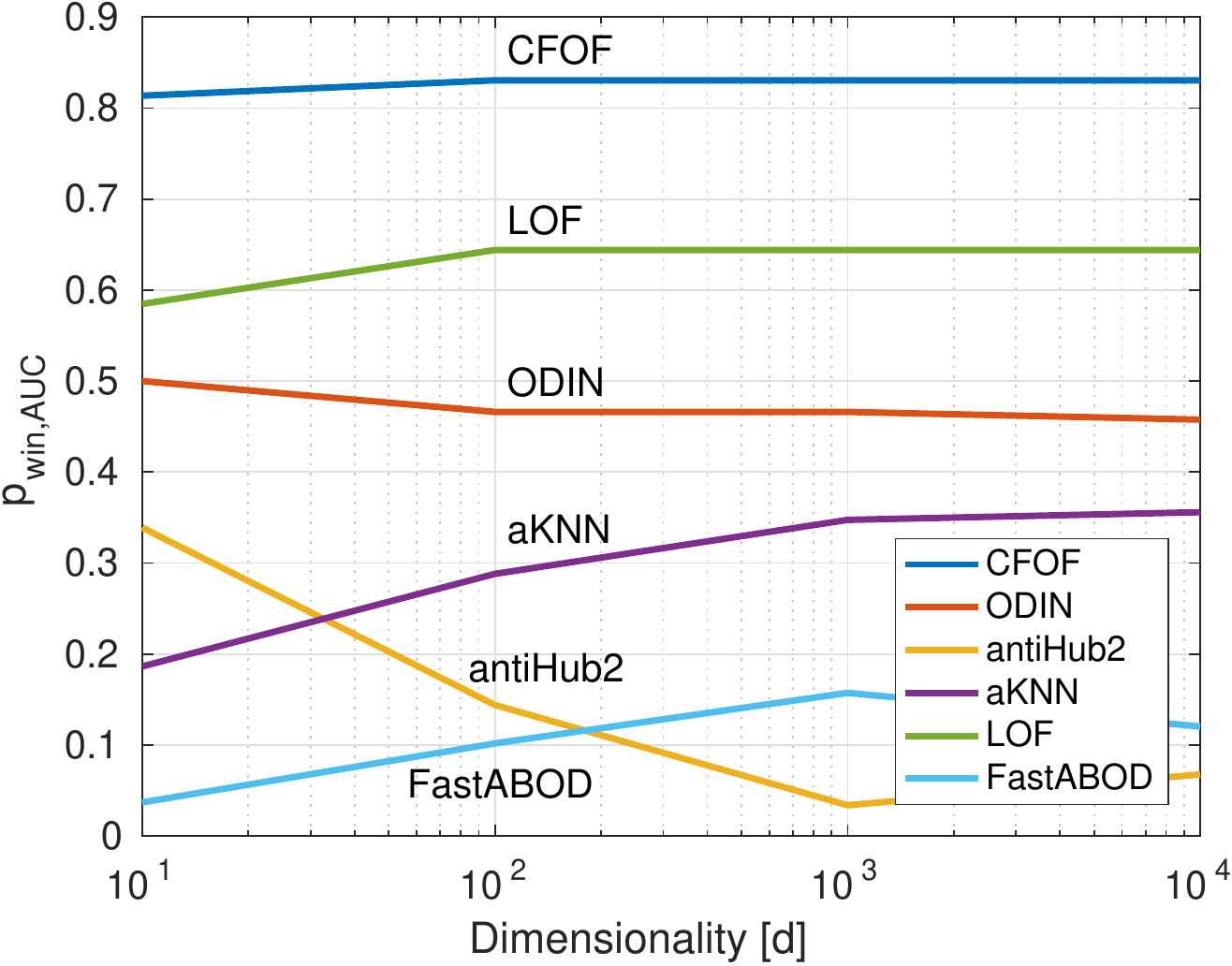}}
~
\subfloat[\label{fig:multimodalart_plot3}\textit{Multimodal}]
{\includegraphics[width=0.33\columnwidth]{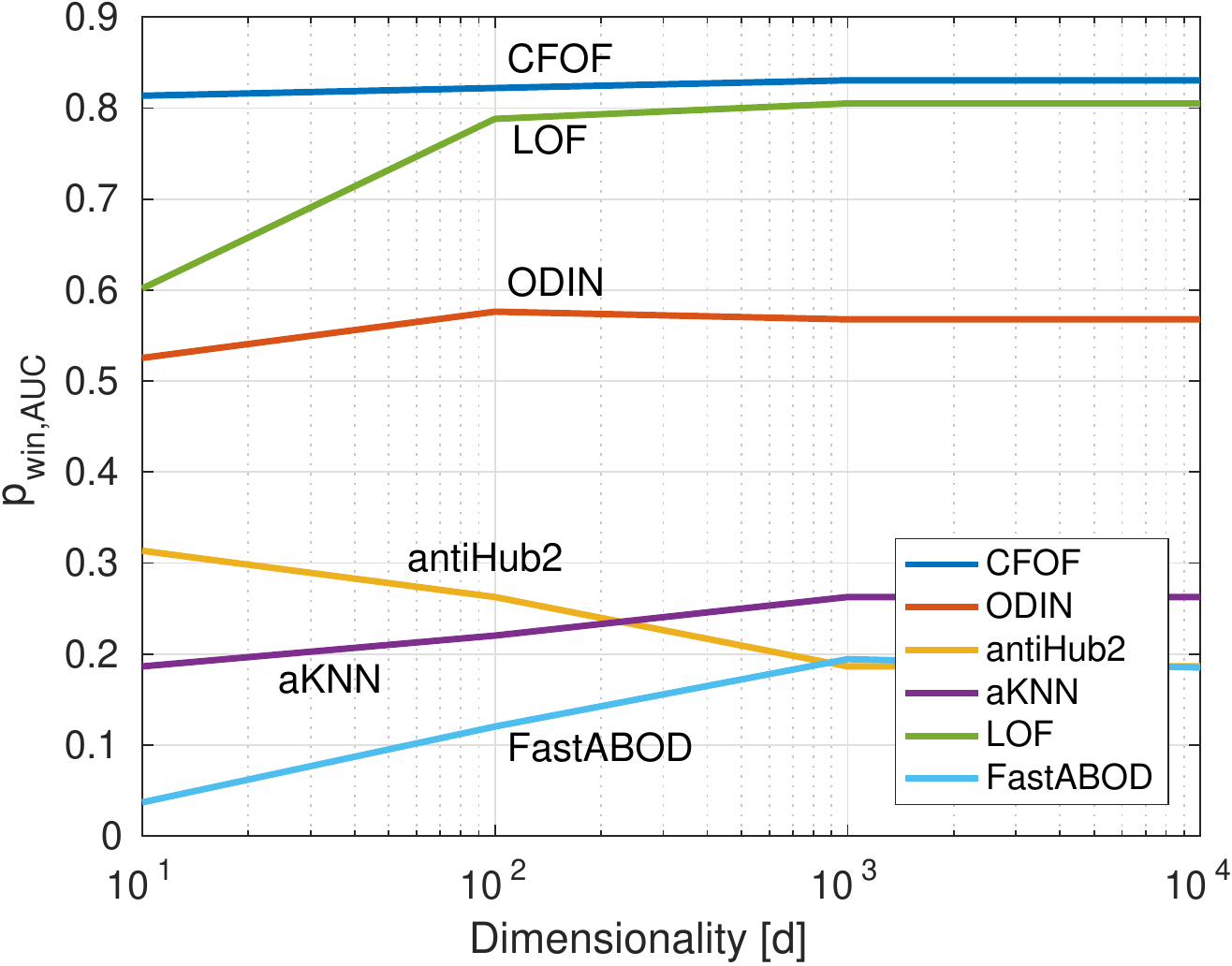}}
\\
\subfloat[\textit{Unimodal}]
{\includegraphics[width=0.33\columnwidth]{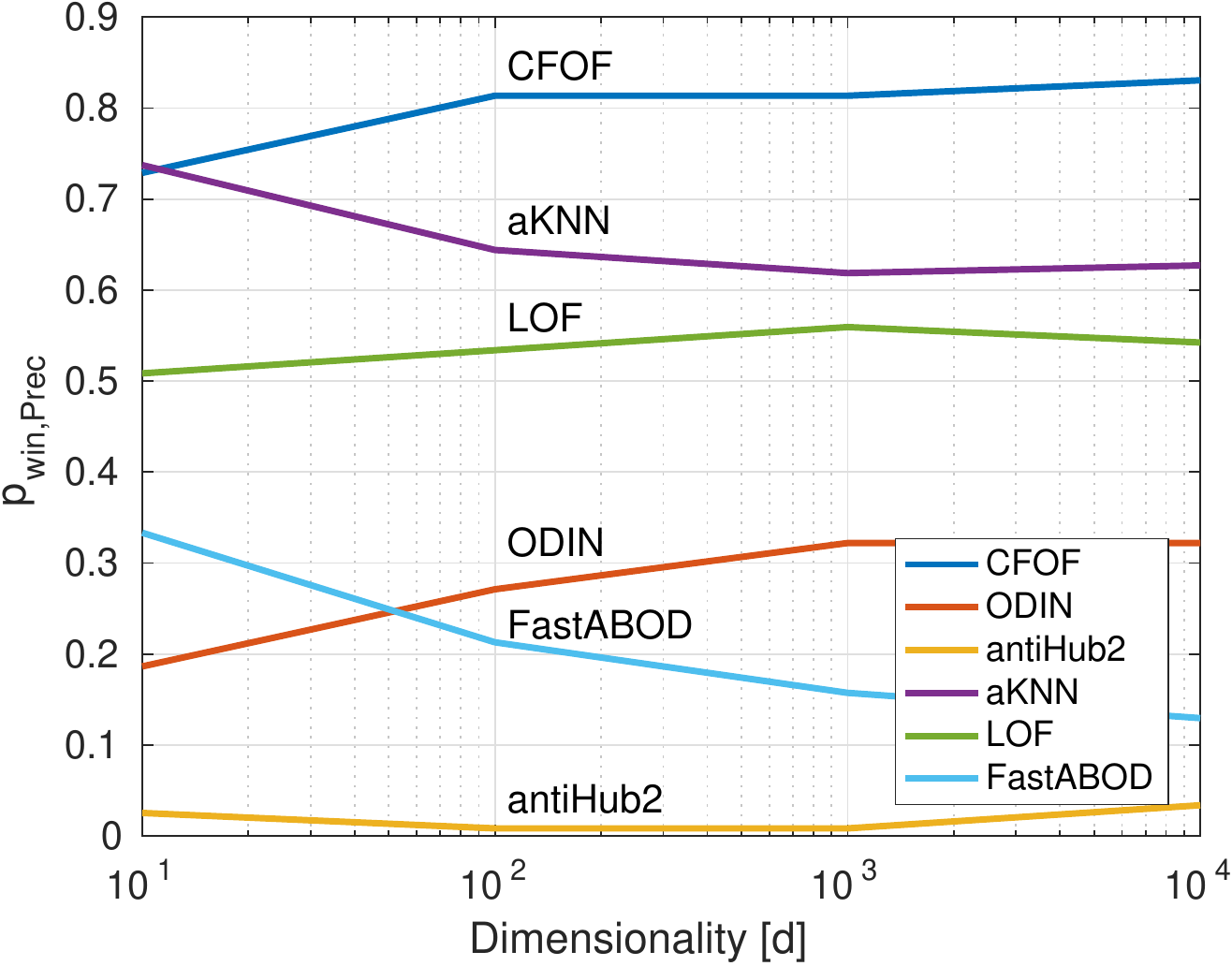}} 
~
\subfloat[\textit{Multimodal}]
{\includegraphics[width=0.33\columnwidth]{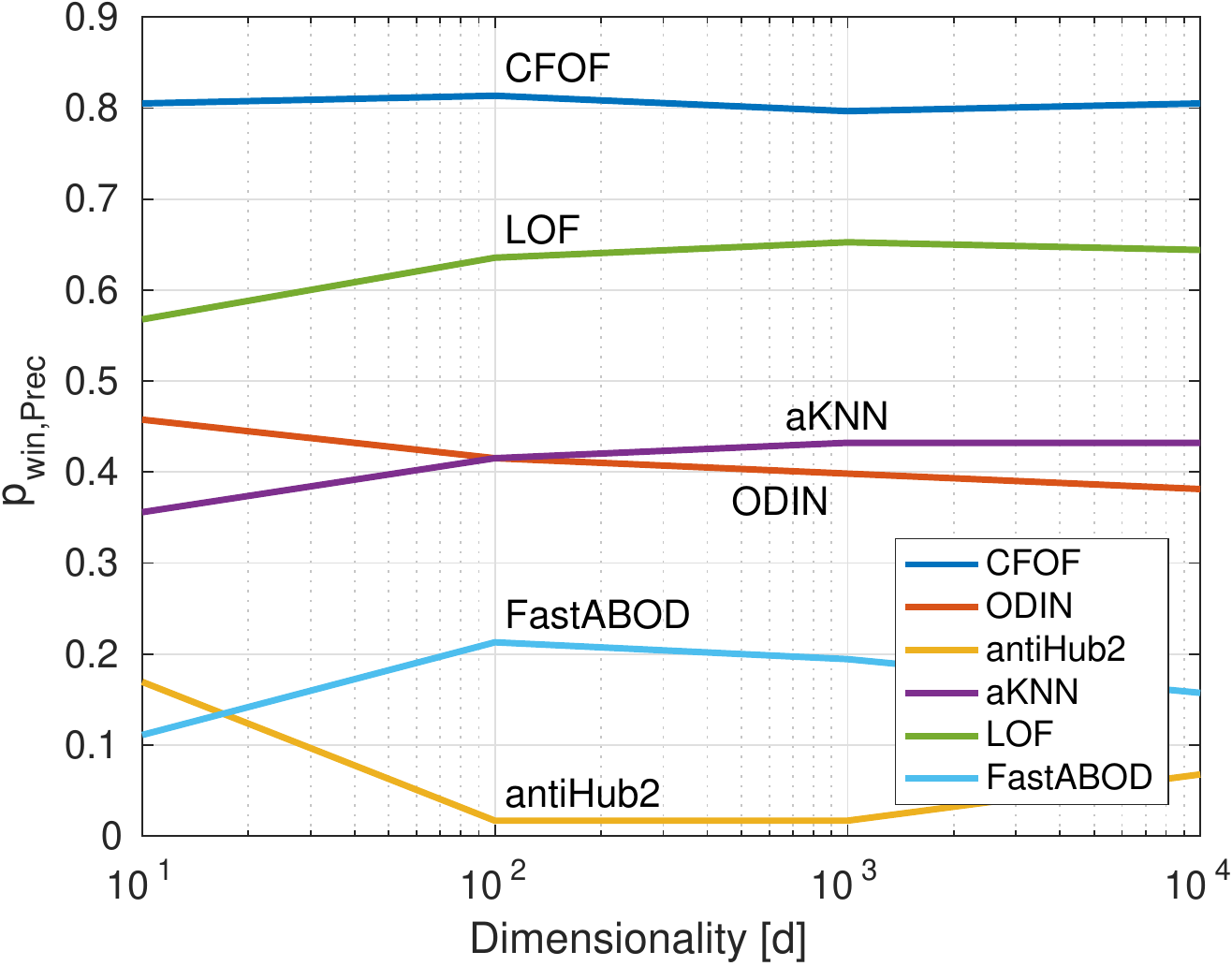}} 
~
\subfloat[\label{fig:multimodalart_plot4}\textit{Multimodal artificial}]
{\includegraphics[width=0.33\columnwidth]{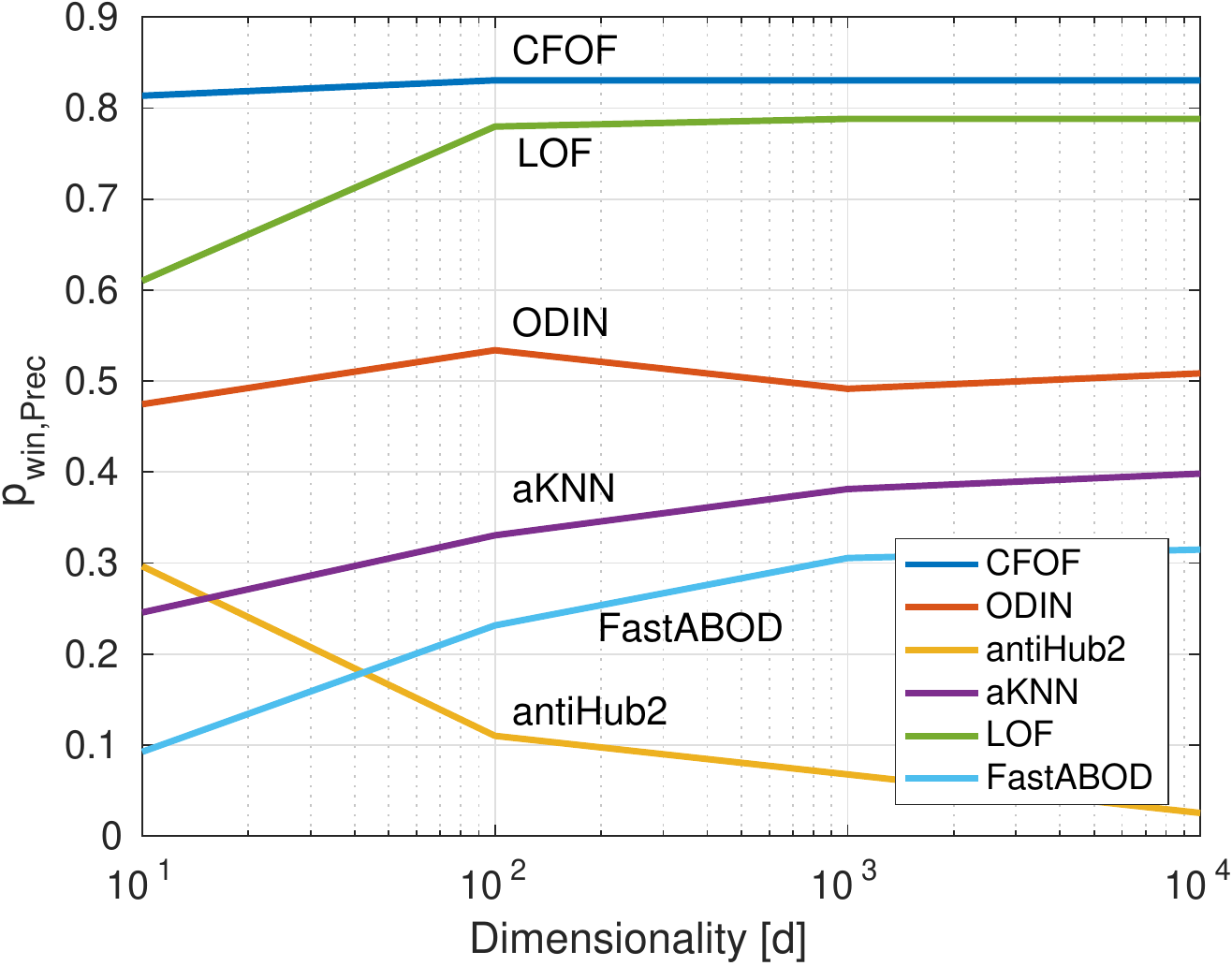}} 
\caption{[Best viewed in color.] $p_{win,AUC}$ and $p_{win,Prec}$ for $d=10,\!000$.}
\label{fig:praucprec}
\end{figure}

\medskip
To measure accuracy,
we used the AUC (Area Under the ROC Curve)
and the Precision, or \textit{Prec} for short.
The ROC curve is the plot of the 
true positive rate (TPR) against 
the false positive rate (FPR) at various threshold settings.
The AUC is defined as the are under the ROC curve.

We also measured the 
$AUC_{max}$ and $AUC_{mean}$,
defined as, respectively, the maximum and the average AUC 
associated with the different values of $k$ employed.
Similarly, we measured the 
$Prec_{max}$ and $Prec_{mean}$.
Within the unsupervised outlier detection
scenario average performances 
are of interest 
due to the unavailability of label information that could be exploited to
determine the optimal values for the parameters.

Moreover, we defined $\pwina{Def_1,Def_2}$ as the fraction of configurations
of the parameter $k$
for which the outlier definition $Def_1$
performs better than the outlier definition $Def_2$
in terms of AUC. By $\pwina{Def_1}$ we denote 
the fraction of configurations
of the parameter $k$
for which the outlier definition $Def_1$
performs better than a generic outlier definition $Def_2$.
The measures $\pwinp{Def_1,Def_2}$ and $\pwinp{Def_1}$
are defined in a similar manner.

\medskip
Table \ref{table:synthetic_data} reports
the $AUC_{mean}$, $AUC_{max}$, $Prec_{mean}$, and $Prec_{max}$
obtained by $\CFOF$, $\ODIN$, $\antiHub$, $\aKNN$, $\LOF$, and $\FastABOD$ 
on the \textit{Unimodal}, \textit{Multimodal}, and \textit{Multimodal artificial}
dataset families.
The table reports also the $AUC_{max}$ an $Prec_{max}$
values obtained by $\iForest$.
We note that $\iForest$ has no neighborhood parameter and, hence, the other
measure are not applicable, since they depend on 
the whole range of values for the parameter $k$ above described.

As for the $AUC_{mean}$ and the $Prec_{mean}$, 
except for the low dimensional \textit{Unimodal} dataset,
$\CFOF$ always obtains the best values of $AUC_{mean}$ and $Prec_{mean}$.
The value associated with $\CFOF$ is most often larger than 
that associated with the other methods, 
thus suggesting that on distribution data
$\CFOF$ is less sensitive to variations of the parameters
than the other methods.
As for the $AUC_{max}$, $\CFOF$ is able to obtain the best values
in all cases. Depending on the characteristics of the dataset,
also other methods are able to reach similar values.
As for the $Prec_{max}$, except for 
\textit{Multimodal artificial} which contains artificially
well-separated outliers, the values obtained by $\CFOF$ 
are not reached by any other method.

The above conclusions are valid for all the three dataset families.
Now we comment on the single families.
As for the \textit{Unimodal} dataset, it contains global outliers
and almost all the techniques considered show good performances.
However, $\FastABOD$ and, specifically, $\antiHub$ sensibly worsen
their quality when the dimensionality increases, while
$\iForest$ tends to the random classifier ($AUC$ close to $0.5$) with a 
negligible precision.

As for the \textit{Multimodal} and \textit{Multimodal artificial} dataset, 
they contain local outliers
and, hence, reverse nearest neighbor-based 
and local techniques should perform 
better than global techniques.
{The behavior of $\FastABOD$ 
can be explained by noticing that this technique
is similar to global techniques, 
since the score of a point depends also on its $k$ nearest neighbors.}
Indeed, the value $Prec_{max}=0.5$ obtained by $\aKNN$ and $\FastABOD$
indicates these methods detect only the outliers coming 
from the sparsest cluster.
Moreover,
the $\antiHub$ method needs well-separated artificial outliers 
to ameliorate its performances. 
As for $\iForest$, its behavior on the \textit{Multimodal} dataset
is similar to the \textit{Unimodal} case ($AUC$ close to $0.5$ and negligible precision). 
Moreover, even in the presence of 
artificially separated outliers, its accuracy remains lower than that
of density-based and angle-based methods.

\medskip
Figures \ref{fig:aucprec}, \ref{fig:praucprec}
show the $AUC$ (on the top left) and the \textit{Precision}
of the techniques for different values of the number of nearest neighbors $k$ 
expressed as a fraction of the dataset size,
together with the $p_{win,AUC}$ and the $p_{win,Prec}$.
The behavior of $\antiHub$ 
reported in Figure \ref{fig:multimodalart_plot1}
for $d=10,\!000$
follows that shown in \cite{RadovanovicNI15}
(cf. Fig. 8) for $d=100$,
where the method reaches its best AUC values
for $k$ ranging to about the $5\%$ to about the $50\%$
of $n$.
On these kind of distribution data,
$\CFOF$ shows best performances for all the measures throughout the whole
range of values for the parameter $k$ and the dimensionality $d$.
Remarkably, the probabilities $p_{win,AUC}(\CFOF)$ and $p_{win,Prec}(\CFOF)$
are practically always the largest.

\subsection{Comparison on labelled data}
\label{sect:exp_real}

In this section we considered some labelled datasets
as ground truth. 
This kind of experiment is usually 
considered in the context of unsupervised outlier detection,
though it must be pointed out that 
these of methods 
are not designed to take advantage from the availability of labelled data.

We considered the following ten datasets,
randomly selected among those 
available
at the UCI ML Repository\footnote{{\tt 
https://archive.ics.uci.edu/ml/index.html}}:
\textit{Breast Cancer Wisconsin Diagnostic} ($n=569$, $d=32$),
\textit{Epileptic Seizure Recognition} ($n=11500$, $d=179$),
\textit{Image segmentation} ($n=2310$, $d=19$),
\textit{Ionosphere} ($n=351$, $d=34$),
\textit{Live disorders} ($n=345$, $d=7$), 
\textit{Ozone Level Detection} ($n=2536$, $d=73$), 
\textit{Pima indians diabetes} ($n=768$, $d=8$),
\textit{QSAR biodegradation} ($n=1055$, $d=41$),
\textit{Wine} ($n=178$, $d=13$),
\textit{Yeast} ($n=1484$, $d=8$),

Each class in turn is marked as normal,
and a dataset composed of all the points
of that class plus $10$ outlier points randomly selected 
from the other classes is considered.
Results are averaged over $30$ runs for each choice of the normal class.
For each dataset we considered $20$ distinct values for
the parameter $k$ ranging from $2$ to $n/2$.
Specifically, if $n>100$ then the $k$ values are log-spaced, otherwise
they are linearly spaced. 

\begin{table}[t]
\small
\begin{center}
\begin{tabular}{|l|c|c|c|c|}
\cline{2-5}
\multicolumn{1}{c|}{~} & $\mathbf{E}[AUC_{max}]$ & $\mathbf{E}[AUC_{mean}]$ & $\mathbf{E}[Prec_{max}]$ & $\mathbf{E}[Prec_{mean}]$ \\
\cline{2-5}\hline

$\CFOF$        & $\bf 0.845$ & $0.732$     & $\bf 0.408$ & $0.292$ \\
$\ODIN$        & $0.784$ & $0.717$     & $0.394$ & $0.253$ \\
$\aHub$        & $0.761$     & $0.694$     & $0.369$ & $0.244$ \\
$\aKNN$        & $0.763$     & $\it 0.733$ & $0.373$ & $\it 0.328$ \\
$\LOF$         & $\it 0.799$     & $0.716$     & $\it 0.395$ & $0.295$ \\
$\FastABOD$    & $0.752$     & $\bf 0.737$ & $0.360$ & $\bf 0.336$ \\
$\iForest$     & $0.760$     & ---         & $0.385$ & --- \\
\hline
\multicolumn{5}{c}{~}\\
\cline{2-5}
\multicolumn{1}{c|}{~}  & $p_{win,AUC_{max}}$ & $p_{win,AUC_{mean}}$ & $p_{win,Prec_{max}}$ & $p_{win,Prec_{mean}}$ \\
\cline{2-5} \hline 
$\CFOF$ & $\bf 0.772$ & $\bf 0.694$ & $\bf 0.619$ & $0.542$ \\
$\ODIN$ & $\it 0.582$ & $0.422$ & $0.521$ & $0.450$ \\
$\aHub$ & $0.353$ & $0.237$ & $0.553$ & $0.392$ \\
$\aKNN$ & $0.411$ & $0.562$ & $0.467$ & $0.526$ \\
$\LOF$ & $0.565$ & $0.399$ & $\it 0.612$ & $\it 0.554$ \\
$\FastABOD$ & $0.414$ & $\it 0.686$ & $0.395$ & $\bf 0.576$ \\
$\iForest$ & $0.418$ & --- & $0.454$ & --- \\
\hline
\end{tabular}
\end{center}
\caption{Accuracy comparison on labelled data.}
\label{table:acc_realdata}
\end{table}

We measured the 
$AUC_{max}$ and $AUC_{mean}$,
defined as, respectively, the maximum and the average AUC 
associated with the different values of $k$ employed.
Similarly, we measured the 
$Prec_{max}$ and $Prec_{mean}$ by considering
the top $10$ outliers reported by each method.
As already pointed out, the former measures are not applicable to $\iForest$.
Table \ref{table:acc_realdata} reports the above 
mentioned accuracy
measures averaged over all the datasets,
together with the probabilities $p_{win}$
that a given outlier definition
reports a value for a certain accuracy measure
better than another outlier definition.

As for the average accuracies, $\CFOF$ 
reports the largest values of $AUC_{max}$ 
and $Prec_{max}$, followed by $\LOF$, while $\FastABOD$
reports the largest values of $AUC_{mean}$
and $Prec_{mean}$,
followed by $\aKNN$, though the $AUC_{mean}$ of $\CFOF$ 
is almost identical to that of $\aKNN$.
In addition, within the family of methods perceiving anomalies 
solely on the basis of their reverse nearest neighbor distribution,
$\CFOF$ presents the largest average values.
Since average values can be influenced by large deviations of observed AUC or
$Prec$ values, the probability $p_{win}$ to obtain a better accuracy value
is of interest.
Remarkably,
$\CFOF$ reports the greatest
probability to obtain an AUC higher than any other method,
both for the maximum and the mean AUC values.
The probability associated with the $AUC_{max}$ of $\CFOF$
seems to indicate a superior ability
to reach peak performances.
Moreover, $\CFOF$ reports the greatest
probability to obtain a $Prec$ higher than any other method.
Also the probability concerning the mean $Prec$ is close to
the best value.

\begin{table}[t]
\small
\begin{center}
\begin{tabular}{|c|c|c|c|c|c|c|c|}
\multicolumn{8}{c}{$p_{rank,AUC}$} \\
\cline{2-8}
\multicolumn{1}{c|}{~} & $\CFOF$ & $\ODIN$ & $\aHub$ & $\aKNN$ & $\LOF$ & $\FastABOD$ & $\iForest$ \\
\cline{2-8}\hline
$CFOF$ & --- & $0.924$ & $0.976$ & $0.921$ & $0.844$ & $0.954$ & $0.944$ \\
$ODIN$ & $0.076$ & --- & $0.769$ & $0.604$ & $0.412$ & $0.687$ & $0.659$ \\
$aHub2$ & $0.024$ & $0.231$ & --- & $0.373$ & $0.176$ & $0.474$ & $0.419$ \\
$aKNN$ & $0.079$ & $0.396$ & $0.627$ & --- & $0.335$ & $0.569$ & $0.594$ \\
$LOF$ & $0.156$ & $0.588$ & $0.824$ & $0.665$ & --- & $0.759$ & $0.738$ \\
$ABOD$ & $0.046$ & $0.313$ & $0.526$ & $0.431$ & $0.241$ & --- & $0.483$ \\
$\iForest$ & $0.056$ & $0.341$ & $0.581$ & $0.406$ & $0.262$ & $0.517$ & --- \\
\hline
\end{tabular}
\end{center}
\caption{Wilcoxon rank-sum test for accuracy ($p$-values).}
\label{table:wilcoxon}
\end{table}

In general, 
no correlation can be expected 
between the semantics underlying classes and the 
observed data density and, consequently, 
the peculiar way a certain outlier detection method 
perceives abnormality. 
In order to detect large deviations 
from this assumption, we exploit the Wilcoxon test.
The Wilcoxon rank-sum test \cite{Siegel56}
(also called the Wilcoxon-Mann-Whitney test or Mann-Whitney U test)
is a nonparametric statistical test of the 
null hypothesis that it is equally likely that a randomly selected 
value from one sample will be less than or greater than a randomly selected 
value from a second sample.

Specifically, in the following $p_{rank,AUC}(Def_1,Def_2)$ 
represents the $p$-value from the left-tailed test, 
that is
the probability 
of observing the given sets of $AUC_{max}$ values, or one more extreme,
by chance if the null hypothesis that medians are equal is true
against the alternative hypothesis that the median of the definition
$Def_1$ is smaller than the median of the definition $Def_2$.
Note that $1-p$ is the $p$-value of the test having as alternative
hypothesis that the median of the definition
$Def_1$ is greater than the median of the definition $Def_2$.
Thus, $1-p_{rank,AUC}$ ($p_{rank,AUC}$, resp.)
represents the significance level at which 
the hypothesis that $Def_1$ ($Def_2$, resp.) is better that $Def_2$
($Def_1$, resp.)
can be accepted.
Small values of $p$-values cast doubt on the validity of the null hypothesis.
Fixed a significance level $\gamma$ (usually $\gamma = 0.05$)
if the $p$-value is greater than $\gamma$
the null hypothesis that medians are equal cannot be rejected.
Table \ref{table:wilcoxon} reports the $p$-values
of the Wilcoxon rank-sum test.

Interestingly, the $p_{rank,AUC}$ achieved by  
$\CFOF$ versus any other definition ranges from $0.844$
to $0.976$.
This seem to indicate that it is likely that the $\CFOF$ method will allow 
configurations which ranks outliers better.

Table \ref{table:lab_data} reports the $AUX_{max}$ and $Prec_{max}$ values
obtained by the various methods on the considered dataset 
for each class label.

\begin{table}[t]
\tiny
\centering
\begin{tabular}{|c|r|r|r|r|r|r|r||r|r|r|r|r|r|r|}
\multicolumn{15}{c}{\textit{wdbc} dataset}\\
\cline{2-15}
\multicolumn{1}{c}{~} & \multicolumn{7}{|c||}{$AUC$} & \multicolumn{7}{|c|}{$Prec$} \\
\cline{2-15}
\multicolumn{1}{c|}{~} & CFOF & ODIN & aHub2 & aKNN & LOF & ABOD & iForest & CFOF & ODIN & aHub2 & aKNN & LOF & ABOD & iForest \\
\hline
$0$ & $0.827$ & $\bf 0.855$ & $\it 0.849$ & $0.744$ & $0.802$ & $0.476$ & $0.842$ & $\it 0.263$ & $\bf 0.272$ & $0.119$ & $0.013$ & $0.200$ & $0.032$ & $0.116$ \\
$1$ & $0.950$ & $0.945$ & $0.920$ & $\it 0.979$ & $0.970$ & $\bf 0.981$ & $0.955$ & $0.702$ & $0.705$ & $0.703$ & $\it 0.739$ & $0.706$ & $\bf 0.745$ & $0.587$ \\
\hline
\multicolumn{15}{c}{\textit{yeast} dataset}\\
\cline{2-15}
\multicolumn{1}{c}{~} & \multicolumn{7}{|c||}{$AUC$} & \multicolumn{7}{|c|}{$Prec$} \\
\cline{2-15}
\multicolumn{1}{c|}{~} & CFOF & ODIN & aHub2 & aKNN & LOF & ABOD & iForest & CFOF & ODIN & aHub2 & aKNN & LOF & ABOD & iForest \\
\hline
$0$ & $0.756$ & $0.753$ & $\bf 0.801$ & $\it 0.775$ & $0.769$ & $0.756$ & $0.741$ & $0.160$ & $0.150$ & $\bf 0.198$ & $\it 0.165$ & $0.155$ & $0.158$ & $0.158$ \\
$1$ & $\bf 0.786$ & $\it 0.675$ & $0.654$ & $0.635$ & $0.664$ & $0.644$ & $0.658$ & $0.074$ & $0.076$ & $0.079$ & $0.061$ & $0.081$ & $\it 0.087$ & $\bf 0.093$ \\
$2$ & $\bf 0.809$ & $0.749$ & $0.734$ & $0.786$ & $0.751$ & $\it 0.798$ & $0.746$ & $\bf 0.162$ & $\it 0.161$ & $\it 0.161$ & $0.139$ & $0.155$ & $\it 0.161$ & $0.119$ \\
$3$ & $\bf 0.946$ & $0.915$ & $0.935$ & $\it 0.943$ & $0.942$ & $0.889$ & $0.922$ & $\bf 0.747$ & $0.589$ & $\bf 0.747$ & $0.723$ & $0.723$ & $0.474$ & $0.661$ \\
$4$ & $0.952$ & $0.923$ & $0.940$ & $\bf 0.958$ & $\it 0.957$ & $0.924$ & $0.924$ & $0.720$ & $0.574$ & $0.701$ & $\bf 0.732$ & $\it 0.729$ & $0.681$ & $0.674$ \\
$5$ & $\it 0.794$ & $0.772$ & $0.786$ & $0.779$ & $\bf 0.814$ & $0.731$ & $0.757$ & $\bf 0.385$ & $0.333$ & $0.355$ & $0.371$ & $\it 0.377$ & $0.284$ & $0.329$ \\
$6$ & $0.878$ & $\bf 0.893$ & $0.878$ & $0.868$ & $0.877$ & $0.875$ & $\it 0.879$ & $\it 0.345$ & $0.336$ & $0.316$ & $0.303$ & $0.310$ & $0.306$ & $\bf 0.348$ \\
$7$ & $\bf 0.853$ & $0.736$ & $0.690$ & $0.728$ & $\it 0.781$ & $0.707$ & $0.711$ & $0.415$ & $\it 0.420$ & $0.416$ & $0.381$ & $\bf 0.432$ & $0.374$ & $0.358$ \\
$8$ & $\bf 0.881$ & $0.789$ & $0.758$ & $0.728$ & $\it 0.809$ & $0.735$ & $0.731$ & $0.538$ & $\it 0.540$ & $0.494$ & $0.352$ & $\bf 0.555$ & $0.419$ & $0.419$ \\
$9$ & $\bf 0.928$ & $0.861$ & $0.845$ & $0.862$ & $\it 0.869$ & $0.851$ & $0.770$ & $\bf 0.733$ & $\it 0.717$ & $0.699$ & $0.703$ & $0.706$ & $0.655$ & $0.565$ \\
\hline
\multicolumn{15}{c}{\textit{pima-indians-diabetes} dataset}\\
\cline{2-15}
\multicolumn{1}{c}{~} & \multicolumn{7}{|c||}{$AUC$} & \multicolumn{7}{|c|}{$Prec$} \\
\cline{2-15}
\multicolumn{1}{c|}{~} & CFOF & ODIN & aHub2 & aKNN & LOF & ABOD & iForest & CFOF & ODIN & aHub2 & aKNN & LOF & ABOD & iForest \\
\hline
$0$ & $\bf 0.770$ & $0.730$ & $0.723$ & $0.706$ & $0.712$ & $0.711$ & $\it 0.746$ & $0.074$ & $0.071$ & $\it 0.118$ & $0.090$ & $0.084$ & $0.090$ & $\bf 0.123$ \\
$1$ & $\bf 0.749$ & $\it 0.651$ & $0.612$ & $0.562$ & $0.647$ & $0.561$ & $0.555$ & $0.066$ & $0.065$ & $\it 0.081$ & $0.029$ & $\bf 0.129$ & $0.029$ & $0.039$ \\
\hline
\multicolumn{15}{c}{\textit{segment} dataset}\\
\cline{2-15}
\multicolumn{1}{c}{~} & \multicolumn{7}{|c||}{$AUC$} & \multicolumn{7}{|c|}{$Prec$} \\
\cline{2-15}
\multicolumn{1}{c|}{~} & CFOF & ODIN & aHub2 & aKNN & LOF & ABOD & iForest & CFOF & ODIN & aHub2 & aKNN & LOF & ABOD & iForest \\
\hline
$1$ & $0.992$ & $0.991$ & $0.980$ & $\bf 0.998$ & $0.995$ & $\bf 0.998$ & $\bf 0.998$ & $0.841$ & $0.852$ & $0.752$ & $0.900$ & $\bf 0.923$ & $0.890$ & $\it 0.906$ \\
$2$ & $\bf 1.000$ & $\bf 1.000$ & $\bf 1.000$ & $\bf 1.000$ & $\bf 1.000$ & $0.994$ & $0.999$ & $\bf 1.000$ & $\bf 1.000$ & $0.990$ & $\bf 1.000$ & $\bf 1.000$ & $0.984$ & $0.913$ \\
$3$ & $\it 0.954$ & $\bf 0.961$ & $0.904$ & $0.924$ & $0.952$ & $0.936$ & $0.916$ & $\bf 0.611$ & $\it 0.605$ & $0.474$ & $0.058$ & $0.568$ & $0.171$ & $0.319$ \\
$4$ & $0.938$ & $0.942$ & $0.921$ & $\it 0.960$ & $0.926$ & $\bf 0.962$ & $0.924$ & $0.458$ & $\it 0.467$ & $0.426$ & $0.410$ & $0.416$ & $0.416$ & $\bf 0.619$ \\
$5$ & $0.946$ & $0.940$ & $0.879$ & $\it 0.962$ & $0.940$ & $\bf 0.964$ & $0.938$ & $0.582$ & $0.595$ & $0.529$ & $\it 0.613$ & $0.542$ & $\bf 0.635$ & $0.590$ \\
$6$ & $0.995$ & $0.997$ & $0.980$ & $\bf 0.999$ & $0.994$ & $\bf 0.999$ & $0.992$ & $0.802$ & $0.829$ & $0.572$ & $\it 0.910$ & $0.845$ & $\bf 0.929$ & $0.881$ \\
$7$ & $\bf 1.000$ & $\bf 1.000$ & $0.999$ & $\bf 1.000$ & $\bf 1.000$ & $0.999$ & $0.998$ & $0.992$ & $0.994$ & $0.923$ & $\bf 0.997$ & $\bf 0.997$ & $0.955$ & $0.890$ \\
\hline
\multicolumn{15}{c}{\textit{biodeg} dataset}\\
\cline{2-15}
\multicolumn{1}{c}{~} & \multicolumn{7}{|c||}{$AUC$} & \multicolumn{7}{|c|}{$Prec$} \\
\cline{2-15}
\multicolumn{1}{c|}{~} & CFOF & ODIN & aHub2 & aKNN & LOF & ABOD & iForest & CFOF & ODIN & aHub2 & aKNN & LOF & ABOD & iForest \\
\hline
$0$ & $0.824$ & $0.839$ & $0.810$ & $0.843$ & $0.809$ & $\it 0.849$ & $\bf 0.871$ & $0.253$ & $\it 0.257$ & $0.192$ & $0.165$ & $0.181$ & $0.197$ & $\bf 0.326$ \\
$1$ & $\bf 0.710$ & $0.549$ & $0.571$ & $0.546$ & $\it 0.590$ & $0.556$ & $0.443$ & $0.006$ & $0.012$ & $\bf 0.016$ & $\it 0.013$ & $0.010$ & $0.006$ & $0.003$ \\
\hline
\multicolumn{15}{c}{\textit{eighthr} dataset}\\
\cline{2-15}
\multicolumn{1}{c}{~} & \multicolumn{7}{|c||}{$AUC$} & \multicolumn{7}{|c|}{$Prec$} \\
\cline{2-15}
\multicolumn{1}{c|}{~} & CFOF & ODIN & aHub2 & aKNN & LOF & ABOD & iForest & CFOF & ODIN & aHub2 & aKNN & LOF & ABOD & iForest \\
\hline
$0$ & $\bf 0.749$ & $\it 0.648$ & $0.610$ & $0.428$ & $0.619$ & $0.461$ & $0.493$ & $\it 0.050$ & $0.047$ & $0.013$ & $0.000$ & $\bf 0.065$ & $0.019$ & $0.000$ \\
$1$ & $\it 0.813$ & $0.739$ & $0.678$ & $0.731$ & $0.712$ & $0.746$ & $\bf 0.820$ & $0.260$ & $0.240$ & $0.274$ & $0.281$ & $0.255$ & $\it 0.297$ & $\bf 0.452$ \\
\hline
\multicolumn{15}{c}{\textit{wine} dataset}\\
\cline{2-15}
\multicolumn{1}{c}{~} & \multicolumn{7}{|c||}{$AUC$} & \multicolumn{7}{|c|}{$Prec$} \\
\cline{2-15}
\multicolumn{1}{c|}{~} & CFOF & ODIN & aHub2 & aKNN & LOF & ABOD & iForest & CFOF & ODIN & aHub2 & aKNN & LOF & ABOD & iForest \\
\hline
$0$ & $0.934$ & $0.916$ & $0.925$ & $\it 0.937$ & $0.929$ & $0.820$ & $\bf 0.991$ & $\it 0.669$ & $0.580$ & $0.653$ & $0.581$ & $0.600$ & $0.503$ & $\bf 0.868$ \\
$1$ & $0.818$ & $0.801$ & $0.799$ & $\it 0.851$ & $0.836$ & $0.823$ & $\bf 0.894$ & $\bf 0.545$ & $0.457$ & $0.539$ & $\bf 0.545$ & $0.542$ & $0.381$ & $0.503$ \\
$2$ & $0.873$ & $\it 0.874$ & $0.859$ & $0.843$ & $0.848$ & $0.849$ & $\bf 0.981$ & $0.599$ & $\it 0.610$ & $0.601$ & $0.581$ & $0.581$ & $0.548$ & $\bf 0.832$ \\
\hline
\multicolumn{15}{c}{\textit{bupa} dataset}\\
\cline{2-15}
\multicolumn{1}{c}{~} & \multicolumn{7}{|c||}{$AUC$} & \multicolumn{7}{|c|}{$Prec$} \\
\cline{2-15}
\multicolumn{1}{c|}{~} & CFOF & ODIN & aHub2 & aKNN & LOF & ABOD & iForest & CFOF & ODIN & aHub2 & aKNN & LOF & ABOD & iForest \\
\hline
$0$ & $\bf 0.792$ & $\it 0.640$ & $0.626$ & $0.615$ & $\it 0.640$ & $0.619$ & $0.622$ & $\it 0.110$ & $0.100$ & $\bf 0.125$ & $0.103$ & $\it 0.110$ & $0.087$ & $0.097$ \\
$1$ & $\bf 0.754$ & $\it 0.673$ & $0.602$ & $0.504$ & $0.632$ & $0.547$ & $0.489$ & $\it 0.075$ & $0.074$ & $\bf 0.083$ & $0.019$ & $0.074$ & $0.013$ & $0.052$ \\
\hline
\multicolumn{15}{c}{\textit{ionosphere} dataset}\\
\cline{2-15}
\multicolumn{1}{c}{~} & \multicolumn{7}{|c||}{$AUC$} & \multicolumn{7}{|c|}{$Prec$} \\
\cline{2-15}
\multicolumn{1}{c|}{~} & CFOF & ODIN & aHub2 & aKNN & LOF & ABOD & iForest & CFOF & ODIN & aHub2 & aKNN & LOF & ABOD & iForest \\
\hline
$0$ & $0.935$ & $0.926$ & $0.906$ & $\bf 0.967$ & $0.938$ & $\it 0.961$ & $0.920$ & $0.571$ & $0.549$ & $0.456$ & $\bf 0.742$ & $0.535$ & $\it 0.703$ & $0.542$ \\
$1$ & $\it 0.477$ & $0.401$ & $\bf 0.552$ & $0.339$ & $0.381$ & $0.272$ & $0.373$ & $0.010$ & $0.016$ & $\bf 0.096$ & $0.006$ & $\it 0.029$ & $0.000$ & $0.016$ \\
\hline
\multicolumn{15}{c}{\textit{epileptic} dataset}\\
\cline{2-15}
\multicolumn{1}{c}{~} & \multicolumn{7}{|c||}{$AUC$} & \multicolumn{7}{|c|}{$Prec$} \\
\cline{2-15}
\multicolumn{1}{c|}{~} & CFOF & ODIN & aHub2 & aKNN & LOF & ABOD & iForest & CFOF & ODIN & aHub2 & aKNN & LOF & ABOD & iForest \\
\hline
$0$ & $\bf 0.577$ & $0.103$ & $0.051$ & $0.051$ & $\it 0.453$ & $0.088$ & $0.026$ & $\it 0.000$ & $\bf 0.001$ & $\it 0.000$ & $\it 0.000$ & $\it 0.000$ & $\it 0.000$ & $\it 0.000$ \\
$1$ & $\bf 0.897$ & $0.798$ & $0.559$ & $0.748$ & $\it 0.835$ & $0.791$ & $0.611$ & $\bf 0.292$ & $\it 0.281$ & $0.108$ & $0.087$ & $0.232$ & $0.152$ & $0.074$ \\
$2$ & $\bf 0.902$ & $0.800$ & $0.623$ & $0.750$ & $\it 0.849$ & $0.804$ & $0.638$ & $0.229$ & $0.227$ & $0.200$ & $0.226$ & $\bf 0.261$ & $\it 0.232$ & $0.219$ \\
$3$ & $\bf 0.678$ & $0.502$ & $0.527$ & $0.493$ & $\it 0.591$ & $0.538$ & $0.475$ & $\bf 0.252$ & $\it 0.249$ & $0.197$ & $0.232$ & $0.245$ & $0.248$ & $0.203$ \\
$4$ & $\bf 0.813$ & $0.797$ & $0.772$ & $0.797$ & $\it 0.798$ & $0.793$ & $0.784$ & $0.413$ & $0.410$ & $0.287$ & $\it 0.419$ & $\bf 0.426$ & $0.416$ & $0.332$ \\
\hline
\end{tabular}
\caption{Maximum $AUC$ and $Prec$ values on the labelled datasets.}
\label{table:lab_data}
\end{table}

\section{Conclusions}
\label{sect:conclusions}

In this work we introduced a novel outlier definition,
namely the Concentration Free Outlier Factor ($\CFOF$).
As a main contribution,
we formalized the notion of {concentration of outlier scores}
and theoretically proved that $\CFOF$ does not concentrate 
in the Euclidean space
for any arbitrarily large dimensionality.
To the best of our knowledge, 
there are no other proposals of outlier detection measures,
and probably also of other {data analysis}
measures related to the Euclidean distance,
for which it has been provided the theoretical evidence
that they are immune to the concentration effect.
We also provided evidence that $\CFOF$ does not suffer
of the hubness problem, since points associated with the largest scores 
always correspond to a small fraction
of the data.

We recognized that the \textit{kurtosis} of the data population
is a key parameter for characterizing from the outlier detection perspective
the unknown distribution underlying the data.
We determined the closed form of the distribution of the $\CFOF$
scores for arbitrarily large dimensionalities and showed that
the $\CFOF$ score of a point depends
on its squared norm standard score 
and on the kurtosis of the data distribution.
We theoretically proved that the $\CFOF$ score 
is both translation and scale-invariant
and, hence, 
that the number of outliers coming from each cluster
is directly proportional to its size and to its kurtosis,
a property that we called {semi--locality}.

Moreover,
we determined that the semi--locality is a
peculiarity of reverse nearest neighbor counts:
this discovery clarified the exact nature of 
the reverse nearest neighbor family of outlier scores.
We also proved that 
classic distance-based and density-based outlier scores 
are subject to concentration
both for bounded and unbounded dataset sizes,
and both for fixed and variable values of the neighborhood parameter.

We showed that $\CFOF$ scores can be reliably computed
by exploiting sampling techniques.
Specifically, we introduced the $\FastCFOF$ technique
which does not suffer of the dimensionality curse
affecting (reverse) nearest neighbor search techniques.
The $\FastCFOF$ algorithm
has cost linear both in the dataset
size and dimensionality, supports 
multi-resolution analysis,
and is efficiently parallelizable.
We provided a multi-core (MIMD) vectorized (SIMD)
implementation.
Experimental results highlight that $\FastCFOF$
is able to achieve very good accuracy with small sample
sizes, to efficiently process huge datasets,
and to efficiently manage even large values of the neighborhood parameter, 
a property 
which is considered a challenge for different existing outlier methods.
Moreover, experiments involving the $\CFOF$ score
witness for the absence of concentration
on real data and
show that $\CFOF$ achieves excellent accuracy
performances.
The applicability of the technique is not confined to the Euclidean space
or to vector spaces:
it can be applied both in metric
and non-metric spaces equipped with a distance function.

The $\CFOF$ technique and 
the properties presented in this work
provide insights 
within the scenario
of outlier detection and, more in the general,
of high-dimensional data analysis.
This work offers the opportunity for further investigations, 
including the design of algorithms with strictly bounded
accuracy guarantees, the application of the technique to
specific outlier detection frameworks,
and many others.
In particular, the suitability of the definition of being evaluated
by means of sampling schemes, seems to make it appealing for 
the big data and data streams scenarios.
In the scenario of high-dimensional data analysis,
the $\CFOF$ score represents a novel notion of density measure
that we believe can offer insights also in 
the context of other learning tasks. 
We are currently investigating its application
in other, both unsupervised and supervised, classification contexts.

\bibliographystyle{ACM-Reference-Format}
\bibliography{paper}

\end{document}